\documentclass{article}

\PassOptionsToPackage{numbers,compress}{natbib}
\usepackage[preprint]{neurips_2026}
\usepackage{layout}

\usepackage{tabularx}

\usepackage{amsfonts}
\usepackage{amsmath}
\usepackage{amsthm}
\usepackage{amssymb} % NOT COMATIBLE WITH svjour3
\usepackage{float}
\usepackage{mdframed} 
\usepackage{thmtools}
\usepackage{soul}
\usepackage{mathtools}

\usepackage{subcaption}
\usepackage{graphicx}
\usepackage{adjustbox}
\usepackage{algorithm}
\usepackage{array}
\newcolumntype{C}[1]{>{\centering\arraybackslash}p{#1}}

\usepackage{makecell}

\definecolor{shadecolor}{gray}{0.95}
\declaretheoremstyle[
headfont=\normalfont\bfseries,
notefont=\mdseries, notebraces={(}{)},
bodyfont=\normalfont,
postheadspace=0.5em,
spaceabove=1pt,
mdframed={
  skipabove=8pt,
  skipbelow=8pt,
  hidealllines=true,
  backgroundcolor={shadecolor},
  innerleftmargin=4pt,
  innerrightmargin=4pt}
]{shaded}

\usepackage{xcolor}
\usepackage{color}
\usepackage{algorithmic}
\usepackage{verbatim}
\usepackage{caption} % for \captionof
\newsavebox{\GenBoundsTableBox}

\newcommand{\R}{\mathbb{R}} % Reals
\newcommand{\N}{\mathbb{N}} % Naturals
\newcommand{\cA}{{\cal A}}

\newcommand{\cG}{{\cal G}}

\newcommand{\cK}{{\cal K}}
\newcommand{\cL}{{\cal L}}

\newcommand{\cT}{{\cal T}}

\newcommand{\cV}{{\cal V}}
\newcommand{\cX}{{\cal X}}
\newcommand{\cY}{{\cal Y}}

\usepackage[colorinlistoftodos,bordercolor=orange,backgroundcolor=orange!20,linecolor=orange,textsize=scriptsize]{todonotes}

\DeclareMathOperator{\argmin}{argmin}        % argmin
\newcommand{\w}{\textbf{W}}

\usepackage{mdframed} 
\newcommand{\myNum}[1]{(\emph{#1})}
\newcommand{\smartparagraph}[1]{\vspace{0.5pt} \noindent {\bf #1}}
\newcommand{\inLineComment}[1]{}

\theoremstyle{plain}

\newtheorem{theorem}{Theorem} %[section]
\newtheorem{lemma}{Lemma} %[section]
\newtheorem{proposition}{Proposition} %[section]
\newtheorem{corollary}{Corollary} %[section]
\newtheorem{assumption}{Assumption} %[section]

\theoremstyle{definition}
\newtheorem{definition}{Definition}% [section]
\newtheorem{remark}{Remark} %[section]
\usepackage{bbm}

\newcommand{\wdelw}{\textbf{W}_0+\Delta \textbf{W}}

\usepackage[utf8]{inputenc} % allow utf-8 input
\usepackage[T1]{fontenc}    % use 8-bit T1 fonts
\usepackage{hyperref}       % hyperlinks
\usepackage{url}            % simple URL typesetting
\usepackage{booktabs}       % professional-quality tables
\usepackage{amsfonts}       % blackboard math symbols
\usepackage{nicefrac}       % compact symbols for 1/2, etc.
\usepackage{microtype}      % microtypography
\usepackage{xcolor}         % colors
\usepackage{subcaption}
\usepackage{booktabs} % for professional tables
\usepackage{caption}
\usepackage{multirow}
\usepackage{graphicx}
\definecolor{codeblue}{rgb}{0.21,0.49,0.74}
\definecolor{codegreen}{rgb}{0,0.6,0}

\newcommand{\reddown}{\textcolor{red}{$\downarrow$}}

\expandafter\def\expandafter\normalsize\expandafter{%
    \normalsize%
    \setlength\abovedisplayskip{0pt}%
    \setlength\belowdisplayskip{4pt}%
    \setlength\abovedisplayshortskip{-1.2pt}%
    \setlength\belowdisplayshortskip{1pt}%
}

\title{Beyond LoRA: Is Sparsity-Induced Adaptation Better?}

\author{%
Elijah Cadenhead\textsuperscript{1}
\quad Cristian McGee\textsuperscript{1, 3}
\quad Xin Li\textsuperscript{1}
\quad El Houcine Bergou\textsuperscript{2}
\quad Aritra Dutta\textsuperscript{1,3}
\\[0.5em]
\textsuperscript{1}School of Data, Mathematical and Statistical Sciences, University of Central Florida, United States
\\
\textsuperscript{2}College of Computing, Mohammed VI Polytechnic University (UM6P), Morocco
\\
\textsuperscript{3}Department of Computer Science, University of Central Florida, United States
\\
}

\begin{document}
\maketitle
% \twocolumn[
% \icmltitle{Beyond LoRA: Is Sparsity-Induced Adaptation Better?}
% %

% \begin{icmlauthorlist}
% \icmlauthor{Anonymous Author(s)}{anon}
% \end{icmlauthorlist} 

% \icmlaffiliation{anon}{Anonymous Institution}

% \vskip 0.3in
% ]

% \printAffiliationsAndNotice{}

\begin{abstract}
%Full fine-tuning of large pre-trained models is constrained by computational and memory overhead, motivating parameter-efficient fine-tuning approaches, such as 
Low-rank adaptation (LoRA) and its variants provide a memory- and compute-efficient alternative to full fine-tuning of pre-trained models. However, questions remain about the comparative generalizability of these approaches and how the structural restrictions on low-rank updates preserve effective adaptation performance. We present a historical framing, covering the past (full fine-tuning and original LoRA), the present (different variants of LoRA), and propose simpler, cheaper, parameter-efficient extensions by inducing sparsity within existing LoRA variants: Cheap LoRA (cLA), training a single low-rank factor with the other fixed (deterministically or, in its randomized variant, stochastically), and the chained circulant variant, ${c}^3$LA. {We frame cLA as a structured instance of asymmetric LoRA, serving as a controlled column-subspace restriction of full fine-tuning}. We derive information-theoretic generalization error bounds for these variants, marking one of the first endeavors in this area. Empirically, \emph{we evaluate \textbf{11 fine-tuning methods} across \textbf{10 pre-trained models and 14 datasets}, analyzing the fine-tuned models' performance and generalization using tools such as loss landscapes and spectral analysis.} Despite the sensitivity of fine-tuned models to the pre-trained model, datasets, and other factors, our study suggests that restricting LoRA-based PEFT methods' adaptation to a sparse, structured columnspace remains competitive across tasks with their parameter-matched baselines while reducing \emph{up to 10\% training time} and \emph{peak GPU memory up to 15\%}, even with a naïve, non-optimized, sparse implementation. Our theoretical and empirical generalization measures provide a more consistent and principled approach to their cost-effective adaptation than commonly used analytical tools. GitHub repo available at: \href{https://github.com/EliCaden/Beyond_LoRA}{\texttt{github.com/EliCaden/Beyond\_LoRA}}.
\end{abstract}

\begin{center}
\end{center}
\tableofcontents

\vspace{0mm}
\section{Introduction}\label{sec:Introduction}
\vspace{0mm}
Full fine-tuning~(FFT)~\cite{bert} modifies a pre-trained neural network's parameters on new datasets %that might be relatively expensive to curate, 
and adapts the network to new downstream tasks. As model sizes and datasets grow, FFT is often computationally infeasible or prohibitively expensive. %Another important paradigm in this direction is that 
Additionally, the growth of these complex models and the hardware's compute capacity are incoherent~\cite{fei2021efficient, xu2021grace}. %Large multimodal models (LMMs) such as OpenAI's GPT series~\cite{gpt3}, Meta's LLaMA~\cite{Llama3}, Google's Gemini~\cite{team2023gemini}, image-text model CLIP~\cite{CLIP-radford2021learning}, video-text model DeepMind's Flamingo~\cite{alayrac2022flamingo}, etc., are pre-trained on massive high-quality data corpora and fine-tuned to adapt to different tasks or domains.
E.g., The smallest variant of Llama-3 \cite{Llama3, grattafiori2024llama} has 8B parameters; it requires 32 GB of GPU memory for inference and 64 GB for training with modern protocols. In contrast, the half-precision performance of the NVIDIA H100 is only about $2.4\times$ that of the NVIDIA A100, while their memory capacity remains unchanged \cite{Graphcore-GPU}. %, Graphcore-GPU-2}. 

%%%%% PLACING TEASER FIGURE

\begin{figure*}[t]
\vspace{-2mm}
    \label{fig:roberta-officehome-scapes}
    \centering
     \scalebox{0.90}{\includegraphics[width=\textwidth, height=1.6in]{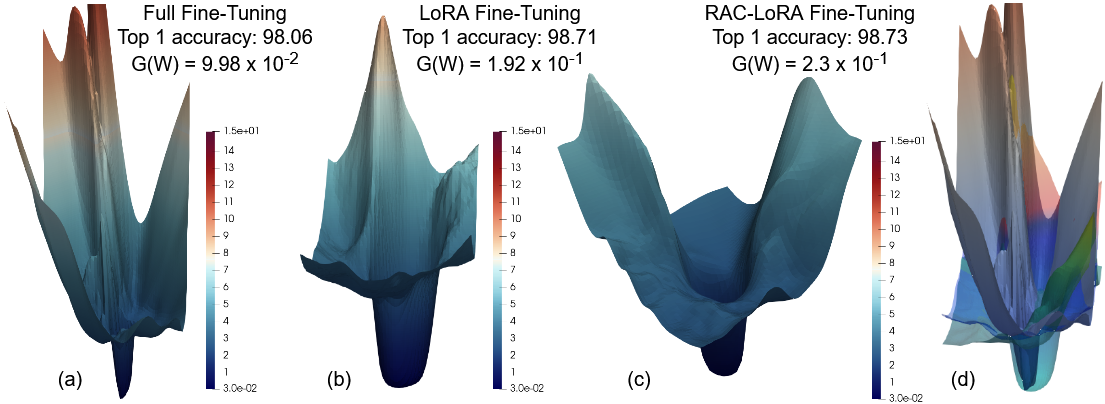}}
    \vspace{-1mm}
  \caption{\small{3D loss landscapes of ViT-Base pretrained on ImageNet-21K and fine-tuned on ImageNet-1K. We fine-tuned this model on CIFAR-10 using different strategies, including FFT. FFT has the narrowest local minima among the other PEFT methods, and yields the worst test accuracy. However, it has the least generalization error, $\cG(\textbf{W})$, among all the methods; see Definition \ref{definition:generalization} and Table \ref{tab:combined}. In (d), when we superimpose the loss landscapes, FFT shows the spikiest landscape; RAC has the smoothest landscape with the highest $\cG(\textbf{W})$. According to \cite{losslandscape}, this is counterintuitive; a model with a spiky landscape and small-volume local minima does not generalize well.}}\label{fig:teaser-fig-ViT-Cifar10}
  \vspace{0mm}
  \end{figure*}
%%%% PLACING TEASER FIGURE

Alternatively, parameter-efficient fine-tuning (PEFT) saves space and time, circumvents overfitting, and is widely used. Low-rank adaptation (LoRA)~\cite{lora} is a PEFT method that achieves performance on par with FFT, by reducing trainable parameters. To mitigate LoRA's flaws, researchers proposed numerous variants, including the chain of LoRA (CoLA)~\cite{cola}, asymmetric LoRA~\cite{fixA}, randomized asymmetric chain of LoRA~\cite{rac}, LoRA+~\cite{plus}, adaptive LoRA~\cite{adalora}, among many; see \cite{PEFT,han2024peft}. 

{Despite the existing LoRA variants, structured restrictions on low-rank updates that preserve effective adaptation under similar parameter counts remain unclear.} Recent works, \cite{PEFT, loravfft, rac}, analyzed and compared these PEFT methods with full fine-tuning, but these benchmarks are inconclusive. Figure \ref{fig:teaser-fig-ViT-Cifar10} {casts one such example, where generalization and loss landscape sharpness contradict our prior understanding --- FFT's resultant model, despite having the spikiest landscape and narrowest valley, has the smallest generalization error, conflicting with the well-known heuristic that models with sharper minima should generalize worse \cite{losslandscape, pmlr-v137-huang20a}.}~Current literature has a limited theoretical grasp on how these methods {behave in parameter-matched comparisons, that is, which extreme sparsity and structured low-rank constraints preserve effective adaptation across tasks and models and offer better generalization, and how far these restrictions can be pushed before adaptation degrades.}

%We further ask how far these restrictions can be pushed before adaptation capacity degrades.%to the point that storage, memory bandwidth, and hardware interface constraints are satisfied during adaptation, and that inference latency benefits from optimized sparse or structured libraries \cite{deepspeed, parseGPU} have become a practical imperative. 

In the era of resource-constrained IoTs and edge deployments \cite{MobileNet, Lottery}, pushing parameter efficiency for sparse or structured libraries \cite{deepspeed, parseGPU} has become a practical imperative. %Can sparse training be the new trend? 
E.g., New OpenAI LLM, GPT-4.5, requires a 10$\times$ increase in compute than GPT-4. Still, it only obtained a marginal performance improvement, and could be indicative that effective parameter reduction may benefit these models \cite{Gpt4o-vsGPT4.5}. Moreover, to reduce activation memory and improve sequential processing of the adapter and pretrained LoRA layers, \cite{paca} introduced partial connection adaptation (PaCA).

These ideas motivate us to explore different structured instances of LoRA that explicitly restrict learning to an established column subspace, allowing for a clearer examination of how far restricted subspace updates can be taken while maintaining competitiveness in performance. At this end, \emph{we propose 4 simpler, cheaper, and parameter-efficient extensions of the existing SOTA LoRA variants}: Cheap LoRA (cLA), which trains only one low-rank factor and sets the other low-rank factor deterministically, its randomized variant, random-cLA, its chain circulant variant, $c^3$LA, and its randomized chain variant, random-$c^3$LA. cLA and r-cLA can be interpreted as structured instances of Asymmetric LoRA that confine learning to an $r$-column subspace, enabling a stark contrast between how partial column-space adaptation compares to alternative low-rank updates. Alternatively, they can be seen as the LoRA adaptation of PaCA, where the restricted fine-tuned columns are set to $r$ columns of the pretrained model; see Figure~\ref{fig:lora_vs_paca}. Therefore, \emph{our proposed sparsity-induced SOTA LoRA variants act as a bridge between the two families of adapters, LoRA and PaCA}; see \S\ref{sec:future} and \S\ref{appendix:paca-cla}. 

But which structured restrictions of low-rank updates remain sufficient for competitive adaptation? Can confining learning to a small, structured fraction of the column space provide performance comparable to fine-tuning all columns? Or, are there significant performance differences among these sparsified LoRA variants? If so, how do these differences vary across PEFT methods, hyperparameter configurations, and models? To answer these questions, we make the following contributions:

%\hl{Let us check this paragraph carefully --- Do we need it? Or, how much do we need?} In practice, are there significant differences and trade-offs in terms of convergence behavior and performance of the reduced-parameter LoRA variants? And if there are, how do these differences vary across PEFT methods, hyperparameter configurations, and DNN models? To answer these questions, we make the following contributions:

% \aritra{Think one compelling line and add why do we think that generalizability will answer some of the questions we posed in the previous paragraph? Same comments for the benchmarking, although less needed.}

\smartparagraph{Theoretical insights through generalization (\S\ref{sec:theory}).}~\emph{Generalizability} measures how well a model's loss on its training dataset represents its loss over the entire feature space, reflecting the model's capacity to avoid overfitting. {Since our questions concern when parameter-reduced fine-tuning subspaces remain competitive, we use generalization bounds to connect structural restrictions (such as adapter rank, chain length, if any, layerwise input-output dimensions, training bitwidth, fine-tuning dataset size, etc.), to overfitting risk.} To this end, we use an \emph{information-theoretic approach} to measure the generalization error bounds of the PEFT methods discussed in this paper, including PaCA. See summary of results in Table \ref{tab:upperbounds}.
%~We also adapt the optimization framework of \cite{rac}, and present the convergence analysis of the PEFT methods for smooth, nonconvex loss functions, under our \emph{layerwise setup}, where each layer's adapters are updated using gradient descent (GD) and show $O(T^{-1})$ convergence rate. % for these methods.%~Following \cite{xu2017information,russo2019much}. 

\smartparagraph{Quantitative evaluation~(\S\ref{sec:empirical results}).}~We evaluate FFT, 9 LoRA-based PEFTs and PaCA, %(LoRA~\cite{lora}, CoLA~\cite{cola}, Asymmetric LoRA~\cite{fixA}, RAC LoRA~\cite{rac}, LoRA Plus~\cite{plus}, cLA, r-cLA, $c^3$LA, and r-$c^3$LA), 
encompassing 10 different pretrained models, %: \myNum{i} GPT2-small~\cite{gpt2}, \myNum{ii} DeBERTa v3 Base~\cite{debertav3}, \myNum{iii} DeBERTa v2 XXL~\cite{deberta}, \myNum{iv} RoBERTa Base~\cite{roberta}, \myNum{v} RoBERTa Large~\cite{roberta}, \myNum{vi} Deepseek-Coder-1.3B-base~\cite{guo2024deepseekcoder}, \myNum{vii} TinyLlama-1.1B~\cite{zhang2024tinyllama}, \myNum{viii} Llama 3-8B~\cite{grattafiori2024llama}, and \myNum{ix} vision Transformers, (ViTs), tiny and base~\cite{dosovitskiy2020vit}, 
on 4 fine-tuning tasks: natural language processing, %on {PAWS}~\cite{zhang2019paws}, {TREC-50}~\cite{li2006learning}, and various {GLUE} benchmarks~\cite{wang2018glue}, 
image recognition, % on {OfficeHome}~\cite{officehome-dataset} and {CIFAR-10}~\cite{cifar10-dataset}, 
coding generation, % on {DJANGO}~\cite{oda2015learning}, 
and logical reasoning. % on {OpenBookQA}~\cite{openbookQA}, {FOLIO}~\cite{han2022folio}, {LogiQA}~\cite{liu2020logiqa}, and {CLUTRR}~\cite{sinha2019clutrr} datasets. %\hl{cite the datasets}.
We report a rich set of metrics, including accuracy, spectral behavior, 3D loss landscape, throughput, runtime, and empirical generalization error. While it is infeasible to be
exhaustive, our comprehensive benchmarking offers broadly applicable insights.

\vspace{0mm}
\section{Fine-Tuning: The Past, Present, and Future}\label{sec:background}
\vspace{0mm}
FFT updates all parameters of deep networks, an approach that becomes increasingly impractical as model size and deployment multiplicity grow. This leads to the advent of LoRA and its variants. Based on their evolutionary timeline, we divide this section into three phases. The \emph{past} contains FFT, and we introduce LoRA, while different LoRA variants dominate the \emph{present}. Finally, extreme compute efficiency characterizes the \emph{future}, where we induce sparsity to SOTA LoRA variants.

\vspace{0mm}
\subsection{The Past: Full fine-tuning (FFT) and LoRA}\label{sec:past}
\vspace{0mm}
\smartparagraph{Pre-training.} Without loss of generality, consider a $L$-layer, fully-connected, neural network whose layers are, $\{W^i\}_{i=1}^L$, where $W^i \in \mathbb{R}^{n_i \times m_i}$ are trainable weights. Let $x\in\R^{m_1}$ be the input and ${\mathbf{W}} = (W^1,...,W^L)$. The network $f_{{\textbf{W}}}(\cdot): \mathbb{R}^{d_{\rm in}}\rightarrow\mathbb{R}^{d_{\rm out}}$ is of the form: 
 \begin{align}\label{eq:FFN}
     \textstyle{f_{{\textbf{W}}}(x) = \sigma_L(W^L\ldots(\sigma_2(W^2\sigma_1(W^1(x))\ldots))},
 \end{align}
 where $\sigma^i(\cdot):\R^{n_i}\to \R^{n_i}$ is a nonlinear activation function for the $i^{\rm th}$ layer. Given a pre-training set, $N_{\rm pre}:=\{(x_i,y_i)\}\subset \mathbb{R}^{m_1}\times \mathbb{R}^{d_{\rm out}}$, and the loss function, $\mathcal{\ell}_{\rm pre}(\cdot): \mathbb{R}^{d_{\rm out}}\times \mathbb{R}^{d_{out}}\rightarrow \mathbb{R}$, we train the network by solving:
\begin{align}\label{eq:Pre-training}
    % {\textbf{W}}_0\approx\argmin_{{\textbf{W}}}\left[\mathcal{L}_{\rm pre}({\textbf{W}}) \stackrel{\text{def}}{=} \frac{1}{|N_{\rm pre}|}\sum_{i=1}^{|N_{\rm pre}|} \mathcal{\ell}_{\rm pre}(f_{{\textbf{W}}}(x_i),y_i)\right],
    {\textbf{W}}_0\approx \argmin_{\textbf{W}} \frac{1}{|N_{\rm pre}|}\sum_{i=1}^{|N_{\rm pre}|} \mathcal{\ell}_{\rm pre}(f_{\textbf{W}}(x_i),y_i),
 \end{align}
obtaining the trained weights ${\textbf{W}}_0 = [W_0^1,\cdots,W_0^L]$. Sophisticated DNNs, such as CNNs, RNNs, Transformers, etc., can be adapted with some modification to \eqref{eq:FFN}.

\smartparagraph{FFT~\cite{bert, peft-1, lora, PEFT}.} Given pre-trained weights, ${\textbf{W}}_0$, FFT updates each DNN layer with corresponding $\Delta W^i$ to adapt the model to a downstream task on a domain-specific training dataset, $N:=\{(x_i',y_i')\}$. Denote $\Delta{\textbf{W}}$ as the update, and define ${\textbf{W}}_0 + \Delta{\textbf{W}} := [W_0^1 + \Delta W^1,\cdots,W_0^L + \Delta W^L]$. Given a loss function, $\mathcal{\ell}(\cdot): \mathbb{R}^{d_{\rm out}\times d_{\rm out}}\rightarrow \mathbb{R}$, FFT updates the model weights via:
\begin{align}\label{eq:FFT}
    % \Delta\hat{\textbf{W}}\approx\argmin_{\Delta \textbf{W}}\left[\mathcal{L}(\textbf{W}_0 + \Delta\textbf{W}) \stackrel{\text{def}}{=} \frac{1}{|N|}\sum_{i=1}^{|N|} \mathcal{\ell}(f_{\textbf{W}_0 + \Delta\textbf{W}}(x_i'),y_i')\right],
    \Delta\hat{\textbf{W}}\approx\argmin_{\Delta \textbf{W}}\frac{1}{|N|}\sum_{i=1}^{|N|} \mathcal{\ell}(f_{\textbf{W}_0 + \Delta\textbf{W}}(x_i'),y_i'),
\end{align}
and obtains the fine-tuned model, $f_{\textbf{W}_0 + \Delta\hat{\textbf{W}}}$, adapted to the downstream task. The computational overhead for FFT can be prohibitively expensive. E.g., LLMs for task-specific fine-tuning. 
%In contrast, parameter-efficient fine-tuning (PEFT) requires orders of magnitude fewer trained parameters and attains similar performance to FFT 
In contrast, parameter-efficient fine-tuning (PEFT) trains orders of magnitude fewer parameters while often attaining performance comparable to FFT \cite{peft-1, PEFT}.

\smartparagraph{LoRA~\cite{lora}} is a popular PEFT method that replaces the layer-wise updates $\Delta W^i$ with a low-rank representation $B^iA^i$, such that $B^i\in \mathbb{R}^{n_i\times r}$, $A^i\in\mathbb{R}^{r\times m_i}$, $r \ll \min(m_i,n_i)$ for all $i\in[L]$.~Denote $\textbf{W}_0 + \frac{\alpha}{r}\textbf{BA} := [W_0^1 + \frac{\alpha}{r}B^1A^1,\cdots,W_0^L + \frac{\alpha}{r}B^LA^L]$, where ${\alpha> 0}$ is a scaling factor. LoRA initializes $B^i$ = 0, $A^i$ $\sim$ $\mathcal{N}(0, 0.02^2)$, and solves:
\begin{align}\label{eq:LoRA}
   % ({\hat{\textbf{B}}, \hat {\textbf{A}}}) \approx\argmin_{\textbf{B},\textbf{A}}\left[\mathcal{L}(\textbf{W}_0 + \frac{\alpha}{r}\textbf{BA}):= \frac{1}{|N|}\sum_{i=1}^{|N|} \mathcal{\ell}(f_{\textbf{W}_0 + \frac{\alpha}{r}\textbf{BA}}(x_i'),y_i')\right],
   ({\hat{\textbf{B}}, \hat {\textbf{A}}}) \approx\argmin_{\textbf{B},\textbf{A}}\frac{1}{|N|}\sum_{i=1}^{|N|} \mathcal{\ell}(f_{\textbf{W}_0 + \frac{\alpha}{r}\textbf{BA}}(x_i'),y_i'),
\end{align}
to obtain $B^i, A^i$ for each tuned layer. LoRA may not need to be applied to all layers; some layers can remain frozen. LoRA substantially reduces trainable parameters, saves training time, and the update $\textbf{BA}$ can be merged into the base weights to avoid additional inference latency. %\aritra{I did not understand the following line; why are we using $\Delta AB$? Maybe for chain, but for LoRA, Assymetric LoRA, we do not need them. When we mean to say update $\theta_0$ to $\Delta \theta$, it makes sense, but for LoRA, it is simply $W_0+AB.$ No need for $\Delta AB.$} 
%\eli{Dropping the $\Delta$s definitely can work. The next line was for switching from one trained LoRA module to another (both using the same base pre-trained model). Can write $\hat{B}_1\hat{A}_1$ for each instead. Will update.} 
%For adapting the same pre-trained model to multiple downstream $K$ tasks, each update, $\{\hat {\textbf{B}}_j\hat {\textbf{A}}_j\}_{j=1}^K$, is stored separately. Then each task can be switched to by taking the current model $f_{\textbf{W}_0+{ \frac{\alpha}{r}\hat {\textbf{B}}_j\hat {\textbf{A}}_j}}$, for $j\in[K]$, subtracting the current update ${\hat {\textbf{B}}_j\hat {\textbf{A}}_j}$, and adding the update corresponding to the new task. 
LoRA is compute- and storage-efficient, but renders worse generalization than FFT~\cite{loravfft}; LoRA may also fail~\cite{lora-loudly}. 

\vspace{0mm}
\subsection{The Present: Evolution of LoRA}\label{sec:present}
\vspace{0mm}

Many variants of LoRA exist to enhance efficiency while addressing weaknesses. 
%Many LoRA adaptations have been proposed to balance higher efficiency with mitigating identified weaknesses. 
They excel in certain tasks but are less optimal in others. Including FFT, empirical evidence suggests that no single fine-tuning method is the best fit for all cases, and that different variations are successful in varying circumstances~\cite{PEFT}. Thus, there exists compelling reasoning as to why new variants of LoRA emerge. Below, we discuss a few popular LoRA variants.

\smartparagraph{Chain of LoRA (CoLA)~\cite{cola}} increases LoRA's performance without substantially increasing compute or memory costs. After fine-tuning $\textbf{B}^1\textbf{A}^1$ for the downstream task to obtain ${\hat {\textbf{B}}^1\hat {\textbf{A}}^1}$, CoLA merges ${ \hat {\textbf{B}}^1\hat {\textbf{A}}^1}$ into the base weights and continues training with a new $\textbf{B}^2\textbf{A}^2$ on the same task, treating $\textbf{W}_0 + \frac{\alpha}{r}{\hat {\textbf{B}}^1\hat {\textbf{A}}^1}$ as the base weights. Denote $\textbf{W}^{(k,BA)}:= \textbf{W}_0 + \sum_{j=1}^k \frac{\alpha}{r}{\hat {\textbf{B}}^j\hat {\textbf{A}}^j}$ and $\textbf{W}^{(0,BA)} = \textbf{W}_0$ for convenience. CoLA of chain length $k$ solves:
\begin{align}\label{eq:COLA}
    \text{For }j\in[k],~~~~~\hat {\textbf{B}}^{j}\hat {\textbf{A}}^j\approx\argmin_{{\textbf{B}^j}\textbf{A}^j}\left[\mathcal{L}(\textbf{W}_0^{(j-1,BA)} + \frac{\alpha}{r}\hat {\textbf{B}}^j\hat {\textbf{A}}^j)\right]
\end{align}
to obtain the fine-tuned model, $f_{\textbf{W}^{(k,BA)}}$.~CoLA simulates a higher-rank approximation of a single LoRA update~\cite{relora} and claims to reduce LoRA's failure~\citep{lora-loudly}.

\smartparagraph{Asymmetric LoRA~\cite{fixA}} modifies LoRA adaptation for each layer by freezing one of the low-rank matrices, conventionally, $A$ to $A_0$, initializing the frozen matrix via a Normal distribution, and setting the trainable matrix to 0, and solves:
\begin{align}\label{eq:Asymm}
    \hat {\textbf{B}}\approx \argmin_{\textbf{B}} [\mathcal{L}(\textbf{W}_0 + \frac{\alpha}{r}\textbf{B}\textbf{A}_0) =\frac{1}{|N|}\sum_{i=1}^{|N|} \mathcal{\ell}(f_{\textbf{W}_0 + \frac{\alpha}{r}\textbf{B}\textbf{A}_0}(x_i),y_i)],
\end{align}
%\eli{Used subscript $0$ to indicate that $\textbf{A}$ is frozen, but it may look like it applies to B?}
to obtain the fine-tuned model $f_{\textbf{W}_0 + {\hat {\textbf{B}}\textbf{A}_0}}$. Under trainable-parameter constraints, Asymmetric LoRA competes with LoRA~\cite{fixA} and retains the Lipschitz smoothness of the loss function, which LoRA does not~\cite{LoRA-federated}.

\smartparagraph{Randomized Asymmetric Chain of LoRA (RAC-LoRA)~\cite{rac}} combines Asymmetric LoRA and CoLA. RAC-LoRA fixes one of the low-rank matrices (conventionally $A$), initializing via some fixed distribution of matrices $\mathcal{D}$, and sets the trainable one to 0. Like CoLA, the trained ${\hat {\textbf{B}}^1\textbf{A}^1_0}$ is then merged into the base weights, and a new $\textbf{BA}_0$ is trained on the same task. Denote $\textbf{W}^{(k,B)}:= \textbf{W}_0 + \sum_{j=1}^k \frac{\alpha}{r}{\hat {\textbf{B}}^j\textbf{A}_0^j}$ and $\textbf{W}^{(0,B)} = \textbf{W}_0$. RAC-LoRA of chain length $k$ solves:
\begin{align}\label{eq:rac}
    \text{For }j\in[k],~~~~~\hat {\textbf{B}}^{j}\approx\argmin_{{\textbf{B}^j}}\left[\mathcal{L}(\textbf{W}_0^{(j-1,B)} + \frac{\alpha}{r}\hat {\textbf{B}}^j{\textbf{A}}_0^j)\right]
\end{align}
to obtain the fine-tuned model $f_{\textbf{W}^{(k,{B)}}}$. 

\smartparagraph{{LoRA}$+$~\cite{plus}} applies separate learning rates $\{\gamma_B^i,\gamma_A^i\}$ to the adapter matrices, $\{B^i, A^i\}$ of each layer, respectively, and maintains the identical structure to LoRA.~${\rm LoRA}+$ prioritizes a substantially higher learning rate ($2-16 \times$) for $B$. 
We discuss some other LoRA variants in \S\ref{Appendix:present}.

\vspace{0mm}
\subsection{The Future: How Can We Achieve More Efficiency?}\label{sec:future}
\vspace{0mm}
With rapidly increasing model dimensionality, memory, and adaptation costs, we characterize this phase as a key next evolutionary step for LoRA: maximizing efficiency while maintaining parity with current LoRA variants. Training $B$ generally performs better~\cite{fixA}, together with insights from structured chaining methods \cite{cola,rac}, leads us to two simple, easy-to-analyze and implement variants, where we postulate that the update of the pre-trained parameter can be restricted to $r$ columns of $B$. 

\myNum{i}\smartparagraph{Cheap LoRA (cLA).} %is a simplified instance of Asymmetric LoRA \cite{fixA}, where only the low‑rank factor $B$ is optimized, while $A$ is kept fixed to the identity matrix of rank $r$ concatenated with zeros.
Stemmed from Asym LoRA \cite{fixA}, in cLA, the fixed matrix, $A^i$, for each layer $i$, is set to be an $r\times r$ identity matrix, concatenated with $\textbf{0}_{r\times m_i-r}$, that is, $A^i = \left[ I_r | \mathbf{0}_{r \times (m_i - r)} \right]\in\mathbb{R}^{r\times m_i}$. For each layer, with $W^i\in\mathbb{R}^{n_i\times m_i},$ and $B^i\in\mathbb{R}^{n_i\times r},$ we have $\Delta W^i=B^i \left[ I_r | \mathbf{0}_{r \times (m_i - r)} \right]= \left[ B^i | \mathbf{0}_{n_i \times (m_i - r)} \right].$ Denote $\textbf{B}^c$ as the layer-wise update with $B^i,A^i$ chosen above, and $\textbf{W}_0 + \frac{\alpha}{r}\textbf{B}^c := [W_0^1 + \frac{\alpha}{r}B^1\left[ I_r | \mathbf{0}_{r \times (m_1 - r)} \right],\cdots,W_0^L + \frac{\alpha}{r}B^L\left[ I_r | \mathbf{0}_{r \times (m_L - r)} \right] ]$. Then cLA solves:
\begin{align}\label{eq:cLA}
    \hat{\textbf{B}}^c \approx \argmin_{\textbf{B}^c}\frac{1}{|N|}\sum_{i=1}^{|N|} \mathcal{\ell}(f_{\textbf{W}_0 + \frac{\alpha}{r}\textbf{B}^c}(x_i'),y_i').
\end{align}
We consider two instantiations of the fixed factor, deterministic (\emph{cLA}) and random (\emph{random-cLA}), where we randomly permute the columns of $A^i$ on initialization. Empirical results show that the deterministic choice suffices, the randomized variant does not yield better performance. %, even though the random version is more convenient for convergence analysis.

\myNum{ii}\smartparagraph{Circulant Chain of Cheap LoRA~($c^3$LA).} As noted in CoLA~\cite{cola} and RAC-LoRA~\cite{rac}, chaining LoRA modules leverages repeated initializations to avoid poor minima. We extend this principle to cLA with a structured chaining, $c^3$LA. This method shifts the identity $I_r$ in each matrix $\left[ I_r | \mathbf{0}_{r \times (m_i - r)} \right]$ by $r$ columns to the left. That is, starting with $\left[ I_r | \mathbf{0}_{r \times (m_i - r)} \right]$, the next chain is $\left[ \mathbf{0}_{r \times r} \;\middle|\; I_r \;\middle|\; \mathbf{0}_{r \times (m_i - 2r)} \right]$, and so on. For $j\in[k]$, let $\textbf{B}^{c^3}$ denote $c^3$LA's update and denote $\textbf{W}^{(k,{B^{c^3}})}:= \textbf{W}_0 + \sum_{j=1}^k {\frac{\alpha}{r}\hat {\textbf{B}}^{c^3,j}}$, and $\textbf{W}^{(0,B^{c^3})}=\textbf{W}_0.$ Then $c^3$LA of chain length $k$ solves:
\begin{align}\label{eq:c^3LA}
    \hat {\textbf{B}}^{{c^3},j}\approx\argmin_{{\textbf{B}^{c^3,j}}}\mathcal{L}\left(\textbf{W}_0^{(j-1,B^{c^3})} + \frac{\alpha}{r}\hat {\textbf{B}}^{c^3,j}\right),
\end{align}
to obtain the fine-tuned model $f_{\textbf{W}^{(k,{B^{c^3})}}}$ for a chain of length $k$; see~\eqref{eq:Asymm} for the definition of $\mathcal{L}.$ Given sufficient epochs and chain length, this ensures we can update all elements in each $W_0$. We formalize this in the following proposition.
\begin{proposition}\label{prop:c3la}
Let $k\in\N$ be such that ${d_{in}}=kr$. Let $E$ be the total number of epochs used in $c^3{LA}$ fine-tuning. Then by creating a new chain in every $\left\lfloor\frac{E}{k} \right\rfloor$ epochs, $c^3{LA}$ updates each element in $W_0$.
\end{proposition}
\vspace{-1mm}
The intuition behind $c^3$LA goes beyond chaining cheap LoRA modules; its structured shifts expand the representational capacity of the learned $B$ matrices.~{We provide pseudocode of our variants in}~\S\ref{subsec:algorithm}. 

\smartparagraph{Connection with partial connection adaptation~(PaCA) \cite{paca}.} Although stemming from different research queries, our sparsity-induced LoRA variants and PaCA are closely related as they study restricted column updates of a pre-trained model. While PaCA directly updates a randomly chosen subset of the pretrained model to address training-time and activation-memory costs, cLA, which is sparsity-induced Asymteric LoRA, remains within the LoRA PEFT family and uses a deterministic, structured column-subspace restriction. However, they differ in where the restriction is imposed (directly inside the pre-trained backbone vs. via a LoRA parameterization) and the motivation behind their construction. Their similarities suggest we can lift many results related to cLA and its variants, and apply them to PaCA; see a detailed discussion in \S\ref{appendix:paca-cla}. 

Empirically, our sparsity-induced LoRA variants and PaCA behave similarly; see Tables \ref{tab:vit_acc_all}-\ref{tab:vit_acc_qv}. The main benefit of PaCA is avoiding the additional cost of running the forward pass through the LoRA module entirely; in comparison to LoRA, per layer, PaCA computes $Wx$, not $Wx + B(Ax)$.
However, from a theoretical perspective, our sparsity-induced variants serve as a bridge between LoRA and PaCA (see Figure \ref{fig:lora_vs_paca}), facilitating bidirectional knowledge transfer between these two families. In \S\ref{sec:generalization}, we show that our LoRA-based generalization results can be used directly to PaCA. LoRA has a growing literature on different variants, their nonconvex convergence; see \cite{rac, LoRA-federated, Mu2025OnTC}. %Recently, \cite{Mu2025OnTC} proposed nonconvex convergence of LoRA by designing \emph{Lipschitz-like} reparametrized function; in LoRA, the Lipschitz smoothness is lost even if the loss, $\cL$, is Lipschitz-smooth \cite{rac}. 
These results can also be directly adapted to PaCA if one considers it as a sparsity-induced extension of the LoRA family. Interestingly, we can add a well-known performance enhancer of the LoRA PEFT family, chain construction, to PaCA; see \S\ref{sec:pacavariant}. Similarly, PaCA's loss convergence result in Theorem \ref{thm:paca_thm} can be adapted for cLA, and so on; see Theorem \ref{thm:cla_paca_thm}. We note that LoRI \cite{zhang2025lori} shares some similarities with cLA, as it keeps the projection matrices $A$ fixed as random projections while sparsifying the $B$ matrices with task-specific masks. However, the construction of cLA is general and task agnostic.

%\hl{I will discuss the recent literature suggestions from the Stanford review, and we will add 2 lines here.} % \hl{Additionally, our sparsity-induced variants can enjoy PaCA's computational advantage. --- Can we say this?}

%\eli{We should not be able to say this. While currently the time in the experiments is the same the main benefit of PaCA is avoiding the additional cost of running the forward pass through the LoRA module entirely (so they only, per layer, do Wx computation not Wx + B(Ax) ). Empirically when I ran PaCA it took the same time as cLA, but it should be ~20\% faster than LoRA so if LoRA took 67.4 seconds per epoch on OfficeHome in~\ref{tab:paca_cla_vit} we should expect PaCA to take closer to 54, rather than the 67.6 that I'm finding.}

\vspace{0mm}
\section{Theoretical Insights}\label{sec:theory}
\vspace{0mm}
\begin{table*}
\scriptsize
\centering
\sbox{\GenBoundsTableBox}{%
\begin{tabular}{lcccl}
\toprule
\textbf{Variant} & \textbf{Chaining?} & \textbf{$\cG(\mathbf{W}_0+\Delta \mathbf{W})\le\Phi_{\textbf{W}_0} +$}\\
\midrule
LoRA      \cite{lora} &$\times$  & $ \sqrt{\tfrac{2rq\sigma^2\ln 2\sum_{i=1}^L(m_i + n_i)}{|N|}}$\\
LoRA+      \cite{plus} & $\times$  & $\sqrt{\tfrac{2rq\sigma^2\ln2\sum_{i=1}^L(m_i+n_i)}{|N|}}$\\
Asym-LoRA      \cite{fixA}&$\times$  & $\sqrt{\tfrac{2rq\sigma^2\ln2\sum_{i=1}^Ln_i}{|N|}}$\\
\midrule
CoLA      \cite{cola}  & \checkmark & $\sqrt{\tfrac{2rq\sigma^2k\ln2\sum_{i=1}^L(m_i + n_i)}{|N|}}$\\
RAC       \cite{rac}   & \checkmark & $\sqrt{\tfrac{2rq\sigma^2k\ln2\sum_{i=1}^L n_i}{|N|}}$ \\
\midrule
\emph{cLA \& PaCA} &$\times$  & $\sqrt{\tfrac{2rq\sigma^2\ln2\sum_{i=1}^L n_i}{|N|}}$\\
(This paper, \cite{paca}) &&\\
$c^{3}$LA & \checkmark & $\sqrt{\tfrac{2rq\sigma^2k\ln2\sum_{i=1}^L n_i}{|N|}}$\\
\bottomrule
\end{tabular}%
}
\begin{minipage}[t]{\wd\GenBoundsTableBox}
\vspace{0pt}
\centering
\usebox{\GenBoundsTableBox}
\vspace{-0.1em}
\captionof{table}{\small{\textbf{Generalization error upper bounds} of LoRA variants. The expression, $\Phi_{\textbf{W}_0}$ is in Theorem \ref{theorem:nonlinearGenBound}. $r$ is the adapter rank, $k$ chain length, $|N|$ size of fine-tuned dataset, $q$ bitwidth, $(m_i,n_i)$ are the (input,output) dimensions for layer $i$. The loss, $\cL$ is $\sigma$-sub-Gaussian (Assumption \ref{assumption:sigma_subgaussian}). 
%PaCA, being similar to cLA, has the same generalization error upper bound.
}}
\label{tab:gen_bounds_lora_variants}
\label{tab:upperbounds}
\vspace{-2em}
\end{minipage}%
\hfill
\begin{minipage}[t]{\dimexpr\textwidth-\wd\GenBoundsTableBox-1em\relax}
\vspace{0pt} 
\centering
 \scalebox{0.96}{\includegraphics[width=\linewidth]{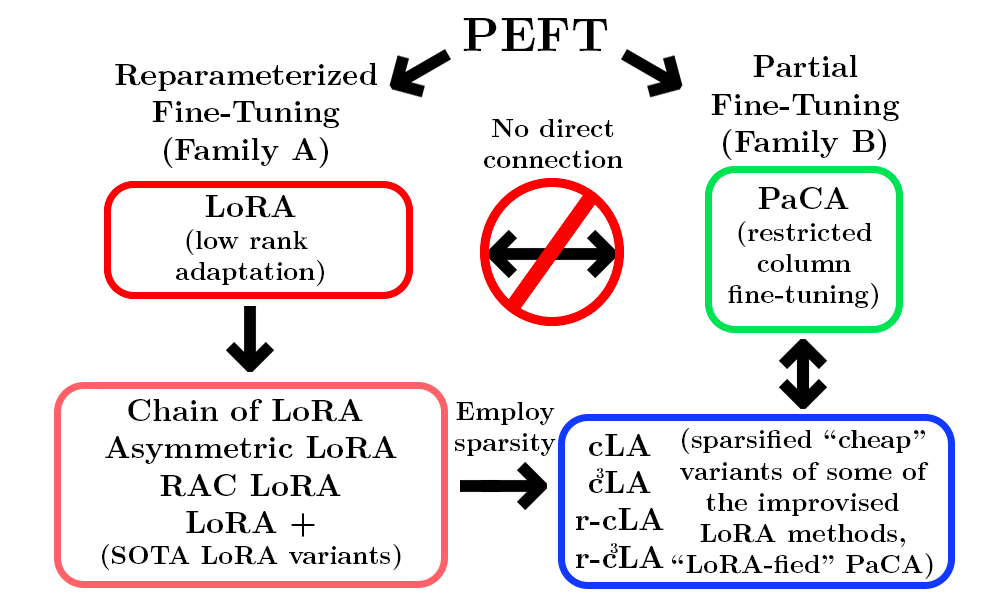}}
\captionof{figure}{\small{\textbf{Evolution of fine-tuning methods.} LoRA falls under the PEFT family (\textbf{Family A}). PaCA fine-tunes randomly selected columns within pretrained weights to improve training speed and reduce overhead (\textbf{Family B}). LoRA has different variants (e.g., Asym, CoLA); RAC is a combination of various techniques applied to LoRA. Our sparsity-induced variants (cLA, $c^3$LA, and their randomized forms) create a bridge between reparameterization (\textbf{Family A}) and partial fine-tuning (\textbf{Family B}).\vspace{-2mm}}}
\label{fig:lora_vs_paca}
\end{minipage}
\end{table*}

In this section, we measure the generalization error upper bounds of the PEFT methods of this paper under an information-theoretic framework \cite{xu2017information, russo2019much} with our layerwise setup, where each layer's adapters are updated using gradient descent (GD). 

%In this Section, we use an information-theoretic approach to measure the generalization error bounds under our layerwise setup, where each layer's adapters are updated using gradient descent (GD). %  \myNum{ii} we adapt the optimization framework of \cite{rac}, and present the convergence analysis of the PEFT methods for smooth, nonconvex loss functions, under our layerwise setup, where each layer's adapters are updated using gradient descent (GD).
\vspace{0mm}
\subsection{On the Generalization of Different LoRA Variants}\label{sec:generalization}
\vspace{0mm}

%\eli{Added comment on the benefit of generalization error bounds as an intro. Noted LoRA-unique bounds for when we use the assumptions of Assym-LoRA/Xu-Raginsky and the PAC-Bayes bound. Will update intro after if we get a similar result to Assym-LoRA paper which seems to imply that asymmetric methods have tighter gen-error bounds.}

%\aritra{Cristian: Work really hard to make this Section nice. Enough of running code and putting everything together without thought. This is the section that has material, but it is horribly written. The definitions are terrible. This is how you read a definition in the book? Go over to the paper \cite{raginsky}, read every word of the generalization error definition, and rephrase it 10 times. Read the section multiple times and see how they write it. Then put the best version here. It looks like a toilet now. Spend a lot of time; writing a paper is not a joke.}

% \aritra{Cristian, since for the fine-tuned model we are using, the convention $\theta_0 + \Delta\theta := [W_0^1 + \Delta W^1,\cdots,W_0^L + \Delta W^L]$, please follow the same convention in this section and replace $W$ and $\Delta W$, accordingly. To define the loss, please follow the convention Eli is using in formulations (2) and (3). }

 Let $\mathcal{X} \times \mathcal{Y}$ be an input space and label space with $\nu$ distribution of pairs $(x, y) \in \mathcal{X} \times     \mathcal{Y}$. Let $N= \{ (x_i,y_i)\}_{i=1}^{|N|}$ represent the training dataset, where each $(x_i, y_i)$ is i.i.d. from $\nu$ distribution of $\mathcal{X} \times \mathcal{Y}$. Given a \emph{hypothesis}, $f_\textbf{W}(\cdot):\mathcal{X} \to \mathcal{Y}$, and a nonnegative \emph{loss function}, $\ell(\cdot): \mathcal{Y} \times \mathcal{Y} \to \mathbb{R}$, the \emph{empirical risk} of a hypothesis on the dataset is defined as,
$\cL(\textbf{W}) := \tfrac{1}{|N|} \sum_{i=1}^{|N|}\ell(f_\textbf{W}(x_i),y_i).$
The \emph{true risk} of the hypothesis, $f_\textbf{W}(\cdot)$ is defined as, 
$\hat{\cL}_{\rm global}(\textbf{W}):=\mathbb{E}_{\mathcal{X}, \mathcal{Y} \sim \nu}[\ell(f_\textbf{W}(X),Y)].$ With the above setup, we define \emph{generalization error}, which tells us how well the hypothesis, $f_\textbf{W}$, generalizes from the training sample to the underlying population distribution.
%\vspace{-1mm}
\begin{definition}(\textbf{Generalization Error}~\cite{xu2017information})\label{definition:generalization}
    Generalization error, $\mathcal{G}(\textbf{W})$, is the difference between a hypothesis's true risk and its empirical risk on the training dataset, i.e., 
    $\mathcal{G}(\textbf{W}):= \hat{\cL}_{\rm global}(\textbf{W}) - \cL(\textbf{W})$.
\end{definition}
%Denote $\theta_0$ as our pretrained model, and $\Delta \theta$ as the update to our trained model. 
%Note that for any matrix $A$, $\sup \{\|Ax\|_2; x \in  \mathcal{X} \} \leq C\|A\|_{\rm op}$.
For our analysis, we make the following general assumptions.

\begin{assumption}\label{ass:bounded feature vector} \textbf{(Boundedness of input vectors)} The input vectors are bounded, i.e., there exists a constant $C\ge0$ such that $\|x\|\leq C$, for all $x \in \mathcal{X}.$
\end{assumption}
\begin{assumption}\label{ass:smoothness}
\textbf{(Lipschitz continuity of the loss)} The loss function, $\ell (\cdot): \R^d\to \R$ is $L_{\cL}$-Lipschitz continuous, i.e., $|\ell(f_\textbf{W}(x), y) - \ell(f_{\textbf{W}'}(x), y)| \leq L_{\cL}\|f_\textbf{W}(x) - f_{\textbf{W}'}(x)\|$ for all $\textbf{W}, \textbf{W}'\in\R^d \text{ and } (x,y) \in \mathcal{X} \times \mathcal{Y}$.
\end{assumption}
\begin{assumption}\label{ass:activationlipschitz}
    \textbf{(Lipschitz continuity of activation)}  
    The vector-valued activation function, $\sigma_{i}(\cdot):\R^{n_i}\to\R^{n_i}$, for each layer, $i$, is $L_{\sigma_i}$-Lipschitz continuous, i.e., $\|\sigma_{i}(u) - \sigma_{i}(v)\|\leq L_{\sigma_i}\|u - v\|,$ for all $u, v \in \R^{n_i}$.
\end{assumption}

%Now, we are set to prove Theorem \ref{theorem:nonlinearGenBound}. %However, first, we furnish the full version of the theorem, including the details of the correction terms, $\Phi_{\Delta \textbf{W}}$ and $\Phi_{ \textbf{W}_0}$, which we deliberately omitted in the main paper to not overwhelm the readers. 
Based on the assumptions, Theorem \ref{theorem:nonlinearGenBound} upper bounds the generalization error of a fine-tuned, $L$-layer fully connected DNN, parameterized by $\textbf{W}_0 + \Delta \textbf{W}$, by the better of two alternatives: the generalization error of $\textbf{W}_0$ and a correction term, or the generalization error of $\Delta \textbf{W}$ and a different correction term. 
%This \emph{general} error bound holds for any DNN with the setup above.
%\aritra{How are we considering the activation? Is the activation a vector-valued function or a scalar-valued function? It looks like we are using vector-valued convention. Then why in Assumption 6: $|\sigma^{i}(u) - \sigma^{i}(v)| \leq L_{\sigma_i}|u - v|,$ for all $u, v \in \R$. I would write for Assumption 6: The vector-valued activation function,  $\sigma^i(\cdot):\R^{m_i}\to\R^{m_i},$ for each layer $i$ is $L_\sigma$-Lipschitz continuous, i.e., $\|\sigma^{i}(u) - \sigma^{i}(v)\|\le L_{\sigma_i}\|u - v\|,$ for $u, v \in \R^m.$}
\begin{theorem}(\textbf{Generalization bounds})\label{theorem:nonlinearGenBound}  Let $f_{\textbf{W}_0+\Delta\textbf{W}}(x)=\sigma_{L}([{W_0}^{L}+\Delta W^{L}](\cdots\sigma_{2}([(W_0^{2}+\Delta W^{2}]\sigma_{1}([W_0^{1}+\Delta W^{1}]x))\cdots))$ be a $L$-layers fine-tuned DNN, where ${\textbf{W}_0+\Delta \textbf{W}}$ is a fine-tuned update. Let the loss function, $\mathcal{L}$ for fine-tuning, follow  Assumptions~\ref{ass:bounded feature vector}--\ref{ass:activationlipschitz}.
Then ${\mathcal{G}(\textbf{W}_0 + \Delta \textbf{W}) \leq\min\left(\mathcal{G}(\textbf{W}_0)+\Phi_{\Delta \textbf{W}},\mathcal{G}(\Delta \textbf{W})+\Phi_{ \textbf{W}_0} \right)},$ where 
$${\Phi_{\Delta \textbf{W}}:= 2L_{\cL}\left[ C\prod_{i=1}^L L_{\sigma_i}\sum_{i=1}^{2^L-1}\prod_{j=1}^L P(i,j)+\sum_{i\neq 2^a-1:a\in [L]}^{2^L-2} F(i)\right]} \text{ and }$$  
$${\Phi_{\textbf{W}_0}:= 2L_{\cL}\left[ C\prod_{i=1}^L L_{\sigma_i}\sum_{i=2}^{2^L}\prod_{j=1}^L P(i,j)
       + \sum_{i\neq 2^a:a\in [L]}^{2^L-1} {F}(i)\right]}, $$
       are the correction terms, ${F(i):=\|\sigma_{L-\psi(i)}(0)\|\prod_{j=1}^{\psi(i)}\![L_{\sigma_{L-j+1}}\,H(i,j)]},$
       ${\psi(i):=\lfloor \log_2(i)\rfloor}$, and 
\[
\begin{aligned}
P(i, j):=
\begin{cases}
\|W_{0}^{L-j+1}\|_2~\text{if}~\lfloor\tfrac{i-1}{2^{\,L-1}}\rfloor~\text{is~odd},\\
\|\Delta W^{L-j+1}\|_2~\text{if}~\lfloor\tfrac{i-1}{2^{\,L-1}}\rfloor~\text{is~even}
\end{cases},
\hspace{-3mm}
H(i, j):=
\begin{cases}
\|\Delta W^{L-j+1}\|_2~\text{if }\lfloor\tfrac{i}{2^{\,\psi(i)-j}}\rfloor~\text{is~odd},\\
\|W_0^{L-j+1}\|_2~\text{if }\lfloor\tfrac{i}{2^{\,\psi(i)-j}}\rfloor~\text{is~even}.
\end{cases}
\end{aligned}
\]
\end{theorem}

\vspace{-0.5mm}
\smartparagraph{Intuition behind Theorem~\ref{theorem:nonlinearGenBound}.} Theorem~\ref{theorem:nonlinearGenBound} provides a general framework to find the generalization error of a fine-tuned model using only the generalization properties of either the pretrained backbone or the parameter update. The standalone terms, $\Phi_{\Delta W}$ and $\Phi_{W_0}$, consist of Lipschitz constants of the loss and layerwise activation, spectral norms of $\{\|W_0^{i}\|_2, \|\Delta W^{i}\|_2\}_{i \in [L]},$ and offset terms, $\|\sigma_{i'}({0})\|$ based on the recursive collapse of the difference of $\|f_{\textbf{W}_0+\Delta\textbf{W}}-f_{\textbf{W}_0}\|$; see Figure \ref{fig:AB_recursivecollapse}. 

\smartparagraph{Tightness of the bounds.} {In Theorem \ref{theorem:nonlinearGenBound}, the combinatorial form may appear loose, given that it has $2^L$ components. However, this is simply the expansion of a product-of-sums; each layer contributes either its base spectral magnitude or its update spectral magnitude. This frames Theorem~\ref{theorem:nonlinearGenBound} within a spectral control perspective, where the generalization behavior is upper-bounded by the largest singular values among layers, offering insight by making spectral control a design handle when fine-tuning models. Additionally, it provides a comparison framework across variants, where spectral control can be combined with PAC-Bayes, information-theoretic, or other matrix-based generalization approaches; see \S\ref{appendix:special cases}. In \S\ref{appendix:tightness of bound}, we show that the bounds provided in Theorem \ref{theorem:nonlinearGenBound} are tight.}

\smartparagraph{Theorem~\ref{theorem:nonlinearGenBound} applied to attention mechanism.} Theorem~\ref{theorem:nonlinearGenBound} can handle advanced architectures, such as Transformers, whose main working component is attention, by considering inputs ${x\in\cX}$ after embedding. The key step is showing that multi-head attention (MHA) blocks can be expressed as compositions of linear maps 
and Lipschitz operators.~See~Theorem~\ref{thm:attn_head} in \S\ref{subsubsec:transformer} for the full derivation.

\smartparagraph{Theorem \ref{theorem:nonlinearGenBound} under special conditions.} The generalization upper bound $\cG(\wdelw)$ in Theorem \ref{theorem:nonlinearGenBound} contains two terms: \myNum{i} $\mathcal{G}(\textbf{W}_0)+\Phi_{\Delta \textbf{W}}$ and \myNum{ii} $\mathcal{G}(\Delta \textbf{W})+\Phi_{ \textbf{W}_0}.$ We can adapt some additional assumptions on loss, quantization bit-width, size of fine-tuning datasets, and layer dimensions; see \S\ref{subsubsec:genba} and bound  $\mathcal{G}(\textbf{W}_0)$ and $\mathcal{G}(\Delta \textbf{W})$. 

\myNum{i}\smartparagraph{Bounding $\mathcal{G}(\textbf{W}_0)$.} We use the PAC-Bayes generalization bound for fine-tuning by \cite{regularizationfinetuning}; see Theorem \ref{theorem:PAC-Bayes} in \S \ref{subsubsec:genba}. The loss function, $\cL$, is bounded by $C_2$. Since $\|W_0^{(i)}-W_0^{(i)}\| = 0$, for all $i\in[L]$, in Theorem \ref{theorem:PAC-Bayes}, we obtain $Q_i := 0$. Hence, $\mathcal{G}(\textbf{W}_0) \leq \epsilon + C_2\sqrt{|N|^{-1}(3\ln|N|\delta^{-1}+8)},$ holds with probability at least $1-2\delta$, where $\epsilon, \delta>0,$  are  arbitrary small numbers.~Together with Theorem \ref{theorem:nonlinearGenBound}, we arrive at 
$\mathcal{G}(\textbf{W}_0+\Delta \textbf{W}) \leq \epsilon + C_2\sqrt{|N|^{-1}(3\ln|N|\delta^{-1}+8)} +\Phi_{\Delta \textbf{W}}$; we quote this result formally in Theorem \ref{theorem:PAC-Extension} in \S\ref{subsubsec:genba}. 

\myNum{ii}\smartparagraph{Bounding $\mathcal{G}(\Delta \textbf{W})$.}
For a DNN, let $q$ be the training bitwidth; we use $q=32$ in this work. We assume $\mathcal{L}$ is $\sigma$-sub-gaussian for all $\mathbf{W}$ and use the generalization upper bound of $\cG(\Delta \textbf{W})$ as in Lemma 4.5 of~\cite{fixA}, for each PEFT method. Lemma 4.5 in~\cite{fixA} only bounds the generalization error of the fine-tuned update, $\Delta \mathbf{W}$, that is, $\cG(\Delta \textbf{W}),$ while keeping the pretrained weights $\mathbf{W}_0$ fixed. In contrast, Theorem \ref{theorem:nonlinearGenBound} is a standalone general result, not an extension. It gives an explicit general bound for the full model $\mathbf{W}_0+\Delta \mathbf{W}$ and tells how the pretrained backbone and the fine-tuned update interact, which is characterized by $\Phi_{\mathbf{W}_0}$. Together with Theorem~\ref{theorem:nonlinearGenBound}, we arrive at $\mathcal{G}(\textbf{W}_0+\Delta \textbf{W}) \leq \Phi_{\w_0}+\cG (\textbf{BA})$, where $\cG (\textbf{BA})$ represents the generalization error of different LoRA variants; see Table~\ref{tab:upperbounds} and \S~\ref{subsubsec:genba}. Table~\ref{tab:upperbounds} demonstrates the generalization error upper bounds of different PEFT methods. In practice, some DNN models may deviate from them; see Tables \ref{tab:gen-error-mini} and \ref{tab:combined}. %\hl{Do we want to keep this sentence here?} 
\begin{table*}[t]
  \centering
  \tiny
  \setlength{\tabcolsep}{2pt}
  \renewcommand{\arraystretch}{0.8}
  \caption{\small{Performance of fine-tuned models with adapter rank $r=16$.~We use \textcolor{codegreen}{green}, \textcolor{red}{red}, and \textcolor{blue}{blue} to indicate the best, second best, and third best result. For the sparse variants, \reddown~indicates the accuracy drop percentage compared to the best. Some results are deferred to the Appendix; see Table~\ref{tab:extendedMo-accuracy-table}.}}
  \label{tab:full-accuracy-table}
  \vspace{0mm}
\begin{tabular}{@{} ll *{2}{r} *{4}{r} | *{4}{r} @{} }
  \toprule
  \multirow{2}{*}{\textbf{Model}} & \multirow{2}{*}{\textbf{Dataset}} & \multicolumn{2}{c}{\textbf{The Past}} & \multicolumn{4}{c}{\textbf{The Present}} & \multicolumn{4}{c}{\textbf{The Future}} \\
  \cmidrule(lr){3-4} \cmidrule(lr){5-8} \cmidrule(lr){9-12}
   &  & FFT & LoRA & CoLA & Asym & RAC & LoRA+ & cLA & $c^3$LA & r-cLA & r-$c^3$LA \\
  \midrule
    ViT-Tiny \cite{dosovitskiy2020vit} & OfficeHome \cite{officehome-dataset} & \textcolor{red}{79.68} & \textcolor{codegreen}{80.13} & \textcolor{blue}{79.54} & 78.02 & 78.55 & 77.87 & 78.01 (\reddown2.65\%) & 78.69 (\reddown1.80\%) & 78.01 (\reddown2.65\%) & 79.32 (\reddown1.01\%) \\
    & CIFAR10 \cite{cifar10-dataset} & \textcolor{codegreen}{96.59} & \textcolor{red}{96.17} & \textcolor{blue}{95.85} & 94.80 & 95.36 & 95.29 & 94.94 (\reddown1.71\%) & 95.23 (\reddown1.41\%) & 95.12 (\reddown1.52\%) & 95.22 (\reddown1.42\%) \\
    \midrule
    ViT-Base \cite{dosovitskiy2020vit} & OfficeHome
      & 86.42 & 88.96 & 89.01 & 89.00 & \textcolor{codegreen}{89.33} & 87.87 & \textcolor{red}{89.21} & \textcolor{blue}{89.18} & 88.83 & 89.17 \\
    & CIFAR10
      & 98.06 & 98.71 & 98.48 & 98.68 & \textcolor{red}{98.73} & 98.36 & 98.63 & 98.54 & \textcolor{codegreen}{98.78} & \textcolor{blue}{98.72} \\
    \midrule
    DeBERTa v2 XXL \cite{deberta} & MRPC \cite{wang2018glue} & \textcolor{blue}{87.49} & \textcolor{codegreen}{88.28} & 87.47 & 87.03 & 86.97 & \textcolor{red}{87.53} & 86.13 (\reddown2.44\%) & 85.11 (\reddown3.59\%) & 85.55 (\reddown3.09\%) & 85.15 (\reddown3.55\%) \\
    & TREC-50 \cite{li2006learning} & \textcolor{blue}{91.99} & 91.47 & 85.65 & \textcolor{codegreen}{92.26} & \textcolor{red}{92.02} & 84.92 & 91.73 (\reddown0.57\%) & 90.87 (\reddown1.51\%) & 91.67 (\reddown0.64\%) & 91.07 (\reddown1.29\%) \\
    & PAWS \cite{zhang2019paws} & 94.69 & \textcolor{blue}{94.97} & \textcolor{codegreen}{95.22} & {94.95} & 94.66 & \textcolor{red}{95.20} & 94.77 (\reddown0.47\%) & 94.90 (\reddown0.34\%) & 94.38 (\reddown0.88\%) & 94.71 (\reddown0.54\%) \\
    \midrule
    DeBERTa v3 Base \cite{debertav3} & MRPC & {85.80} & \textcolor{codegreen}{88.33} & \textcolor{red}{87.91} & \textcolor{blue}{86.40} & {86.34} & 84.51 & 84.43 (\reddown4.42\%) & {80.22} (\reddown9.18\%) & 85.42 (\reddown3.29\%) & 84.17 (\reddown4.71\%) \\
    % & RTE \cite{wang2018glue} & 82.47 & \textcolor{codegreen}{86.34} & \textcolor{blue}{83.80} & 78.94 & 79.40 & \textcolor{red}{84.72} & 76.00 (\reddown11.98\%) & 75.08 (\reddown13.04\%) & 79.40 (\reddown8.04\%) & 79.40 (\reddown8.04\%) \\
    & STS-B \cite{wang2018glue} & \textcolor{codegreen}{89.52} & {89.09} & \textcolor{red}{89.34} & 89.04 & 88.71 & \textcolor{blue}{89.15} & 87.56 (\reddown2.19\%) & 87.90 (\reddown1.81\%) & 88.05 (\reddown1.64\%) & 87.90 (\reddown1.81\%) \\
    & TREC-50 & \textcolor{red}{90.15} & {89.29} & \textcolor{blue}{89.88} & \textcolor{codegreen}{90.67} & {89.22} & 85.52 & 86.04 (\reddown5.11\%) & 87.96 (\reddown2.99\%) & 86.04 (\reddown5.11\%) & {87.70} (\reddown3.28\%) \\
    & PAWS & \textcolor{codegreen}{94.76} & \textcolor{red}{94.62} & {94.40} & {94.48} & {94.45} & {94.44} & 94.23 & \textcolor{blue}{94.60} & 94.36 & 94.42\\
    \midrule
    RoBERTa-Base \cite{roberta} & MRPC & \textcolor{codegreen}{87.40} & 86.34 & \textcolor{red}{86.76} & 86.40 & \textcolor{blue}{86.67} & 84.29 & 84.83 (\reddown2.94\%) & 84.39 (\reddown3.44\%) & 85.08 (\reddown2.65\%) & 85.33 (\reddown2.37\%) \\
    %& CoLA \cite{wang2018glue} & \textcolor{blue}{56.08} & \textcolor{red}{57.33} & \textcolor{codegreen}{58.39} & 52.35 & 53.76 & 50.40 & 51.86 (\reddown11.18\%) & 53.29 (\reddown8.73\%) & 52.56 (\reddown9.98\%) & 53.10 (\reddown9.06\%) \\
    \midrule
    RoBERTa-Large \cite{roberta} & MRPC
      & 87.57 & \textcolor{codegreen}{88.46} & \textcolor{red}{88.43} & 87.56 & 87.69 & 72.91 & \textcolor{blue}{87.81} & 86.36 & 86.24 & 86.59 \\
   & CoLA & \textcolor{codegreen}{64.58} & \textcolor{blue}{62.42} & 60.03 & \textcolor{red}{63.42} & 59.84 & 28.80 & 59.47 (\reddown7.91\%) & 59.60 (\reddown7.71\%) & 58.60 (\reddown9.26\%) & 60.24 (\reddown6.72\%) \\
    \midrule
    TinyLlama \cite{zhang2024tinyllama} 
    %& OpenBookQA \cite{openbookQA} & \textcolor{codegreen}{55.47} & 52.41 & \textcolor{blue}{52.47} & 45.96 & 47.59 & \textcolor{red}{53.26} & 44.92(\reddown19.02\%) & 45.12(\reddown18.66\%) & 47.07(\reddown15.14\%) & 27.34(\reddown50.71\%) \\
    & FOLIO \cite{han2022folio} & \textcolor{codegreen}{60.71} & 57.59 & {59.40} & {58.33} & 55.45 & 54.17 & \textcolor{blue}{58.97} & 58.01 & 54.81 & \textcolor{red}{59.82}\\
    & LogiQA \cite{liu2020logiqa} & \textcolor{codegreen}{47.54} & 41.54 & \textcolor{blue}{43.70} & 41.50 & 40.86 & \textcolor{red}{45.83} & 39.09 (\reddown17.77\%) & 39.30 (\reddown17.33\%) & 39.09 (\reddown17.77\%) & 39.31 (\reddown17.31\%) \\
    & CLUTRR \cite{sinha2019clutrr} & \textcolor{codegreen}{42.01} & 37.44 & \textcolor{red}{39.38} & 37.98 & 37.98 & 38.10 & \textcolor{blue}{39.12} & 37.79 & 36.23 & 37.03\\
    \midrule
    Llama3-8B~\cite{grattafiori2024llama} & OpenBookQA & \textcolor{codegreen}{88.80} & \textcolor{blue}{87.53} & 86.47 & \textcolor{red}{88.47} & 87.33 & 86.87 & 87.33(\reddown1.65\%) & 85.07(\reddown4.20\%) & 86.07(\reddown3.07\%) & 53.69(\reddown39.54\%) \\
    & CLUTRR & 50.29 & 48.7 & 47.65 & 51.69 & 49.65 & \textcolor{blue}{52.89} & \textcolor{codegreen}{55.53} & {52.04} & \textcolor{red}{54.9} & 49.94\\
    \midrule
    % DeepseekCoder \cite{guo2024deepseekcoder} & DJANGO \cite{oda2015learning} & 22.73 & 23.60 & 19.79 & \textcolor{codegreen}{35.12} & \textcolor{red}{30.27} & \textcolor{blue}{27.27} & 7.83 (\reddown77.71\%) & 19.48 (\reddown44.53\%) & {19.36} (\reddown44.87\%) & 15.34 (\reddown56.32\%) \\
    % \midrule
    GPT2-Small \cite{gpt2} & E2E \cite{novikova2017e2e} & \textcolor{codegreen}{2.98} & \textcolor{blue}{3.18} & 3.29 & 3.36 & 3.34 & \textcolor{red}{3.23} & 3.34(\textcolor{red}{$\uparrow$}12.08\%) & 3.29(\textcolor{red}{$\uparrow$}10.4\%) & 3.30(\textcolor{red}{$\uparrow$}10.7\%) & 3.29(\textcolor{red}{$\uparrow$}10.4\%) \\
  %\midrule
  \bottomrule
  \end{tabular}
  \vspace{0mm}
\end{table*}
%\eli{Move these to appendix: DeepseekCoder on DJANGO, DeBERTa v3 Base on RTE, the performance of DeepSeekCoder is a surprising case of its own which requires much deeper understanding of the interplay between rank and column space expressivity, due to the complexity and space limitations we move it to the appendix. In the extra space add the comment about DeepseekCoder (and add a comment of full table in appendix, same idea with generalization). Maybe at the BEGINNING of the empirical study section bring a comment similar to the "no method works the best" rebuttal from ICML.}

%%% OLD WORKING CAN SWITCH TO IF CURRENT WORDING NOT CLEAR.
% let $\mathcal{D}^i=\{A^{1,i},A^{2,i},...,A^{k,i}\}$, $A^{j,i} = \left[ \mathbf{0}_{r \times jr} \;\middle|\; I_r \;\middle|\; \mathbf{0}_{r \times (d_\text{in} - (j+1)r)} \right]\in\R^{r\times n_i}.$ Then, we have the following result: 
% \begin{proposition}\label{cor: $C^3LA$ eigen value}
% Let $A^i\in\mathbb{R}^{r\times n_i}$ with $n_i = rk$ for some $k\in\mathbb{N}$. Then for $c^3$LA and \emph{random-cLA},$\lambda^{H,i}_{\rm min} = \frac{r}{n_i}.$
% \end{proposition}

%\eli{switched to proposition since I'm no longer using it to prove Theorem 4 but to apply it?}

%But can we do better, so that we can analyze all PEFT variants discussed in this paper under a unified framework?
% \aritra{Where is the convergence result? What can we comment on COLA, cLA, and Assym LoRA's  convergence?}

\vspace{0mm}
\section{Quantitative Evaluation}\label{sec:empirical results}
\vspace{0mm}
%It is well known that fine-tuning may or may not be optimal, depending on the actual pre-trained model, datasets used, and a multitude of other factors \cite{PEFT}. Our extensive experimental study of 11 fine-tuning methods confirms that. Moreover, the unpredictable performance of LoRA-based PEFT methods suggests that it would be advantageous to use their cheaper variants for cost reduction and better generalizability of pre-trained models.

Our extensive experimental study of 11 fine-tuning methods confirms that fine-tuning may or may not be optimal, depending on the actual pre-trained model, datasets used, and a multitude of other factors \cite{PEFT}. Hence, it is better to use the cheaper LoRA variants for cost reduction and better generalizability. % of pre-trained models.

\smartparagraph{Implementation details and models used.} Our empirical evaluation encompasses 10 pretrained models: \myNum{i} DeBERTa v3 Base, \myNum{ii} DeBERTa v2 XXL, \myNum{iii} GPT2-small, \myNum{iv} RoBERTa Base, \myNum{v} RoBERTa Large, \myNum{vi} DeepseekCoder-1.3B-base, \myNum{vii} TinyLlama-1.1B, \myNum{viii} Llama 3-8B, \myNum{ix} ViT Base, and \myNum{x} ViT-Tiny. See Table~\ref{tab:task-summary} in~\S\ref{subsec:implementation} for a detailed summary of the models and Table~\ref{tab:task-hyperparams} for implementation details and reproducibility. We report the lowest validation loss epoch for each model. We report additional ablation studies to justify the choices of hyperparameters in Table \ref{tab:full-accuracy-table}, such as learning rate, rank, scaling factor, and chain-reset in \S\ref{subsec:ablation-studies}, spanning Tables \ref{tab:deberta3ablation}--\ref{tab:vit-chain_reset}.

\smartparagraph{Fine-tuning tasks and datasets.} %We perform 4 different tasks:
\myNum{i} \smartparagraph{Natural Language Processing (NLP).} We use {PAWS}, {TREC-50}, and various {GLUE} benchmarks, including {MRPC}, {CoLA}, {STS-B}, and {RTE} for NLP tasks. \myNum{ii} \smartparagraph{Image Classification.} We fine-tuned on {OfficeHome} and {CIFAR-10}. \myNum{iii}\smartparagraph{Coding Generation.}~Code generation presents unique challenges; minor deviations can lead to runtime errors or semantic mismatches. There is relatively limited LoRA-focused literature on programming tasks; we evaluate how different LoRA variants adapt to these tasks on {DJANGO}, and report results using {Exact Match~(EM)}. \myNum{iv} \smartparagraph{Logical Reasoning.} We use {OpenBookQA} for elementary science multiple-choice reasoning, {FOLIO} for natural language reasoning with first-order logic, {LogiQA} for logical comprehension, and {CLUTRR} for compositional relational reasoning from text. %See Table~\ref{tab:task-summary} for a summary and Table~\ref{tab:task-hyperparams} for hyperparameter configuration. 

\begin{table*}[t]
  \centering
  \scriptsize
  \setlength{\tabcolsep}{2pt}
  \renewcommand{\arraystretch}{0.8}
  \caption{\small{Empirical generalization error, $\cG(\textbf{W})$, of the fine-tuning methods over various models and datasets.}}
  \label{tab:gen-error-mini}
  \vspace{0mm}
\begin{tabular}{@{} ll *{2}{r}  *{4}{r} | *{4}{r} @{} }
  \toprule
   \multirow{2}{*}{\textbf{Model}} & \multirow{2}{*}{\textbf{Dataset}} & \multicolumn{2}{c}{\textbf{The Past}} & \multicolumn{4}{c}{\textbf{The Present}} & \multicolumn{4}{c}{\textbf{The Future}} \\
  \cmidrule(lr){3-4} \cmidrule(lr){5-8} \cmidrule(lr){9-12}
   &  & FFT & LoRA & CoLA & Asym & RAC & LoRA+ & cLA & $c^3$LA & r-cLA & r-$c^3$LA \\
  \midrule
    ViT-Tiny \cite{dosovitskiy2020vit} & OfficeHome & 4.85$e^{-1}$ & 6.96$e^{-2}$ & \textcolor{codegreen}{9.55$e^{-3}$} & 7.22$e^{-2}$ & 6.17$e^{-2}$ & 7.39$e^{-2}$ & \textcolor{red}{1.98$e^{-2}$} & 3.40$e^{-2}$ & \textcolor{blue}{2.16$e^{-2}$} & 3.51$e^{-2}$ \\
    %& CIFAR-10 & \textcolor{codegreen}{1.42$e^{-1}$} & \textcolor{red}{2.64$e^{-1}$} & 2.87$e^{-1}$ & 3.36$e^{-1}$ & 3.18$e^{-1}$ & \textcolor{blue}{2.80$e^{-1}$} & 3.13$e^{-1}$ & 3.03$e^{-1}$ & 3.12$e^{-1}$ & 2.92$e^{-1}$ \\
    \midrule
   %ViT-Base \cite{dosovitskiy2020vit} & OfficeHome & 3.66$e^{-1}$ & 1.07$e^{-1}$ & \textcolor{blue}{1.43$e^{-2}$} & \textcolor{codegreen}{8.52$e^{-3}$} & \textcolor{red}{1.02$e^{-2}$} & 1.41$e^{-1}$ & 3.16$e^{-2}$ & 3.62$e^{-2}$ & 5.53$e^{-2}$ & 3.00$e^{-2}$ \\
    %& CIFAR-10 & \textcolor{codegreen}{9.98$e^{-2}$} & \textcolor{blue}{1.92$e^{-1}$} & 2.21$e^{-1}$ & 2.38$e^{-1}$ & 2.30$e^{-1}$ & \textcolor{red}{1.84$e^{-1}$} & 2.33$e^{-1}$ & 2.34$e^{-1}$ & 2.26$e^{-1}$ & 2.15$e^{-1}$ \\
    DeBERTa v2 XXL \cite{deberta} %& MRPC & 8.15$e^{-2}$ & \textcolor{red}{6.89$e^{-2}$} & \textcolor{codegreen}{6.53$e^{-2}$} & 8.09$e^{-2}$ & \textcolor{blue}{8.02$e^{-2}$} & 9.08$e^{-2}$ & 9.31$e^{-2}$ & 1.10$e^{-1}$ & 9.47$e^{-2}$ & 1.22$e^{-1}$\\
    %& TREC50 & 3.38$e^{-1}$& 2.36$e^{-1}$& \color{codegreen}{7.04$e^{-2}$} & \textcolor{blue}{1.53$e^{-1}$} & 2.24$e^{-1}$& \color{red}{1.36$e^{-1}$}& 1.85$e^{-1}$& 2.22$e^{-1}$& 1.93$e^{-1}$& 1.92$e^{-1}$\\
    & PAWS & 6.07$e^{-2}$& \textcolor{red}{1.99$e^{-2}$} & 3.63$e^{-2}$ & \textcolor{blue}{3.26$e^{-2}$} & 3.95$e^{-2}$ & 5.41$e^{-2}$ & 6.68$e^{-2}$ & 5.11$e^{-2}$ & \textcolor{codegreen}{1.98$e^{-2}$} & 6.99 $e^{-2}$\\
    \midrule
    DeBERTa v3 Base \cite{debertav3} & MRPC & 1.06$e^{-1}$ & 8.90$e^{-2}$ & {2.59$e^{-2}$} & 7.28$e^{-2}$ & 9.86$e^{-2}$ & \textcolor{red}{1.52$e^{-2}$} & {2.58$e^{-2}$} & \textcolor{codegreen}{8.52$e^{-3}$} & 1.16$e^{-1}$ & \textcolor{blue}{2.57$e^{-2}$}\\
    & TREC50 & 4.56$e^{-1}$ & 2.73$e^{-1}$ & 3.99$e^{-1}$ & \textcolor{blue}{2.16$e^{-1}$} & 2.67$e^{-1}$ & \textcolor{red}{2.61$e^{-2}$} & {2.25$e^{-1}$} & {3.70$e^{-1}$} & 3.36$e^{-1}$ & \textcolor{codegreen}{2.63$e^{-2}$}\\
%    & PAWS & 2.62$e^{-2}$ & 6.43$e^{-2}$ & \textcolor{codegreen}{2.40$e^{-2}$} & 6.27$e^{-2}$ & 8.17$e^{-2}$ & \textcolor{red}{5.55$e^{-2}$} & 7.39$e^{-2}$ & \textcolor{blue}{5.77$e^{-2}$} & 1.01$e^{-1}$ &  5.82$e^{-2}$\\
    \midrule
    %RoBERTa-Base \cite{roberta} %& MRPC & 9.48$e^{-1}$ & 6.01$e^{-1}$ & \textcolor{red}{2.05$e^{-1}$} & \textcolor{codegreen}{1.64$e^{-1}$} & \textcolor{blue}{2.20$e^{-1}$} & 5.33$e^{-1}$ & 4.37$e^{-1}$ & 3.78$e^{-1}$ & 3.35$e^{-1}$ & 3.21$e^{-1}$ \\
   % & CoLA & 1.39 & 7.74$e^{-1}$ & 4.04$e^{-1}$ & \textcolor{red}{2.22$e^{-1}$} & \textcolor{codegreen}{1.96$e^{-1}$} & 8.10$e^{-1}$ & 4.70$e^{-1}$ & 4.43$e^{-1}$ & 4.38$e^{-1}$ & \textcolor{blue}{4.01$e^{-1}$} \\
    %\midrule
    %RoBERTa-Large \cite{roberta} & MRPC & 7.29$e^{-1}$ & 4.64$e^{-1}$ & 4.71$e^{-1}$ & \textcolor{blue}{2.77$e^{-1}$} & \textcolor{red}{2.68$e^{-1}$} & \textcolor{codegreen}{2.64$e^{-1}$} & 6.54$e^{-1}$ & 5.57$e^{-1}$ & 5.27$e^{-1}$ & 3.84$e^{-1}$ \\
    %& CoLA & 8.06$e^{-1}$ & 4.25$e^{-1}$ & 4.18$e^{-1}$ & \textcolor{blue}{2.36$e^{-1}$} & \textcolor{codegreen}{1.75$e^{-1}$} & \textcolor{red}{2.28$e^{-1}$} & 4.96$e^{-1}$ & 4.56$e^{-1}$ & 6.14$e^{-1}$ & 4.05$e^{-1}$ \\
    TinyLlama  \cite{zhang2024tinyllama} & OpenBookQA & \textcolor{blue}{1.78$e^{-1}$} & 2.82$e^{-1}$ & 3.41$e^{-1}$ & {2.15$e^{-1}$} & 1.86$e^{-1}$ & 2.07$e^{-1}$ & \textcolor{red}{1.51$e^{-1}$} & 2.20$e^{-1}$ & 3.16$e^{-1}$ & \textcolor{codegreen}{7.59$e^{-2}$} \\
    & FOLIO & {1.82$e^{-1}$} & 2.37$e^{-1}$ & {2.17$e^{-1}$} &  \textcolor{blue}{1.75$e^{-1}$} & 1.93$e^{-1}$ & \textcolor{codegreen}{5.11$e^{-2}$} & 2.35$e^{-1}$ & 1.91$e^{-1}$ & \textcolor{red}{1.05$e^{-1}$} & {2.49$e^{-1}$}\\
    %& LogiQA & 3.61$e^{-1}$ & \textcolor{codegreen}{6.12$e^{-3}$} & 1.45$e^{-1}$ & \textcolor{red}{1.16$e^{-2}$} & 1.75$e^{-1}$ & 2.37$e^{-1}$ & 8.60$e^{-2}$ & 1.1$e^{-1}$ & 6.64$e^{-2}$ & \textcolor{blue}{6.25$e^{-2}$}\\
    & CLUTRR & 4.  29 & 2.25 & \textcolor{codegreen}{1.55} & 2.34 & 2.27 & 5.48 & \textcolor{red}{2.16} & \textcolor{blue}{2.19} & 2.59 & 4.23\\
    \midrule
    Llama3-8B~\cite{grattafiori2024llama}
    & CLUTRR & \textcolor{red}{2.53} & 2.66 & 2.97 & 2.9 & 5.49 & \textcolor{blue}{2.65} & 2.69 & 5.02 & \textcolor{codegreen}{2.51} & 4.33 \\
    \midrule
    DeepseekCoder \cite{guo2024deepseekcoder} & DJANGO & \textcolor{red}{3.48$e^{-2}$} & {4.65$e^{-2}$} & \textcolor{codegreen}{3.4$e^{-2}$} & 5.16$e^{-2}$ & 4.64$e^{-2}$ & {3.87$e^{-2}$} & {4.19$e^{-2}$} & {3.89$e^{-2}$} & {3.64$e^{-2}$} & \textcolor{blue}{3.62$e^{-2}$}\\
    \midrule
    GPT2-Small~\cite{gpt2} & E2E & \textcolor{codegreen}{1.65$e^{-1}$} & 1.93$e^{-1}$ & 1.85$e^{-1}$ & 1.83$e^{-1}$ & 1.85$e^{-1}$ & 1.87$e^{-1}$ & \textcolor{red}{1.77$e^{-1}$} & 1.82$e^{-1}$ & 1.88$e^{-1}$ & \textcolor{blue}{1.82$e^{-1}$} \\
  %\midrule
  \bottomrule
  \end{tabular}
  \vspace{0mm}
\end{table*}

\vspace{0mm}
\subsection{Quality of the Fine-Tuned Models}\label{subsec:quality_of_models}
\vspace{0mm}
In Table~\ref{tab:full-accuracy-table}, we present fine-tuning performance of various models with FFT and LoRA-based PEFTs. For the CoLA dataset, we report the Matthews Correlation Coefficient (higher is better)~\cite{mcc}. For reporting GPT2-small's results, we use perplexity (lower is better); for the rest of the models and datasets, we report test accuracies (higher is better). Each model is trained over 3 seeds, and we average the results. We find that no single method substantially outperforms the others for adapting the model to their downstream tasks, including FFT, which confirms the previous findings in \cite{PEFT}. In many cases, FFT performs rather poorly (e.g., ViT-Base on OfficeHome, DeBERTa v3 on RTE, DeepseekCoder on DJANGO). The sparsity-induced SOTA LoRA variants outperform FFT and LoRA in some tasks by a large margin~(e.g., ViT-Base on OfficeHome, DeBERTA v3 on MRPC); in many cases, their performance drop is modest. By reducing the memory footprint, the sparse PEFT methods perform well for large models (e.g., Llama 3), even with low batch sizes and short sequence lengths. However, the sparse variants cannot always produce the best accuracy in low-epoch fine-tuning, but they still generalize well; see Table~\ref{tab:combined}. Even when the sparse PEFT methods perform undesirably, their performance improves significantly by increasing the rank; see cLA's substantial EM improvement when fine-tuning DeepseekCoder for Django with a higher rank. cLA's parent method, Asym LoRA, performs well in a lower rank budget; this trend switches at a higher rank. It is a surprising case of its own, which requires a deeper understanding of the rank and column space interplay; we have a dedicated discussion on DeepSeek's performance in~\S\ref{subsec:deepseekcoder_perf}. This suggests that when fine-tuning a model for a downstream task, it may be optimal to select a fine-tuning method based on its other characteristics and user-specific needs, rather than just the generated accuracy.~To highlight this, in \S\ref{sec:performance analysis}, we analyze the performance of each method based on training time, generalizability, and robustness for adapting to further downstream tasks. Although the sparse variants do not reduce the number of trainable parameters compared to their non-sparse LoRA counterparts, they \emph{reduce training time by 5-10\% and peak GPU memory by 5-15\%, with a naïve, non-optimized, sparse implementation}; see throughput, peak memory, and runtime in \S\ref{subsec:efficiency}. %, Tables \ref{tab:combined_metrics_all}-\ref{tab:combined_metrics_qv}. 

\subsection{Generalizability of the Fine-Tuned Models}\label{sec:generalizability} 
%\vspace{-1mm}

The generalization error, $\cG(\textbf{W})$ (Definition \ref{definition:generalization}), is hard to realize in practice, as the true distribution of a feature space and label space, $\mathcal{X}\times\mathcal{Y}$, cannot be obtained. Therefore, we cannot use the theoretical bounds on $\cG(\textbf{W})$ in Table \ref{tab:upperbounds} without modification. Since test samples are i.i.d. from ($\mathcal{X} \times \mathcal{Y}$), as an alternative, the difference between the loss of a model on a collection of unseen test samples and the loss on its training set approximates how well the model generalizes to the true distribution of the instance space it was trained on. Therefore, we approximate ${\cG}(\textbf{W})\approx\mathbb{E}(\mathcal{L_{\rm test}}) - \mathcal{L_{\rm train}}.$ As the size of the test set increases, the difference approaches the actual ${\cG}(\textbf{W})$ of the model; see discussion in \S\ref{subsec:genconnect}. {We report the approximate generalizability of all fine-tuned models in Tables \ref{tab:gen-error-mini} via the epoch with the lowest validation loss averaged over three seeds to avoid overfitting; see additional results in \S\ref{subsec:genconnect}, Table \ref{tab:combined}.} {Since loss scales differently across various tasks, we interpret the loss gaps within the same task and model setting; any cross-task comparisons are presented for trend inspection rather than validation of the bounds. We report the normalized scores and their observed trends in \S~\ref{subsubsec:gen_norm}}. %\hl{Eli and Cristian, we need to understand the normalized G(W) results and see if they are a better fit in the main paper.}\aritra{12/24: ELi, please write down your comments that we discussed on 12/22 about normalized $\cG(W).$}
Drawing a connection from our theoretical upper bounds in Table \ref{tab:upperbounds}, we find PEFT methods with the same upper bounds perform similarly in practice. More precisely, cLA has a smaller upper bound on $\cG(\textbf{W})$ than r-$c^3$LA in practice, indicating the validity of theoretical upper bounds. This observation also holds for cLA and RAC, and c$^{3}$LA and Asymmetric LoRA pairs. On the other hand, cLA and r-cLA have the same upper bound on $\cG(\textbf{W})$, and they also perform almost similarly in practice. Nevertheless, there are some discrepancies, and we attribute them to the fact that Table \ref{tab:upperbounds} gives us an upper bound on $\cG(\textbf{W}).$ E.g., \myNum{i} although the upper bound on $\cG(\textbf{W})$ of Asymmetric LoRA is smaller than RAC by a factor of $\sqrt{k}$, they behave similarly in practice. \myNum{ii} Similarly, r-cLA performs marginally worse than RAC, although RAC has a higher theoretical bound on $\cG(\textbf{W})$. \myNum{iii} $c^3$LA and RAC-LoRA have similar theoretical bounds; in practice, we notice stronger generalization trends for $c^3$LA in comparison to all other variants. \myNum{iv} In an extreme case, r-$c^3$LA empirically outperforms r-cLA while having a higher theoretical bound on $\cG(\textbf{W}).$

\subsection{Performance Analysis}\label{sec:performance analysis}

With trained model quality and empirical $\cG(\textbf{W})$, we were curious to dissect the performance of different LoRA variants using two popular and practically useful tools, {\myNum{i} loss landscape \cite{losslandscape}, and \myNum{ii} intruder dimensions \cite{loravfft}, that assess a model's quality. We examine whether these tools can aid in our understanding of which fine-tuning method to use when a model and task are at hand, given that we measured their performance and empirical $\cG(\textbf{W})$ beforehand and have a basis for comparison. Instead of limiting these tools to a single task, we unleash them across tasks, models, and modalities.}

\myNum{i}\smartparagraph{3D-loss landscape} visualizes how a model's empirical loss differs under small parameter perturbations; see details in~\S\ref{appendix:loss-landscape}.~A sharper loss landscape indicates worse generalization, smoother landscapes indicate the PEFT method is more robust to initialization~\cite{losslandscape}.~In Figure~\ref{fig:roberta-officehome-landscapes}, the top row shows the loss-landscapes of ViT-Base, pretrained on {Imagenet-21K}, and fine-tuned on {OfficeHome}, % by cLA, $c^3$LA, LoRA, FFT, and RAC, 
while the bottom row shows the loss-landscapes of RoBERTa-Base, pretrained on a large corpus of English data and fine-tuned on CoLA. %by r-$c^3$LA, LoRA, CoLA, Asymmetric LoRA, and FFT.
~For ViT-Base, we used PCA directions, whereas for RoBERTa-Base, we used random directions; see~\S\ref{appendix:loss-landscape}, for their comparison. % of these two implementations. 
We present a direct comparison of non-chain LoRA methods (LoRA, Asymmetric LoRA, cLA) with their chain counterparts (CoLA, RAC-LoRA, $c^3$LA) in Figure~\ref{fig:direct-chain-method-comparison}. In \S\ref{appendix:loss-landscape}, we plotted the optimizer path in 2D contour plots.

Based on the loss landscapes' characteristics in \cite{losslandscape}, {for ViT-Base fine-tuned on OfficeHome, FFT would generalize worse with the worst test accuracy; Figure \ref{fig:roberta-officehome-landscapes} confirms this. But recall from Figure \ref{fig:teaser-fig-ViT-Cifar10}, for ViT-Base fine-tuned on CIFAR-10, FFT produced the lowest $\cG(\textbf{W})$, and yielded the worst test accuracy. In both cases, FFT produced the spikiest loss and small-volume minima.} We witnessed {from Figures \ref{fig:teaser-fig-ViT-Cifar10} and \ref{fig:roberta-officehome-landscapes}, chain methods (e.g., RAC, c$^3$LA), sharpen the minima of the fine-tuned DNN models,} and these sharper valleys indicate that these models should generalize worse. However, in practice, this is not the case. E.g., For ViT-Base, in Figure \ref{fig:roberta-officehome-landscapes}, RAC {and c$^3$LA have the lowest $\cG(\textbf{W})$, although cLA has a wider valley, its $\cG(\textbf{W})$ is higher, indicating it should generalize worse. Contrastingly, all three PEFT methods produced similar high-quality fine-tuned models on OfficeHome and CIFAR-10, albeit with slight differences.} %\hl{Eli, based on the above two paragraphs, write the results for Roberta-Base in Blue text.}
{This discrepancy between practice and theory is consistent across video and text model modalities. For RoBERTa-Base in Figure~\ref{fig:roberta-officehome-landscapes}, the chain methods, CoLA and r-$c^3$LA, produce sharper landscapes than the non-chain methods. Still, CoLA has a lower $\cG(\textbf{W}).$ }

% The loss landscape rhetoric, as we also observed in Figure \ref{fig:teaser-fig-ViT-Cifar10}, \textcolor{blue}{does not always match with our empirically evaluated $\cG(\textbf{W})$ of the fine-tuned models.}

\smartparagraph{Key Takeaway.} {Generalizability and model performance are two sides of one coin, as is generally agreed \cite{losslandscape}. However, we observed that this perspective can be conflicting for fine-tuned models, as the narrative surrounding loss-landscape sharpness versus empirically measured $\cG(\textbf{W})$ of these models is mostly contradictory. }  

%This is consistent with the intruder dimension analysis of the PEFT methods. %\hl{You will write the Messages. We also plotted the 2D contour plots to show the optimizer path; see Appendix.}

%\aritra{1. Loss landscape for ViT-Base on OfficeHome for r-cLA, r-C3LA, LoRA, and RAC. The rest of the methods we will put in the Appendix.}
\begin{figure*}
  \centering
  % SET ONE (OfficeHome PCA Fourset, LoRA + CoLA, Asym-LoRA + RAC-LoRA)
  \scalebox{0.99}{\includegraphics[width=0.95\textwidth, height=2in]{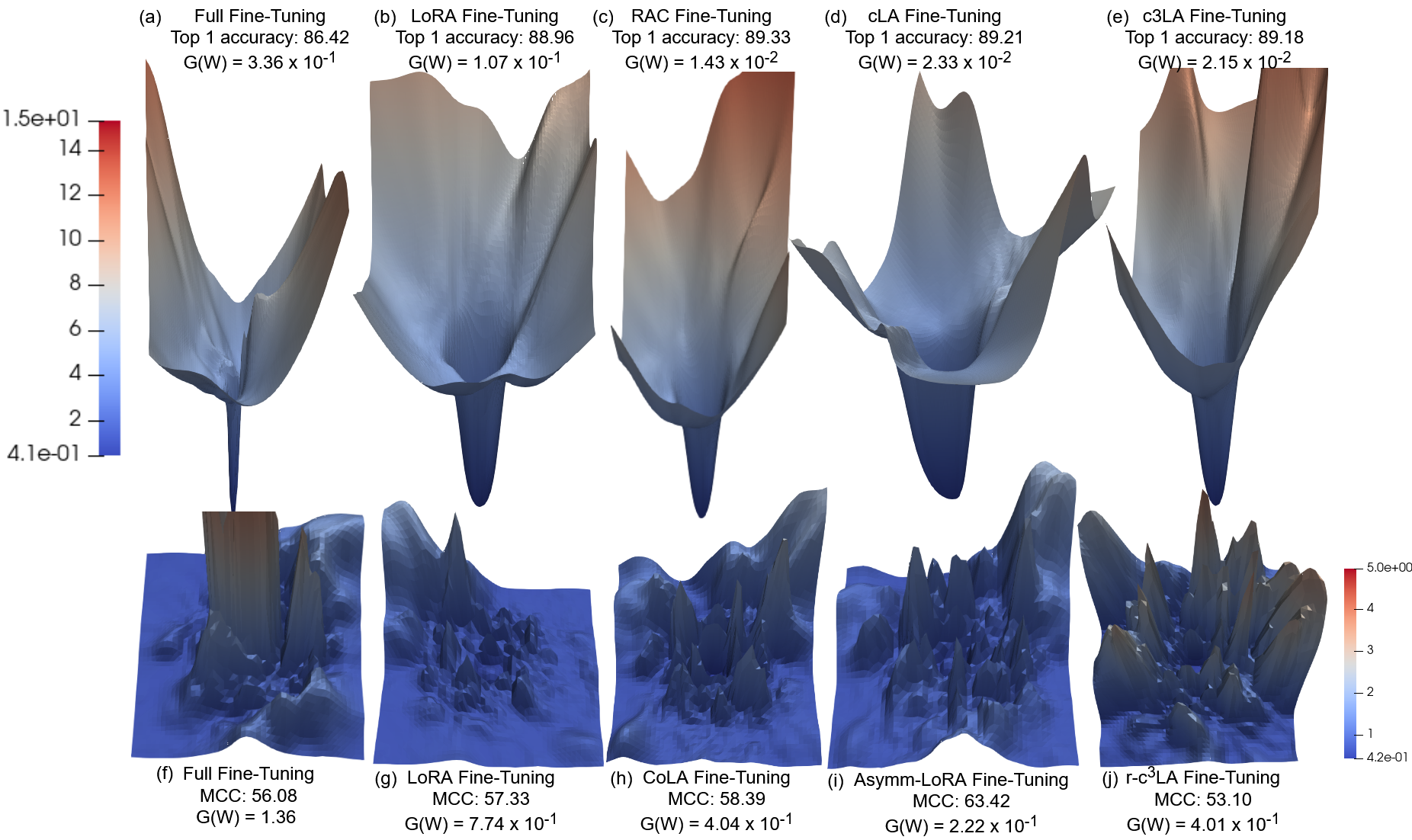}}
  \vspace{-2mm}
\caption{\small{Loss landscapes of {ViT-Base} fined tuned on {OfficeHome} (top row) with PCA directions, and {RoBERTa-Base} fine-tuned on {CoLA} (bottom row) with random directions. In both cases, we observe the worst generalization error, $\cG(\textbf{W})$, in (a) and (f), respectively, which are the spikiest landscapes in their class of models. Additionally, chain methods consistently produce spikier landscapes. }}
  \label{fig:roberta-officehome-landscapes}
\end{figure*}
\begin{figure*}
%\vspace{-2ex}
      \centering
	 \scalebox{0.95}{\begin{subfigure}{0.32\linewidth}
		\includegraphics[width=\linewidth]{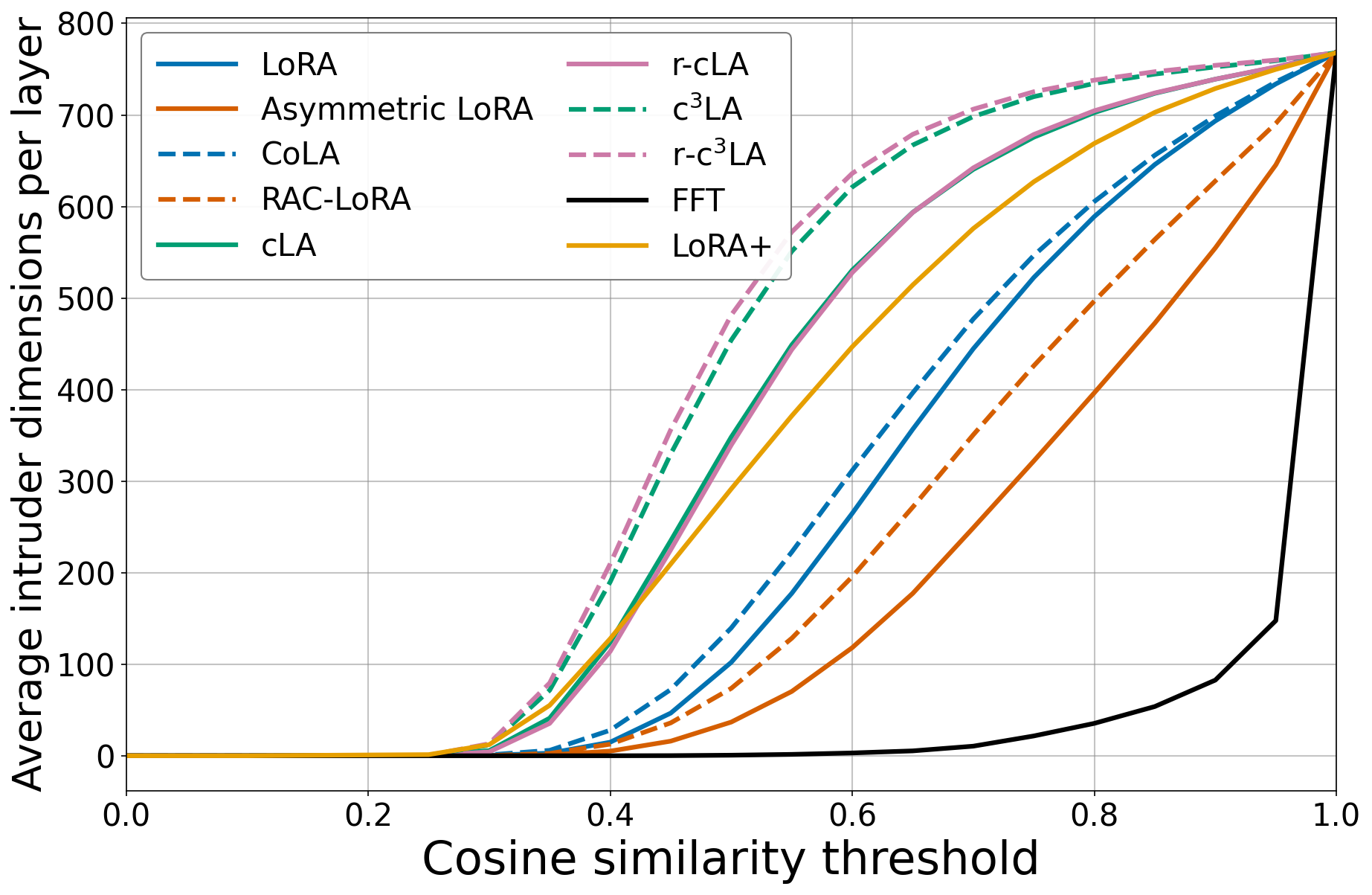}
		\caption{\tiny{RoBERTa-Base~(CoLA)}}
		\label{fig:RoBERTa-intruder-figures}
	\end{subfigure}
	\begin{subfigure}{0.32\linewidth}
		\includegraphics[width=\linewidth]{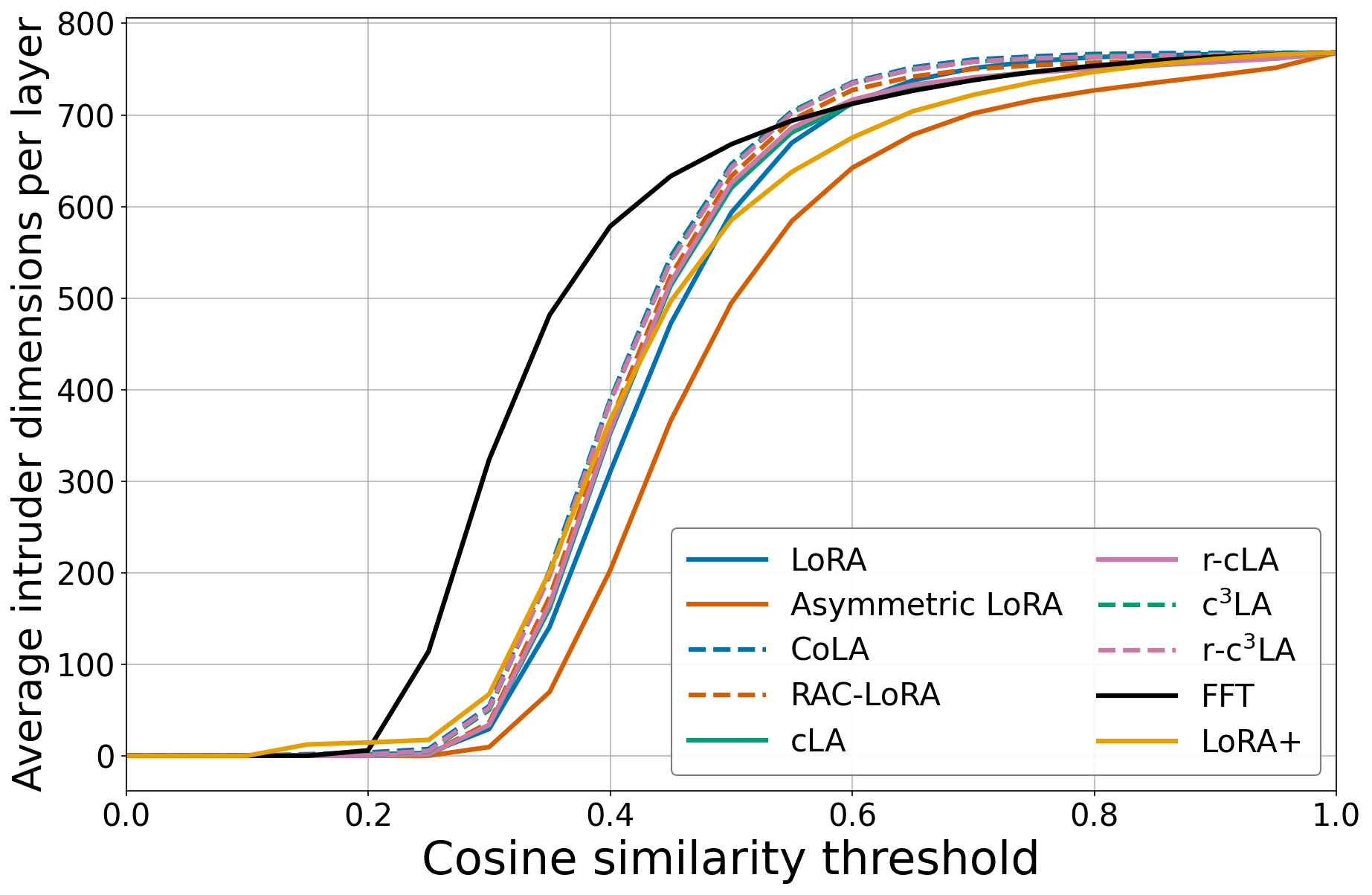}
		\caption{\tiny{ViT-Base~(OfficeHome)}}
		\label{fig:ViT-intruder-figures}
	\end{subfigure}
    \begin{subfigure}{0.32\linewidth}
		\includegraphics[width=\linewidth]{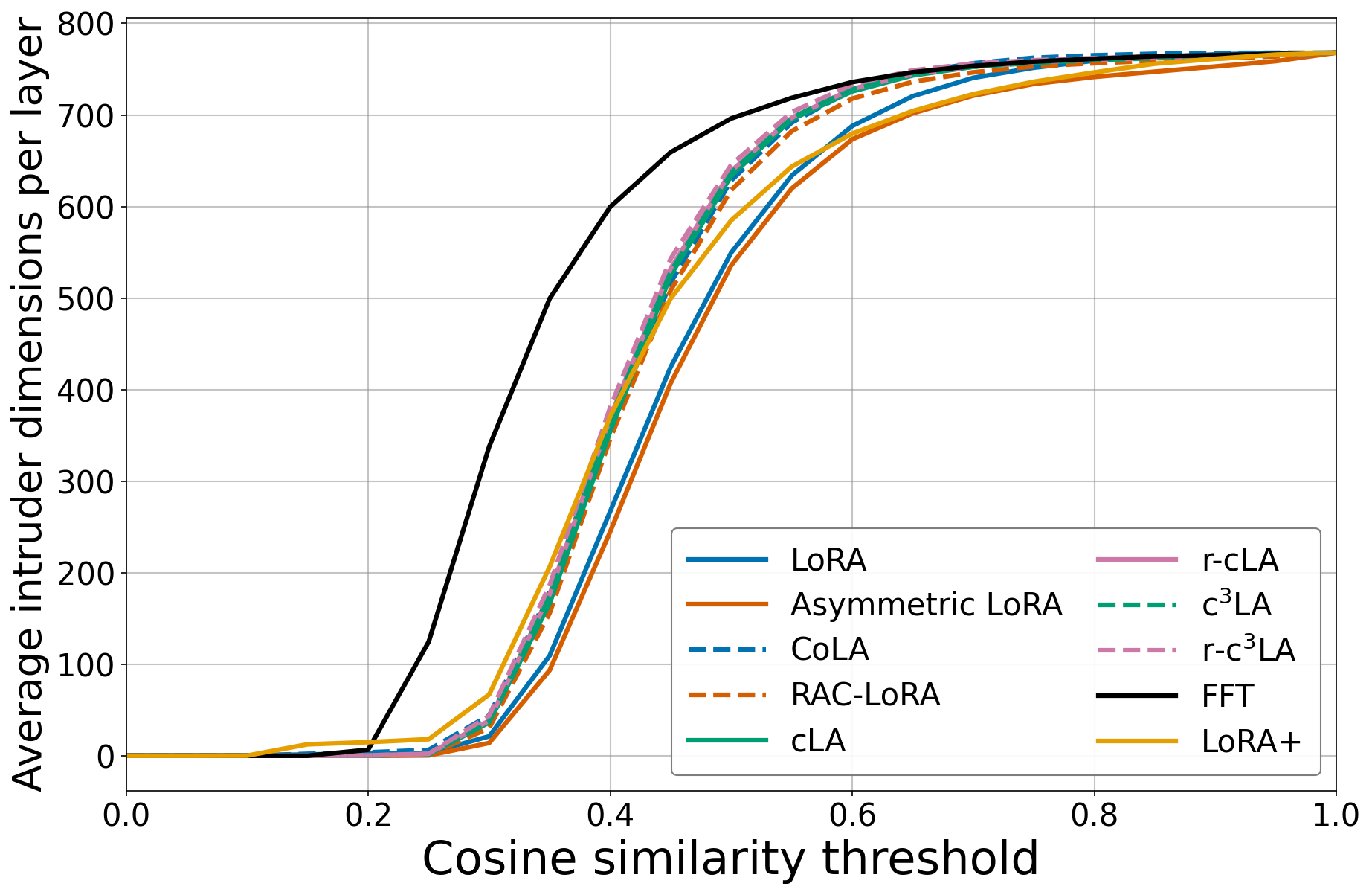}
		\caption{\tiny{ViT-Base~(CIFAR-10)}}
		\label{fig:vit-c10-intruder}
	\end{subfigure}}
	\caption{\small{The average number of intruder dimensions present in different fine-tuned models at the end of training. For each method that has a corresponding chain variant (LoRA to CoLA, Asymmetric LoRA to RAC-LoRA), their colors are the same where the chain method is a dashed line.}
    }\label{fig:intruder_dimension}
\end{figure*}

\myNum{ii}\smartparagraph{Intruder dimension} 
{\cite{loravfft} compares the performance between the fine-tuned models of LoRA and FFT.} Given the pretrained and fine-tuned models, $\textbf{W}_0$ and $\textbf{W}_0 + \Delta \textbf{W}$, the number of \emph{intruder dimensions} correlates with their performance on the pretraining task; \emph{higher intruder dimensions correlate to a worse performance.} % (see \S\ref{appendix:Cosine Similarity})}~\cite{loravfft}. 
We ask: \emph{Will forgetting less of the more diverse dataset indicate better generalizability?} and examine LoRA variants and FFT via this perspective, and how intruder count aligns with our empirical $\cG(\textbf{W})$. In Figure~\ref{fig:intruder_dimension}, we present the number of intruder dimensions present in FFT and various LoRA-based PEFT methods for RoBERTa-Base fine-tuned on CoLA and ViT-Base fine-tuned on OfficeHome and CIFAR-10 over varying threshold ranges, ${\varepsilon\in(0,1]}$. 
%As shown in the Figures~\ref{fig:RoBERTa-intruder-figures} and \ref{fig:ViT-intruder-figures},

%\hl{12/24: I do not understand most of this paragraph; we need to rewrite.} 
{In Figure~\ref{fig:intruder_dimension}(a), RoBERTa-Base fine-tuned on FFT has some of the least intruders yet has the worst generalizability; see Table~\ref{tab:gen-error-mini}. The methods that produce the least generalization error for RoBERTa-Base, Asymmetric LoRA, and RAC both produce relatively few intruders. Conversely, ViT-Base fine-tuned on both OfficeHome and CIFAR-10 via FFT produces some of the most intruders while having the best generalization for Cifar10, and relatively poor for OfficeHome; see Table~\ref{tab:combined}. For all three experiments, the chain variant of any LoRA PEFT method produces more intruders than its non-chain counterpart; see LoRA compared to CoLA, Asymmetric LoRA to RAC, cLA to c$^3$LA in Figure~\ref{fig:intruder_dimension}. This correlates with our loss landscapes, where chain variants produce sharper landscapes. However, the expected worse generalizability of these chain methods is not observed empirically.}

%The chain variant of each LoRA-based PEFT method produces more intruders than its non-chain counterpart. This effect is least pronounced in LoRA and CoLA, which produced almost the same number of intruders for RoBERTa. This is consistent with results in Figure \ref{fig:roberta-officehome-landscapes}---If CoLA produced more intruders than LoRA, it would never have a better $\cG(\textbf{W})$ than LoRA. But, for ViT-Base, this observation does not hold. Also, we note that RAC has the best $\cG(\textbf{W})$ on ViT-Base, while producing substantially more intruders than LoRA and LoRA+. Additionally, from Figure~\ref{fig:vit-c10-intruder}, we find that FFT has the highest intruder dimensions but the least $\cG(\textbf{W})$ (Table \ref{tab:combined}), LoRA+ has more than average intruder dimensions but the second best $\cG(\textbf{W})$ (Table \ref{tab:combined}); only LoRA's intruder dimensions and $\cG(\textbf{W})$ follow the correct trend for OfficeHome and CIFAR-10. 

%\hl{For Roberta-Base fine-tuned on COLA, FFT has fewer intruder dimensions than LoRA, and this follows the observation from Shuttleworth et al., 2024; LoRA has better $\cG(W)$ than FFT. When considering the extreme cases of this experiment, LoRA+ has the highest intruder dimensions and the second-worst $\cG(W)$, Asym has the lowest intruders, and the second-best $\cG(W)$, RAC has the best $\cG(W)$, but contains more intruders than FFT.}

\smartparagraph{Key Takeaway.}~Although a more promising tool than loss-landscape, and an excellent addition to analyze FFT and PEFT, intruder count does not always match with our empirically evaluated $\cG(\textbf{W})$ either. This is more prominent in the vision task. The original paper reports experiments involving only language models; we encourage the community to explore diverse modalities and models.

%and see how different DNN models fine-tuned for different tasks respond to this tool.} 

\subsubsection{Discussion}
{Using alternative diagnostics such as loss landscapes and intruder dimensions to analyze the fine-tuned model quality, we observe that these tools do not consistently align the empirical trends between test accuracy (Table~\ref{tab:full-accuracy-table}) and generalizability (Table~\ref{tab:gen-error-mini}). %They offer valuable insights in some cases. 
While they provide informative post hoc rationalization in some settings, their results are often ambiguous, producing false positive scenarios, suggesting good generalization even when $\cG(\mathbf{W})$ and test accuracy do not support it.
%For some models and modalities, these tools perform as desired. In other cases, the results are ambiguous. In particular, across tasks and datasets, these perspectives yield false positives, i.e., they may suggest that a method generalizes even when the observed $\cG(\mathbf{W})$ and obtained test accuracy do not support that. 
These discrepancies motivate us to present an alternative analysis that better aligns with empirically observed $\cG(\textbf{W})$ and the obtained accuracy. This perspective produces fewer false positives.} {We observe that the theoretical results in Table~\ref{tab:upperbounds} and the empirical results in Table~\ref{tab:gen-error-mini} are relatively comparable, and the strong generalizable models in Table~\ref{tab:gen-error-mini} typically perform well in terms of rendering higher accuracy in Table~\ref{tab:full-accuracy-table}. Consequently, %while proxy diagnostics can be informative, 
jointly measuring test accuracy and experimentally evaluating the resulting model's $\cG(\mathbf{W})$ provides a more consistent predictive criterion for selecting whether a fine-tuning method is suitable for a particular model, task, and specific dataset. In the era of artificial general intelligence, when we want a model to behave in a human-like way across many tasks, we conclude that it is better to choose PEFT methods that generalize well and are computationally efficient.}

%After using alternative diagnostics such as loss landscapes and intruder dimensions to analyze model quality, we observe that these tools do not consistently connect the empirical trends between test accuracy (Table 2) and generalizability (Table 3). They offer valuable insights in some cases. For some models and modalities, these tools perform as desired. In other cases, the results are ambiguous. In particular, across tasks and datasets, these perspectives yield false positives, i.e., they may suggest that a method generalizes even when the observed G(W) and obtained test accuracy do not support that. These discrepancies motivate us to present an alternative analysis that aligns well with our empirical observation on the G(W) and the obtained accuracy, and produces fewer false-positive implications. We observe that the theoretical results in Table 1 and the empirical results from Table 3 are relatively comparable. We also observe that the more generalizable models in Table 3 typically perform well in terms of rendered test accuracy in Table 2. Hence, while proxy diagnostics can be informative, jointly measuring test accuracy and experimentally evaluating the resulting model’s G(W) provides a more consistent criterion for selecting whether a fine-tuning method is suitable for a particular model and task on a specific dataset. In the advent of artificial general intelligence, when we want a model to behave human-like across many tasks, it is better to choose PEFT methods that generalize well and are computationally efficient.

%\myNum{iii}

\section{Conclusion}\label{sec:conclusion}
Through extensive benchmarking %spanning four fine-tuning tasks, ten models, and fourteen datasets, 
we show the complex, task-dependent nature of PEFT performance, including FFT. This aligns with prior findings. % that dissect LoRA's efficacy. 
%However, our core contributions are two-fold. 
%As the future of computing and hardware interfaces moves towards memory- and compute-efficiency, 
Our proposed sparse extensions of SOTA LoRA variants perform well across multiple modalities and models while substantially reducing training time and memory requirements. From a theoretical perspective, our sparsity-induced variants serve as a bridge between LoRA and PaCA, two different families of PEFT methods. While these sparse variants may require larger budgets to maintain robustness in certain settings (e.g., code generation with DeepSeekCoder), they remain overall effective, highlighting the importance of selecting fine-tuning methods based on task characteristics and user constraints. To support this, we analyzed various common LoRA PEFT variants through the lens of generalizability. %; {this is our second major contribution}. 
We show that, in theory, the sparse methods have the same generalization error upper bounds as their non-sparse counterparts, and closely track the empirical generalization trend across most models and modalities. %When comparing our theoretical generalization error bounds with experimentally observed generalization error, we find that our upper bounds closely track the empirical generalization trend across most models and modalities. 
This insight provides a more consistent and guided pathway for selecting PEFT methods, complementing existing diagnostic tools such as loss-landscape and intruder-dimension analyses.

\smartparagraph{Acknowledgment.} Aritra Dutta is partially supported by the Florida Department of Health Grant, AWD00007072, and the National Science Foundation Grant, 2321986.

%\addcontentsline{toc}{section}{Impact Statement}
%\input{Sections/Impact_statement.tex}

%%%%%%%%%%%%%%%%%%%%%%%%%%%%%%%%%%%%%%%%%%%%%%%%%%%%%%%%%%%%

\bibliographystyle{plain}
\bibliography{References}
\clearpage

%
%\appendix

\appendix
%\tableofcontents

%\phantomsection
%\addcontentsline{toc}{section}{Appendix}

\textbf{Organization of Appendix.} We organize the Appendix with the following structure:In \S\ref{Appendix:present}, we discuss the popular contemporary LoRA variants; this is a continuation of \S\ref{sec:present} of the main paper. In \S\ref{subsec:algorithm}, we give the pseudocode of our proposed LoRA variants, cLA, random-cLA, and $c^3$LA. \S\ref{Appendix:Theoretical_Results} contains the proofs to the theorems in \S\ref{sec:theory}. Particularly, it contains the proofs for %Theorem~\ref{theorem:layer-wise convergence}, 
Theorem~\ref{theorem:nonlinearGenBound} and Theorem~\ref{theorem:PAC-Extension}. 
%and Theorem~\ref{theorem:LoRa smoothness}. 
Additionally, this section contains the statement and proof of Theorem \ref{theorem:nonlinearGenBound} adapted to the attention mechanism; see Theorem \ref{thm:attn_head} in \S\ref{subsubsec:transformer}.
In \S\ref{appendix:benchmarking}, we discuss the implementation details and extend our empirical study by including various ablation studies and developing discussion topics. This section acts as an addendum to \S\ref{sec:empirical results} of the main paper. For notations used in this paper, we refer to Table \ref{tab:notation} in \S\ref{app:notation}. 

\section{The Present: Evolution of LoRA --- Continued}\label{Appendix:present}

\smartparagraph{Other LoRA variants.} There are other popular LoRA variants, such as HydraLoRA \cite{hydra}, designed for fine-tuning on datasets with high heterogeneity. LoRA-SB~\cite{lora-sb} simulates the FFT process within low-rank subspaces by adding a trainable $r\times r$ matrix $R$, initializing $BRA$ based on the SVD of the first step of FFT, and freezing $B, A$. QLoRA~\cite{qlora} fine-tunes quantized LLMs. AdaLoRA~\cite{adalora} uses varying rank by layer and uses an SVD initialization. SoRA~\cite{sora} introduces sparsity in the low-rank updates. DoRA~\cite{dora} separates fine-tuning the direction and magnitude components of the model. AutoLoRA~\cite{autolora} trains each LoRA update as a sum of rank-one matrices and learns which to discard during training. DyLoRA~\cite{dylora} concentrates the more important features in the first columns and rows of $B$ and $A$, respectively. LoRA-GA~\cite{wang2024lora} improves LoRA convergence by initializing low-rank adapters using a gradient approximation of full fine-tuned updates. LoRTA~\cite{hounie2024lorta} replaces matrix-based adapters with a low rank tensor factorization that unifies model updates across layers and attention heads. Flat-LoRA~\cite{li2024flat}  seeks flat minima in the full-parameter space to improve model generalization. BayesLoRA~\cite{doyle2025bayeslora} learns the effective adapter rank through Bayesian low-rank variational dropout. rsLoRA~\cite{rsLoRA} stabilizes LoRA training across ranks by correcting the scaling factor from $1/r$ to $1/\sqrt{r}$, preventing gradient collapse. GoRA~\cite{gora} is a gradient-based LoRA variant that performs adaptive rank allocation and initializes adapter weights using compressed gradients. {PaCA~\cite{paca} proposes fine-tuning randomly selected connections within pretrained weights to improve training speed and reduce overhead.} 

{The above-mentioned PEFT methods are important extensions of the LoRA family of PEFTs. Although we do not claim any technical or algorithmic novelty in designing another PEFT method in this paper, our research question is still a well-timed and important one. We asked, instead of task-specific tailoring, can we induce simple sparsity, which is easy to implement, in the present SOTA LoRA variants and witness their superior performance across diverse tasks; See Tables \ref{tab:full-accuracy-table}, \ref{tab:gen_bounds_lora_variants}, and Figure \ref{fig:cla_deepseek}.  At the same time, through a unified information-theoretic generalization error bound analysis framework, we demonstrate that the sparsity-induced variants share similar upper bounds with their parent PEFT methods. However, in practice, they may differ and show substantially better performance in many cases compared to their parent PEFT methods. In summary, except for PaCA, we share little to no conceptual proximity to other LoRA variants. In the next section, we draw this connection deeper, and based on our understanding of the SOTA LoRA variants, we introduce a few new artifacts to PaCA.}

\section{Relationship between PaCA and cLA}\label{appendix:paca-cla}
% \textcolor{blue}{PaCA and cLA are closely related, as they both exist as restricted column updates to a pre-trained model. However, they differ in where the restriction is imposed (directly inside the pre-trained backbone vs. via a LoRA parameterization) and the motivation behind their construction.}

% \begin{figure}
%     \centering
%     \includegraphics[width=0.8\linewidth]{Images/first_round_graph.png}
%     \caption{Representation of the evolution of fine-tuning methods under PEFT~\cite{PEFT}, connecting reparameterization and partial fine-tuning through our sparsity-induced LoRA variants.}
%     \label{fig:evolution_of_models}
% \end{figure}

{During fine-tuning, PEFT methods, such as LoRA adapters, must be processed with the backbone and cannot be merged, which limits their hardware utilization. PaCA~\cite{paca} is motivated by the training-time inefficiencies of adapter-based PEFT. PaCA fine-tunes a random subset of the pretrained model's columns explicitly and uses partial activations to form gradients, improving throughput and lowering activation-memory cost~\cite{paca}. In contrast, cLA stays within the LoRA family to analyze extreme sparse LoRA structures and for simple LoRA-based adapter deployment. cLA fixes its $A$ matrix to $[I_r|\textbf{0}]$, forcing a column update through $B$'s projection of $A$; $\Delta W=B[I_r|\textbf{0}]=[B|\textbf{0}].$ Additionally, an alternate variant of cLA has the opportunity to update the $A$ matrix and freeze the $B$ matrix, forcing cLA to fine-tune a restricted subset of the row space of the pretrained model. As fine-tuning $B$ is inherently more effective than fine-tuning $A$~\cite{fixA}, this suggests that applying PaCA to a fixed subset of the rows of the pretrained matrix instead of the columns would degrade performance.}

\subsection{Introducing New Artifacts to PaCA}\label{sec:pacavariant} 

Since we established a connection between the LoRA PEFT family and PaCA via the sparsity-induced LoRA variants, we were interested to see how some performance-enhancing artifacts of LoRA PEFT methods migrate to PaCA, such as chain construction to PaCA, which we call C-PaCA, where we periodically resample the columns PaCA is updating. To have a fairer comparison to our sparsity-induced method, cLA, we introduce a deterministic PaCA variant, D-PaCA, which updates the same columns as cLA. PaCA and r-cLA both update a random subset of $r$ columns of each adapted layer, and C-PaCA and r-c$^3$LA effectively implement chaining behavior to PaCA and r-cLA, respectively.

\smartparagraph{Experiments.} To showcase the performance of PaCA and the sparsity-induced LoRA variants compared to baseline SOTA LoRA variants and full fine-tuning, we fine-tune ViT-Tiny and ViT-Base on the OfficeHome and CIFAR-10 datasets, and RoBERTa-Base and RoBERTa-Large on the MRPC and CoLA datasets using the aforementioned methods with rank $r = 8$ and 30 epochs, averaged over the other three seeds. To align our experiments with those done in PaCA's introduction paper~\cite {paca}, we adapt all layers of the models, and fully fine-tune the classification heads, then report the accuracy in Table~\ref{tab:vit_acc_all}. To align with the other experiments in our paper as well as LoRA's introduction paper~\cite{lora}, we adapt only the query and value matrices of each attention head, fully fine-tune the classification heads of the models, and report those results in Table~\ref{tab:vit_acc_qv}.

\smartparagraph{Discussion of results.} In both Tables~\ref{tab:vit_acc_all} and~\ref{tab:vit_acc_qv}, the difference in accuracy of each sparsity-induced LoRA method and their PaCA counterpart (cLA and D-PaCA, r-cLA and PaCA, r-c$^3$LA and C-PaCA) is negligible, empirically validating how these methods connect LoRA and PaCA. In Table~\ref{tab:vit_acc_qv}, C-PaCA's relative performance to PaCA was consistent with applying chain construction to the LoRA PEFT family. E.g., if CoLA outperformed LoRA for a particular row, then often C-PaCA outperformed PaCA. This effect was less consistent in Table~\ref{tab:vit_acc_all}. We realized this inconsistency is due to adapting more layers; the increased capability of updating more columns throughout training was unnecessary, as we were already updating far more columns due to adapting more layers. For efficiency details of PaCA and the sparse-induced LoRA variants compared to SOTA variants, see \S\ref{subsec:efficiency}.

\begin{table*}
  \centering
  \tiny
  \setlength{\tabcolsep}{5pt}
  \renewcommand{\arraystretch}{1.25}
  \caption{\small{Test accuracy (\%) of ViT-Tiny and ViT-Base models fine-tuned on OfficeHome and CIFAR-10, and RoBERTa-Base and RoBERTa-Large models fine-tuned on MRPC and CoLA, averaged over three seeds (0,1,2) for various LoRA and PaCA methods. The value in parentheses is normalized to LoRA's test accuracy, for $r = 16$, adapting the token and map embeddings, query, key, value, and output matrices of the attention layers, and both fully connected layers. We \underline{underline} sparsity-induced variants that remain very competitive with the best performing methods for each model and dataset combination.}}
  \label{tab:vit_acc_all}
  \vspace{0mm}
  \scalebox{0.95}{
  \begin{tabular}{@{} ll c c c c c c c c c c c c @{}}
    \toprule
    \multirow{2}{*}{\textbf{Model}} & \multirow{2}{*}{\textbf{Dataset}} & \multicolumn{1}{c}{} & \multicolumn{4}{c}{\textbf{LoRA Variants}} & \multicolumn{4}{c}{\textbf{Sparsity-Induced Variants}} & \multicolumn{3}{c}{\textbf{PaCA Variants}} \\
    \cmidrule(lr){4-7} \cmidrule(lr){8-11} \cmidrule(lr){12-14}
     &  & FFT & LoRA & CoLA & Asym & RAC & cLA & c$^3$LA & r-cLA & r-c$^3$LA & D-PaCA & PaCA & C-PaCA \\
    \midrule
    \multirow{2}{*}{ViT-Tiny} & OfficeHome & \shortstack{54.2\\(0.810)} & \shortstack{\textcolor{red}{66.9}\\(\textcolor{red}{1.000})} & \shortstack{\textcolor{red}{66.9}\\(\textcolor{red}{1.000})} & \shortstack{\textcolor{codegreen}{67.4}\\(\textcolor{codegreen}{1.008})} & \shortstack{\textcolor{codegreen}{67.4}\\(\textcolor{codegreen}{1.008})} & \shortstack{\textcolor{blue}{64.7}\\(\textcolor{blue}{0.968})} & \shortstack{\textcolor{blue}{64.7}\\(\textcolor{blue}{0.968})} & \shortstack{\underline{64.6}\\(0.965)} & \shortstack{64.4\\(0.963)} & \shortstack{63.7\\(0.952)} & \shortstack{64.0\\(0.956)} & \shortstack{63.9\\(0.956)} \\
    \cmidrule(lr){2-14}
      & CIFAR-10 & \shortstack{92.7\\(0.990)} & \shortstack{\textcolor{red}{93.6}\\(\textcolor{red}{1.000})} & \shortstack{\textcolor{blue}{93.5}\\(\textcolor{blue}{0.998})} & \shortstack{\textcolor{codegreen}{94.8}\\(\textcolor{codegreen}{1.013})} & \shortstack{\textcolor{codegreen}{94.8}\\(\textcolor{codegreen}{1.013})} & \shortstack{92.5\\(0.988)} & \shortstack{92.4\\(0.987)} & \shortstack{93.2\\(0.996)} & \shortstack{92.1\\(0.984)} & \shortstack{\textcolor{blue}{93.5}\\(\textcolor{blue}{0.999})} & \shortstack{92.9\\(0.992)} & \shortstack{92.4\\(0.987)} \\
    \midrule
    \multirow{2}{*}{ViT-Base} & OfficeHome & \shortstack{68.4\\(0.862)} & \shortstack{79.3\\(1.000)} & \shortstack{79.3\\(1.000)} & \shortstack{\textcolor{codegreen}{80.4}\\(\textcolor{codegreen}{1.013})} & \shortstack{\textcolor{codegreen}{80.4}\\(\textcolor{codegreen}{1.013})} & \shortstack{\underline{79.9}\\(1.007)} & \shortstack{\underline{79.9}\\(1.007)} & \shortstack{\textcolor{red}{80.3}\\(\textcolor{red}{1.011})} & \shortstack{\underline{79.9}\\(1.007)} & \shortstack{\textcolor{blue}{80.0}\\(\textcolor{blue}{1.008})} & \shortstack{79.7\\(1.004)} & \shortstack{\textcolor{blue}{80.0}\\(\textcolor{blue}{1.009})} \\
    \cmidrule(lr){2-14}
      & CIFAR-10 & \shortstack{95.0\\(0.972)} & \shortstack{97.8\\(1.000)} & \shortstack{98.0\\(1.002)} & \shortstack{\textcolor{codegreen}{98.8}\\(\textcolor{codegreen}{1.011})} & \shortstack{\textcolor{codegreen}{98.8}\\(\textcolor{codegreen}{1.011})} & \shortstack{\underline{98.5}\\(1.007)} & \shortstack{\underline{98.5}\\(1.007)} & \shortstack{\textcolor{blue}{98.6}\\(\textcolor{blue}{1.008})} & \shortstack{\textcolor{red}{98.7}\\(\textcolor{red}{1.009})} & \shortstack{98.2\\(1.005)} & \shortstack{98.4\\(1.006)} & \shortstack{98.5\\(1.007)} \\
    \midrule
    \multirow{2}{*}{RoBERTa-Base} & MRPC & \shortstack{\textcolor{red}{90.8}\\(\textcolor{red}{0.998})} & \shortstack{\textcolor{codegreen}{91.0}\\(\textcolor{codegreen}{1.000})} & \shortstack{90.3\\(0.992)} & \shortstack{\textcolor{blue}{90.6}\\(\textcolor{blue}{0.995})} & \shortstack{90.2\\(0.990)} & \shortstack{89.3\\(0.981)} & \shortstack{89.3\\(0.981)} & \shortstack{90.1\\(0.989)} & \shortstack{89.8\\(0.987)} & \shortstack{89.4\\(0.982)} & \shortstack{90.2\\(0.991)} & \shortstack{89.7\\(0.985)} \\
    \cmidrule(lr){2-14}
      & CoLA & \shortstack{\textcolor{blue}{62.7}\\(\textcolor{blue}{1.031})} & \shortstack{60.9\\(1.000)} & \shortstack{\textcolor{codegreen}{65.1}\\(\textcolor{codegreen}{1.069})} & \shortstack{60.7\\(0.997)} & \shortstack{62.1\\(1.021)} & \shortstack{59.3\\(0.974)} & \shortstack{59.2\\(0.973)} & \shortstack{62.3\\(1.024)} & \shortstack{\textcolor{red}{64.3}\\(\textcolor{red}{1.056})} & \shortstack{60.5\\(0.994)} & \shortstack{60.3\\(0.990)} & \shortstack{62.1\\(1.020)} \\
    \midrule
    \multirow{2}{*}{RoBERTa-Large} & MRPC & \shortstack{\textcolor{codegreen}{91.3}\\(\textcolor{codegreen}{1.000})} & \shortstack{\textcolor{codegreen}{91.3}\\(\textcolor{codegreen}{1.000})} & \shortstack{90.9\\(0.996)} & \shortstack{\textcolor{blue}{91.1}\\(\textcolor{blue}{0.997})} & \shortstack{90.8\\(0.995)} & \shortstack{\underline{91.0}\\(0.997)} & \shortstack{90.9\\(0.996)} & \shortstack{90.9\\(0.995)} & \shortstack{90.9\\(0.996)} & \shortstack{\textcolor{codegreen}{91.3}\\(\textcolor{codegreen}{1.000})} & \shortstack{\textcolor{blue}{91.1}\\(\textcolor{blue}{0.998})} & \shortstack{\textcolor{red}{91.2}\\(\textcolor{red}{0.999})} \\
    \cmidrule(lr){2-14}
      & CoLA & \shortstack{\textcolor{red}{69.7}\\(\textcolor{red}{1.032})} & \shortstack{67.6\\(1.000)} & \shortstack{64.2\\(0.950)} & \shortstack{67.0\\(0.992)} & \shortstack{67.2\\(0.994)} & \shortstack{64.3\\(0.952)} & \shortstack{65.8\\(0.974)} & \shortstack{64.0\\(0.947)} & \shortstack{66.8\\(0.988)} & \shortstack{66.3\\(0.982)} & \shortstack{\textcolor{blue}{69.0}\\(\textcolor{blue}{1.021})} & \shortstack{\textcolor{codegreen}{70.1}\\(\textcolor{codegreen}{1.037})} \\
    \bottomrule
  \end{tabular}}
  \vspace{-2mm}
\end{table*}

\begin{table*}
  \centering
  \tiny
  \setlength{\tabcolsep}{5pt}
  \renewcommand{\arraystretch}{1.25}
  \caption{\small{Test accuracy (\%) of ViT-Tiny and ViT-Base models fine-tuned on OfficeHome and CIFAR-10, and RoBERTa-Base and RoBERTa-Large models fine-tuned on MRPC and CoLA, averaged over three seeds (0,1,2) for various LoRA and PaCA methods. The value in parentheses is normalized to LoRA's test accuracy, for $r = 16$, adapting the query and value matrices of the attention layers as well as the classification head. We \underline{underline} sparsity-induced variants that remain very competitive with the best performing methods for each model and dataset combination.}}
  \label{tab:vit_acc_qv}
  \vspace{0mm}
  \scalebox{0.95}{
  \begin{tabular}{@{} ll c c c c c c c c c c c c @{}}
    \toprule
    \multirow{2}{*}{\textbf{Model}} & \multirow{2}{*}{\textbf{Dataset}} & \multicolumn{1}{c}{} & \multicolumn{4}{c}{\textbf{LoRA Variants}} & \multicolumn{4}{c}{\textbf{Sparsity-Induced Variants}} & \multicolumn{3}{c}{\textbf{PaCA Variants}} \\
    \cmidrule(lr){4-7} \cmidrule(lr){8-11} \cmidrule(lr){12-14}
     &  & FFT & LoRA & CoLA & Asym & RAC & cLA & c$^3$LA & r-cLA & r-c$^3$LA & D-PaCA & PaCA & C-PaCA \\
    \midrule
    \multirow{2}{*}{ViT-Tiny} & OfficeHome & \shortstack{54.7\\(0.803)} & \shortstack{\textcolor{red}{68.2}\\(\textcolor{red}{1.000})} & \shortstack{\textcolor{red}{68.2}\\(\textcolor{red}{1.000})} & \shortstack{\textcolor{codegreen}{68.4}\\(\textcolor{codegreen}{1.003})} & \shortstack{\textcolor{codegreen}{68.4}\\(\textcolor{codegreen}{1.003})} & \shortstack{67.4\\(0.989)} & \shortstack{67.4\\(0.989)} & \shortstack{\underline{67.8}\\(0.995)} & \shortstack{\textcolor{blue}{67.9}\\(\textcolor{blue}{0.995})} & \shortstack{66.6\\(0.976)} & \shortstack{67.4\\(0.989)} & \shortstack{67.2\\(0.985)} \\
    \cmidrule(lr){2-14}
      & CIFAR-10 & \shortstack{93.5\\(0.997)} & \shortstack{93.7\\(1.000)} & \shortstack{93.8\\(1.001)} & \shortstack{\textcolor{codegreen}{94.2}\\(\textcolor{codegreen}{1.005})} & \shortstack{\textcolor{codegreen}{94.2}\\(\textcolor{codegreen}{1.005})} & \shortstack{\textcolor{blue}{94.0}\\(\textcolor{blue}{1.003})} & \shortstack{\underline{93.9}\\(1.002)} & \shortstack{\underline{93.9}\\(1.001)} & \shortstack{93.8\\(1.001)} & \shortstack{\textcolor{red}{94.1}\\(\textcolor{red}{1.004})} & \shortstack{\textcolor{blue}{94.0}\\(\textcolor{blue}{1.003})} & \shortstack{\textcolor{blue}{94.0}\\(\textcolor{blue}{1.003})} \\
    \midrule
    \multirow{2}{*}{ViT-Base} & OfficeHome & \shortstack{68.4\\(0.850)} & \shortstack{\textcolor{codegreen}{80.5}\\(\textcolor{codegreen}{1.000})} & \shortstack{\textcolor{codegreen}{80.5}\\(\textcolor{codegreen}{1.000})} & \shortstack{80.2\\(0.995)} & \shortstack{\textcolor{red}{80.4}\\(\textcolor{red}{0.998})} & \shortstack{79.9\\(0.992)} & \shortstack{79.9\\(0.992)} & \shortstack{79.8\\(0.991)} & \shortstack{79.8\\(0.991)} & \shortstack{\textcolor{blue}{80.3}\\(\textcolor{blue}{0.997})} & \shortstack{80.1\\(0.994)} & \shortstack{80.0\\(0.994)} \\
    \cmidrule(lr){2-14}
      & CIFAR-10 & \shortstack{95.0\\(0.963)} & \shortstack{98.6\\(1.000)} & \shortstack{98.6\\(1.000)} & \shortstack{\textcolor{codegreen}{99.0}\\(\textcolor{codegreen}{1.004})} & \shortstack{\textcolor{codegreen}{99.0}\\(\textcolor{codegreen}{1.004})} & \shortstack{\underline{98.6}\\(0.999)} & \shortstack{\underline{98.6}\\(0.999)} & \shortstack{\textcolor{blue}{98.7}\\(\textcolor{blue}{1.001})} & \shortstack{\textcolor{red}{98.8}\\(\textcolor{red}{1.001})} & \shortstack{\textcolor{red}{98.8}\\(\textcolor{red}{1.002})} & \shortstack{98.6\\(1.000)} & \shortstack{98.5\\(0.999)} \\
    \midrule
    \multirow{2}{*}{RoBERTa-Base} & MRPC & \shortstack{\textcolor{codegreen}{90.8}\\(\textcolor{codegreen}{1.010})} & \shortstack{\textcolor{blue}{89.9}\\(\textcolor{blue}{1.000})} & \shortstack{\textcolor{red}{90.1}\\(\textcolor{red}{1.003})} & \shortstack{\textcolor{red}{90.1}\\(\textcolor{red}{1.002})} & \shortstack{89.7\\(0.997)} & \shortstack{88.7\\(0.986)} & \shortstack{88.7\\(0.987)} & \shortstack{89.5\\(0.996)} & \shortstack{89.5\\(0.996)} & \shortstack{88.9\\(0.988)} & \shortstack{89.3\\(0.993)} & \shortstack{88.8\\(0.987)} \\
    \cmidrule(lr){2-14}
      & CoLA & \shortstack{\textcolor{red}{62.7}\\(\textcolor{red}{0.953})} & \shortstack{\textcolor{codegreen}{65.8}\\(\textcolor{codegreen}{1.000})} & \shortstack{\textcolor{blue}{61.2}\\(\textcolor{blue}{0.931})} & \shortstack{59.3\\(0.901)} & \shortstack{60.0\\(0.911)} & \shortstack{58.1\\(0.883)} & \shortstack{58.4\\(0.887)} & \shortstack{59.5\\(0.904)} & \shortstack{57.6\\(0.876)} & \shortstack{57.1\\(0.868)} & \shortstack{57.9\\(0.880)} & \shortstack{57.6\\(0.875)} \\
    \midrule
    \multirow{2}{*}{RoBERTa-Large} & MRPC & \shortstack{\textcolor{codegreen}{91.3}\\(\textcolor{codegreen}{1.019})} & \shortstack{89.6\\(1.000)} & \shortstack{90.8\\(1.014)} & \shortstack{\textcolor{blue}{91.0}\\(\textcolor{blue}{1.016})} & \shortstack{\textcolor{red}{91.2}\\(\textcolor{red}{1.019})} & \shortstack{90.1\\(1.006)} & \shortstack{90.5\\(1.010)} & \shortstack{89.5\\(1.000)} & \shortstack{90.0\\(1.005)} & \shortstack{90.4\\(1.010)} & \shortstack{89.6\\(1.000)} & \shortstack{90.3\\(1.008)} \\
    \cmidrule(lr){2-14}
      & CoLA & \shortstack{\textcolor{codegreen}{69.7}\\(\textcolor{codegreen}{1.047})} & \shortstack{66.6\\(1.000)} & \shortstack{\textcolor{blue}{68.4}\\(\textcolor{blue}{1.027})} & \shortstack{\textcolor{red}{69.3}\\(\textcolor{red}{1.042})} & \shortstack{68.1\\(1.023)} & \shortstack{65.0\\(0.976)} & \shortstack{66.4\\(0.997)} & \shortstack{67.1\\(1.008)} & \shortstack{66.4\\(0.997)} & \shortstack{67.1\\(1.009)} & \shortstack{67.3\\(1.012)} & \shortstack{65.4\\(0.982)} \\
    \bottomrule
  \end{tabular}}
  \vspace{-2mm}
\end{table*}

\subsection{Applying PaCA's Convergence Result to cLA}

PaCA updates only a subset of columns in each layer. Writing the $l$-th layer weight as a list of columns
$W_l = [ w^l_1, w^l_2, \ldots, w_{d_{\rm out}}^l]$ and letting
$P_l = [i_{1}^l, i_{2}^l,\cdots,i_{r}^l]$ denote the selected column indices,
PaCA's masked column update is 
$$
W_l^{t+1}
= W_l^t - \eta \Delta W_l^t
= W_l^t - \eta [0, \nabla_{{i_1^l} w_l^t}, \ldots, \nabla_{{i_r^l} w_l^t}, \ldots, 0].
$$
\begin{theorem} ({Loss Convergence of Partial Connections} [Theorem 1~\cite{paca}]) If the gradient of the loss $f(W, X)$ is Lipschitz continuous with Lipschitz Constant $L_\cL$, and only the partial connections are updated, then
$$f(W^{t+1}, X^{t+1}) 
\le 
f(W^t, X^t)
-\eta (1 - \frac{\eta L_{\cL}}{2}) 
\|\nabla_{P^t}
\|^2,
$$
where $\eta$ is the learning rate and $\|\nabla_{P^t}\|^2$ is the masked/restricted gradient only consisting of the gradient for the columns in $P_l$ for each layer, $l \in [L]$.
    \label{thm:paca_thm}
\end{theorem}

In particular, the significance of Theorem~\ref{thm:paca_thm} can be accredited to the fact that a learning rate of  $0<\eta<2/L_{\cL}$ ensures that the loss decreases after each iteration. We use PaCA's Theorem~\ref{thm:paca_thm} to demonstrate similar convergence behavior as cLA and formalize the result in Theorem~\ref{thm:cla_paca_thm}. 

\begin{theorem}\label{thm:cla_paca_thm}
    (Loss Convergence of cLA [Theorem~\ref{thm:paca_thm} applied to cLA]) If the gradient of the loss $f(W, X)$ is Lipschitz continuous with Lipschitz Constant $L_\cL$, and only the columns from cLA are updated, then
$$f(W^{t+1}, X^{t+1}) 
\le 
f(W^t, X^t)
-\eta (1 - \frac{\eta L_{\cL}}{2}) 
\|\nabla_{W}\cL(W^t) A_0^TA_0
\|^2.
$$
\end{theorem}
\begin{proof}
In cLA, we consider the frozen factor as
$A_0=[I_r|0]$, so the effective weight can be represented by
$$
W^t = W_0 + \frac{\alpha}{r} B^t A_0.
$$
Observe that only the first $r$ columns of $W$ can change ($\Delta W = B[I_r|0]=[B|0]$). Since $W$ depends linearly on $B$, $W^t = W_0+\frac{\alpha}{r}B^tA$, by the chain rule we have:
$$\nabla_B \cL(W^t)=\nabla_B \cL(W_0+\frac{\alpha}{r}B^tA)=\frac{\alpha}{r}\nabla_W \cL(W^t)A^\top.$$
Through gradient descent, we have:
$$B^{t+1}=B^{t}-\gamma \nabla_B\cL(W^t)=B^t-\gamma\frac{\alpha}{r}\nabla_W \cL(W^t)A^\top$$
Denote $\cK = A_0^TA_0=\operatorname{Diag}(\overbrace{1,...,1}^{1\text{ to } r},0,...,0).$Observe the following:
$$W^{t+1}-W^t=\frac{\alpha}{r}(B^{t+1}-B^t)A=-\gamma\frac{\alpha^2}{r^2}\nabla_W \cL(W^t) A^\top A=-\gamma\frac{\alpha^2}{r^2}\nabla_W \cL(W^t)\cK=-\eta \nabla_W \cL(W^t)\cK$$

In particular, for cLA we have $\cK$ acting as a mask that selects only the first $r$ columns in each layer. Thus, if we define $P_l^{\text{cLA}}=\{1,2,\ldots,r\}$ for every layer $l$, then the masked gradient $\nabla_W \cL(W^t)\cK$ is exactly the restricted gradient $\nabla_{P^t}$ in Theorem~\ref{thm:paca_thm} (with a deterministic choice of $P_l$ rather than a random subset). Therefore, under the same Lipschitz-gradient assumption on $\cL$ with constant $L_\cL$, Theorem~\ref{thm:paca_thm} applies directly to the effective weights $W^t$ induced by cLA, resulting in
$$
\cL(W^{t+1},X^{t+1})
\le
\cL(W^{t},X^{t})
-\eta\Big(1-\frac{\eta L_\cL}{2}\Big)\big\|\nabla_W \cL(W^t)\cK\big\|^2.
$$
Hence the result. 
\end{proof}
The same step-size condition $0<\eta<2/L_\cL$ ensures a decrease in the loss for cLA updates. Equivalently, in terms of the $B$ update step size, using $\gamma=\eta\frac{r^2}{\alpha^2}$ implies the condition
$$
0<\gamma<\frac{2r^2}{L_\cL \alpha^2},
$$
which guarantees that each iteration decreases the loss.

\section{Pseudo Code of sparsity-induced LoRA variants}\label{subsec:algorithm}
 In this Section, we present the pseudocode of our proposed LoRA variants, cLA (Algorithm \ref{alg:cLA}), random-cLA (Algorithm \ref{alg:random-cLA}), $c^3$LA (Algorithm \ref{alg:c^3LA}), and r-$c^3$LA (Algorithm \ref{alg:rc^3LA}).

\begin{algorithm}
\caption{\textbf{Cheap LoRA (cLA)}}
\label{alg:cLA}
\begin{algorithmic}[1]
\STATE \textbf{Parameters:} Loss function $\mathcal{L}$ and model $f_{\mathbf{W}}(\cdot)$. Pretrained weights $\mathbf{W}_0 = (W^1_0,...,W^L_0)$, where $W^i_0\in\mathbb{R}^{n_i\times m_i}$. rank $r\ll\min\{m_i,n_i\}_{i\in[L]}$, learning rate $\gamma>0$, scaling factor $\alpha>0$, total training iterations $T$.
\STATE \textbf{Initialize} $A_0^j = [I_{r}\,|\,\mathbf{0}_{\,r\times(m_j-r)}];\ B^{0,j} = \mathbf{0}$ for $j\in[L]$
\FOR{$t = 1,...,T$}
    \STATE forward pass with LoRA modules
    \STATE backward pass then update $\mathbf{B}^{t}$
    \FOR{$j = 1,...,L$}
        \STATE $B^{t,j} = B^{t-1,j} - \gamma\frac{\alpha}{r}\nabla_j\mathcal{L}(\mathbf{W}_0 + \frac{\alpha}{r}\mathbf{B}^{t-1}\mathbf{A}_0)\operatorname{Diag}(\overbrace{1,...,1}^{1\text{ to } r},0,...,0)$
    \ENDFOR
    % \FOR{$j = 1,...,L$}
    %     \STATE $B^{t,j} = B^{t-1,j} - \gamma\frac{\alpha}{r}\nabla_j\mathcal{L}(\mathbf{W}_0 + \frac{\alpha}{r}\mathbf{B}^{t-1}\mathbf{A}_0)\left[\begin{array}{c|c}
    %     I_r & \mathbf{0}_{r\times(n_i-r)}\\\hline \mathbf{0}_{(n_i-r)\times r} & \mathbf{0}_{(n_i-r)\times(n_i-r)}\end{array}\right]$
    % \ENDFOR
\ENDFOR
\STATE $\hat{j}=\argmin_{j\in[T]}\mathcal{L}(\mathbf{W}_0 + \frac{\alpha}{r}\mathbf{B}^j\mathbf{A}_0)$ or task-based metric.
\STATE \textbf{Return} Fine-tuned weights $\mathbf{W}_0 + \frac{\alpha}{r}\mathbf{B}^{\hat{j}}\mathbf{A}_0$
\end{algorithmic}
\end{algorithm}

\begin{algorithm}
\caption{\textbf{random Cheap LoRA (r-cLA)}}
\label{alg:random-cLA}
\begin{algorithmic}[1]
\STATE \textbf{Parameters:} Loss function $\mathcal{L}$ and model $f_{\mathbf{W}}(\cdot)$. Pretrained weights $\mathbf{W}_0 = (W^1_0,...,W^L_0)$, where $W^i_0\in\mathbb{R}^{m_i\times n_i}$. rank $r\ll\min\{m_i,n_i\}_{i\in[L]}$, learning rate $\gamma>0$, scaling factor $\alpha>0$, total training iterations $T$.
\STATE \textbf{Initialize}
    \STATE \hspace*{1em} $\xi_j = \operatorname{randint}(0,\lfloor \frac{n_j}{r}\rfloor-1)\text{ for }j\in[L]$
    \STATE \hspace*{1em} $A_0^j = \left[\mathbf{0}_{r\times\xi_j}\mid I_{r}\mid\mathbf{0}_{r\times(n_j-\xi_j-r)}\right];\ B^{0,j} = \mathbf{0}$ for $j\in[L]$
\FOR{$t = 1,...,T$}
    \STATE forward pass with LoRA modules
    \STATE backward pass then update $\mathbf{B}^{t}$
    \FOR{$j = 1,...,L$}
        \STATE $B^{t,j} = B^{t-1,j} - \gamma\frac{\alpha}{r}\nabla_j\mathcal{L}(\mathbf{W}_0 + \frac{\alpha}{r}\mathbf{B}^{t-1}\mathbf{A}_0)\operatorname{Diag}(0,...,0,\overbrace{1,...,1}^{\xi_j + 1\text{ to }\xi_j + r},0,...,0)$
    \ENDFOR
\ENDFOR
\STATE $\hat{j}=\argmin_{j\in[T]}\mathcal{L}(\mathbf{W}_0 + \frac{\alpha}{r}\mathbf{B}^j\mathbf{A}_0)$ or task-based metric.
\STATE \textbf{Return} Fine-tuned weights $\mathbf{W}_0 + \frac{\alpha}{r}\mathbf{B}^{\hat{j}}\mathbf{A}_0$
\end{algorithmic}
\end{algorithm}

\begin{algorithm}
\caption{\textbf{Circulant Chain of Cheap LoRA ($c^3$LA)}}
\label{alg:c^3LA}
\begin{algorithmic}[1]
\STATE \textbf{Parameters:} Loss function $\mathcal{L}$ and model $f_{\mathbf{W}}(\cdot)$. Pretrained weights $\mathbf{W}_0^{(0)} = (W^1,...,W^L)$, where $W^i\in\mathbb{R}^{m_i\times n_i}$. rank $r\ll\min\{m_i,n_i\}_{i\in[L]}$, learning rate $\gamma>0$, scaling factor $\alpha>0$, total training iterations $T$, chain-length $k <= T$.
\STATE \textbf{Initialize} $A_0^j = [I_{r}\,|\,\mathbf{0}_{\,r\times(n_j-r)}];\ B^{0,j} = \mathbf{0}$ for $j\in[L]$, current chain $c = 0$.
\FOR{$t = 1,...,T$}
    \IF{$t\equiv0\pmod{\lfloor\frac{T}{k}\rfloor}$}
        \STATE $c = c + 1$
        \STATE Merge LoRA to backbone weights $\mathbf{W}_0^{(c)} = \mathbf{W}_0^{(c-1)} + \frac{\alpha}{r}\mathbf{B}^{t-1}\mathbf{A}_0$
        \STATE Re-initialize with $\mathbf{A}_0$ shifted by $r$:
            \STATE \hspace*{1em} $A^j_0 = \left[\mathbf{0}_{r\times cr}\mid I_{r}\mid \mathbf{0}_{r\times{n_i-r-cr}}\right]$; $B^{t-1,j} = \mathbf{0}$ for $j\in[L]$
    \ENDIF
    \STATE forward pass with LoRA modules
    \STATE backward pass then update $\mathbf{B}^{t}$
    \FOR{$j = 1,...,L$}
        \STATE $B^{t,j} = B^{t-1,j} - \gamma\frac{\alpha}{r}\nabla_j\mathcal{L}(\mathbf{W}_0^{(c)} + \frac{\alpha}{r}\mathbf{B}^{t-1}\mathbf{A}_0)\operatorname{Diag}(0,...,0,\overbrace{1,...,1}^{cr\text{ to }(c+1)_r},0,...,0)$
    \ENDFOR
\ENDFOR
\STATE $\hat{c},\hat{j}=\argmin_{j\in[\lfloor\frac{T}{k}\rfloor],c\in[k]}\mathcal{L}(\mathbf{W}_0^{(c)} + \frac{\alpha}{r}\mathbf{B}^{cj}\mathbf{A}_0)$ or task-based metric.
\STATE \textbf{Return} Fine-tuned weights $\mathbf{W}_0^{\hat{c}} + \frac{\alpha}{r}\mathbf{B}^{\hat{c}\hat{j}}\mathbf{A}_0$
\end{algorithmic}
\end{algorithm}%\newpage

\begin{algorithm}
\caption{\textbf{Random Circulant Chain of Cheap LoRA (r-$c^3$LA)}}
\label{alg:rc^3LA}
\begin{algorithmic}[1]
\STATE \textbf{Parameters:} Loss function $\mathcal{L}$ and model $f_{\mathbf{W}}(\cdot)$. Pretrained weights $\mathbf{W}_0^{(0)} = (W^1,...,W^L)$, where $W^i\in\mathbb{R}^{m_i\times n_i}$. rank $r\ll\min\{m_i,n_i\}_{i\in[L]}$, learning rate $\gamma>0$, scaling factor $\alpha>0$, total training iterations $T$, chain-length $k <= T$.
\STATE \textbf{Initialize}
\STATE  $\xi_j = \operatorname{randint}(0,\lfloor \frac{n_j}{r}\rfloor-1)\text{ for }j\in[L].$
\STATE $A_0^j = \left[\mathbf{0}_{r\times\xi_j}\mid I_{r}\mid\mathbf{0}_{r\times(n_j-\xi_j-r)}\right];\ B^{0,j} = \mathbf{0}$ for $j\in[L]$, current chain $c = 0$.

\FOR{$t = 1,...,T$}
    \IF{$t\equiv0\pmod{\lfloor\frac{T}{k}\rfloor}$}
        \STATE $c = c + 1$
        \STATE Merge LoRA to backbone weights $\mathbf{W}_0^{(c)} = \mathbf{W}_0^{(c-1)} + \frac{\alpha}{r}\mathbf{B}^{t-1}\mathbf{A}_0$
        \STATE Re-initialize with $\mathbf{A}_0$ shifted by a new random variable $\xi'_j$:
            \STATE \hspace*{1em} $A_0^j = \left[\mathbf{0}_{r\times\xi'_j}\mid I_{r}\mid\mathbf{0}_{r\times(n_j-\xi'_j-r)}\right]$; $B^{t-1,j} = \mathbf{0}$ for $j\in[L]$
    \ENDIF
    \STATE forward pass with LoRA modules
    \STATE backward pass then update $\mathbf{B}^{t}$
    \FOR{$j = 1,...,L$}
        \STATE $B^{t,j} = B^{t-1,j} - \gamma\frac{\alpha}{r}\nabla_j\mathcal{L}(\mathbf{W}_0 + \frac{\alpha}{r}\mathbf{B}^{t-1}\mathbf{A}_0)\operatorname{Diag}(0,...,0,\overbrace{1,...,1}^{\xi_j + 1\text{ to }\xi_j + r},0,...,0)$
    \ENDFOR
\ENDFOR
\STATE $\hat{c},\hat{j}=\argmin_{j\in[\lfloor\frac{T}{k}\rfloor],c\in[k]}\mathcal{L}(\mathbf{W}_0^{(c)} + \frac{\alpha}{r}\mathbf{B}^{cj}\mathbf{A}_0)$ or task-based metric.
\STATE \textbf{Return} Fine-tuned weights $\mathbf{W}_0^{\hat{c}} + \frac{\alpha}{r}\mathbf{B}^{\hat{c}\hat{j}}\mathbf{A}_0$
\end{algorithmic}
\end{algorithm}
% ELI - I had left that \newpage there since it stopped the algorithm from, for some reason, choosing to jump forward into the nonconvex convergence section. This may not end up mattering if we add shuffle's pseudocode. Leaving it commented out for now.

% \aritra{Use Algorithm mode, not Table. Look internet: https://www.overleaf.com/learn/latex/Algorithms}

\section{Theoretical Results}\label{Appendix:Theoretical_Results}

This section complements Section~\ref{sec:theory} in the main paper.

\subsection{Generalization}\label{Appendix:generalization}
Theorem~\ref{theorem:nonlinearGenBound} provides a point-wise deterministic bound that relates the generalization error of the pretrained model $\textbf{W}_0$ to that of the fine-tuned model $\textbf{W}_0+\Delta\textbf{W},$ rather than a uniform probabilistic PAC bound over a hypothesis class. This construction matches our intended use so that practitioners can be theoretically guided by the structural characteristics of a fine-tuning procedure.  

In this section, we give a detailed proof of the generalization error bound. We start by listing the inequalities used in this section. 
\subsubsection{Inequalities used}
\begin{enumerate}
    \item If $A,B \in \mathbb{R}^{m\times n}$ and $x \in \mathbb{R}^{n}$, then the Triangle-Inequality gives:
    \begin{eqnarray}\label{eq:triangle_inequality}
    \|(A+B)x\| \leq \|Ax\| + \|Bx\|.
    \end{eqnarray}
    \item For $A \in \mathbb{R}^{m\times n}$ and $x \in \mathbb{R}^n$, we have:
    \begin{eqnarray}\label{eq:norm_bound_inequality}
    \|Ax\| \le \|A\|_2\|x\|.
    \end{eqnarray}
    \item By Assumption~\ref{ass:activationlipschitz}, we have:
    \begin{eqnarray}\label{eq:zero_inequality}
    \|\sigma(Ax) \| \leq \|\sigma(Ax) - \sigma({0})\| + \|\sigma({0})\| \leq L_{\sigma}\| Ax\| + \|\sigma({0})\|.
    \end{eqnarray}
    \item For a finite collection of matrices, $\{A_1, \cdots, A_k\}; A_i\in\mathbb{R}^{m\times n}$, we have:
\begin{eqnarray}\label{eq:rank_inequality}
        \text{rank}(\sum_{i=1}^kA_i) \leq \sum_{i=1}^k\text{rank}(A_i).
    \end{eqnarray}
    \item Let $\mathbf{I}(X;Y)$ denote the mutual information between random variables $X$ and $Y$. It measures how much the knowledge of one random variable reveals about measuring the other, i.e., $$\mathbf{I}(X;Y)=D(P_{(X,Y)}\|P_X \otimes P_Y)=\sup_F\{\int FdP_{XY}-\log\int e^{F}d(P_X\otimes P_Y)\},$$ where $F$ is a bounded, measurable function \cite{xu2017information}. Let $T$ be a deterministic map for $A \in \mathbb{R}^{m \times n}$. Then the \textbf{Data Processing Inequality (DPI)} gives us $\mathbf{I}(T(A); N )\le\textbf{I}(A; N),$ where $N$ denotes the training dataset. If $T$ is a bijective mapping then \textbf{DPI} gives us \cite{fixA}:
    \begin{eqnarray}\label{eq:information_inequality}
        \mathbf{I}(A; N )=\textbf{I}(T(A); N).
    \end{eqnarray}
\end{enumerate}

% We make the following general assumptions to prove our generalization result. 

% \begin{assumption}\label{ass:bounded feature vector} \textbf{(Boundedness of input vectors)} The input vectors are bounded, i.e., there exists a constant $C\ge0$ such that $\|x\|\leq C$, for all $x \in \mathcal{X}.$
% \end{assumption}
% \begin{assumption}\label{ass:smoothness}
% \textbf{(Lipschitz continuity of the loss)} The loss function, $\ell (\cdot): \R^d\to \R$ is $L_{\cL}$-Lipschitz continuous, i.e., $|\ell(f_\textbf{W}(x), y) - \ell(f_{\textbf{W}'}(x), y)| \leq L_{\cL}\|f_\textbf{W}(x) - f_{\textbf{W}'}(x)\|$ for all $\textbf{W}, \textbf{W}'\in\R^d \text{ and } (x,y) \in \mathcal{X} \times \mathcal{Y}$.
% \end{assumption}
% \begin{assumption}\label{ass:activationlipschitz}
%     \textbf{(Lipschitz continuity of activation)}  
%     The vector-valued activation function, $\sigma_{i}(\cdot):\R^{n_i}\to\R^{n_i}$, for each layer, $i$, is $L_{\sigma_i}$-Lipschitz continuous, i.e., $\|\sigma_{i}(u) - \sigma_{i}(v)\|\leq L_{\sigma_i}\|u - v\|,$ for all $u, v \in \R^{n_i}$.
% \end{assumption}

Now, we are set to prove Theorem \ref{theorem:nonlinearGenBound}. %However, first, we furnish the full version of the theorem, including the details of the correction terms, $\Phi_{\Delta \textbf{W}}$ and $\Phi_{ \textbf{W}_0}$, which we deliberately omitted in the main paper to not overwhelm the readers. 

\subsubsection{Proof of Theorem~\ref{theorem:nonlinearGenBound}}
\label{subsubsec:proof_of_theorem}
\setcounter{theorem}{0}
\begin{theorem}(\textbf{Generalization bounds}) Let $f_{\textbf{W}_0+\Delta\textbf{W}}(x)=\sigma_{L}([{W_0}^{L}+\Delta W^{L}](\cdots\sigma_{2}([(W_0^{2}+\Delta W^{2}]\sigma_{1}([W_0^{1}+\Delta W^{1}]x))\cdots))$ be a $L$-layers fine-tuned DNN, where ${\textbf{W}_0+\Delta \textbf{W}}$ is a fine-tuned update. Let the loss function, $\mathcal{L}$ for fine-tuning, follow  Assumptions~\ref{ass:bounded feature vector}--\ref{ass:activationlipschitz}.
Then ${\mathcal{G}(\textbf{W}_0 + \Delta \textbf{W}) \leq\min\left(\mathcal{G}(\textbf{W}_0)+\Phi_{\Delta \textbf{W}},\mathcal{G}(\Delta \textbf{W})+\Phi_{ \textbf{W}_0} \right)},$ where 
$${\Phi_{\Delta \textbf{W}}:= 2L_{\cL}\left[ C\prod_{i=1}^L L_{\sigma_i}\sum_{i=1}^{2^L-1}\prod_{j=1}^L P(i,j)+\sum_{i\neq 2^a-1:a\in [L]}^{2^L-2} F(i)\right]} \text{ and }$$  
$${\Phi_{\textbf{W}_0}:= 2L_{\cL}\left[ C\prod_{i=1}^L L_{\sigma_i}\sum_{i=2}^{2^L}\prod_{j=1}^L P(i,j)
       + \sum_{i\neq 2^a:a\in [L]}^{2^L-1} {F}(i)\right]}, $$
       are the correction terms, ${F(i):=\|\sigma_{L-\psi(i)}(0)\|\prod_{j=1}^{\psi(i)}\![L_{\sigma_{L-j+1}}\,H(i,j)]},$
       ${\psi(i):=\lfloor \log_2(i)\rfloor}$, and 
\[
\begin{aligned}
P(i, j):=
\begin{cases}
\|W_{0}^{L-j+1}\|_2~\text{if}~\lfloor\tfrac{i-1}{2^{\,L-1}}\rfloor~\text{is~odd},\\
\|\Delta W^{L-j+1}\|_2~\text{if}~\lfloor\tfrac{i-1}{2^{\,L-1}}\rfloor~\text{is~even}
\end{cases},
\hspace{-3mm}
H(i, j):=
\begin{cases}
\|\Delta W^{L-j+1}\|_2~\text{if }\lfloor\tfrac{i}{2^{\,\psi(i)-j}}\rfloor~\text{is~odd},\\
\|W_0^{L-j+1}\|_2~\text{if }\lfloor\tfrac{i}{2^{\,\psi(i)-j}}\rfloor~\text{is~even}.
\end{cases}
\end{aligned}
\]
\end{theorem}
\setcounter{theorem}{2}

\begin{proof} 
    Let $$f_{\mathbf{W}_0 + \Delta \mathbf{W}} := \sigma_{L}([W_0^{L}+\Delta W^{L}]\sigma_{L-1}(...\sigma_{1}([W_0^{1}+\Delta W^{1}]x)...))$$
    represent our fine-tuned model and
    $$f_{\mathbf{W}_0} := \sigma_{L}(W_0^{L}\sigma_{L-1}(...\sigma_{1}(W_0^{1}x)...))$$ 
    represent our pretrained model. First, we upper bound the quantity $\|f_{\mathbf{W}_0 + \Delta \mathbf{W}}-f_{\mathbf{W}_0}\|$. We have
    \begin{align*}
&\|f_{\mathbf{W}_0 + \Delta \mathbf{W}}-f_{\mathbf{W}_0}\|\\
   &= \big\|\sigma_{L}([W_0^{L}+\Delta W^{L}]\,
      \sigma_{L-1}(\cdots\sigma_{1}([W_0^{1}+\Delta W^{1}]x)\cdots)) \nonumber \\
   &\quad - \sigma_{L}(W_0^{L}\,
      \sigma_{L-1}(\cdots\sigma_{(1)}(W_0^{1}x)\cdots))\big\| \\
   &\overset{{\rm Assumption~\ref{ass:activationlipschitz}}}{\leq} \;\; L_{\sigma_L}\big \|[W_0^{L} + \Delta W^{L}]\sigma_{L-1}(\cdots\sigma_{1}([W_0^{1}+\Delta W^{1}]x)\cdots)) \nonumber\\
   &\quad - [W_0^{L}] \sigma_{L-1}(\cdots\sigma_{1}(W_0^{1}x)\cdots)) \big\| \nonumber \\
   &= L_{\sigma_L}\big\|\Delta W^{L} \sigma_{L-1}(\cdots\sigma_{1}([W_0^{1}+\Delta W^{1}]x)\cdots)) \nonumber \\
   &\quad + W_0^{L}[(\sigma_{L-1}(\cdots\sigma_{1}([W_0^{1}+\Delta W^{1}]x)\cdots)) - \sigma_{L-1}(\cdots\sigma_{1}(W_0^{1}x)\cdots)))]\big\| \nonumber \\
   & \overset{{\rm Triangle~Inequality~and~Inequality~\eqref{eq:norm_bound_inequality}}}{\leq} L_{\sigma_L}[\|\Delta W^{L}\|_2 \|\sigma_{L-1}(\cdots\sigma_{1}([W_0^{1}+\Delta W^{1}]x)\cdots))\| \nonumber \\ 
   & \quad + \| W_0^{L}\|_2 \|(\sigma_{L-1}(\cdots\sigma_{1}([W_0^{1}+\Delta W^{1}]x)\cdots) - \sigma_{L-1}(\cdots\sigma_{1}(W_0^{1}x)\cdots))\|].
\end{align*}

If $f_{\mathbf{W}_0}$ and $f_{\mathbf{W}_0 + \Delta \mathbf{W}}$ are both 1-layer, we can expand out their difference by:
$$\|f_{\mathbf{W}_0 + \Delta \mathbf{W}} -  f_{\mathbf{W}_0}\| \leq CL_{\sigma_1}\|\Delta W^{1}\|_2.$$
If $f_{\mathbf{W}_0}$ and $f_{\mathbf{W}_0 + \Delta \mathbf{W}}$ are both 2-layer, we can expand out their difference by:
\begin{align*}
\|f_{\mathbf{W}_0+\Delta\mathbf{W}}-f_{\mathbf{W}_0}\|
&\le
C\,L_{\sigma_2}L_{\sigma_1}\,\|W_0^{2}\|_2\,\|\Delta W^{1}\|_2  +\;
CL_{\sigma_2}L_{\sigma_1}\|\Delta W^{2}\|_2\|W_0^{1}\|_2 \\
&\quad +CL_{\sigma_2}L_{\sigma_2}\|\Delta W^{2}\|_2 \|\Delta W^{1}\|_2 + L_{\sigma_2}\|\Delta W^{2}\|_2\|\sigma_1({0})\|.
\end{align*}
If $f_{\mathbf{W}_0}$ and $f_{\mathbf{W}_0 + \Delta \mathbf{W}}$ are both 3-layer, we can expand out their difference by:
\begin{align*}
\|f_{\mathbf{W}_0+\Delta\mathbf{W}}-f_{\mathbf{W}_0}\|
&\le
C\,L_{\sigma_1}L_{\sigma_2}L_{\sigma_3}\,\Big(
 \|W_0^{3}\|_2\,\|W_0^{2}\|_2\,\|\Delta W^{1}\|_2
 \\& + \|W_0^{3}\|_2\,\|\Delta W^{2}\|_2\,\|W_0^{1}\|_2 
 + \|W_0^{3}\|_2\,\|\Delta W^{2}\|_2\,\|\Delta W^{1}\|_2
 \\
&+ \|\Delta W^{3}\|_2\,\|W_0^{2}\|_2\,\|W_0^{1}\|_2 
 + \|\Delta W^{3}\|_2\,\|W_0^{2}\|_2\,\|\Delta W^{1}\|_2
  \\
&+ \|\Delta W^{3}\|_2\,\|\Delta W^{2}\|_2\,\|W_0^{1}\|_2 
 + \|\Delta W^{3}\|_2\,\|\Delta W^{2}\|_2\,\|\Delta W^{1}\|_2
\Big) \\
& +\;
L_{\sigma_3}L_{\sigma_2} \|\sigma_{1}({0})\| \,\Big(
\|W_0^{3}\|_2\,\|\Delta W^{2}\|_2\,
 +\;\,
\|\Delta W^{3}\|_2\,\|W_0^{2}\|_2\, \\
& +\;
\,
\|\Delta W^{3}\|_2\,\|\Delta W^{2}\|_2\Big)  +\;
L_{\sigma_3}\,
\|\Delta W^{3}\|_2\,\|\sigma_{2}({0})\|,
\end{align*}
and so on. Thus a proof by induction indicates the difference between $f_{\mathbf{W}_0}$ and $f_{\mathbf{W}_0 + \Delta \mathbf{W}}$ for a $L$-layered model can be upper bounded by:
$$\|f_{\mathbf{W}_0 + \Delta \mathbf{W}} - f_{\mathbf{W}}\| \leq C\prod_{i=1}^LL_{\sigma_i}[\sum_{i=1}^{2^L-1}\prod_{j=1}^LP_L(i, j)] + \sum_{i=2; i\neq2^a-1, a\in [L]}^{2^L-2}F(i).$$
If we treat $\Delta W^{i}$ and $W_0^{i}$ as binary classes, we can give each identity 0 and 1 respectively; thus $W_0^{3}W_0^{2}W_0^{1}$ corresponds to $111_2$ or 7 and $\Delta W^{3}W_0^{2} \Delta W^{1}$ corresponds to $010_2$ or 2. Thus, using this pattern, we can expand our summation using the following expression:
\[
P_L(i,j) =
\begin{cases}
\mathrm\|{W_0}^{L-j+1}\|_2, & \text{if } \left\lfloor \dfrac{i-1}{2^{\,L-j}} \right\rfloor \bmod 2 = 1, \\[6pt]
\mathrm\|\Delta W^{L-j+1}\|_2, & \text{if } \left\lfloor \dfrac{i-1}{2^{\,L-j}} \right\rfloor \bmod 2 = 0.
\end{cases},
\]
\[
F(i) =
\|\sigma_{(L-\lfloor \log_2(i)\rfloor)}({0})|| \prod_{j=1}^{\lfloor \log_2(i)\rfloor}[L_{\sigma_{(L-j+1)}}H(i,j)]
\]
and
\[
H(i, j) =
\begin{cases}
\mathrm \|\Delta W^{L-j+1}\|_2 &\text{if } \lfloor \frac{i}{2^{\lfloor \log_2(i)\rfloor - j}} \rfloor \mod 2 = 1,\\[6pt]
\mathrm \|W_0^{L-j+1}\|_2 &\text{if }  \lfloor \frac{i}{2^{\lfloor \log_2(i)\rfloor - j}} \rfloor \mod2=0.
\end{cases},
\]
where $F(i)$ and $H(i, j)$ are index functions that can be visualized in Figure \ref{fig:AB_recursivecollapse}. For representational purposes, every vertex that has three red edges adds the $\ell_2$ norm of the layer below its activation function on the zero vector. When a vertex has two different colored edges strictly below it, it collapses into an $A$ and $B$ sub-component. When this occurs, no additional offset term is added to our summation. A total of $2^L-(L+1)$ of these offset terms will be added. Both $P(i,j)$ and $H(i, j)$ can also take cases by even and odd inputs as their indexing requires modulus arithmetic over binary classifications $(\| W_0^{i}\|\text{ and } \|\Delta W^{i}\|).$

Now that we have an upper bound for the difference of our hypotheses, we estimate the difference in terms of true loss and empirical loss:
\begin{align*}
\cL_{\rm global}(\mathbf{W}_0 + \Delta \mathbf{W}) - \cL_{\rm global}(\mathbf{W}_0)
&= \mathbb{E}_{ \mathcal{X}, \mathcal{Y}\sim\nu}[\ell(f_{\mathbf{W}_0 + \Delta \mathbf{W}}(X), Y)]
 - \mathbb{E}_{ \mathcal{X}, \mathcal{Y}\sim\nu }[\ell(f_{\mathbf{W}_0}(X),Y)] \\
&= \mathbb{E}_{ \mathcal{X}, \mathcal{Y}\sim\nu}[\ell(f_{\mathbf{W}_0 + \Delta \mathbf{W}}(X),Y)
 - \ell(f_{\mathbf{W}_0}(X),Y)].\\
 &\leq \mathbb{E}_{ \mathcal{X}, \mathcal{Y}\sim\nu}[L_{\cL} \|f_{\mathbf{W}_0 + \Delta \mathbf{W}}(X)-f_{\mathbf{W}_0}(X)\|] \\
 &\leq  \mathbb{E}_{ \mathcal{X}, \mathcal{Y}\sim\nu}\left[L_{\cL}\left(C\prod_{k=1}^LL_{\sigma_i}\left[\sum_{i=1}^{2^L-1}\prod_{j=1}^LP_L(i, j)\right] + \sum_{i\neq2^a-1}^{2^L-2}F(i)\right)\right] \\
 &= L_{\cL}\left[C\prod_{k=1}^LL_{\sigma_i}\left[\sum_{i=1}^{2^L-1}\prod_{j=1}^LP_L(i, j)\right] + \sum_{i\neq 2^a-1}^{2^L-2}F(i)\right].
\end{align*}

 \begin{figure}
    \centering
    \includegraphics[width=0.6\linewidth]{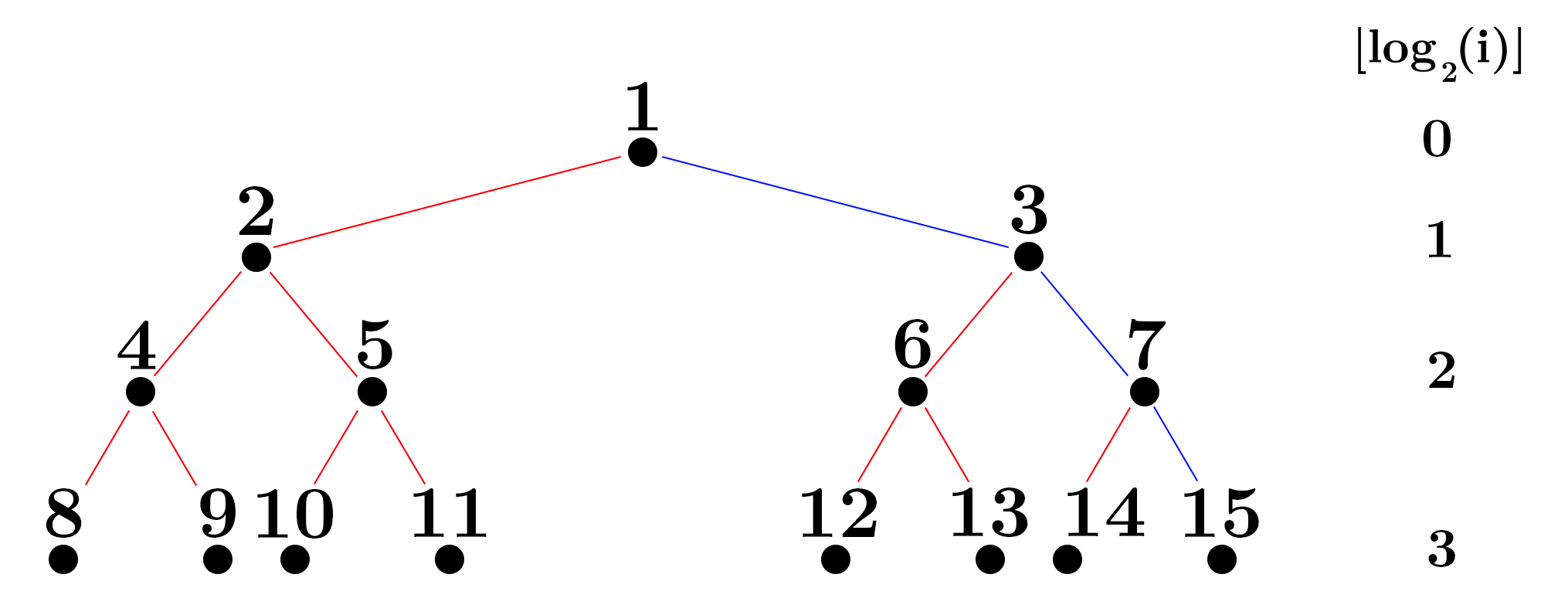}
    \caption{$\|f_{\textbf{W}_0+\Delta \textbf{W}} - f_{\textbf{W}_0} \|$ Visual representation of the recursive collapse of differences.}
    \label{fig:AB_recursivecollapse}
\end{figure}
Similarly,
\begin{align*}
\cL(\mathbf{W}_0 + \Delta \mathbf{W}) - \cL(\mathbf{W}_0)
&= \frac{1}{L}\sum_{i'=1}^{L} \ell\!\big(f_{\mathbf{W}_0 + \Delta \mathbf{W}}(x_i'),\, y_i'\big)
 \;-\; \frac{1}{L}\sum_{i'=1}^{L} \ell\!\big(f_{\mathbf{W}_0}(x_i'),\, y_i'\big) \\
&= \frac{1}{L}\sum_{i'=1}^{L}
\Big[\ell\!\big(f_{\mathbf{W}_0 + \Delta \mathbf{W}}(x_i'),\, y_i'\big)
      - \ell\!\big(f_{\mathbf{W}_0}(x_i'),\, y_i'\big)\Big] \\
&\le \frac{1}{L}\sum_{i'=1}^{L}
L_{\cL}\,\big\|f_{\mathbf{W}_0 + \Delta \mathbf{W}}(x_i') - f_{\mathbf{W}_0}(x_i')\big\| \\
&\le \frac{1}{L}\sum_{i'=1}^{L}
L_{\cL}\,\Big(
 C\,\prod_{k=1}^{L} L_{\sigma_k}\,
 \Big[\sum_{i=1}^{2^{L}-1}\,\prod_{j=1}^{L} P_L(i,j)\Big]
 \;+\; \sum_{i\neq2^a-1}^{2^{L}-2} F(i)\Big) \\
&= L_{\cL}\,\Big(
 C\,\prod_{k=1}^{L} L_{\sigma_k}\,
 \Big[\sum_{i=1}^{2^{L}-1}\,\prod_{j=1}^{L} P_L(i,j)\Big]
 \;+\; \sum_{i\neq2^a-1}^{2^{L}-2} F(i)\Big).
\end{align*}
Using the triangle inequality, we reach:
\begin{align*}
\big|\mathcal{G}(\mathbf{W}_0 + \Delta \mathbf{W}) - \mathcal{G}(\mathbf{W}_0)\big|
&= \big|\cL_{\rm global}(\mathbf{W}_0 + \Delta \mathbf{W}) - \cL(\mathbf{W}_0 + \Delta \mathbf{W})
      - \cL_{\rm global}(\mathbf{W}_0) + \cL(\mathbf{W}_0)\big| \\
&\le \big|\cL_{\rm global}(\mathbf{W}_0 + \Delta \mathbf{W}) - \cL_{\rm global}(\mathbf{W}_0)\big|
   + \big|\cL(\mathbf{W}_0 + \Delta \mathbf{W}) - \cL(\mathbf{W}_0)\big| \\
&\le 2\,L_{\cL}\,\Big(
 C\,\prod_{k=1}^{L} L_{\sigma_k}\,
 \Big[\sum_{i=1}^{2^{L}-1}\,\prod_{j=1}^{L} P_L(i,j)\Big]
 \;+\; \sum_{i\neq 2^a-1}^{2^{L}-2} F(i)\Big).
\end{align*}
Finally, we obtain the inequality:
$$\mathcal{G}(\mathbf{W}_0 + \Delta \mathbf{W}) \leq \mathcal{G}\left(\mathbf{W}_0\right) + 2\,L_{\cL}\,\left(
 C\,\prod_{k=1}^{L} L_{\sigma_k}\,
 \Big[\sum_{i=1}^{2^{L}-1}\,\prod_{j=1}^{L} P_L(i,j)\Big]
 \;+\; \sum_{i=1; i\neq2^a-1;a\in[L]}^{2^{L}-2} F(i)\right).$$

\smartparagraph{Bound around $f_{\Delta\mathbf{W}}$.} We can also perturb around $\mathcal{G}(\Delta \textbf{W})$
by swapping the roles or conditions of $W_0^{(i)}$ and $\Delta W^{(i)}$ in the zero–activation bookkeeping function $H(i, j)$. This requires us to ignore the indices $2^a, a\in [L]$ as opposed to $2^a-1, a\in[L]$ as viewable in Figure~\ref{fig:AB_recursivecollapse}. Similarly, the function $P_L(\cdot,\cdot)$ can be kept unchanged by shifting the summation index range from $1\!:\!2^{L}\!-\!1$ to $2\!:\!2^{L}$. Thus
\[
\mathcal{G}(\mathbf{W}_0+\Delta\mathbf{W})
\;\le\;
\mathcal{G}(\Delta\mathbf{W})
\;+\;
2\,L_{\cL}\,\left(
 C\,\prod_{k=1}^{L} L_{\sigma_k}\,
 \Big[\sum_{i=2}^{2^{L}}\,\prod_{j=1}^{L} P_L(i,j)\Big]
 \;+\; \sum_{i=3;i\neq2^a; a \in [L]}^{2^{L}-1} {F}(i)\right).
\]

Consequently, we can conclude with:
$$\mathcal{G}(\mathbf{W}_0 + \Delta \mathbf{W}) \leq \min\left( \mathcal{G}(\mathbf{W}_0) + \Phi_{\Delta \mathbf{W}}, \mathcal{G}(\Delta \mathbf{W}) + \Phi_{\mathbf{W}_0}\right)$$

\[
\Phi_{\mathbf{W}} =
\begin{cases}
2\,L_{\cL}\,\left(
 C\,\prod_{k=1}^{L} L_{\sigma_k}\,
 \Big[\sum_{i=1}^{2^{L}-1}\,\prod_{j=1}^{L} P_L(i,j)\Big]
 \;+\; \sum_{i\neq 2^a-1}^{2^{L}-2} F(i)\right),
& \text{for }\mathbf{W}=\Delta\mathbf{W}, \\[0.6em]
2\,L_{\cL}\,\left(
 C\,\prod_{k=1}^{L} L_{\sigma_k}\,
 \Big[\sum_{i=2}^{2^{L}}\,\prod_{j=1}^{L} P_L(i,j)\Big]
 \;+\; \sum_{i\neq2^a}^{2^{L}-1} {F}(i)\right),
& \text{for }\mathbf{W}=\mathbf{W}_0.
\end{cases}
\]
Hence, the result. 
\end{proof}

\subsubsection{Neural Network with No activation Function---Special case of Theorem \ref{theorem:nonlinearGenBound}}\label{appendix:special cases}
We can upper bound the generalization error of a neural network with no nonlinear activation functions, i.e., $\sigma_i=I_{n_i}$ for all $i\in[L]$. We additionally include the simplest case of a one-layer linear network.

\begin{corollary}\label{corollary:lineargen} 
Let Assumption \ref{ass:bounded feature vector} hold, and $\cL$ follow Assumption \ref{ass:smoothness}. Let $\sigma_i=I_{n_i},$ for all $i\in[L]$, and $f_{\mathbf{W}_0+BA}(x)=(W_0^{L}+B^LA^{L}(\cdots(W_0^{2}+B^2A^{2}(W_0^{1}+B^1A^{1})x)\cdots))$. Then we have:
$${\mathcal{G}(\mathbf{W}_0 + \Delta\textbf{W}) \leq \min(\mathcal{G}(\mathbf{W}_0) + 2CL_{\cL}\sum_{i=1}^{2^L-1}\prod_{j=1}^L P_L(i, j), \mathcal{G}(\Delta\textbf{W}) +2CL_{\cL}\sum_{i=2}^{2^L}\prod_{j=1}^LP_L(i, j)}).$$
\end{corollary}

\begin{remark} \label{rmk:lineargenbound}
Let Assumption \ref{ass:bounded feature vector} hold, and $\cL$ follow Assumption \ref{ass:smoothness}. If $L=1$, i.e., the model consists of only 1 layer, then we have:
$${\mathcal{G}(\mathbf{W}_0 + \Delta \textbf{W}) \leq \min(\mathcal{G}(\mathbf{W}_0) + 2CL_{\cL}\| \Delta W\|_2, \mathcal{G}(\Delta \textbf{W}) + 2CL_{\cL}\|W_0\|_2)}.$$
\end{remark}
\begin{table*}
  \centering
  \scriptsize
  \setlength{\tabcolsep}{2pt}
  \renewcommand{\arraystretch}{0.8}
  \caption{\small{\textbf{Summary of the benchmarks, quality metrics, and trainable parameters.} For LoRA and Asymmetric LoRA methods, we report their percentage of trainable parameters relative to FFT. }}
  \label{tab:task-summary}
\scalebox{0.93}{\begin{tabular}{@{} l l l r l l l l @{} }
  \toprule
 \textbf{Task} & \textbf{Model} & \makecell[l]{\textbf{Pretrained}\\\textbf{On}} & \makecell[l]{\textbf{Fine-Tuned}\\\textbf{On}} & \makecell[l]{\textbf{Trainable}\\\textbf{Parameters (FFT)}} & \textbf{LoRA} & \makecell[l]{\textbf{Asymmetric}\\\textbf{LoRA}} & \makecell[l]{\textbf{Quality}\\\textbf{Metric}}\\
  \midrule
  \multirow{13}{*}{Natural Language Processing}
      & \multirow{2}{*}{RoBERTa-Base}
        & \multirow{2}{*}{English language corpora} & MRPC & 124.6M & 0.944 & 0.708 & Accuracy\\
        && & CoLA & 124.6M & 0.944 & 0.708 & MCC\\
        \cmidrule(lr){2-8}
      & \multirow{2}{*}{RoBERTa-Large}
        & \multirow{2}{*}{English language corpora} & MRPC & 355.4M & 0.735 & 0.515 & Accuracy\\
        && & CoLA & 355.4M & 0.735 & 0.515 & MCC\\
        \cmidrule(lr){2-8}
      & \multirow{3}{*}{DeBERTa v2 XXL}
        & \multirow{3}{*}{English language corpora} & MRPC & 1.56B & 0.301 & 0.151 & Accuracy\\
        && & TREC-50 & 1.56B & 0.301 & 0.151 & Accuracy\\
        && & PAWS & 1.56B & 0.301 & 0.151 & Accuracy\\
        \cmidrule(lr){2-8}
      & \multirow{5}{*}{DeBERTa v3 Base}
        & \multirow{5}{*}{English language corpora} & MRPC & 184.4M & 0.641 & 0.481 & Accuracy\\
        && & RTE & 184.4M & 0.641 & 0.481 & Accuracy\\
        && & STS-B & 184.4M & 0.640 & 0.481 & Accuracy\\
        && & TREC-50 & 184.4M & 0.661 & 0.501 & Accuracy\\
        && & PAWS & 184.4M & 0.641 & 0.481 & Accuracy\\
        \cmidrule(lr){2-8}
      & GPT2-Small & WebText & E2E & 124.4M & 6.140 & 5.904 & Accuracy\\
  \midrule
  \multirow{4}{*}{Image Classification}
      & \multirow{2}{*}{ViT-Tiny}
            & \multirow{2}{*}{ImageNet-1K} & OfficeHome & 5.54M & 2.815 & 1.518 & Accuracy\\
            && & Cifar10 & 5.53M & 2.633 & 1.333 & Accuracy\\
            \cmidrule(lr){2-8}
      & \multirow{2}{*}{ViT-Base}
            & \multirow{2}{*}{ImageNet-21K then ImageNet-1K} & OfficeHome & 85.8M & 0.740 & 0.399 & Accuracy\\
            && & Cifar10 & 85.8M & 0.692 & 0.350 & Accuracy\\
  \midrule
  Coding Generation
      & DeepSeek-Coder-Base & Repo-Level Code Corpus & DJANGO & 1.35B & 0.233 & 0.117 & Exact Match\\
  \midrule
  \multirow{6}{*}{Logical Reasoning}
      & \multirow{4}{*}{TinyLlama} & \multirow{4}{*}{SlimPajama}
            & OpenBookQA & 1.03B & 0.218 & 0.079 & Accuracy\\
            && & FOLIO & 1.03B & 0.218 & 0.079 & Accuracy\\
            && & LogiQA & 1.03B & 0.218 & 0.079 & Accuracy\\
            && & CLUTRR & 1.03B & 0.221 & 0.082 & Accuracy\\
        \cmidrule(lr){2-8}
      & \multirow{2}{*}{LlaMA3-8B} & \multirow{2}{*}{Large-scale multilingual corpora}
            & OpenBookQA & 8.03B & 0.091 & 0.035 & Accuracy\\
            && & CLUTRR & 8.03B & 0.092 & 0.036 & Accuracy\\
  \bottomrule
\end{tabular}}
\end{table*}

\subsubsection{Tightness of the bounds in Theorem \ref{theorem:nonlinearGenBound}}\label{appendix:tightness of bound} 
% \aritra{Christian: Rewrite the section such that every sentence makes sense. Please do a good job this time. You should now understand what I want. The first sentence does not make any sense. Please work sensibly.}

% \hl{We do not wish or attempt, we confirm. Use a direct sentence. Stop hyperbolizing.} 

% \hl{Write small sentences. Poor people write long sentences.}

We demonstrate Theorem \ref{theorem:nonlinearGenBound} as an appropriate upper bound on the generalization error. We show the case where $f_{\textbf{W}_0 + \Delta\textbf{W}}=f_{\textbf{W}_0}$ and guarantee that $\cG (\textbf{W}_0+\Delta \textbf{W})=\cG(\textbf{W}_0).$

Assume $\Delta \textbf{W}$ was never trained, i.e., $\|\Delta W^{i}\|=0,$ for all $i \in [L]$. Denote $\hat{F}:=\sum_{i=1|i\neq2^a-1;i\in[L]}^{2^L-2} F(i)$ Then we have:
\begin{align*}
    |\cG(\textbf{W}_0 + \Delta \textbf{W})-\cG(\textbf{W}_0)| 
    &\overset{\rm Theorem~\ref{theorem:nonlinearGenBound}}{\leq} 
    2L_{\cL}(C\prod_{i=1}^LL_{\sigma_i}\sum_{i=1}^{2^L-1}\prod_{j=1}^LP_L(i,j) 
    + \sum_{i \neq 2^a-1:a\in[L]}^{2^L-2}F(i)) \\
    &\overset{}{=} 2L_{\cL}\big(C\prod_{i=1}^LL_{\sigma_i}(\|W_0^{i}\|_2+ \|\Delta W^{i}\|_2) 
    - C\prod_{i=1}^LL_{\sigma_i}\| W_0^{i} \|_2 + \hat{F}\big)\\
    &\overset{\rm \| \Delta W^{i}\|_2=0;}{=} 2L_{\cL}\big(C\prod_{i=1}^LL_{\sigma_i}\|W_0^{i}\|_2 - C\prod_{i=1}^LL_{\sigma_i}\|W_0^{i}\|_2+\hat{F}\big)\\
    &= 2L_{\cL}\hat{F}.
\end{align*}
Since each $F(i)$ does not take entries from $2^a-1,$ where $a\in[L]$, at least one $H(i,j)$ returns the spectral norm of one of the $\Delta \textbf{W}$ layers, returning 0 by construction. Hence, each $F(i)$ returns 0 and we obtain the result: $|\cG(\textbf{W}_0+\Delta\textbf{W})-\cG(\textbf{W}_0)|\leq0$ confirming that $\cG(\textbf{W}_0+\Delta\textbf{W})=\cG(\textbf{W}_0),$ if $\Delta \textbf{W}$ was never trained. This way, we make sure the generalization measure would be unchanged and does not risk including unnecessary terms.

\subsubsection{Adapting Theorem~\ref{theorem:nonlinearGenBound} to Attention Mechanism} 
\label{subsubsec:transformer}
Theorem~\ref{theorem:nonlinearGenBound} applies to any architecture that can be written as a composition of linear maps and Lipschitz maps, under bounded input. We therefore view transformer blocks as fitting the theorem. Let embedded inputs be bounded, and let $X$ denote the current input sequence.

The MLP sub-blocks follow the same structure as standard DNNs, so it suffices if their activation functions are Lipschitz~\cite{goodfellow2016deep}. Transformers also include residual connections and normalization. Residual connections take the form $X + F(X)$, which preserves the same decomposition up to a constant~\cite{he2016deep}. LayerNorm is applied elementwise across features; under bounded activations, LayerNorm is Lipschitz on the bounded set, and therefore it can be treated as another Lipschitz map in the composition~\cite{xu2019understanding}. For the multi-head attention (MHA) sub-blocks, the sequence $X$ is projected into queries, keys, and values such that
$Q=XW_Q$, $K=XW_K$, $V=XW_V$, and the attention output is projected by $W_O$. Each head forms attention weights from $QK^\top$ using softmax and masks and then directly applies them to $V$~\cite{vaswani2017attention}. 

Define the linear map $\cT(X) := [XW_Q,| XW_K| XW_V]$ and define
$\sigma(\cT(X)) =\sigma_{(Q,K,V)}(\cdot) = \mathrm{Softmax}(\frac{QK^\top}{\sqrt{d_k}}+M)V$,
where $M$ is an attention mask. Then one attention head can be written as
$\cV(X) = (\sigma_{(Q,K,V)}\circ \cT(X))\,W_O$.
We treat $\sigma_{(Q,K,V)}$ as a Lipschitz operator and denote the Lipschitz constant by $L_\sigma$. In particular, Softmax is $\frac{1}{2}$-Lipschitz uniformly across all $\ell_p$ norms~\cite{nair2025softmax}.

Consider $\|\sigma(\cT_{\textbf{W}_0+\Delta\textbf{W}}(X))W^{\textbf{W}_0+\Delta\textbf{W}}_O - \sigma(\cT_{\textbf{W}_0}(X))W^{\textbf{W}_0}_O\|$. We expand it in the same way as in the proof of Theorem~\ref{theorem:nonlinearGenBound}. Denote $Z_\Delta = \sigma_{(Q, K, V)}\circ(\cT_{\textbf{W}_0+\Delta \textbf{W}}(X))$ and $Z_0 = \sigma_{(QKV)}\circ(\cT_{\textbf{W}_0}(X))$. Let $W_O^{\textbf{W}_0+\Delta \textbf{W}} = W_{O, 0} + \Delta W_{O}$. Then, 

\begin{align*}
\|\sigma(\cT_{\textbf{W}_0+\Delta\textbf{W}}(X))W^{\textbf{W}_0+\Delta\textbf{W}}_O - \sigma(\cT_{\textbf{W}_0}(X))W^{\textbf{W}_0}_O\| 
& \overset{\rm By\; construction}{=}
\| Z_\Delta W_O^{\mathbf W_0+\Delta \mathbf W} - Z_0 W_O^{\mathbf W_0}\| \\ 
& \overset{By \;W_O~\rm Definition}{=}\| Z_\Delta (W_{O,0}+\Delta W_O) - Z_0 W_{O,0} \| \\ 
%&= \| (Z_\Delta - Z_0) W_{O,0} + Z_\Delta \Delta W_O \| \\ 
& \overset{\rm By\; Triangle~Inequality}{\le}
\| (Z_\Delta - Z_0) W_{O,0} \| + \| Z_\Delta \Delta W_O \| \\ 
& \overset{\rm Inequality~\eqref{eq:norm_bound_inequality}}{\le}
\|Z_\Delta - Z_0\| \, \|W_{O,0}\|_2 + \|Z_\Delta\| \, \|\Delta W_O\|_2. &(\star)
\end{align*}
% \[
% \|\sigma(\cT_{\textbf{W}_0+\Delta\textbf{W}}(X))W^{\textbf{W}_0+\Delta\textbf{W}}_O - \sigma(\cT_{\textbf{W}_0}(X))W^{\textbf{W}_0}_O\|
% \overset{\rm Construction}{=}
% \| Z_\Delta W_O^{\mathbf W_0+\Delta \mathbf W} - Z_0 W_O^{\mathbf W_0} \|
% \]

% \[
% \| Z_\Delta W_O^{\mathbf W_0+\Delta \mathbf W} - Z_0 W_O^{\mathbf W_0} \|
% \overset{W_O~\rm Definition}{=}
% \| Z_\Delta (W_{O,0}+\Delta W_O) - Z_0 W_{O,0} \|
% \]

% \[
% =
% \| Z_\Delta W_{O,0} + Z_\Delta \Delta W_O - Z_0 W_{O,0} \|
% =
% \| (Z_\Delta - Z_0) W_{O,0} + Z_\Delta \Delta W_O \|
% \]

% \[
% \overset{\rm Triangle~Inequality}{\le}
% \| (Z_\Delta - Z_0) W_{O,0} \| + \| Z_\Delta \Delta W_O \|
% \overset{\rm Inequality~\ref{eq:norm_bound_inequality}}{\le}
% \|Z_\Delta - Z_0\| \, \|W_{O,0}\|_2 + \|Z_\Delta\| \, \|\Delta W_O\|_2 ~(\star)
% \]
Now, we consider bounding the term $\|Z_\Delta-Z_0\|$ as follows:
\begin{align*}
    \|Z_\Delta - Z_0\| & \overset{\rm By\;construction}{=}
\|\sigma(\cT_{\mathbf W_0+\Delta \mathbf W}(X))-\sigma(\cT_{\mathbf W_0}(X))\| 
\\ & \overset{\rm Assumption~\ref{ass:activationlipschitz}}{\le}
L_\sigma \|\cT_{\mathbf W_0+\Delta \mathbf W}(X)-\cT_{\mathbf W_0}(X)\|.
\end{align*}

% \[
% \|Z_\Delta - Z_0\|
% \overset{\rm Construction}{=}
% \|\sigma(\cT_{\mathbf W_0+\Delta \mathbf W}(X))-\sigma(\cT_{\mathbf W_0}(X))\|
% \overset{\rm Assumption~\ref{ass:activationlipschitz}}{\le}
% L_\sigma \|\cT_{\mathbf W_0+\Delta \mathbf W}(X)-\cT_{\mathbf W_0}(X)\|
% \]
Expanding $\cT_{\textbf{W}_0+\Delta \textbf{W}}(X) - \cT_{\textbf{W}_0}(X)$ we obtain:
\begin{align*}
    \cT_{\mathbf W_0+\Delta \mathbf W}(X)-\cT_{\mathbf W_0}(X) & \overset{\rm By\;construction}{=}
X(W_{QKV,0}+\Delta W_{QKV})-XW_{QKV,0}\\ &= X\Delta W_{QKV}, 
\end{align*}
which implies
\begin{align*}
\|\cT_{\mathbf W_0+\Delta \mathbf W}(X)-\cT_{\mathbf W_0}(X)\| & \overset{\rm Inequality~\ref{eq:norm_bound_inequality}}{\le}
\|X\| \, \|\Delta W_{QKV}\|_2. &(\triangle)
\end{align*}

% \[
% \cT_{\mathbf W_0+\Delta \mathbf W}(X)-\cT_{\mathbf W_0}(X)
% \overset{\rm Construction}{=}
% X(W_{QKV,0}+\Delta W_{QKV})-XW_{QKV,0}
% =
% X\Delta W_{QKV}
% \]

% \[
% \|\cT_{\mathbf W_0+\Delta \mathbf W}(X)-\cT_{\mathbf W_0}(X)\|
% \overset{\rm Inequality~\ref{eq:norm_bound_inequality}}{\le}
% \|X\| \, \|\Delta W_{QKV}\|_2~(\triangle)
% \]
Finally, we obtain the result: 
\begin{align*}
    &\| \sigma(\cT_{\textbf{W}_0+\Delta\textbf{W}}(X))W^{\textbf{W}_0+\Delta\textbf{W}}_O - \sigma(\cT_{\textbf{W}_0}(X))W^{\textbf{W}_0}_O\|\\& \overset{\rm(\star)~and~(\triangle)}{\le}
L_\sigma \|X\| \, \|\Delta W_{QKV}\|_2 \, \|W_{O,0}\|_2
+
\|Z_\Delta\| \, \|\Delta W_O\|_2.
\end{align*}
% \[
% \| \sigma(\cT_{\textbf{W}_0+\Delta\textbf{W}}(X))W^{\textbf{W}_0+\Delta\textbf{W}}_O - \sigma(\cT_{\textbf{W}_0}(X))W^{\textbf{W}_0}_O\|
% \overset{\rm(\star)~and~(\triangle)}{\le}
% L_\sigma \|X\| \, \|\Delta W_{QKV}\|_2 \, \|W_{O,0}\|_2
% +
% \|Z_\Delta\| \, \|\Delta W_O\|_2.
% \]

Hence, the main difference for Theorem~\ref{theorem:nonlinearGenBound} between transformers and deep neural networks appears in the MHA block, where subtracting the fine-tuned and pretrained models produces an additional term $\|Z_\Delta\|\|\Delta W_O\|_2$ from the output projection update. This term is controlled under our setup since $X$ is bounded and $V=X(W_V+\Delta W_V)$ is linear, we have $\|V\|\le \|X\|\|W_V+\Delta W_V\|_2$. Thus $\|Z_\Delta\|=\|\rm Softmax(\cdot)V\|\le C_Z$ for some constant $C_Z$ depending on the input bound and $\|W_V+\Delta W_V\|_2$.

If $W_O$ is frozen ($\Delta W_O=0$), the additional term vanishes, and the bound simplifies to $L_\sigma\|X\|\,\|\Delta W_{QKV}\|_2\,\|W_{O,0}\|_2$. This has the same proof structure as in Theorem~\ref{theorem:nonlinearGenBound}, with $\|W_{O,0}\|_2$ contributing as an additional spectral factor. The attention mask $M$ is fixed, so it does not contribute to the parameter difference. Also, with $W_{QKV}=[W_Q\,W_K\,W_V]$ where we concatenate the projection matrices, we have $\|\Delta W_{QKV}\|_2 \le \sqrt{\|\Delta W_Q\|^2_2 + \|\Delta W_K\|^2_2 + \|\Delta W_V\|^2_2}$ by properties of block matrices. This leads to the development of Theorem~\ref{thm:attn_head}.
\setcounter{theorem}{3}
\begin{theorem}\textbf{(Upperbound on self-attention fine-tuned update)}
\label{thm:attn_head}
Let an input sequence $X$ be bounded, such that $\|X\|\le C$. Consider one attention head as
$\mathcal V_{\mathbf W}(X) \;=\; (\sigma_{(QKV)}\circ \mathcal T_{\mathbf W}(X))W_O,
\mathcal T_{\mathbf W}(X) := [XW_Q \mid XW_K \mid XW_V],
$ where $\sigma_{(QKV)}(\cdot)=\mathrm{Softmax}(\frac{QK^\top}{\sqrt{d_k}}+M)V$ and $M$ is a fixed mask, let $X \in \mathbb{R}^{\hat{n} \times \hat{m}}$.
Then the difference in attention-head outputs satisfies
\[
\|\mathcal V_{\mathbf W_0+\Delta\mathbf W}(X)-\mathcal V_{\mathbf W_0}(X)\|
\le
C(L_\sigma\|\Delta W_{QKV}\|_2\|W_{O,0}\|_2
+
\sqrt{\hat{n}}\|W_{V,0}+\Delta W_{V}\|_2\|\Delta W_O\|_2),
\]
where $\Delta W_{QKV}$ denotes a concatenated update $\Delta W_{QKV}=[\Delta W_Q\mid \Delta W_K\mid \Delta W_V]$, $L_\sigma$ represents the Lipschitz constant of $\rm Softmax(\frac{QK^{\top}+M}{\sqrt{d_k}})V.$
\end{theorem}

\begin{proof}
We have already shown that $$\|\mathcal V_{\mathbf W_0+\Delta\mathbf W}(X)-\mathcal V_{\mathbf W_0}(X)\| \leq L_\sigma \|X\| \, \|\Delta W_{QKV}\|_2 \, \|W_{O,0}\|_2
+
\|Z_\Delta\| \, \|\Delta W_O\|_2.$$
We assume $\|X\| \le C$, hence
$$\|\mathcal V_{\mathbf W_0+\Delta\mathbf W}(X)-\mathcal V_{\mathbf W_0}(X)\| \leq CL_\sigma \, \|\Delta W_{QKV}\|_2 \, \|W_{O,0}\|_2
+
\|Z_\Delta\| \, \|\Delta W_O\|_2.$$
We aim to upper bound $Z_\Delta$; $Z_\Delta = \sigma_{(QKV)}\circ(\cT_{\textbf{W}_0+\Delta \textbf{W}}(X))=\text{Softmax}(\frac{QK^\top}{\sqrt{d_k}})V$ which gives us the relationship
$$\|Z_\Delta\| = \| \text{Softmax}(\frac{QK^\top}{\sqrt{d_k}}+M)V\| \leq \| \text{Softmax}(\frac{QK^\top}{\sqrt{d_k}}+M)\|_2 \|V\| \le \sqrt{\hat{n}}\|V\|$$
$$\sqrt{\hat{n}}\|V\| = \sqrt{\hat{n}}\|X(W_{V, 0}+\Delta W_V)\| \le \sqrt{\hat{n}}\|X\|\|W_{V, 0}+\Delta W_V\|_2\le C\sqrt{\hat{n}}\|W_{V,0}+\Delta W_V\|_2.$$
To justify the $\sqrt{\hat n}$ factor, define the attention weight matrix
\[
S := \mathrm{Softmax}(\frac{QK^\top}{\sqrt{d_k}}+M)\in\mathbb R^{\hat n\times \hat n},
\]
where Softmax is applied row-wise. Each row of $S$ is a probability vector, hence $S_{ij}\ge 0$ and
$\sum_{j=1}^{\hat n} S_{ij}=1$ for all $i$. Therefore $\|S\|_\infty = 1$.
Moreover, since each column sum satisfies $\sum_{i=1}^{\hat n} S_{ij}\le \hat n$, we have $\|S\|_1\le \hat n$.
Using the inequality $\|A\|_2 \le \sqrt{\|A\|_1\|A\|_\infty}$ gives
$
\|S\|_2 \le \sqrt{\|S\|_1\|S\|_\infty}\le \sqrt{\hat n}.
$
Consequently,
\[
\|Z_\Delta\| = \|SV\| \le \|S\|_2\|V\| \le \sqrt{\hat n}\|V\|.
\] 
Hence, we conclude with
\[
\|\mathcal V_{\mathbf W_0+\Delta\mathbf W}(X)-\mathcal V_{\mathbf W_0}(X)\|
\le
C(L_\sigma\|\Delta W_{QKV}\|_2\|W_{O,0}\|_2
+
\sqrt{\hat{n}}\|W_{V,0}+\Delta W_{V}\|_2\|\Delta W_O\|_2).
\]
\end{proof}

\subsubsection{Adapting Theorem \ref{theorem:nonlinearGenBound} under special cases}\label{subsubsec:genba}
To adapt Theorem \ref{theorem:nonlinearGenBound} under special cases, we need the following general assumptions.  
\begin{assumption}
\label{assumption:lipchitz_1}
The loss function, $\ell(\cdot):\mathbb{R}^d\to\mathbb{R}$, is 1-Lipschitz, i.e, $|\ell(f_{\textbf{W}}(x), y)-\ell(f_{\textbf{W}'}(x),y)| \leq \|f_{\textbf{W}}(x)-f_{\textbf{W}'}(x)\|$ for all \textbf{W}, \textbf{W}' $\in \mathbb{R}^d$ and $(x,y) \in \cX \times \cY.$
\end{assumption}

\begin{assumption}
\label{assumption:bounded_func}
    The loss function, $\ell(\cdot):\mathbb{R}^d\to\mathbb{R}$, is bounded, i.e., there exists a constant $C_2\geq 0$ such that $\|\ell(f_{\textbf{W}}(x), y)\| \leq C_2,$ for all $\textbf{W} \in \mathbb{R}^d$ and $(x,y) \in \cX \times \cY.$
\end{assumption}

\myNum{i}\smartparagraph{Perturbing around $\mathcal{G}(\textbf{W}_0)$.}
First, we adapt Theorem 4.1 in \cite{regularizationfinetuning} into our notation and quote it below. 

\begin{theorem}\label{theorem:PAC-Bayes}
    {\textbf{(PAC-Bayes generalization bound for fine-tuning)}[\cite{regularizationfinetuning}, Theorem 4.1]}  Let Assumption~\ref{ass:bounded feature vector} hold with the requirement that $C \geq 1$. Let the loss function, $\cL$, follow Assumptions~\ref{assumption:lipchitz_1} and \ref{assumption:bounded_func}. Let 
 $\|W_0^{(i)}\|_2 \leq$ $\cA_i$ with fixed $\cA_i>1$, $\|\Delta W^{(i)}\| \leq Q_i$, for all $i \in [L]$ and $V=\max_{i \in [L]} \{m_i,n_i\}$. Let $\epsilon$ and $\delta$ be arbitrary small values. Then with probability $1-2\delta$, the following inequality holds: 
$$\mathcal{G}(\mathbf{W}_0 + \Delta \mathbf{W}) \leq \epsilon + C_2\sqrt{\frac{ \frac{36}{\epsilon^2}C^2V\log(4LVC_2)(\sum_{i=1}^L \frac{\prod_{j=1}^L(\cA_j+Q_j)}{\cA_i+Q_i})^2(\sum_{i=1}^LQ_i^2) +3 \ln\frac{|N|}{\delta}+8 }{|N|}}.$$
\end{theorem}

We now use Theorem \ref{theorem:PAC-Bayes} to obtain a bound for $\cG(\textbf{W}_0).$ The following Theorem gives that. 

\begin{theorem}\label{theorem:PAC-Extension}
    Using the Assumptions made for Theorem \ref{theorem:nonlinearGenBound} and Theorem \ref{theorem:PAC-Bayes}, the following inequality holds with probability at least $1-2\delta:$
    $$\cG(\textbf{W}_0+\Delta \textbf{W}) \le \epsilon + C_2\sqrt{\frac{3\ln \frac{|N|}{\delta}+8}{|N|}}+2L_{\cL}\big(C\prod_{i=1}^L L_{\sigma_i}\sum_{i=1}^{2^L-1}\prod_{j=1}^L P(i,j)+\sum_{i\neq2^a-1:a\in[L]}^{2^L-2}F(i)\big).$$
\end{theorem}
%\smartparagraph{Proof of Theorem \ref{theorem:PAC-Extension}.}
\begin{proof}
    We wish to find $\cG (\textbf{W}_0)$, and note that if we never train the model, we obtain the expression $\|W_{0}^{i}-W_{0}^{i}\|_2=0$. Thus, we can use $Q_i=0$ for all $i \in [L]$ and obtain:
    \begin{align*}
        \cG(\textbf{W}_0) &\overset{\rm Theorem~\ref{theorem:PAC-Bayes}}{\leq} \epsilon + C_2\sqrt{\frac{ \frac{36}{\epsilon^2}C^2V\log(4LVC_2)(\sum_{i=1}^L \frac{\prod_{j=1}^L(\cA_j+Q_j)}{\cA_i+Q_i})^2(\sum_{i=1}^LQ_i^2) +3 \ln\frac{|N|}{\delta}+8 }{|N|}}\\
        &\overset{Q_i=0;i~\in~[L]}{=} \epsilon + C_2\sqrt{\frac{ \frac{36}{\epsilon^2}C^2V\log(4LVC_2)(\sum_{i=1}^L \frac{\prod_{j=1}^L(\cA_j+0)}{\cA_i+0})^2(\sum_{i=1}^L0^2) +3 \ln\frac{|N|}{\delta}+8 }{|N|}}\\
        &= \epsilon + C_2\sqrt{\frac{3\ln \frac{|N|}{\delta}+8}{|N|}}.
    \end{align*}
    Now that we have an upper bound for $\cG(\textbf{W}_0)$, we can apply Theorem \ref{theorem:nonlinearGenBound} and obtain the following:
    \begin{align*}
        \cG(\textbf{W}_0+\Delta \textbf{W}) &\overset{\rm Theorem~\ref{theorem:nonlinearGenBound}}{\leq} \cG(\textbf{W}_0)+\Phi_{\Delta \textbf{W}}\\
        &\overset{}{\leq} \epsilon + C_2\sqrt{\frac{3\ln \frac{|N|}{\delta}+8}{|N|}}+\Phi_{\Delta\textbf{W}}. 
    \end{align*}
    By substituting the expression for $\Phi_{\Delta\textbf{W}},$ in the above expression we have:
    $$\cG(\textbf{W}_0+\Delta \textbf{W}) \le \epsilon + C_2\sqrt{\frac{3\ln \frac{|N|}{\delta}+8}{|N|}}+2L_{\cL}\big(C\prod_{i=1}^L L_{\sigma_i}\sum_{i=1}^{2^L-1}\prod_{j=1}^L P(i,j)+\sum_{i\neq2^a-1:a\in[L]}^{2^L-2}F(i)\big).$$
    This concludes the proof. 
\end{proof}

 % \myNum{v} $I(\{\Delta W_i\}_{i \in L}| \textbf{W}) \leq \hat{H}(\{\Delta W_i\}_{i \in L}|\textbf{W}) \leq \log2^{qp}=qp\log 2$, where q is the number of quantized bits, and p is the number of directly trainable parameters[\cite{xu2017information}]
 
 % \myNum{vi} $\cG(\Delta \textbf{W}) \leq \sqrt{\frac{2\sigma^2\mathbf{I}( \{\Delta W_i \}_{i\in[L]}|\textbf{W})}{|N|}}$ [\cite{xu2017information}, Lemma 1]

\myNum{i}\smartparagraph{Perturbing around $\mathcal{G} 
(\mathbf{\mathcal{A}})$.} First, we make another assumption on the loss function and then adapt Theorem 1 in \cite{xu2017information} to our notation. 

\begin{assumption}\label{assumption:sigma_subgaussian} 
The loss function, $\ell(\cdot):\mathbb{R}^d\to\mathbb{R}$, is $\sigma$-sub-gaussian, i.e., $\mathbb{E}(e^{\lambda[\ell(f_{\mathbf{W}}(X),Y)-\mathbb{E}(\ell(f_{\mathbf{W}}(X),Y))]})\leq e^{\frac{\lambda^2\sigma^2}{2}}$ for all $\lambda \in \mathbb{R}$, $\mathbf{W} \in \mathbb{R}^d.$
\end{assumption}

\begin{theorem}\label{Theorem:Information_Generalization_Error} \textbf{(Upper bound on generalization error using mutual information)}[Theorem 1~\cite{xu2017information}] Let $\mathbf{\cA}$ denote a LoRA-based algorithm that outputs $\{\mathbf{\Delta W}_i\}_{i \in [L]}$ on a fine-tuning dataset, $N$. By $\nu$ we denote the underlying distribution of the input space, $\cX$, of which the elements of the fine-tuning dataset $N$ are chosen following i.i.d. Let Assumption~\ref{assumption:sigma_subgaussian} hold. Then we have the following:
$$\cG (\mathbf{\cA})_{\nu} \leq \sqrt{\frac{2\sigma^2\mathbf{I}(\{\mathbf{\Delta W}_i\}_{i\in[L]};N|\mathbf{\cA};\mathbf{W})}{|N|}}.$$ 
\end{theorem}

Let the loss function $\cL$ follow Assumption~\ref{assumption:sigma_subgaussian}. We present the generalization error upper bounds of the LoRA variants in Table~\ref{tab:upperbounds}. For this, we use the inequality $\mathcal{G}(\mathbf{W}_0 + \mathbf{\cA}) \leq \mathcal{G}(\mathbf{\cA}) + \Phi_{\mathbf{W}_0}$, where $\mathcal{G}(\mathbf{\cA})$ is upper bounded by the use of Lemma~\ref{lemma:information_inequality} quoted below.

\begin{lemma} \label{lemma:information_inequality}(Upperbound on mutual-information)[\cite{xu2017information}] Let $\{\mathbf{\Delta W}_i\}_{i\in [L]}$ be an update to a learning algorithm. Then the mutual information is upper bounded by the uniform distribution over an updated support set, i.e., 
    $\mathbf{I}(\Delta \{\mathbf{W}_i\}_{i\in [L]}; N| \mathbf{\cA}; \mathbf{W}) \leq \ln2^{qp}=qp\ln 2,$ where $q$ represents the number of bits the learning algorithm is designed on, and $p$ is the number of trainable parameters. Thus, with the use of Theorem~\ref{Theorem:Information_Generalization_Error}, if Assumption~\ref{assumption:sigma_subgaussian} holds, then
    $\cG(\mathbf{\cA}) \leq \sqrt{\frac{2\sigma^2qp\ln2}{|N|}}$.
\end{lemma}

\smartparagraph{How do we arrive at the bounds of different LoRA variants?}

\myNum{a}\smartparagraph{LoRA+} has $\cG(\mathbf{\cA})$ upper bounded by $\sqrt{\frac{2rq\sigma^2\ln2\sum_{i=1}^L(m_i+n_i)}{|N|}}.$ The learning rate does not alter the number of trainable parameters, which leads LoRA+ to possess the same upper bounds as LoRA. We note a unique observation regarding this claim, as $\gamma_A\to0$, LoRA+ takes the lowered generalization error bound of Asymmetric LoRA since the adapter matrix, $A$, is no longer trainable.

\myNum{b}\smartparagraph{cLA} has the fine-tuned update $ B[I_r|\mathbf{0}_{m_i-r}]$, where $[I_r|\mathbf{0}_{m_i-r}]$ is a fixed constant orthogonal matrix. Thus, by using data processing inequality~\eqref{eq:information_inequality}, the mutual information between the two is preserved, i.e, $$\mathbf{I}(\{B_i[I_r|\mathbf{0}_{m_i-r}]_i \}_{i\in[L]}; N|\mathbf{\cA};\mathbf{W}) = \mathbf{I}(\{B_i\}_{i\in[L]}; N|\mathbf{\cA};\mathbf{W}).$$ Similar to~\cite{fixA}, we upper bound mutual information by the uniform distribution of a model's support; particularly $\mathbf{I}(\{B_i\}_{i\in[L]}; N|\Delta\mathbf{W};\mathbf{W}) \leq qr\ln 2\sum_{i=1}^Ln_i$, by Lemma~\ref{lemma:information_inequality}. Finally, by  Theorem~\ref{Theorem:Information_Generalization_Error}, we obtain the result $\cG(\cA) \leq \sqrt{\frac{2rq\sigma^2\ln 2\sum_{i=1}^Ln_i}{|N|}}$.

\myNum{c}\smartparagraph{c$^3$LA} has the fine-tuned update $B_1[I_r|\mathbf{0}_{m_i-r}] + B_2[\mathbf{0}_r|I_r|\mathbf{0}_{m_i-2r}] + \cdots + B_k[\mathbf{0}_{r(k-1)}|I_r|\mathbf{0}_{m_i-kr}].$ This expansion can be simplified by $\sum_{j=1}^k B_j[0_{r\times r(j-1)} \mid I_r \mid 0_{r\times (m_i-rj)}]
= [B_1 \mid B_2 \mid \cdots \mid B_k|\mathbf{0}_{n_i(m_i-kr)}]$. Using~\eqref{eq:rank_inequality}, we can upper bound the rank of $ \sum_{j=1}^k B^j[0_{r\times r(j-1)} \mid I_r \mid 0_{r\times (m_i-rj)}]$ by $kr$. Thus, the mapping $[B_1|\cdots|B_k]\to \Delta\mathbf{W}$ is injective and can be inverted by slicing the last $n_i-kr$ columns. Using DPI, this leads to the expression $$\mathbf{I}(\{\sum_{j=1}^k B_i^j[0_{r\times r(j-1)} \mid I_r \mid 0_{r\times (m_i-rj)}]\}_{i \in [L]};N|\mathbf{\cA};\mathbf{W})=\mathbf{I}(\{ [B_1|\cdots|B_k]_i\}_{i \in [L]};N|\mathbf{\cA};\mathbf{W}).$$ 
We upper bound $\mathbf{I}(\{ [B_1|\cdots|B_k]_i\}_{i \in [L]};N|\mathbf{\cA};\mathbf{W})$ by $qrk\ln2\sum_{i=1}^Ln_i$, using Lemma~\ref{lemma:information_inequality}. Hence, by Theorem~\ref{Theorem:Information_Generalization_Error}, we obtain: $\cG(\mathbf{\cA}) \leq \sqrt{\frac{2rq\sigma^2k\ln 2\sum_{i=1}^Ln _i}{|N|}}.$

\myNum{d}\smartparagraph{CoLA} has the update structure $\Delta\textbf{W}=\sum_{j=1}^kB^jA^j$. Using inequality~\eqref{eq:rank_inequality}, we upper bound the rank of each layer's update by $kr.$ By Lemma~\ref{lemma:information_inequality},  we upper bound $\mathbf{I}(\{ \sum_{j=1}^LB_i^jA_i^j\}_{i \in [L]};N|\mathbf{\cA};\textbf{W})$ by $qrk\ln2\sum_{i=1}^L(m_i+n_i).$ Hence, we obtain $\cG(\mathbf{\cA}) \leq \sqrt{\tfrac{2rq\sigma^2k\ln2\sum_{i=1}^L(m_i + n_i)}{|N|}}$, by Theorem~\ref{Theorem:Information_Generalization_Error}. 

\myNum{e}\smartparagraph{RAC-LoRA} has the fine-tuned update $\sum_{j=1}^kB^jQ^j,$ where we consider each $Q^j$ to be a frozen orthogonal matrix. This update can be represented by $\sum_{j=1}^kB^jQ^j = [B^1|B^2|\cdots|B^k][Q^1|Q^2|\cdots|Q^k]^T,$ where we can invert $[B^1|B^2|\cdots|B^k][Q^1|Q^2|\cdots|Q^L]^T]$ to $[B^1|B^2|\cdots|B^L].$ Thus by using inequality~\eqref{eq:rank_inequality}, DPI, and Lemma~\ref{lemma:information_inequality} we have $$\mathbf{I}(\{[B^1|B^2|\cdots|B^L]_i[Q^1|Q^2|\cdots|Q^L]_i^T]\}_{i \in [L]};N|\mathbf{\cA};\textbf{W})=\mathbf{I}(\{B^1|B^2|\cdots|B^L]_i]\}_{i \in [L]};N|\mathbf{\cA}; \textbf{W}),$$ which is 
$$
\mathbf{I}(\{B^1|B^2|\cdots|B^L]_i]\}_{i \in [L]};N|\mathbf{\cA}; \textbf{W}) \leq  qrk\ln2\sum_{i=1}^{L} n_i.
$$
Hence, by Theorem~\ref{Theorem:Information_Generalization_Error}, we have the result: $\cG (\mathbf{\cA}) \leq \sqrt{\tfrac{2rq\sigma^2k\ln2\sum_{i=1}^Ln_i}{|N|}}$. 

\smartparagraph{Generalization Upperbound on PaCA}

\myNum{f}\smartparagraph{PaCA} has an identical update to cLA if the fine-tuned columns are first $r$ columns of the pretrained backbone. This suggests that any generalization guarantee driven by the update motivates the same generalization behavior for PaCA as cLA. Particularly, since PaCA updates $r\sum_{i=1}^L n_i$ trainable entries (the same degrees of freedom as cLA’s $B$ matrices), Lemma~\ref{lemma:information_inequality} yields $I(\{\Delta W_i\}_{i\in[L]};N|\cA; \textbf{W})\le qr\ln 2\sum_{i=1}^L n_i.$ Plugging into Theorem~\ref{Theorem:Information_Generalization_Error}, we get $\cG(\cA)\le \dfrac{2rq\sigma^2\ln 2\sum_{i=1}^L n_i}{|N|}$, matching the generalization upper bound of cLA from Table~\ref{tab:upperbounds}.

\section{Addendum to Benchmarking and Evaluation}\label{appendix:benchmarking}

In~\S\ref{subsec:implementation}, we summarize the quality metrics and trainable parameters used for training the models in Table~\ref{tab:full-accuracy-table} and provide the specific hyperparameters for fine-tuning each model for each dataset in Table~\ref{tab:task-hyperparams}. In ~\S\ref{subsec:ablation-studies}, we present ablation studies on the effects of learning rate ($\gamma)$, scaling factor $(\alpha)$, and chain reset indices on the resulting test accuracy and test loss for varying ranks. In~\S\ref{subsec:efficiency}, we comment on the potential of our methods by naively leveraging the sparsity of our $A$ matrices. In~\S\ref{appendix:Cosine Similarity} and~\S\ref{appendix:loss-landscape}, we extend~\ref{sec:performance analysis} with the implementation details of the loss landscapes and provide additional loss landscapes and intruder dimension results.  In~\S\ref{subsec:genconnect}, we extend section~\S\ref{sec:generalization} with empirical results on generalization.

\subsection{Implementation Details}\label{subsec:implementation}

We implement the framework in Python using PyTorch~\cite{paszke2019pytorch}. We train all models with the ADAM optimizer~\cite{adam}. Training (of most models) was performed on one 80GB NVIDIA H100 GPU. The ablation studies on ViT-Tiny in Tables~\ref{tab:vittinyablation},~\ref{tab:vit-scaling-factor}, and ~\ref{tab:vit-chain_reset} were trained using one NVIDIA V100 GPU. We provide the hyperparameter settings, i.e., learning rates, learning rate scheduler, chain reset frequency, weight decay, batch size, training epochs, maximum token length or image resolution, and random seeds for the experiments in Table~\ref{tab:task-hyperparams}.

\begin{table*}
  \centering
  \tiny
  \setlength{\tabcolsep}{2pt}
  \renewcommand{\arraystretch}{0.8}
  \caption{\small{\textbf{Summary of the hyperparameters.} We used the same learning rate for LoRA methods that train $B, A$, and Asymmetric LoRA methods that only train $B$; we write (FFT, LoRA, Asym) to indicate those three sets. We selected the best model out of all epochs based on the lowest validation loss, except for the CoLA dataset, where we used the lowest Matthews Correlation Coefficient. We used rank $r=16$ and scaling factor $\alpha=2r$ for all LoRA PEFT methods. For all models, we used the ADAM optimizer~\cite{adam} with $(\beta_1,\beta_2,\epsilon)=(0.9,0.999,1e^{-8}).$ For ViT, RoBERTa, and GPT2, we used gradient clipping on global $L_2$ norm with a max of 1, and did not otherwise. For LoRA+, the learning rate for our $B$ matrix is 16 times that of $A.$}}
  \label{tab:task-hyperparams}
\scalebox{0.9}{\begin{tabular}{@{} l l c c c c c c c c c @{} }
  \toprule
   \textbf{Model} & \textbf{Dataset} & \makecell[l]{\textbf{Scheduler}\\\textbf{(Warmup LR, Ratio)}} & \makecell[l]{\textbf{Learning Rates}\\\textbf{(FFT,LoRA,Asym)}} & \makecell[l]{\textbf{Chain reset}\\\textbf{frequency}} & \makecell[l]{\textbf{Weight decay}\\\textbf{(FFT,LoRA)}} & \textbf{Batch size} & \textbf{Epochs} & \makecell[l]{\textbf{Max length or}\\\textbf{Image size}} & \textbf{Seeds}\\
   \midrule
   \multirow{2}{*}{RoBERTa-Base} &
        MRPC & Linear$(1e^{-6},0.1)$ & $(1e^{-5},3e^{-4},3e^{-4})$ & 3 & $(0.01,0)$ & 32 & 20 & 128 & (12,22,32)\\
        \cmidrule(lr){2-10}
       &CoLA & Linear$(1e^{-6},0.1)$ & $(1e^{-5},3e^{-4},3e^{-4})$ & 3 & $(0.01,0)$ & 32 & 20 & 128 & (12,22,32)\\
    \midrule
    \multirow{2}{*}{RoBERTa-Large} &
        MRPC & Linear$(1e^{-6},0.1)$ & $(1e^{-5},3e^{-4},3e^{-4})$ & 3 & $(0.01,0)$ & 32 & 20 & 128 & (12,22,32)\\
        \cmidrule(lr){2-10}
       &CoLA & Linear$(1e^{-6},0.1)$ & $(1e^{-5},3e^{-4},3e^{-4})$ & 3 & $(0.01,0)$ & 32 & 20 & 128 & (12,22,32)\\
    \midrule
    \multirow{3}{*}{DeBERTa v2 XXL} &
        MRPC & Constant & ( $1e^{-5.5}$,$1e^{-4.5}$,$1e^{-4}$) & 5 & 0 & 8 & 25 & 512 & (100,101,102)\\
        \cmidrule(lr){2-10}
       &TREC-50 & Constant & ( $1e^{-5.5}$,$1e^{-4.5}$,$1e^{-4}$) & 5 & 0 & 8 & 25 & 512 & (100,101,102)\\
       \cmidrule(lr){2-10}
       &PAWS & Constant & (1e$^{-6.5}$, $1e^{-4.5}$, $1e^{-4}$) & 5 & 0 & 8 & 10 & 512 & (100,101,102)\\
    \midrule
    \multirow{5}{*}{DeBERTa v3 Base} &
        MRPC & Constant & ($1e^{-5}$, $1e^{-3.5}$, $1e^{-3}$) & 5 & 0 & 8 & 40 & 512 & (100,101,102)\\
        \cmidrule(lr){2-10}
       &RTE  & Constant & ($1e^{-4.75}$, $1e^{-3.5}$, $1e^{-3}$) & 5 & 0 & 8 & 40 & 512 & (100,101,102)\\
       \cmidrule(lr){2-10}
       &STS-B  & Constant & ($1e^{-4.75}$, $1e^{-3.5}$, $1e^{-3}$) & 5 & 0 & 8 & 40 & 512 & (100,101,102)\\
       \cmidrule(lr){2-10}
       &TREC-50 & Constant & ($1e^{-4.75}$, $1e^{-3.25}$, $1e^{-3}$) & 5 & 0 & 8 & 40 & 512 & (100,101,102)\\
       \cmidrule(lr){2-10}
       &PAWS & Constant & ($1e^{-5}$, $1e^{-3.5}$, $1e^{-3}$) & 5 & 0 & 8 & 20 & 512 & (100,101,102)\\
    \midrule
    GPT2-Small &
        E2E & Linear$(1e^{-6},0.1)$ & $(5e^{-5},3e^{-4},3e^{-4})$ & 1 & (0.01,0) & 16 & 30 & 64 & (12,22,32)\\
    \midrule
    \multirow{2}{*}{ViT-Tiny} &
        OfficeHome & Cosine$(1e^{-6},0.05)$ & $(3e^{-4},1e^{-3},1e^{-3})$ & 5 & (0.05,0) & 64 & 30 & 224 & (12,22,32)\\
        \cmidrule(lr){2-10}
        &CIFAR-10 & Cosine$(1e^{-6},0.05)$ & $(3e^{-4},1e^{-3},1e^{-3})$ & 5 & (0.05,0) & 64 & 30 & 224 & (12,22,32)\\
    \midrule
    \multirow{2}{*}{ViT-Base} &
        OfficeHome & Cosine$(1e^{-6},0.05)$ & $(3e^{-4},1e^{-3},1e^{-3})$ & 5 & (0.05,0) & 64 & 30 & 224 & (12,22,32)\\
        \cmidrule(lr){2-10}
        &CIFAR-10 & Cosine$(1e^{-6},0.05)$ & $(3e^{-4},1e^{-3},1e^{-3})$ & 5 & (0.05,0) & 64 & 30 & 224 & (12,22,32)\\
    \midrule
    DeepSeek-Coder Base &
        DJANGO & Constant & ($1e^{-5.5}$,$1e^{-4.5}$,$1e^{-4}$) & 1 & 0 & 8 & 5 & 512 & (100,101,102)\\
    \midrule
    \multirow{4}{*}{TinyLlama} &
        OpenBookQA & Constant & ($1e^{-6.25}$,$1e^{-3.75}$,$1e^{-3.25}$) & 2 & 0 & 8 & 10 & 512 & (100,101,102)\\
        \cmidrule(lr){2-10}
        &FOLIO & Constant & ($1e^{-5}$,$1e^{-3.75}$,$1e^{-3.5}$) & 2 & 0 & 8 & 10 & 512 & (100,101,102)\\
        \cmidrule(lr){2-10}
        &LogiQA & Constant & ($1e^{-5.75}$,$1e^{-4}$,$1e^{-3.25}$) & 2 & 0 & 8 & 10 & 512 & (100,101,102)\\
        \cmidrule(lr){2-10}
        &CLUTRR & Constant & ($1e^{-6.25}$,$1e^{-5.25}$,$1e^{-4.75}$) & 2 & 0 & 8 & 10 & 512 & (100,101,102)\\  
    \midrule
    \multirow{2}{*}{Llama 3} & OpenBookQA & Constant & ($1e^{-5.25}$, $1e^{-3.75}$, $1e^{-3.5}$) & 2 & 0 & 4 & 5 & 384 & (100,101,102)\\
    \cmidrule(lr){2-10}
    & CLUTRR & Constant & ($1e^{-5.25}$, $1e^{-4.25}$, $1e^{-3.25}$) & 2 & 0 & 4 & 5 & 384 & (100,101,102)\\
    % ELI - I removed Table 10 (Configuration on DeepSeekCoder and TinyLlama regarding chain reset, num epochs, and max length) since I transferred it here. I'm noting this so you know they're accurate.
    % FOR CRISTIAN - REVIEW FOR HYPERPARAM ASSIST
    % On DeBERTa v2 XXL we use $r=16, \alpha=32,$ batch size = 8, (16 for FFT), chain reset $=5$ (2 on PAWS). We report the results for all DeBERTa v2 XXL datasets on 25 epochs with the exception of PAWS which we trained on with 10 epochs. We report the learning rate configuration for DeBERTa v2 XXL in Table~\ref{tab:deberta_v2xxl}. Similarly, we average our results over the seeds 100, 101, and 102.
  \bottomrule
\end{tabular}}
\vspace{-3mm}
\end{table*}

\begin{figure}
    \centering
    \includegraphics[width=0.5\linewidth]{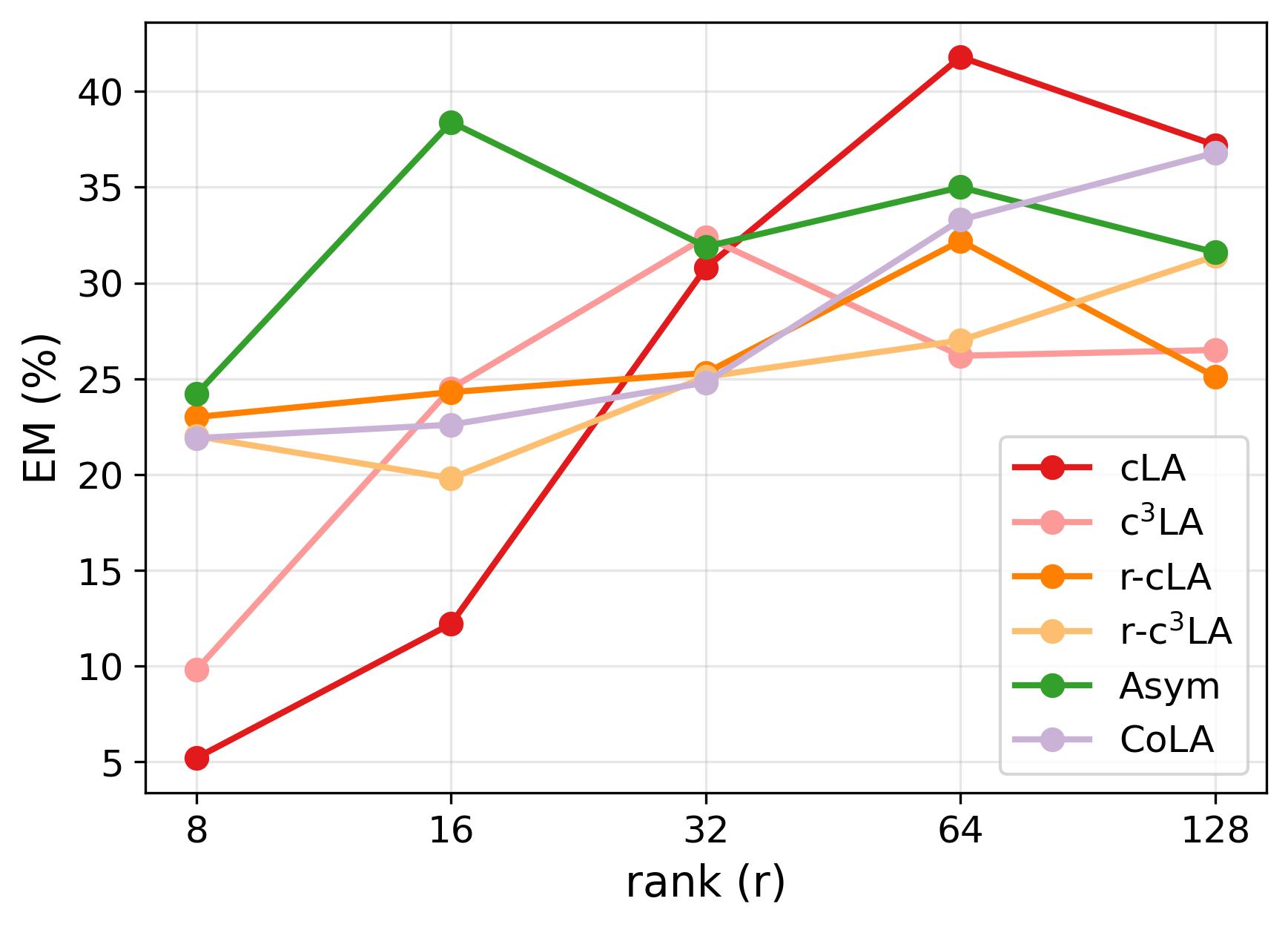}
    \caption{Performance of cLA, c$^3$La, r-cLA, r-c$^3$LA, CoLA, and Asymmetric LoRA on DeepseekCoder fine-tuned on the DJANGO Dataset by Exact Match (EM). We vary the target rank. The \emph{epoch selection} is based on the highest validation accuracy instead of the lowest validation loss. The figure demonstrates that the difference between the sparsity-induced LoRA variants cLA and their non-sparse counterparts tends to decrease with increasing rank.}
    \label{fig:cla_deepseek}
\end{figure}

\subsection{The Effects of Learning Rate, Scaling Factor, and Chain Reset Frequency on Quality Metric Over Various Ranks}\label{subsec:ablation-studies}

The ideal learning rate of an LLM tends to scale inversely with its size~\cite{scalelaw}. Many papers suggest a default scaling factor of $2r$~\cite{bidermanlora,loravfft}.~\cite{rac} suggests that for sufficiently low learning rates, performing a chain reset every epoch is optimal. We validate the first claim under LoRA fine-tuning methods via ablation studies over learning rates presented in Tables~\ref{tab:deberta3ablation}-\ref{tab:vittinyablation}. Similarly, we assess the scaling factor baseline choice in Tables~\ref{tab:deberta-scaling-factor} and~\ref{tab:vit-scaling-factor}, and the optimal chain reset frequency in Table~\ref{tab:chain_reset}. For the ablation studies, we fine-tuned DeBERTaV3-Base on the MRPC, TREC-50, and PAWS for learning rate, MRPC and TREC-50 for scaling factor,and MRPC, CoLA, RTE, and TREC-50 for chain reset over various ranks. We then re-ran the same experiments on ViT-Tiny fine-tuned on the OfficeHome and CIFAR-10 datasets. We ran for 30 epochs.

As shown in Tables~\ref{tab:deberta3ablation} and~\ref{tab:vittinyablation}, Asymmetric LoRA methods are more sensitive to varying learning rates than methods that train both matrices $B,A.$ We notice that the cLA has a wide variety of acceptable learning rates. Furthermore, across varying ranks, cLA and $c^3$LA often underperform compared to other LoRA variants. As rank increases, this gap tends to narrow. This is a byproduct of their structure, limiting how much of the pretrained weights they can update at any one time.

For our ablation study on scaling factor shown in Tables~\ref{tab:deberta-scaling-factor} and~\ref{tab:vit-scaling-factor}, the use of $\alpha=2r$ works as a baseline given how often it was the best choice. With Asymmetric methods, the ideal scaling factor tends to be larger; this follows from the number of trainable parameters decreasing, requiring a larger effective learning rate, as the scaling factor can be interpreted as a scale on the learning rate.

Our ablation study on chain reset frequency, shown in Tables~\ref{tab:chain_reset}, revealed no clear correlation between the frequency of chain resets.

\subsubsection{DeepseekCoder Performance Analysis}\label{subsec:deepseekcoder_perf}

Figure~\ref{fig:cla_deepseek} displays the performance of DeepSeekCoder fine-tuned on the DJANGO dataset using Asymmetric LoRA, CoLA, and our sparsity-induced LoRA variants using Exact Match. We vary rank over $r\in\{8,16,32,64,128\}$, otherwise we use the same hyperparameters in Table~\ref{tab:task-hyperparams}. While cLA at $r=16$ performs substantially worse than other variants, its EM accuracy increases relatively greatly until $r=64$, outperforming its non-sparse derivative variant Asymmetric LoRA. A similar trend in EM increasing with rank is observed in c$^3$LA, but peaks at $r=32$. These trends are far weaker in r-cLA and r-c$^3$LA, suggesting that this is a consequence of applying contiguous column updates to each layer rather than a random permutation.

As the rank increases, the largest difference between any sparsity-induced LoRA variant and Asymmetric LoRA or CoLA decreases. Combined with Table~\ref{tab:extendedMo-accuracy-table}, these results indicate that restricting adaptation to a structured $r$-column subspace (as in cLA) can still be effective for code generation, but the loss of column space expressivity is most noticeable at low ranks. In our DJANGO setup, larger ranks substantially improve performance and can even surpass alternative PEFT scores, suggesting that restricted column-space fine-tuning remains a viable strategy for code tasks, provided that enough columns are made available for adaptation. These trends suggest that code generation may tolerate less aggressive restriction than other tasks.

\begin{table*}
\centering
\captionsetup[subtable]{labelformat=simple, labelsep=colon}
\renewcommand\thesubtable{\alph{subtable}}
\caption{Test accuracies obtained by fine-tuning DeBERTa v3 on MRPC, TREC-50, and PAWS varying learning rates (columns), ranks (rows), and LoRA PEFT methods. We center our search at $1e^{-4}$. The learning rate for all methods decreases with increasing rank; the relationship between learning rate and model size observed in LLMs~\cite{scalelaw} persists when fine-tuning via LoRA methods. Chain methods and their non-chain counterparts produce the best results in similar learning rate ranges; therefore, chain resets do not influence the optimal learning rate. We repeated the experiment with ViT-Tiny on Table~\ref{tab:vittinyablation}.}
\label{tab:deberta3ablation}
\footnotesize
\setlength{\tabcolsep}{3pt}
\renewcommand{\arraystretch}{0.9}

\begin{tabular}{@{}ccc@{}}

% ===================== ROW 1: LoRA (filled) =====================
\begin{subtable}[t]{0.32\linewidth}
\centering
\caption*{DeBERTa v3 LoRA MRPC}
{\setlength{\tabcolsep}{2pt}\renewcommand{\arraystretch}{0.9}
\resizebox{\linewidth}{!}{%
\begin{tabular}{@{}lccccccccc@{}}
\toprule
Rank/LR& 1e-6 & 1e-5.5 & 1e-5 & 1e-4.5 & 1e-4 & 1e-3.5 & 1e-3 & 1e-2.5 & 1e-2 \\
\midrule
2   & 66.4 & 66.4 & 79.9 & 84.2 & \textcolor{blue}{85.5} & \textcolor{red}{87.3} & \textcolor{codegreen}{88.1} & 66.4 & 66.4 \\
4   & 66.4 & 74.4 & 81.8 & 84.6 & \textcolor{blue}{85.6} & \textcolor{codegreen}{87.9} & \textcolor{red}{87.7} & 66.4 & 66.4 \\
8   & 66.4 & 76.1 & 83.0 & 85.1 & \textcolor{blue}{86.9} & \textcolor{codegreen}{87.3} & \textcolor{red}{87.1} & 66.4 & 66.4 \\
16  & 66.4 & 78.1 & 83.1 & \textcolor{blue}{85.0} & \textcolor{red}{86.8} & \textcolor{codegreen}{87.9} & 72.9 & 66.4 & 66.4 \\
32  & 66.4 & 80.1 & 83.9 & \textcolor{blue}{84.9} & \textcolor{red}{87.1} & \textcolor{codegreen}{87.9} & 66.4 & 66.4 & 66.4 \\
64  & 76.1 & 81.7 & \textcolor{blue}{84.6} & \textcolor{red}{85.1} & \textcolor{codegreen}{88.0} & 81.5 & 66.4 & 66.4 & 66.4 \\
128 & 77.5 & 81.9 & 84.9 & \textcolor{blue}{86.2} & \textcolor{codegreen}{88.5} & \textcolor{red}{87.6} & 66.4 & 66.4 & 66.4 \\
\bottomrule
\end{tabular}}}
\end{subtable}
&
\begin{subtable}[t]{0.32\linewidth}
\centering
\caption*{DeBERTa v3 CoLA MRPC}
{\setlength{\tabcolsep}{2pt}\renewcommand{\arraystretch}{0.9}
\resizebox{\linewidth}{!}{%
\begin{tabular}{@{}lccccccccc@{}}
\toprule
Rank/LR& 1e-6 & 1e-5.5 & 1e-5 & 1e-4.5 & 1e-4 & 1e-3.5 & 1e-3 & 1e-2.5 & 1e-2 \\
\midrule
2   & 66.5 & 66.5 & 75.9 & 82.2 & \textcolor{blue}{85.7} & \textcolor{red}{87.7} & \textcolor{codegreen}{88.3} & 66.5 & 66.5 \\
4   & 66.5 & 66.5 & 76.7 & 82.7 & \textcolor{blue}{85.1} & \textcolor{codegreen}{87.8} & \textcolor{red}{87.2} & 66.5 & 66.5 \\
8   & 66.5 & 66.5 & 79.1 & 84.0 & \textcolor{blue}{86.8} & \textcolor{codegreen}{88.3} & \textcolor{red}{86.9} & 66.5 & 66.5 \\
16  & 66.5 & 75.3 & 80.8 & \textcolor{blue}{84.0} & \textcolor{red}{86.4} & \textcolor{codegreen}{88.9} & 79.6 & 66.5 & 66.5 \\
32  & 66.5 & 76.8 & 82.3 & \textcolor{blue}{83.9} & \textcolor{red}{86.5} & \textcolor{codegreen}{88.0} & 66.5 & 66.5 & 66.5 \\
64  & 69.3 & 79.6 & 83.3 & \textcolor{blue}{85.7} & \textcolor{red}{87.1} & \textcolor{codegreen}{88.5} & 66.5 & 66.5 & 66.5 \\
128 & 75.8 & 80.7 & 83.1 & \textcolor{blue}{85.9} & \textcolor{codegreen}{88.2} & \textcolor{red}{88.0} & 66.5 & 66.5 & 66.5 \\
\bottomrule
\end{tabular}}}
\end{subtable}
&
\begin{subtable}[t]{0.32\linewidth}
\centering
\caption*{DeBERTa v3 Asym MRPC}
\resizebox{\linewidth}{!}{%
\begin{tabular}{@{}lccccccccc@{}}
\toprule
Rank/LR& 1e-6 & 1e-5.5 & 1e-5 & 1e-4.5 & 1e-4 & 1e-3.5 & 1e-3 & 1e-2.5 & 1e-2 \\
\midrule
2   & 66.4 & 66.4 & 66.4 & 68.1 & 81.4 & 85.3 & \textcolor{red}{86.1} & \textcolor{codegreen}{86.9} & \textcolor{blue}{86.0} \\
4   & 66.4 & 66.4 & 66.4 & 76.0 & 82.9 & 84.2 & \textcolor{red}{85.7} & \textcolor{codegreen}{86.8} & \textcolor{blue}{85.6} \\
8   & 66.4 & 66.4 & 66.5 & 80.6 & 84.2 & \textcolor{blue}{85.2} & \textcolor{codegreen}{86.5} & \textcolor{red}{86.3} & 72.3 \\
16  & 66.4 & 66.4 & 76.7 & 82.6 & 84.3 & \textcolor{blue}{86.2} & \textcolor{red}{86.8} & \textcolor{codegreen}{87.3} & 66.4 \\
32  & 66.4 & 66.6 & 79.6 & 83.1 & \textcolor{blue}{84.6} & \textcolor{red}{86.1} & \textcolor{codegreen}{87.3} & 75.0 & 66.4 \\
64  & 66.4 & 77.7 & 81.7 & 82.5 & \textcolor{blue}{84.7} & \textcolor{red}{86.1} & \textcolor{codegreen}{87.4} & 66.4 & 66.4 \\
128 & 69.0 & 79.6 & \textcolor{blue}{82.2} & \textcolor{red}{84.3} & \textcolor{codegreen}{86.1} & 80.1 & 66.4 & 66.4 & 66.4 \\

\bottomrule
\end{tabular}}
\end{subtable}
\\

% ===================== ROW 2: CoLA (placeholders) =====================
\begin{subtable}[t]{0.32\linewidth}
\centering
\caption*{DeBERTa v3 RAC MRPC}
\resizebox{\linewidth}{!}{%
\begin{tabular}{@{}lccccccccc@{}}
\toprule
Rank/LR& 1e-6 & 1e-5.5 & 1e-5 & 1e-4.5 & 1e-4 & 1e-3.5 & 1e-3 & 1e-2.5 & 1e-2 \\
\midrule
2   & 66.5 & 66.5 & 66.5 & 66.5 & 66.5 & 76.5 & \textcolor{blue}{82.0} & \textcolor{codegreen}{87.0} & \textcolor{red}{86.2} \\
4   & 66.5 & 66.5 & 66.5 & 66.5 & 74.0 & 79.4 & \textcolor{blue}{84.6} & \textcolor{red}{85.4} & \textcolor{codegreen}{85.5} \\
8   & 66.5 & 66.5 & 66.5 & 66.5 & 76.9 & \textcolor{blue}{82.8} & \textcolor{red}{86.3} & \textcolor{codegreen}{87.5} & 76.7 \\
16  & 66.5 & 66.5 & 66.5 & 74.1 & 78.8 & \textcolor{blue}{83.0} & \textcolor{red}{86.9} & \textcolor{codegreen}{87.2} & 66.5 \\
32  & 66.5 & 66.5 & 66.5 & 76.7 & \textcolor{blue}{83.9} & \textcolor{red}{86.3} & \textcolor{codegreen}{87.7} & 72.4 & 66.5 \\
64  & 66.5 & 66.5 & 74.2 & 80.4 & \textcolor{blue}{84.1} & \textcolor{codegreen}{87.9} & \textcolor{red}{87.3} & 66.5 & 66.5 \\
128 & 66.5 & 66.5 & 76.1 & \textcolor{blue}{82.3} & \textcolor{red}{86.4} & \textcolor{codegreen}{87.6} & 72.0 & 66.5 & 66.5 \\
\bottomrule
\end{tabular}}
\end{subtable}
&
\begin{subtable}[t]{0.32\linewidth}
\centering
\caption*{DeBERTa v3 cLA MRPC}
\resizebox{\linewidth}{!}{%
\begin{tabular}{@{}lccccccccc@{}}
\toprule
Rank/LR& 1e-6 & 1e-5.5 & 1e-5 & 1e-4.5 & 1e-4 & 1e-3.5 & 1e-3 & 1e-2.5 & 1e-2 \\
\midrule
2   & 66.4 & 66.4 & 66.4 & 71.0 & 82.4 & \textcolor{red}{84.0} & \textcolor{codegreen}{85.0} & \textcolor{blue}{83.1} & 80.4 \\
4   & 66.4 & 66.4 & 70.2 & 79.0 & \textcolor{blue}{84.9} & \textcolor{red}{85.5} & \textcolor{codegreen}{85.9} & 85.5 & 66.4 \\
8   & 66.4 & 66.4 & 73.7 & 81.5 & \textcolor{blue}{84.6} & 84.5 & \textcolor{red}{85.2} & \textcolor{codegreen}{85.5} & 66.4 \\
16  & 66.4 & 67.0 & 76.1 & 81.9 & \textcolor{blue}{83.3} & \textcolor{red}{84.4} & \textcolor{codegreen}{85.3} & 78.8 & 66.4 \\
32  & 66.4 & 74.4 & 79.6 & 80.8 & \textcolor{blue}{83.2} & \textcolor{red}{85.0} & \textcolor{codegreen}{85.2} & 66.4 & 66.4 \\
64  & 66.4 & 77.0 & 80.8 & 81.8 & \textcolor{blue}{84.5} & \textcolor{red}{85.1} & \textcolor{codegreen}{86.7} & 66.4 & 66.4 \\
128 & 71.3 & 78.6 & 79.5 & 82.0 & \textcolor{red}{85.5} & \textcolor{codegreen}{87.1} & \textcolor{blue}{84.8} & 66.4 & 66.4 \\

\bottomrule
\end{tabular}}
\end{subtable}
&
\begin{subtable}[t]{0.32\linewidth}
\centering
\caption*{DeBERTa v3 $c^3$LA MRPC}
\resizebox{\linewidth}{!}{%
\begin{tabular}{@{}lccccccccc@{}}
\toprule
Rank/LR& 1e-6 & 1e-5.5 & 1e-5 & 1e-4.5 & 1e-4 & 1e-3.5 & 1e-3 & 1e-2.5 & 1e-2 \\
\midrule
2   & 66.5 & 66.5 & 66.5 & 66.5 & 73.0 & \textcolor{blue}{80.1} & \textcolor{red}{85.6} & \textcolor{codegreen}{86.4} & 72.3 \\
4   & 66.5 & 66.5 & 66.5 & 66.8 & 74.6 & \textcolor{blue}{83.1} & \textcolor{red}{85.7} & \textcolor{codegreen}{87.2} & 66.5 \\
8   & 66.5 & 66.5 & 66.5 & 72.6 & 79.5 & \textcolor{red}{86.2} & \textcolor{codegreen}{86.8} & \textcolor{blue}{85.5} & 66.5 \\
16  & 66.5 & 66.5 & 66.5 & 75.6 & \textcolor{blue}{82.8} & \textcolor{red}{86.0} & \textcolor{codegreen}{86.5} & 78.6 & 66.5 \\
32  & 66.5 & 66.5 & 71.9 & \textcolor{blue}{79.1} & \textcolor{red}{84.5} & \textcolor{codegreen}{86.8} & 86.8 & 66.5 & 66.5 \\
64  & 66.5 & 66.5 & 73.9 & 82.4 & \textcolor{blue}{86.3} & \textcolor{codegreen}{87.2} & \textcolor{red}{86.4} & 66.5 & 66.5 \\
128 & 66.5 & 66.5 & 81.9 & \textcolor{blue}{86.3} & \textcolor{codegreen}{87.9} & \textcolor{red}{86.7} & 72.3 & 66.5 & 66.5 \\
\bottomrule
\end{tabular}}
\end{subtable}
\\

\begin{subtable}[t]{0.32\linewidth}
\centering
\caption*{DeBERTa v3 LoRA TREC-50}
\resizebox{\linewidth}{!}{%
\begin{tabular}{@{}lccccccccc@{}}
\toprule
Rank/LR& 1e-6 & 1e-5.5 & 1e-5 & 1e-4.5 & 1e-4 & 1e-3.5 & 1e-3 & 1e-2.5 & 1e-2 \\
\midrule
2 & 3.2 & 10.9 & 10.9 & 39.1 & \textcolor{blue}{59.5} & \textcolor{red}{76.6} & \textcolor{codegreen}{86.9} & 10.9 & 10.9 \\
4 & 10.1 & 10.9 & 10.9 & 42.3 & \textcolor{blue}{70.6} & \textcolor{red}{82.3} & \textcolor{codegreen}{87.7} & 10.9 & 10.9 \\
8 & 10.9 & 10.9 & 10.9 & 50.0 & \textcolor{blue}{70.6} & \textcolor{red}{84.7} & \textcolor{codegreen}{90.1} & 10.9 & 10.9 \\
16 & 1.4 & 10.9 & 10.9 & 50.0 & \textcolor{blue}{73.0} & \textcolor{codegreen}{89.3} & \textcolor{red}{88.3} & 10.9 & 10.9 \\
32 & 1.4 & 10.9 & 42.9 & \textcolor{blue}{59.5} & \textcolor{red}{76.4} & \textcolor{codegreen}{89.1} & 10.9 & 10.9 & 10.9 \\
64 & 10.9 & 10.9 & 48.2 & \textcolor{blue}{66.1} & \textcolor{red}{82.9} & \textcolor{codegreen}{87.1} & 10.9 & 10.9 & 10.9 \\
128 & 10.9 & 10.9 & \textcolor{blue}{58.1} & \textcolor{red}{71.6} & \textcolor{codegreen}{86.1} & 10.9 & 10.9 & 10.9 & 10.9 \\

\bottomrule
\end{tabular}}
\end{subtable}
&
\begin{subtable}[t]{0.32\linewidth}
\centering
\caption*{DeBERTa v3 CoLA TREC-50}
{\setlength{\tabcolsep}{2pt}\renewcommand{\arraystretch}{0.9}
\resizebox{\linewidth}{!}{%
\begin{tabular}{@{}lccccccccc@{}}
\toprule
Rank/LR& 1e-6 & 1e-5.5 & 1e-5 & 1e-4.5 & 1e-4 & 1e-3.5 & 1e-3 & 1e-2.5 & 1e-2 \\
\midrule
2 & 10.9 & 10.9 & 42.1 & 54.4 & \textcolor{blue}{71.8} & \textcolor{red}{88.1} & \textcolor{codegreen}{89.1} & 10.9 & 10.9 \\
4 & 10.9 & 10.9 & 42.7 & 58.3 & \textcolor{blue}{81.7} & \textcolor{red}{84.5} & \textcolor{codegreen}{88.3} & 10.9 & 10.9 \\
8 & 10.9 & 10.9 & 42.9 & 65.5 & \textcolor{blue}{82.3} & \textcolor{red}{87.1} & \textcolor{codegreen}{90.9} & 10.9 & 10.9 \\
16 & 10.9 & 10.9 & 39.9 & 66.7 & \textcolor{red}{84.9} & \textcolor{codegreen}{87.1} & \textcolor{blue}{68.1} & 10.9 & 10.9 \\
32 & 10.9 & 26.8 & 53.2 & \textcolor{blue}{71.4} & \textcolor{red}{85.3} & \textcolor{codegreen}{86.7} & 10.9 & 10.9 & 10.9 \\
64 & 10.9 & 37.5 & \textcolor{blue}{58.5} & \textcolor{red}{75.6} & \textcolor{codegreen}{86.9} & 10.9 & 10.9 & 10.9 & 10.9 \\
128 & 10.9 & 43.5 & \textcolor{blue}{66.1} & \textcolor{red}{82.7} & \textcolor{codegreen}{86.7} & 10.9 & 10.9 & 10.9 & 10.9 \\
\bottomrule
\end{tabular}}}
\end{subtable}
&
\begin{subtable}[t]{0.32\linewidth}
\centering
\caption*{DeBERTa v3 Asym TREC-50}
\resizebox{\linewidth}{!}{%
\begin{tabular}{@{}lccccccccc@{}}
\toprule
 Rank/LR& 1e-6 & 1e-5.5 & 1e-5 & 1e-4.5 & 1e-4 & 1e-3.5 & 1e-3 & 1e-2.5 & 1e-2 \\
\midrule
2 & 10.9 & 10.9 & 10.9 & 32.5 & 46.2 & \textcolor{blue}{80.6} & \textcolor{codegreen}{86.5} & \textcolor{red}{82.5} & 26.6 \\
4 & 10.9 & 10.9 & 10.9 & 33.3 & 58.9 & \textcolor{blue}{85.1} & \textcolor{codegreen}{88.1} & \textcolor{red}{85.7} & 10.9 \\
8 & 10.9 & 10.9 & 10.9 & 40.1 & 73.6 & \textcolor{codegreen}{87.3} & \textcolor{red}{86.7} & \textcolor{blue}{84.5} & 10.9 \\
16 & 10.9 & 10.9 & 10.9 & 42.9 & 78.8 & \textcolor{codegreen}{89.1} & \textcolor{red}{86.9} & \textcolor{blue}{82.7} & 10.9 \\
32 & 10.9 & 10.9 & 10.9 & 57.5 & \textcolor{blue}{83.1} & \textcolor{red}{90.9} & \textcolor{codegreen}{91.7} & 56.5 & 10.9 \\
64 & 10.9 & 10.9 & 41.9 & 73.0 & \textcolor{red}{88.5} & \textcolor{codegreen}{90.7} & \textcolor{blue}{87.5} & 10.9 & 10.9 \\
128 & 10.9 & 10.9 & 52.2 & \textcolor{blue}{78.8} & \textcolor{red}{89.9} & \textcolor{codegreen}{90.9} & 10.9 & 10.9 & 10.9 \\

\bottomrule
\end{tabular}}
\end{subtable}
\\

% ===================== ROW 4: RAC (placeholders) =====================
\begin{subtable}[t]{0.32\linewidth}
\centering
\caption*{DeBERTa v3 RAC TREC-50}
\resizebox{\linewidth}{!}{%
\begin{tabular}{@{}lccccccccc@{}}
\toprule
 Rank/LR& 1e-6 & 1e-5.5 & 1e-5 & 1e-4.5 & 1e-4 & 1e-3.5 & 1e-3 & 1e-2.5 & 1e-2 \\
\midrule
2 & 10.9 & 10.9 & 10.9 & 10.9 & 34.9 & 46.6 & \textcolor{blue}{72.8} & \textcolor{red}{85.9} & \textcolor{codegreen}{87.5} \\
4 & 1.2 & 10.9 & 10.9 & 10.9 & 38.5 & 59.7 & \textcolor{blue}{82.9} & \textcolor{red}{88.3} & \textcolor{codegreen}{89.5} \\
8 & 10.9 & 10.9 & 10.9 & 10.9 & 43.8 & 68.1 & \textcolor{red}{82.9} & \textcolor{codegreen}{88.1} & \textcolor{blue}{72.4} \\
16 & 10.1 & 10.9 & 10.9 & 13.3 & 59.5 & \textcolor{blue}{78.0} & \textcolor{red}{87.1} & \textcolor{codegreen}{90.3} & 10.9 \\
32 & 10.9 & 10.9 & 10.9 & 45.6 & 70.6 & \textcolor{blue}{84.9} & \textcolor{codegreen}{88.5} & \textcolor{red}{88.1} & 10.9 \\
64 & 10.9 & 10.9 & 34.9 & 50.6 & \textcolor{blue}{74.2} & \textcolor{red}{86.7} & \textcolor{codegreen}{87.7} & 10.9 & 10.9 \\
128 & 2.0 & 10.9 & 42.9 & 62.3 & \textcolor{blue}{84.5} & \textcolor{codegreen}{88.9} & \textcolor{red}{87.3} & 10.9 & 10.9 \\

\bottomrule
\end{tabular}}
\end{subtable}
&
\begin{subtable}[t]{0.32\linewidth}
\centering
\caption*{DeBERTa v3 cLA TREC-50}
\resizebox{\linewidth}{!}{%
\begin{tabular}{@{}lccccccccc@{}}
\toprule
 Rank/LR& 1e-6 & 1e-5.5 & 1e-5 & 1e-4.5 & 1e-4 & 1e-3.5 & 1e-3 & 1e-2.5 & 1e-2 \\
\midrule
2 & 10.9 & 10.9 & 10.9 & 10.9 & 10.9 & \textcolor{blue}{40.9} & \textcolor{red}{71.2} & \textcolor{codegreen}{81.3} & 34.5 \\
4 & 9.5 & 10.9 & 10.9 & 10.9 & 35.3 & \textcolor{blue}{61.9} & \textcolor{red}{79.6} & \textcolor{codegreen}{82.7} & 10.9 \\
8 & 0.4 & 10.9 & 10.9 & 10.9 & 53.0 & \textcolor{blue}{71.8} & \textcolor{red}{83.1} & \textcolor{codegreen}{86.3} & 10.9 \\
16 & 10.9 & 10.9 & 10.9 & 40.7 & 60.7 & \textcolor{blue}{84.3} & \textcolor{red}{85.5} & \textcolor{codegreen}{87.5} & 10.9 \\
32 & 10.1 & 10.1 & 10.9 & 47.0 & \textcolor{blue}{70.0} & \textcolor{red}{86.1} & \textcolor{codegreen}{88.9} & 10.9 & 10.9 \\
64 & 10.9 & 10.9 & 42.7 & 62.3 & \textcolor{blue}{76.2} & \textcolor{red}{86.9} & \textcolor{codegreen}{89.1} & 52.4 & 10.9 \\
128 & 3.6 & 10.9 & 50.2 & \textcolor{blue}{67.1} & \textcolor{red}{85.1} & \textcolor{codegreen}{86.7} & 66.7 & 10.9 & 10.9 \\

\bottomrule
\end{tabular}}
\end{subtable}
&
\begin{subtable}[t]{0.32\linewidth}
\centering
\caption*{DeBERTa v3 $c^3$LA TREC-50}
\resizebox{\linewidth}{!}{%
\begin{tabular}{@{}lccccccccc@{}}
\toprule
Rank/LR& 1e-6 & 1e-5.5 & 1e-5 & 1e-4.5 & 1e-4 & 1e-3.5 & 1e-3 & 1e-2.5 & 1e-2 \\
\midrule
2 & 10.9 & 10.9 & 10.9 & 34.7 & 42.3 & \textcolor{codegreen}{79.2} & \textcolor{red}{66.7} & \textcolor{blue}{62.1} & 10.9 \\
4 & 10.9 & 10.9 & 19.4 & 34.7 & 56.9 & \textcolor{red}{86.5} & \textcolor{codegreen}{87.7} & \textcolor{blue}{73.8} & 10.9 \\
8 & 10.9 & 10.9 & 20.6 & 36.9 & \textcolor{blue}{68.7} & \textcolor{codegreen}{87.3} & \textcolor{red}{78.8} & 66.7 & 10.9 \\
16 & 10.9 & 10.9 & 10.9 & 43.8 & \textcolor{blue}{76.8} & \textcolor{codegreen}{88.5} & \textcolor{red}{83.9} & 71.0 & 10.9 \\
32 & 10.9 & 10.9 & 38.7 & 58.1 & \textcolor{red}{84.3} & \textcolor{codegreen}{88.1} & \textcolor{blue}{80.4} & 34.3 & 10.9 \\
64 & 10.9 & 10.9 & 45.8 & 74.2 & \textcolor{red}{85.9} & \textcolor{codegreen}{89.1} & \textcolor{blue}{80.4} & 10.9 & 10.9 \\
128 & 10.9 & 35.1 & 56.5 & \textcolor{blue}{79.2} & \textcolor{red}{85.9} & \textcolor{codegreen}{90.9} & 10.9 & 10.9 & 10.9 \\
\bottomrule
\end{tabular}}
\end{subtable}
\\

% ===================== ROW 5: cLA (placeholders) =====================
\begin{subtable}[t]{0.32\linewidth}
\centering
\caption*{DeBERTa v3 LoRA PAWS}
\resizebox{\linewidth}{!}{%
\begin{tabular}{@{}lccccccccc@{}}
\toprule
 Rank/LR& 1e-6 & 1e-5.5 & 1e-5 & 1e-4.5 & 1e-4 & 1e-3.5 & 1e-3 & 1e-2.5 & 1e-2 \\
\midrule
2 & 92.1 & 93.7 & 94.0 & \textcolor{blue}{94.2} & \textcolor{codegreen}{94.7} & \textcolor{red}{94.5} & 94.0 & 55.8 & 55.8 \\
4 & 92.3 & 94.1 & 94.0 & \textcolor{red}{94.4} & \textcolor{blue}{94.2} & \textcolor{codegreen}{94.8} & 94.0 & 55.8 & 55.8 \\
8 & 92.4 & 93.5 & \textcolor{blue}{94.3} & \textcolor{blue}{94.3} & \textcolor{codegreen}{94.7} & \textcolor{red}{94.5} & 93.5 & 55.8 & 55.8 \\
16 & 93.1 & \textcolor{blue}{94.0} & \textcolor{red}{94.4} & \textcolor{codegreen}{94.6} & \textcolor{codegreen}{94.6} & 93.6 & 55.8 & 55.8 & 55.8 \\
32 & 93.8 & \textcolor{blue}{93.9} & \textcolor{codegreen}{94.7} & \textcolor{red}{94.5} & \textcolor{codegreen}{94.7} & 55.8 & 55.8 & 55.8 & 55.8 \\
64 & \textcolor{blue}{93.6} & \textcolor{red}{94.1} & \textcolor{codegreen}{94.6} & \textcolor{codegreen}{94.6} & \textcolor{codegreen}{94.6} & 93.5 & 55.8 & 55.8 & 50.0 \\
128 & 94.0 & \textcolor{blue}{94.2} & \textcolor{red}{94.4} & \textcolor{codegreen}{94.7} & \textcolor{codegreen}{94.7} & 55.8 & 55.8 & 55.8 & 50.0 \\
\bottomrule
\end{tabular}}
\end{subtable}
&
\begin{subtable}[t]{0.32\linewidth}
\centering
\caption*{DeBERTa v3 CoLA PAWS}
{\setlength{\tabcolsep}{2pt}\renewcommand{\arraystretch}{0.9}
\resizebox{\linewidth}{!}{%
\begin{tabular}{@{}lccccccccc@{}}
\toprule
Rank/LR & 1e-6 & 1e-5.5 & 1e-5 & 1e-4.5 & 1e-4 & 1e-3.5 & 1e-3 & 1e-2.5 & 1e-2 \\
\midrule
2 & 55.8 & 89.8 & 92.5 & 93.9 & \textcolor{red}{94.6} & \textcolor{codegreen}{94.7} & \textcolor{blue}{94.2} & 55.8 & 55.8 \\
4 & 55.8 & 90.6 & \textcolor{blue}{92.9} & \textcolor{red}{94.0} & \textcolor{codegreen}{94.8} & \textcolor{codegreen}{94.8} & \textcolor{red}{94.0} & 55.8 & 55.8 \\
8 & 55.8 & 89.3 & \textcolor{blue}{93.2} & \textcolor{red}{94.1} & \textcolor{codegreen}{94.3} & \textcolor{red}{94.1} & 93.0 & 55.8 & 55.8 \\
16 & 55.8 & 91.8 & 93.7 & \textcolor{blue}{94.3} & \textcolor{red}{94.5} & \textcolor{codegreen}{94.7} & 55.8 & 55.8 & 55.8 \\
32 & 89.9 & 92.9 & \textcolor{blue}{94.3} & \textcolor{red}{94.5} & \textcolor{blue}{94.3} & \textcolor{codegreen}{94.8} & 55.8 & 55.8 & 55.8 \\
64 & 90.5 & \textcolor{blue}{93.3} & \textcolor{red}{94.5} & \textcolor{red}{94.5} & \textcolor{codegreen}{94.8} & 93.2 & 55.8 & 55.8 & 55.8 \\
128 & 91.8 & 93.8 & \textcolor{blue}{94.7} & \textcolor{codegreen}{95.1} & \textcolor{red}{95.0} & 92.7 & 55.8 & 55.8 & 44.2 \\
\bottomrule
\end{tabular}}}
\end{subtable}
&
\begin{subtable}[t]{0.32\linewidth}
\centering
\caption*{DeBERTa v3 Asym PAWS}
\resizebox{\linewidth}{!}{%
\begin{tabular}{@{}lccccccccc@{}}
\toprule
 Rank/LR& 1e-6 & 1e-5.5 & 1e-5 & 1e-4.5 & 1e-4 & 1e-3.5 & 1e-3 & 1e-2.5 & 1e-2 \\
\midrule
2 & 55.8 & 55.8 & 86.9 & 92.3 & \textcolor{blue}{93.7} & \textcolor{red}{93.9} & \textcolor{red}{93.9} & \textcolor{codegreen}{94.4} & 93.4 \\
4 & 55.8 & 55.8 & 90.3 & 92.5 & 93.2 & \textcolor{blue}{94.0} & \textcolor{codegreen}{94.4} & \textcolor{red}{94.2} & 92.7 \\
8 & 55.8 & 55.8 & 92.3 & 93.1 & \textcolor{blue}{94.1} & \textcolor{codegreen}{94.7} & \textcolor{codegreen}{94.7} & \textcolor{red}{94.2} & 55.8 \\
16 & 55.8 & 55.8 & 92.4 & \textcolor{blue}{94.1} & 94.0 & \textcolor{red}{94.5} & \textcolor{codegreen}{94.6} & 94.0 & 55.8 \\
32 & 55.8 & 92.9 & 93.6 & \textcolor{red}{94.4} & \textcolor{red}{94.4} & \textcolor{codegreen}{94.6} & \textcolor{blue}{94.1} & 92.9 & 55.8 \\
64 & 55.8 & 92.5 & \textcolor{red}{94.1} & \textcolor{red}{94.1} & \textcolor{codegreen}{94.8} & \textcolor{codegreen}{94.8} & \textcolor{blue}{93.4} & 55.8 & 55.8 \\
128 & 91.7 & 92.9 & 94.1 & \textcolor{blue}{94.4} & \textcolor{red}{94.6} & \textcolor{codegreen}{94.7} & 91.8 & 55.0 & 44.2 \\
\bottomrule
\end{tabular}}
\end{subtable}
\\

% ===================== ROW 6: c^3LA (placeholders) =====================
\begin{subtable}[t]{0.32\linewidth}
\centering
\caption*{DeBERTa v3 RAC PAWS}
\resizebox{\linewidth}{!}{%
\begin{tabular}{@{}lccccccccc@{}}
\toprule
 Rank/LR& 1e-6 & 1e-5.5 & 1e-5 & 1e-4.5 & 1e-4 & 1e-3.5 & 1e-3 & 1e-2.5 & 1e-2 \\
\midrule
2 & 55.8 & 55.8 & 55.8 & 89.0 & 93.3 & \textcolor{red}{94.1} & \textcolor{blue}{93.8} & \textcolor{codegreen}{94.4} & 93.4 \\
4 & 55.8 & 55.8 & 91.0 & 93.0 & 93.5 & \textcolor{red}{93.9} & \textcolor{codegreen}{94.4} & \textcolor{blue}{93.8} & 90.6 \\
8 & 55.8 & 55.8 & 89.4 & 93.4 & 93.8 & \textcolor{codegreen}{94.5} & \textcolor{blue}{94.2} & \textcolor{red}{94.3} & 88.9 \\
16 & 55.8 & 55.8 & 92.6 & 92.8 & \textcolor{blue}{94.2} & \textcolor{codegreen}{95.1} & \textcolor{red}{94.7} & 93.5 & 55.8 \\
32 & 55.8 & 91.0 & 92.6 & 93.8 & \textcolor{red}{94.2} & \textcolor{blue}{94.0} & \textcolor{codegreen}{94.7} & 55.8 & 55.8 \\
64 & 55.8 & 92.7 & 93.5 & \textcolor{blue}{94.3} & \textcolor{codegreen}{94.8} & \textcolor{red}{94.5} & 93.8 & 55.8 & 55.8 \\
128 & 91.8 & 93.1 & \textcolor{blue}{94.2} & \textcolor{red}{94.3} & \textcolor{codegreen}{94.6} & \textcolor{codegreen}{94.6} & 55.8 & 55.8 & 55.8 \\
\bottomrule
\end{tabular}}
\end{subtable}
&
\begin{subtable}[t]{0.32\linewidth}
\centering
\caption*{DeBERTa v3 cLA PAWS}
\resizebox{\linewidth}{!}{%
\begin{tabular}{@{}lccccccccc@{}}
\toprule
Rank/LR& 1e-6 & 1e-5.5 & 1e-5 & 1e-4.5 & 1e-4 & 1e-3.5 & 1e-3 & 1e-2.5 & 1e-2 \\
\midrule
2 & 55.8 & 55.8 & 55.8 & 90.5 & 92.5 & \textcolor{codegreen}{94.0} & \textcolor{red}{93.8} & \textcolor{blue}{93.7} & 55.8 \\
4 & 55.8 & 55.8 & 89.3 & 90.9 & 93.1 & \textcolor{red}{93.8} & \textcolor{codegreen}{94.0} & \textcolor{blue}{93.7} & 55.8 \\
8 & 55.8 & 55.8 & 89.7 & 92.6 & 92.9 & \textcolor{codegreen}{94.3} & \textcolor{red}{93.9} & \textcolor{blue}{93.7} & 55.8 \\
16 & 55.8 & 55.8 & 91.5 & 93.0 & \textcolor{blue}{93.7} & \textcolor{codegreen}{94.5} & \textcolor{red}{94.1} & 55.8 & 55.8 \\
32 & 55.8 & 89.5 & 91.6 & 93.4 & \textcolor{blue}{93.8} & \textcolor{codegreen}{94.2} & \textcolor{red}{94.1} & 55.8 & 55.8 \\
64 & 55.8 & 90.0 & 93.2 & \textcolor{blue}{93.8} & \textcolor{red}{93.9} & \textcolor{codegreen}{94.3} & 93.5 & 55.8 & 55.8 \\
128 & 87.2 & 92.5 & 93.6 & \textcolor{blue}{93.9} & \textcolor{codegreen}{94.7} & \textcolor{red}{94.2} & 55.8 & 55.8 & 55.8 \\
\bottomrule
\end{tabular}}
\end{subtable}
&
\begin{subtable}[t]{0.32\linewidth}
\centering
\caption*{DeBERTa v3 $c^3$LA PAWS}
\resizebox{\linewidth}{!}{%
\begin{tabular}{@{}lccccccccc@{}}
\toprule
 Rank/LR& 1e-6 & 1e-5.5 & 1e-5 & 1e-4.5 & 1e-4 & 1e-3.5 & 1e-3 & 1e-2.5 & 1e-2 \\
\midrule
2 & 55.8 & 55.8 & 55.8 & 91.4 & \textcolor{blue}{93.3} & \textcolor{red}{93.8} & \textcolor{codegreen}{93.9} & \textcolor{codegreen}{93.9} & 55.8 \\
4 & 55.8 & 55.8 & 89.8 & 91.2 & \textcolor{red}{93.7} & \textcolor{codegreen}{94.0} & \textcolor{codegreen}{94.0} & \textcolor{blue}{93.3} & 55.8 \\
8 & 55.8 & 55.8 & 91.8 & \textcolor{blue}{93.7} & \textcolor{red}{93.9} & \textcolor{codegreen}{94.7} & 93.5 & 93.3 & 55.8 \\
16 & 55.8 & 55.8 & 92.6 & \textcolor{blue}{93.5} & \textcolor{codegreen}{94.1} & \textcolor{codegreen}{94.1} & \textcolor{red}{93.7} & 55.8 & 55.8 \\
32 & 55.8 & 92.0 & 93.2 & \textcolor{blue}{94.0} & \textcolor{red}{94.1} & \textcolor{codegreen}{94.4} & 93.9 & 55.8 & 55.8 \\
64 & 91.3 & 92.2 & 93.8 & \textcolor{red}{94.1} & \textcolor{blue}{93.9} & \textcolor{codegreen}{94.2} & 93.7 & 55.8 & 55.8 \\
128 & 92.5 & 93.3 & \textcolor{blue}{94.0} & \textcolor{red}{94.3} & \textcolor{codegreen}{94.8} & 55.8 & 55.8 & 55.8 & 55.8 \\
\bottomrule
\end{tabular}}
\end{subtable}

\end{tabular}
\end{table*}

\begin{table*}
\centering
\captionsetup[subtable]{labelformat=simple, labelsep=colon}
\renewcommand\thesubtable{\alph{subtable}}
\caption{Test accuracies obtained by fine-tuning ViT-Tiny on CIFAR-10 and OfficeHome over varying learning rates (columns), ranks (rows), and LoRA PEFT methods. We center our search at $1e^{-3}.$ Consistent with the results of~\ref{tab:deberta3ablation}, the learning rate for all methods decreases with increasing rank. Chain methods and their non-chain counterparts produce the best results in similar learning rate ranges.
}
\label{tab:vittinyablation}
\footnotesize
\setlength{\tabcolsep}{3pt}
\renewcommand{\arraystretch}{0.9}

\begin{tabular}{@{}ccc@{}}

% ===================== ROW 1: LoRA =====================
\begin{subtable}[t]{0.32\linewidth}
\centering
\caption*{ViT-Tiny LoRA CIFAR-10}
{\setlength{\tabcolsep}{2pt}\renewcommand{\arraystretch}{0.9}
\resizebox{\linewidth}{!}{%
\begin{tabular}{@{}lccccccccc@{}}
\toprule
 & 1e-5 & 1e-4.5 & 1e-4 & 1e-3.5 & 1e-3 & 1e-2.5 & 1e-2 & 1e-1.5 & 1e-1 \\
\midrule
2   & 89.08 & 90.91 & \textcolor{blue}{92.85} & \textcolor{red}{92.92} & \textcolor{codegreen}{93.56} & 91.76 & 87.13 & 49.75 & 10.70 \\
4   & 89.38 & 91.55 & 93.49 & \textcolor{red}{94.50} & \textcolor{blue}{94.18} & \textcolor{codegreen}{95.11} & 81.16 & 17.82 & 11.15 \\
8   & 89.56 & 92.04 & 94.00 & \textcolor{blue}{95.20} & \textcolor{codegreen}{95.68} & \textcolor{red}{95.65} & 77.10 & 11.85 & 10.00 \\
16   & 90.03 & 92.69 & \textcolor{blue}{94.36} & \textcolor{red}{95.69} & \textcolor{codegreen}{95.91} & 91.27 & 59.09 & 13.52 & 10.08 \\
32   & 90.85 & 93.05 & \textcolor{blue}{95.14} & \textcolor{red}{96.14} & \textcolor{codegreen}{96.26} & 87.87 & 17.79 & 18.07 & 10.67 \\
64   & 91.79 & 94.00 & \textcolor{blue}{95.30} & \textcolor{codegreen}{96.44} & \textcolor{red}{96.43} & 81.73 & 14.33 & 11.30 & 13.21 \\
128   & 92.47 & 94.71 & \textcolor{blue}{96.03} & \textcolor{codegreen}{96.50} & \textcolor{red}{96.17} & 64.03 & 11.16 & 12.00 & 11.41 \\
\bottomrule
\end{tabular}}}
\end{subtable}
&
\begin{subtable}[t]{0.32\linewidth}
\centering
\caption*{ViT-Tiny CoLA CIFAR-10}
{\setlength{\tabcolsep}{2pt}\renewcommand{\arraystretch}{0.9}
\resizebox{\linewidth}{!}{%
\begin{tabular}{@{}lccccccccc@{}}
\toprule
 & 1e-5 & 1e-4.5 & 1e-4 & 1e-3.5 & 1e-3 & 1e-2.5 & 1e-2 & 1e-1.5 & 1e-1 \\
\midrule
2   & 87.90 & 90.43 & 92.13 & \textcolor{codegreen}{93.79} & \textcolor{red}{93.77} & \textcolor{blue}{93.57} & 87.13 & 20.52 & 12.61 \\
4   & 88.43 & 90.77 & 92.73 & \textcolor{blue}{94.30} & \textcolor{codegreen}{95.21} & \textcolor{red}{94.49} & 83.74 & 17.82 & 11.15 \\
8   & 88.96 & 91.26 & 93.37 & \textcolor{red}{94.95} & \textcolor{codegreen}{95.39} & \textcolor{blue}{94.74} & 77.09 & 11.85 & 10.00 \\
16   & 89.50 & 92.00 & \textcolor{blue}{93.86} & \textcolor{red}{95.26} & \textcolor{codegreen}{95.72} & 91.27 & 62.12 & 15.42 & 11.25 \\
32   & 90.02 & 92.63 & \textcolor{blue}{94.54} & \textcolor{red}{95.72} & \textcolor{codegreen}{95.96} & 87.87 & 19.10 & 18.07 & 12.41 \\
64   & 91.18 & 93.51 & \textcolor{blue}{95.13} & \textcolor{codegreen}{96.07} & \textcolor{red}{96.01} & 81.73 & 14.33 & 17.70 & 13.23 \\
128   & 92.06 & 94.20 & \textcolor{red}{95.56} & \textcolor{codegreen}{96.17} & \textcolor{blue}{95.54} & 65.78 & 17.45 & 10.30 & 11.03 \\
\bottomrule
\end{tabular}}}
\end{subtable}
&
\begin{subtable}[t]{0.32\linewidth}
\centering
\caption*{ViT-Tiny Asym CIFAR-10}
{\setlength{\tabcolsep}{2pt}\renewcommand{\arraystretch}{0.9}
\resizebox{\linewidth}{!}{%
\begin{tabular}{@{}lccccccccc@{}}
\toprule
 & 1e-5 & 1e-4.5 & 1e-4 & 1e-3.5 & 1e-3 & 1e-2.5 & 1e-2 & 1e-1.5 & 1e-1 \\
\midrule
2   & 85.34 & 89.63 & 90.86 & \textcolor{blue}{91.64} & \textcolor{codegreen}{92.03} & \textcolor{red}{91.88} & 90.85 & 90.18 & 80.81 \\
4   & 86.82 & 90.29 & 91.66 & \textcolor{blue}{92.58} & \textcolor{codegreen}{93.32} & \textcolor{red}{92.73} & 91.63 & 87.88 & 78.91 \\
8   & 88.34 & 90.87 & 92.35 & 93.34 & \textcolor{blue}{93.71} & \textcolor{codegreen}{93.81} & \textcolor{red}{93.74} & 86.88 & 64.43 \\
16   & 89.23 & 91.61 & 93.18 & \textcolor{blue}{94.23} & \textcolor{red}{94.65} & \textcolor{codegreen}{95.08} & 90.39 & 82.52 & 50.44 \\
32   & 90.12 & 92.17 & 93.81 & \textcolor{red}{95.20} & \textcolor{blue}{95.15} & \textcolor{codegreen}{95.36} & 93.94 & 73.55 & 34.13 \\
64   & 91.20 & 93.11 & \textcolor{blue}{94.66} & \textcolor{red}{95.77} & \textcolor{codegreen}{96.08} & 92.99 & 92.18 & 53.07 & 24.52 \\
128   & 92.27 & 94.03 & \textcolor{red}{95.36} & \textcolor{codegreen}{96.09} & 94.56 & \textcolor{blue}{94.97} & 69.81 & 27.77 & 16.58 \\
\bottomrule
\end{tabular}}}
\end{subtable}
\\

\begin{subtable}[t]{0.32\linewidth}
\centering
\caption*{ViT-Tiny RAC CIFAR-10}
{\setlength{\tabcolsep}{2pt}\renewcommand{\arraystretch}{0.9}
\resizebox{\linewidth}{!}{%
\begin{tabular}{@{}lccccccccc@{}}
\toprule
 & 1e-5 & 1e-4.5 & 1e-4 & 1e-3.5 & 1e-3 & 1e-2.5 & 1e-2 & 1e-1.5 & 1e-1 \\
\midrule
2   & 85.61 & 89.68 & 90.98 & \textcolor{blue}{91.87} & \textcolor{codegreen}{92.72} & \textcolor{red}{91.94} & 91.47 & 89.51 & 80.23 \\
4   & 86.64 & 90.35 & 91.89 & 92.96 & \textcolor{codegreen}{93.71} & \textcolor{blue}{93.46} & \textcolor{red}{93.49} & 87.88 & 77.56 \\
8   & 88.26 & 90.87 & 92.40 & \textcolor{blue}{93.65} & \textcolor{red}{94.09} & \textcolor{codegreen}{94.33} & 90.69 & 86.88 & 64.43 \\
16   & 89.26 & 91.75 & 93.26 & \textcolor{red}{94.40} & \textcolor{codegreen}{95.31} & \textcolor{blue}{94.24} & 89.56 & 82.52 & 50.44 \\
32   & 90.27 & 92.28 & 94.05 & \textcolor{blue}{95.43} & \textcolor{red}{95.54} & \textcolor{codegreen}{95.68} & 93.14 & 73.55 & 24.92 \\
64   & 91.12 & 93.20 & \textcolor{blue}{94.86} & \textcolor{codegreen}{95.91} & 94.78 & \textcolor{red}{95.35} & 82.86 & 53.07 & 24.52 \\
128   & 92.19 & 94.03 & \textcolor{blue}{95.56} & \textcolor{codegreen}{96.11} & \textcolor{red}{96.10} & 94.21 & 69.81 & 27.77 & 16.58 \\
\bottomrule
\end{tabular}}}
\end{subtable}
&
\begin{subtable}[t]{0.32\linewidth}
\centering
\caption*{ViT-Tiny cLA CIFAR-10}
{\setlength{\tabcolsep}{2pt}\renewcommand{\arraystretch}{0.9}
\resizebox{\linewidth}{!}{%
\begin{tabular}{@{}lccccccccc@{}}
\toprule
 & 1e-5 & 1e-4.5 & 1e-4 & 1e-3.5 & 1e-3 & 1e-2.5 & 1e-2 & 1e-1.5 & 1e-1 \\
\midrule
2   & 87.30 & 89.65 & 91.13 & \textcolor{red}{92.12} & \textcolor{codegreen}{92.49} & \textcolor{blue}{91.70} & 88.92 & 82.43 & 51.06 \\
4   & 88.06 & 90.47 & 91.91 & \textcolor{blue}{92.00} & \textcolor{codegreen}{93.45} & \textcolor{red}{92.05} & 88.78 & 73.27 & 11.47 \\
8   & 89.14 & 91.38 & \textcolor{blue}{93.13} & \textcolor{codegreen}{93.36} & \textcolor{red}{93.22} & 90.72 & 85.78 & 68.40 & 13.39 \\
16   & 90.43 & 92.27 & 93.91 & \textcolor{codegreen}{94.86} & \textcolor{red}{94.53} & \textcolor{blue}{94.11} & 78.83 & 44.03 & 17.11 \\
32   & 91.32 & 93.37 & \textcolor{blue}{94.91} & \textcolor{codegreen}{95.63} & \textcolor{red}{95.33} & 89.55 & 71.95 & 30.90 & 19.23 \\
64   & 92.50 & 94.17 & \textcolor{red}{95.82} & \textcolor{codegreen}{96.26} & \textcolor{blue}{95.54} & 83.63 & 50.91 & 21.61 & 15.50 \\
128   & 93.86 & \textcolor{blue}{95.39} & \textcolor{codegreen}{96.50} & \textcolor{red}{96.45} & 94.85 & 75.52 & 28.65 & 19.95 & 31.06 \\
\bottomrule
\end{tabular}}}
\end{subtable}
&
\begin{subtable}[t]{0.32\linewidth}
\centering
\caption*{ViT-Tiny $c^3$LA CIFAR-10}
{\setlength{\tabcolsep}{2pt}\renewcommand{\arraystretch}{0.9}
\resizebox{\linewidth}{!}{%
\begin{tabular}{@{}lccccccccc@{}}
\toprule
 & 1e-5 & 1e-4.5 & 1e-4 & 1e-3.5 & 1e-3 & 1e-2.5 & 1e-2 & 1e-1.5 & 1e-1 \\
\midrule
2   & 87.29 & 89.83 & \textcolor{blue}{91.49} & \textcolor{red}{92.48} & \textcolor{codegreen}{93.04} & 91.35 & 88.65 & 80.39 & 53.28 \\
4   & 88.57 & 90.76 & \textcolor{blue}{92.50} & \textcolor{red}{93.36} & \textcolor{codegreen}{93.98} & 91.55 & 87.25 & 73.27 & 11.47 \\
8   & 89.69 & 91.84 & \textcolor{blue}{93.52} & \textcolor{red}{94.68} & \textcolor{codegreen}{94.80} & 90.72 & 85.78 & 68.40 & 27.84 \\
16   & 90.63 & 92.69 & \textcolor{blue}{94.33} & \textcolor{codegreen}{95.35} & \textcolor{red}{95.30} & 90.04 & 81.13 & 44.03 & 17.11 \\
32   & 91.39 & 93.68 & \textcolor{blue}{95.12} & \textcolor{codegreen}{95.57} & \textcolor{red}{95.30} & 89.55 & 71.95 & 21.94 & 17.60 \\
64   & 92.86 & 94.71 & \textcolor{red}{95.92} & \textcolor{codegreen}{96.41} & \textcolor{blue}{95.07} & 83.63 & 41.60 & 21.61 & 14.03 \\
128   & 93.83 & \textcolor{blue}{95.38} & \textcolor{codegreen}{96.43} & \textcolor{red}{96.22} & 94.89 & 75.52 & 27.84 & 19.95 & 14.48 \\
\bottomrule
\end{tabular}}}
\end{subtable}
\\

\begin{subtable}[t]{0.32\linewidth}
\centering
\caption*{ViT-Tiny LoRA OfficeHome}
{\setlength{\tabcolsep}{2pt}\renewcommand{\arraystretch}{0.9}
\resizebox{\linewidth}{!}{%
\begin{tabular}{@{}lccccccccc@{}}
\toprule
 & 1e-5 & 1e-4.5 & 1e-4 & 1e-3.5 & 1e-3 & 1e-2.5 & 1e-2 & 1e-1.5 & 1e-1 \\
\midrule
2   & 47.33 & 65.03 & 70.80 & 75.55 & \textcolor{codegreen}{77.81} & \textcolor{blue}{75.93} & \textcolor{red}{77.08} & 40.57 & 1.80 \\
4   & 48.57 & 64.90 & 71.36 & 75.33 & \textcolor{red}{77.85} & \textcolor{codegreen}{78.50} & \textcolor{blue}{77.55} & 26.51 & 2.01 \\
8   & 49.42 & 64.81 & 72.08 & \textcolor{blue}{75.93} & \textcolor{codegreen}{79.39} & \textcolor{red}{78.88} & 59.09 & 4.28 & 2.61 \\
16   & 50.41 & 65.33 & 72.30 & \textcolor{blue}{77.21} & \textcolor{red}{79.26} & \textcolor{codegreen}{79.69} & 49.17 & 1.97 & 4.53 \\
32   & 50.53 & 65.50 & \textcolor{blue}{73.54} & \textcolor{red}{79.09} & \textcolor{codegreen}{79.56} & 66.01 & 36.43 & 2.44 & 2.05 \\
64   & 51.86 & 66.35 & \textcolor{blue}{74.99} & \textcolor{red}{78.97} & \textcolor{codegreen}{79.86} & 61.48 & 10.77 & 1.75 & 2.09 \\
128   & 54.68 & 66.82 & \textcolor{blue}{76.70} & \textcolor{red}{79.91} & \textcolor{codegreen}{80.33} & 53.53 & 2.78 & 1.75 & 2.91 \\
\bottomrule
\end{tabular}}}
\end{subtable}
&
\begin{subtable}[t]{0.32\linewidth}
\centering
\caption*{ViT-Tiny CoLA OfficeHome}
{\setlength{\tabcolsep}{2pt}\renewcommand{\arraystretch}{0.9}
\resizebox{\linewidth}{!}{%
\begin{tabular}{@{}lccccccccc@{}}
\toprule
 & 1e-5 & 1e-4.5 & 1e-4 & 1e-3.5 & 1e-3 & 1e-2.5 & 1e-2 & 1e-1.5 & 1e-1 \\
\midrule
2   & 45.28 & 64.64 & 70.50 & \textcolor{blue}{74.48} & \textcolor{codegreen}{76.96} & \textcolor{red}{76.70} & 74.05 & 40.57 & 2.01 \\
4   & 46.43 & 64.94 & 70.63 & \textcolor{blue}{74.39} & \textcolor{codegreen}{76.87} & \textcolor{red}{76.61} & 72.60 & 26.51 & 2.01 \\
8   & 47.71 & 65.11 & 70.97 & \textcolor{blue}{75.12} & \textcolor{codegreen}{77.73} & \textcolor{red}{76.96} & 59.09 & 3.04 & 5.77 \\
16   & 49.47 & 64.51 & 71.40 & \textcolor{blue}{75.63} & \textcolor{codegreen}{78.71} & \textcolor{red}{77.68} & 49.17 & 1.97 & 2.22 \\
32   & 50.32 & 65.20 & \textcolor{blue}{72.12} & \textcolor{red}{77.38} & \textcolor{codegreen}{80.38} & 66.01 & 36.43 & 2.44 & 2.05 \\
64   & 51.05 & 65.63 & \textcolor{blue}{73.79} & \textcolor{red}{78.32} & \textcolor{codegreen}{79.35} & 61.48 & 10.77 & 1.75 & 2.69 \\
128   & 52.29 & 66.35 & \textcolor{blue}{75.29} & \textcolor{red}{79.48} & \textcolor{codegreen}{79.52} & 52.29 & 2.78 & 1.75 & 4.10 \\
\bottomrule
\end{tabular}}}
\end{subtable}
&
\begin{subtable}[t]{0.32\linewidth}
\centering
\caption*{ViT-Tiny Asym OfficeHome}
{\setlength{\tabcolsep}{2pt}\renewcommand{\arraystretch}{0.9}
\resizebox{\linewidth}{!}{%
\begin{tabular}{@{}lccccccccc@{}}
\toprule
 & 1e-5 & 1e-4.5 & 1e-4 & 1e-3.5 & 1e-3 & 1e-2.5 & 1e-2 & 1e-1.5 & 1e-1 \\
\midrule
2   & 43.91 & 63.06 & 70.20 & 73.45 & \textcolor{red}{74.65} & \textcolor{blue}{74.48} & 74.22 & \textcolor{codegreen}{75.29} & 74.09 \\
4   & 44.72 & 63.83 & 70.59 & 73.71 & \textcolor{red}{75.76} & 73.92 & \textcolor{codegreen}{76.14} & \textcolor{blue}{75.37} & 51.69 \\
8   & 45.79 & 64.47 & 71.31 & 74.65 & \textcolor{codegreen}{75.93} & 75.50 & \textcolor{red}{75.84} & \textcolor{blue}{75.72} & 40.49 \\
16   & 47.20 & 65.80 & 72.21 & 75.12 & \textcolor{codegreen}{77.13} & \textcolor{blue}{75.67} & \textcolor{red}{76.87} & 54.77 & 21.42 \\
32   & 49.68 & 66.27 & 72.47 & 76.66 & \textcolor{codegreen}{78.50} & \textcolor{red}{77.94} & \textcolor{blue}{77.04} & 45.28 & 13.85 \\
64   & 51.52 & 67.04 & 74.13 & \textcolor{blue}{77.81} & \textcolor{codegreen}{79.31} & \textcolor{red}{78.37} & 57.46 & 23.56 & 8.85 \\
128   & 53.53 & 68.06 & 75.25 & \textcolor{red}{79.35} & \textcolor{codegreen}{80.72} & \textcolor{blue}{77.85} & 46.81 & 10.65 & 2.05 \\
\bottomrule
\end{tabular}}}
\end{subtable}
\\

\begin{subtable}[t]{0.32\linewidth}
\centering
\caption*{ViT-Tiny RAC OfficeHome}
{\setlength{\tabcolsep}{2pt}\renewcommand{\arraystretch}{0.9}
\resizebox{\linewidth}{!}{%
\begin{tabular}{@{}lccccccccc@{}}
\toprule
 & 1e-5 & 1e-4.5 & 1e-4 & 1e-3.5 & 1e-3 & 1e-2.5 & 1e-2 & 1e-1.5 & 1e-1 \\
\midrule
2   & 44.38 & 63.75 & 70.54 & 73.54 & \textcolor{red}{75.29} & \textcolor{codegreen}{76.14} & \textcolor{blue}{75.16} & 70.97 & 54.85 \\
4   & 44.76 & 64.04 & 70.71 & \textcolor{blue}{74.01} & \textcolor{codegreen}{76.36} & \textcolor{red}{74.22} & 72.85 & 66.44 & 51.69 \\
8   & 45.83 & 65.07 & 71.65 & 74.78 & \textcolor{red}{76.31} & \textcolor{codegreen}{77.04} & \textcolor{blue}{75.97} & 64.73 & 40.49 \\
16   & 47.29 & 65.71 & 72.55 & 75.29 & \textcolor{codegreen}{77.94} & \textcolor{red}{77.47} & \textcolor{blue}{75.46} & 67.38 & 21.42 \\
32   & 49.64 & 66.10 & 72.60 & \textcolor{blue}{77.47} & \textcolor{codegreen}{79.26} & \textcolor{red}{78.79} & 75.07 & 45.28 & 13.85 \\
64   & 51.60 & 67.12 & 74.39 & \textcolor{red}{78.32} & \textcolor{codegreen}{79.91} & \textcolor{blue}{75.33} & 57.46 & 23.56 & 9.49 \\
128   & 53.53 & 67.76 & 75.72 & \textcolor{red}{79.26} & \textcolor{codegreen}{80.63} & \textcolor{blue}{76.74} & 46.81 & 10.65 & 2.05 \\
\bottomrule
\end{tabular}}}
\end{subtable}
&
\begin{subtable}[t]{0.32\linewidth}
\centering
\caption*{ViT-Tiny cLA OfficeHome}
{\setlength{\tabcolsep}{2pt}\renewcommand{\arraystretch}{0.9}
\resizebox{\linewidth}{!}{%
\begin{tabular}{@{}lccccccccc@{}}
\toprule
 & 1e-5 & 1e-4.5 & 1e-4 & 1e-3.5 & 1e-3 & 1e-2.5 & 1e-2 & 1e-1.5 & 1e-1 \\
\midrule
2   & 44.04 & 64.04 & 70.12 & 73.45 & \textcolor{codegreen}{75.89} & \textcolor{red}{75.72} & \textcolor{blue}{74.35} & 73.19 & 30.14 \\
4   & 46.09 & 64.86 & 70.84 & 74.90 & \textcolor{codegreen}{76.53} & \textcolor{red}{75.37} & \textcolor{blue}{74.99} & 54.81 & 2.01 \\
8   & 47.50 & 65.16 & 72.12 & 75.16 & \textcolor{red}{76.57} & \textcolor{codegreen}{76.87} & \textcolor{blue}{75.25} & 45.10 & 2.86 \\
16   & 50.75 & 65.63 & 72.98 & \textcolor{blue}{77.00} & \textcolor{codegreen}{78.37} & \textcolor{red}{77.17} & 59.30 & 36.55 & 3.42 \\
32   & 52.93 & 66.99 & \textcolor{blue}{74.39} & \textcolor{codegreen}{77.85} & \textcolor{red}{76.83} & \textcolor{red}{76.83} & 51.13 & 12.61 & 2.35 \\
64   & 55.96 & 68.53 & \textcolor{blue}{75.67} & \textcolor{red}{79.14} & \textcolor{codegreen}{79.26} & 63.06 & 35.49 & 4.02 & 2.69 \\
128   & 61.22 & 71.95 & \textcolor{blue}{77.34} & \textcolor{codegreen}{79.78} & \textcolor{red}{77.77} & 53.91 & 17.44 & 3.42 & 3.42 \\
\bottomrule
\end{tabular}}}
\end{subtable}
&
\begin{subtable}[t]{0.32\linewidth}
\centering
\caption*{ViT-Tiny $c^3$LA OfficeHome}
{\setlength{\tabcolsep}{2pt}\renewcommand{\arraystretch}{0.9}
\resizebox{\linewidth}{!}{%
\begin{tabular}{@{}lccccccccc@{}}
\toprule
 & 1e-5 & 1e-4.5 & 1e-4 & 1e-3.5 & 1e-3 & 1e-2.5 & 1e-2 & 1e-1.5 & 1e-1 \\
\midrule
2   & 45.02 & 64.51 & 70.29 & \textcolor{blue}{73.88} & \textcolor{codegreen}{76.36} & \textcolor{red}{75.93} & 73.07 & 57.46 & 30.14 \\
4   & 46.30 & 64.94 & 70.80 & \textcolor{blue}{74.31} & \textcolor{codegreen}{77.17} & \textcolor{red}{76.19} & 66.65 & 54.81 & 2.01 \\
8   & 49.21 & 65.37 & \textcolor{blue}{72.68} & \textcolor{red}{75.93} & \textcolor{codegreen}{77.51} & 72.60 & 64.73 & 45.10 & 3.04 \\
16   & 51.13 & 66.44 & \textcolor{blue}{73.15} & \textcolor{red}{77.55} & \textcolor{codegreen}{78.71} & 71.23 & 59.30 & 36.55 & 2.48 \\
32   & 53.31 & 67.55 & \textcolor{blue}{74.22} & \textcolor{codegreen}{79.26} & \textcolor{red}{79.22} & 67.38 & 51.13 & 12.01 & 2.35 \\
64   & 56.18 & 69.22 & \textcolor{blue}{76.06} & \textcolor{codegreen}{79.52} & \textcolor{red}{78.24} & 63.06 & 35.49 & 6.50 & 3.59 \\
128   & 61.18 & 72.00 & \textcolor{blue}{77.38} & \textcolor{codegreen}{79.91} & \textcolor{red}{78.45} & 53.91 & 17.44 & 8.38 & 3.42 \\
\bottomrule
\end{tabular}}}
\end{subtable}
\end{tabular}
\end{table*}

\begin{table*}
\centering
\captionsetup[subtable]{labelformat=simple, labelsep=colon}
\renewcommand\thesubtable{\alph{subtable}}
\caption{Test accuracies of fine-tuning DeBERTa v3 on MRPC and TREC-50 over varying scaling factors (columns), ranks (rows), and LoRA PEFT methods. The standard baseline $2r$ often was the best, and asymmetric methods preferred higher scaling factors.
% Consistent with the results of~\ref{tab:deberta3ablation}, the learning rate for all methods decreases with increasing rank. Chain methods and their non chain counterparts produce the best results in simlar learning rat ranges. %
}
\label{tab:deberta-scaling-factor}
\footnotesize
\setlength{\tabcolsep}{3pt}
\renewcommand{\arraystretch}{0.9}
\begin{tabular}{@{}ccc@{}}
% ===================== ROW 1: LoRA =====================
\begin{subtable}[t]{0.32\linewidth}
\centering
\caption*{DeBERTa v3 LoRA MRPC}
{\setlength{\tabcolsep}{2pt}\renewcommand{\arraystretch}{0.9}
\resizebox{\linewidth}{!}{%
\begin{tabular}{@{}lccccc@{}}
\toprule
 Rank/$\alpha$ & $\frac{r}{4}$ & $\frac{r}{2}$ & $r$ & $2r$ & $4r$\\
\midrule
4 & 87.2 & \textcolor{red}{88.3} & \textcolor{codegreen}{88.5} & \textcolor{blue}{88.1} & 87.4 \\
8 & \textcolor{blue}{86.9} & 86.1 & \textcolor{codegreen}{89.2} & \textcolor{red}{87.0} & 66.5 \\
16 & \textcolor{blue}{87.8} & \textcolor{red}{88.9} & \textcolor{codegreen}{89.1} & 66.5 & 66.5 \\
\bottomrule
\end{tabular}}}
\end{subtable}
&
\begin{subtable}[t]{0.32\linewidth}
\centering
\caption*{DeBERTa v3 CoLA MRPC}
{\setlength{\tabcolsep}{2pt}\renewcommand{\arraystretch}{0.9}
\resizebox{\linewidth}{!}{%
\begin{tabular}{@{}lccccc@{}}
\toprule
 Rank/$\alpha$ & $\frac{r}{4}$ & $\frac{r}{2}$ & $r$ & $2r$ & $4r$\\
\midrule
4 & 87.8 & \textcolor{red}{88.9} & \textcolor{blue}{88.5} & \textcolor{codegreen}{89.2} & 87.1 \\
8 & \textcolor{codegreen}{89.6} & \textcolor{blue}{87.4} & \textcolor{red}{88.7} & 87.2 & 86.3 \\
16 & \textcolor{codegreen}{89.2} & \textcolor{red}{87.6} & 86.9 & \textcolor{blue}{87.2} & 66.5 \\
\bottomrule
\end{tabular}}}
\end{subtable}
&
\begin{subtable}[t]{0.32\linewidth}
\centering
\caption*{DeBERTa v3 Asym MRPC}
{\setlength{\tabcolsep}{2pt}\renewcommand{\arraystretch}{0.9}
\resizebox{\linewidth}{!}{%
\begin{tabular}{@{}lccccc@{}}
\toprule
 Rank/$\alpha$ & $\frac{r}{4}$ & $\frac{r}{2}$ & $r$ & $2r$ & $4r$\\
\midrule
4 & 75.5 & 79.9 & \textcolor{blue}{80.4} & \textcolor{codegreen}{85.0} & \textcolor{red}{84.2} \\
8 & 76.7 & 82.1 & \textcolor{blue}{83.6} & \textcolor{red}{86.1} & \textcolor{codegreen}{86.9} \\
16 & 79.2 & \textcolor{blue}{81.4} & \textcolor{red}{84.8} & \textcolor{red}{84.8} & \textcolor{codegreen}{86.1} \\
\bottomrule
\end{tabular}}}
\end{subtable}
\\
\begin{subtable}[t]{0.32\linewidth}
\centering
\caption*{DeBERTa v3 RAC MRPC}
{\setlength{\tabcolsep}{2pt}\renewcommand{\arraystretch}{0.9}
\resizebox{\linewidth}{!}{%
\begin{tabular}{@{}lccccc@{}}
\toprule
 Rank/$\alpha$ & $\frac{r}{4}$ & $\frac{r}{2}$ & $r$ & $2r$ & $4r$\\
\midrule
4 & 75.6 & 79.4 & \textcolor{blue}{82.2} & \textcolor{red}{85.0} & \textcolor{codegreen}{85.7} \\
8 & 77.8 & 81.6 & \textcolor{blue}{84.6} & \textcolor{red}{85.7} & \textcolor{codegreen}{87.2} \\
16 & 80.4 & 84.5 & \textcolor{blue}{85.0} & \textcolor{red}{85.6} & \textcolor{codegreen}{87.0} \\
\bottomrule
\end{tabular}}}
\end{subtable}
&
\begin{subtable}[t]{0.32\linewidth}
\centering
\caption*{DeBERTa v3 cLA MRPC}
{\setlength{\tabcolsep}{2pt}\renewcommand{\arraystretch}{0.9}
\resizebox{\linewidth}{!}{%
\begin{tabular}{@{}lccccc@{}}
\toprule
 Rank/$\alpha$ & $\frac{r}{4}$ & $\frac{r}{2}$ & $r$ & $2r$ & $4r$\\
\midrule
4 & \textcolor{blue}{86.2} & 86.0 & \textcolor{red}{86.3} & \textcolor{codegreen}{86.4} & \textcolor{codegreen}{86.4} \\
8 & \textcolor{codegreen}{86.6} & 84.8 & 85.4 & \textcolor{blue}{85.5} & \textcolor{red}{85.9} \\
16 & \textcolor{red}{86.8} & \textcolor{codegreen}{86.9} & 86.2 & 86.2 & \textcolor{blue}{86.4} \\
\bottomrule
\end{tabular}}}
\end{subtable}
&
\begin{subtable}[t]{0.32\linewidth}
\centering
\caption*{DeBERTa v3 $c^3$LA MRPC}
{\setlength{\tabcolsep}{2pt}\renewcommand{\arraystretch}{0.9}
\resizebox{\linewidth}{!}{%
\begin{tabular}{@{}lccccc@{}}
\toprule
 Rank/$\alpha$ & $\frac{r}{4}$ & $\frac{r}{2}$ & $r$ & $2r$ & $4r$\\
\midrule
4 & 79.3 & 83.3 & \textcolor{red}{86.5} & \textcolor{codegreen}{88.5} & \textcolor{blue}{86.1} \\
8 & 78.3 & \textcolor{blue}{84.9} & \textcolor{red}{86.9} & \textcolor{codegreen}{87.6} & \textcolor{red}{86.9} \\
16 & 85.0 & \textcolor{blue}{85.7} & \textcolor{codegreen}{87.3} & \textcolor{red}{85.8} & 66.5 \\
\bottomrule
\end{tabular}}}
\end{subtable}
\\
\begin{subtable}[t]{0.32\linewidth}
\centering
\caption*{DeBERTa v3 LoRA TREC-50}
{\setlength{\tabcolsep}{2pt}\renewcommand{\arraystretch}{0.9}
\resizebox{\linewidth}{!}{%
\begin{tabular}{@{}lccccc@{}}
\toprule
 Rank/$\alpha$ & $\frac{r}{4}$ & $\frac{r}{2}$ & $r$ & $2r$ & $4r$\\
\midrule
4 & 88.9 & \textcolor{blue}{89.7} & \textcolor{codegreen}{90.7} & 83.1 & \textcolor{red}{90.3} \\
8 & \textcolor{blue}{88.7} & \textcolor{red}{90.7} & \textcolor{codegreen}{91.3} & 85.3 & 75.4 \\
16 & \textcolor{red}{91.1} & \textcolor{codegreen}{91.5} & \textcolor{blue}{90.7} & 88.5 & 10.9 \\
\bottomrule
\end{tabular}}}
\end{subtable}
&
\begin{subtable}[t]{0.32\linewidth}
\centering
\caption*{DeBERTa v3 CoLA TREC-50}
{\setlength{\tabcolsep}{2pt}\renewcommand{\arraystretch}{0.9}
\resizebox{\linewidth}{!}{%
\begin{tabular}{@{}lccccc@{}}
\toprule
 Rank/$\alpha$ & $\frac{r}{4}$ & $\frac{r}{2}$ & $r$ & $2r$ & $4r$\\
\midrule
4 & \textcolor{red}{92.1} & \textcolor{blue}{91.9} & \textcolor{codegreen}{92.5} & 90.7 & 91.1 \\
8 & \textcolor{codegreen}{91.7} & 89.7 & \textcolor{red}{90.9} & \textcolor{blue}{90.3} & 85.5 \\
16 & \textcolor{red}{91.9} & \textcolor{codegreen}{92.3} & 86.1 & \textcolor{blue}{87.3} & 10.9 \\
\bottomrule
\end{tabular}}}
\end{subtable}
&
\begin{subtable}[t]{0.32\linewidth}
\centering
\caption*{DeBERTa v3 Asym TREC-50}
{\setlength{\tabcolsep}{2pt}\renewcommand{\arraystretch}{0.9}
\resizebox{\linewidth}{!}{%
\begin{tabular}{@{}lccccc@{}}
\toprule
 Rank/$\alpha$ & $\frac{r}{4}$ & $\frac{r}{2}$ & $r$ & $2r$ & $4r$\\
\midrule
4 & 79.8 & 84.7 & \textcolor{blue}{87.7} & \textcolor{codegreen}{90.5} & \textcolor{red}{89.9} \\
8 & 82.9 & \textcolor{blue}{87.7} & 84.5 & \textcolor{red}{89.3} & \textcolor{codegreen}{90.7} \\
16 & 89.3 & 86.7 & \textcolor{red}{90.7} & \textcolor{codegreen}{91.3} & \textcolor{blue}{89.7} \\
\bottomrule
\end{tabular}}}
\end{subtable}
\\
\begin{subtable}[t]{0.32\linewidth}
\centering
\caption*{DeBERTa v3 RAC TREC-50}
{\setlength{\tabcolsep}{2pt}\renewcommand{\arraystretch}{0.9}
\resizebox{\linewidth}{!}{%
\begin{tabular}{@{}lccccc@{}}
\toprule
 Rank/$\alpha$ & $\frac{r}{4}$ & $\frac{r}{2}$ & $r$ & $2r$ & $4r$\\
\midrule
4 & 60.1 & 72.6 & \textcolor{blue}{81.0} & \textcolor{red}{85.5} & \textcolor{codegreen}{88.7} \\
8 & 75.2 & 81.5 & \textcolor{blue}{85.7} & \textcolor{red}{88.1} & \textcolor{codegreen}{89.7} \\
16 & 83.1 & \textcolor{blue}{86.3} & \textcolor{red}{87.7} & \textcolor{codegreen}{90.3} & 78.0 \\
\bottomrule
\end{tabular}}}
\end{subtable}
&
\begin{subtable}[t]{0.32\linewidth}
\centering
\caption*{DeBERTa v3 cLA TREC-50}
{\setlength{\tabcolsep}{2pt}\renewcommand{\arraystretch}{0.9}
\resizebox{\linewidth}{!}{%
\begin{tabular}{@{}lccccc@{}}
\toprule
 Rank/$\alpha$ & $\frac{r}{4}$ & $\frac{r}{2}$ & $r$ & $2r$ & $4r$\\
\midrule
4 & 57.9 & 74.8 & \textcolor{blue}{80.0} & \textcolor{red}{82.9} & \textcolor{codegreen}{86.3} \\
8 & 73.6 & 76.6 & \textcolor{red}{83.7} & \textcolor{blue}{82.5} & \textcolor{codegreen}{87.3} \\
16 & 79.6 & 80.2 & \textcolor{red}{87.9} & \textcolor{codegreen}{88.1} & \textcolor{blue}{86.1} \\
\bottomrule
\end{tabular}}}
\end{subtable}
&
\begin{subtable}[t]{0.32\linewidth}
\centering
\caption*{DeBERTa v3 $c^3$LA TREC-50}
{\setlength{\tabcolsep}{2pt}\renewcommand{\arraystretch}{0.9}
\resizebox{\linewidth}{!}{%
\begin{tabular}{@{}lccccc@{}}
\toprule
 Rank/$\alpha$ & $\frac{r}{4}$ & $\frac{r}{2}$ & $r$ & $2r$ & $4r$\\
\midrule
4 & 73.8 & 81.5 & \textcolor{blue}{83.1} & \textcolor{codegreen}{89.3} & \textcolor{red}{88.7} \\
8 & 78.2 & \textcolor{blue}{82.3} & \textcolor{red}{83.9} & \textcolor{codegreen}{84.7} & 81.7 \\
16 & 83.1 & \textcolor{blue}{85.5} & 85.3 & \textcolor{codegreen}{87.5} & \textcolor{red}{86.3} \\
\bottomrule
\end{tabular}}}
\end{subtable}

\end{tabular}
\end{table*}

\begin{table*}
\centering
\captionsetup[subtable]{labelformat=simple, labelsep=colon}
\renewcommand\thesubtable{\alph{subtable}}
\caption{Test accuracies obtained by fine-tuning ViT-Tiny on OfficeHome and CIFAR-10 over varying scaling factors (columns), ranks (rows), and LoRA PEFT methods. The standard baseline $2r$ often was the best, and asymmetric methods preferred higher scaling factors.
% Consistent with the results of~\ref{tab:deberta3ablation}, the learning rate for all methods decreases with increasing rank. Chain methods and their non chain counterparts produce the best results in simlar learning rat ranges. %
}
\label{tab:vit-scaling-factor}
\footnotesize
\setlength{\tabcolsep}{3pt}
\renewcommand{\arraystretch}{0.9}

\begin{tabular}{@{}ccc@{}}

% ===================== ROW 1: LoRA =====================
\begin{subtable}[t]{0.32\linewidth}
\centering
\caption*{ViT-Tiny LoRA CIFAR-10}
{\setlength{\tabcolsep}{2pt}\renewcommand{\arraystretch}{0.9}
\resizebox{\linewidth}{!}{%
\begin{tabular}{@{}lcccc@{}}
\toprule
 Rank/$\alpha$ & $\frac{r}{4}$ & $\frac{r}{2}$ & r & 2r\\
\midrule
4 & 93.9 & \textcolor{codegreen}{94.1} & \textcolor{red}{94.0} & \textcolor{red}{94.0}\\
8 & 94.8 & \textcolor{red}{95.7} & \textcolor{red}{95.7} & \textcolor{codegreen}{95.8}\\
16 & 95.8 & \textcolor{codegreen}{96.1} & \textcolor{codegreen}{96.1} & \textcolor{red}{95.2} \\
\bottomrule
\end{tabular}}}
\end{subtable} &
\begin{subtable}[t]{0.32\linewidth}
\centering
\caption*{ViT-Tiny CoLA CIFAR-10}
{\setlength{\tabcolsep}{2pt}\renewcommand{\arraystretch}{0.9}
\resizebox{\linewidth}{!}{%
\begin{tabular}{@{}lccccc@{}}
\toprule
 Rank/$\alpha$ & $\frac{r}{4}$ & $\frac{r}{2}$ & r & 2r\\
\midrule
4 & 94.3 & \textcolor{red}{94.5} & \textcolor{red}{94.5} & \textcolor{codegreen}{95.3}\\
8 & 94.7 & \textcolor{blue}{94.9} & \textcolor{codegreen}{95.3} & \textcolor{red}{95.1}\\
16 & 95.0 & \textcolor{blue}{95.5} & \textcolor{red}{95.7} & \textcolor{codegreen}{96.2}\\
\bottomrule
\end{tabular}}}
\end{subtable} &
\begin{subtable}[t]{0.32\linewidth}
\centering
\caption*{ViT-Tiny Asym CIFAR-10}
{\setlength{\tabcolsep}{2pt}\renewcommand{\arraystretch}{0.9}
\resizebox{\linewidth}{!}{%
\begin{tabular}{@{}lccccc@{}}
\toprule
 Rank/$\alpha$ & $\frac{r}{4}$ & $\frac{r}{2}$ & r & 2r\\
\midrule
4 & 91.7 & \textcolor{blue}{92.2} & \textcolor{codegreen}{92.8} & \textcolor{red}{92.4}\\
8 & 92.6 & \textcolor{blue}{93.1} & \textcolor{red}{93.7} & \textcolor{codegreen}{94}.\\
16 & 93.1 & \textcolor{red}{94.0} & \textcolor{codegreen}{94.4} & \textcolor{codegreen}{94.4}\\
\bottomrule
\end{tabular}}}
\end{subtable}\\
\\
\begin{subtable}[t]{0.32\linewidth}
\centering
\caption*{ViT-Tiny RAC CIFAR-10}
{\setlength{\tabcolsep}{2pt}\renewcommand{\arraystretch}{0.9}
\resizebox{\linewidth}{!}{%
\begin{tabular}{@{}lccccc@{}}
\toprule
 Rank/$\alpha$ & $\frac{r}{4}$ & $\frac{r}{2}$ & r & 2r\\
\midrule
4 & 91.8 & \textcolor{blue}{92.6} & \textcolor{red}{93.2} & \textcolor{codegreen}{94.2}\\
8 & 92.8 & \textcolor{blue}{93.4} & \textcolor{red}{94.0} & \textcolor{codegreen}{94.8}\\
16 & 93.6 & \textcolor{blue}{94.3} & \textcolor{red}{94.8} & \textcolor{codegreen}{95.6}\\
\bottomrule
\end{tabular}}}
\end{subtable} &
\begin{subtable}[t]{0.32\linewidth}
\centering
\caption*{ViT-Tiny cLA CIFAR-10}
{\setlength{\tabcolsep}{2pt}\renewcommand{\arraystretch}{0.9}
\resizebox{\linewidth}{!}{%
\begin{tabular}{@{}lccccc@{}}
\toprule
 Rank/$\alpha$ & $\frac{r}{4}$ & $\frac{r}{2}$ & r & 2r\\
\midrule
4 & 92.0 & \textcolor{blue}{92.6} & \textcolor{red}{93.1} & \textcolor{codegreen}{93.5}\\
8 & 93.4 & \textcolor{red}{93.4} & \textcolor{codegreen}{93.5} & \textcolor{red}{93.4}\\
16 & 94.3 & \textcolor{codegreen}{94.5} & \textcolor{codegreen}{94.5} & \textcolor{codegreen}{94.5}\\
\bottomrule
\end{tabular}}}
\end{subtable} &
\begin{subtable}[t]{0.32\linewidth}
\centering
\caption*{ViT-Tiny $c^3$LA CIFAR-10}
{\setlength{\tabcolsep}{2pt}\renewcommand{\arraystretch}{0.9}
\resizebox{\linewidth}{!}{%
\begin{tabular}{@{}lccccc@{}}
\toprule
 Rank/$\alpha$ & $\frac{r}{4}$ & $\frac{r}{2}$ & r & 2r\\
\midrule
4 & 92.7 & \textcolor{red}{93.4} & \textcolor{red}{93.4} & \textcolor{codegreen}{94.4}\\
8 & 93.8 & \textcolor{blue}{94.4} & \textcolor{red}{94.6} & \textcolor{codegreen}{94.8}\\
16 & 94.8 & \textcolor{codegreen}{95.3} & \textcolor{red}{95.2} & \textcolor{codegreen}{95.3}\\
\bottomrule
\end{tabular}}}
\end{subtable}
\\
\begin{subtable}[t]{0.32\linewidth}
\centering
\caption*{ViT-Tiny LoRA OfficeHome}
{\setlength{\tabcolsep}{2pt}\renewcommand{\arraystretch}{0.9}
\resizebox{\linewidth}{!}{%
\begin{tabular}{@{}lccccc@{}}
\toprule
 Rank/$\alpha$ & $\frac{r}{4}$ & $\frac{r}{2}$ & r & 2r\\
\midrule
4 & 76.8 & \textcolor{blue}{77.1} & \textcolor{red}{77.8} & \textcolor{codegreen}{77.9}\\
8 & 76.9 & \textcolor{blue}{77.9} & \textcolor{red}{78.5} & \textcolor{codegreen}{79.4}\\
16 & 77.9 & \textcolor{blue}{78.4} & \textcolor{red}{79.2} & \textcolor{codegreen}{79.4}\\
\bottomrule
\end{tabular}}}
\end{subtable} &
\begin{subtable}[t]{0.32\linewidth}
\centering
\caption*{ViT-Tiny CoLA OfficeHome}
{\setlength{\tabcolsep}{2pt}\renewcommand{\arraystretch}{0.9}
\resizebox{\linewidth}{!}{%
\begin{tabular}{@{}lccccc@{}}
\toprule
 Rank/$\alpha$ & $\frac{r}{4}$ & $\frac{r}{2}$ & r & 2r\\
\midrule
4 & 75.5 & \textcolor{blue}{75.9} & \textcolor{red}{76.6} & \textcolor{codegreen}{77.8}\\
8 & 76.0 & \textcolor{blue}{76.4} & \textcolor{red}{77.4} & \textcolor{codegreen}{79.6}\\
16 & 76.3 & \textcolor{blue}{77.2} & \textcolor{red}{78.2} & \textcolor{codegreen}{79.4}\\
\bottomrule
\end{tabular}}}
\end{subtable} &
\begin{subtable}[t]{0.32\linewidth}
\centering
\caption*{ViT-Tiny Asym OfficeHome}
{\setlength{\tabcolsep}{2pt}\renewcommand{\arraystretch}{0.9}
\resizebox{\linewidth}{!}{%
\begin{tabular}{@{}lccccc@{}}
\toprule
 Rank/$\alpha$ & $\frac{r}{4}$ & $\frac{r}{2}$ & r & 2r\\
\midrule
4 & 74.0 & \textcolor{blue}{74.5} & \textcolor{red}{75.2} & \textcolor{codegreen}{75.6}\\
8 & 74.5 & \textcolor{blue}{75.1} & \textcolor{red}{75.6} & \textcolor{codegreen}{76.2}\\
16 & 75.2 & \textcolor{blue}{75.9} & \textcolor{red}{76.4} & \textcolor{codegreen}{76.9}\\
\bottomrule
\end{tabular}}}
\end{subtable}\\
\\
\begin{subtable}[t]{0.32\linewidth}
\centering
\caption*{ViT-Tiny RAC OfficeHome}
{\setlength{\tabcolsep}{2pt}\renewcommand{\arraystretch}{0.9}
\resizebox{\linewidth}{!}{%
\begin{tabular}{@{}lccccc@{}}
\toprule
 Rank/$\alpha$ & $\frac{r}{4}$ & $\frac{r}{2}$ & r & 2r\\
\midrule
4 & 74.3 & \textcolor{blue}{74.6} & \textcolor{red}{75.6} & \textcolor{codegreen}{76.0}\\
8 & 74.8 & \textcolor{blue}{75.2} & \textcolor{red}{75.9} & \textcolor{codegreen}{77.1}\\
16 & 75.2 & \textcolor{blue}{75.5} & \textcolor{red}{76.4} & \textcolor{codegreen}{77.7}\\
\bottomrule
\end{tabular}}}
\end{subtable} &
\begin{subtable}[t]{0.32\linewidth}
\centering
\caption*{ViT-Tiny cLA OfficeHome}
{\setlength{\tabcolsep}{2pt}\renewcommand{\arraystretch}{0.9}
\resizebox{\linewidth}{!}{%
\begin{tabular}{@{}lccccc@{}}
\toprule
 Rank/$\alpha$ & $\frac{r}{4}$ & $\frac{r}{2}$ & r & 2r\\
\midrule
4 & 75.3 & \textcolor{red}{75.9} & \textcolor{codegreen}{76.5} & \textcolor{codegreen}{76.5}\\
8 & 76.3 & \textcolor{red}{76.5} & \textcolor{codegreen}{76.6} & \textcolor{red}{76.5}\\
16 & 76.4 & \textcolor{blue}{76.9} & \textcolor{red}{77.3} & \textcolor{codegreen}{78.4}\\
\bottomrule
\end{tabular}}}
\end{subtable} &
\begin{subtable}[t]{0.32\linewidth}
\centering
\caption*{ViT-Tiny $c^3$LA OfficeHome}
{\setlength{\tabcolsep}{2pt}\renewcommand{\arraystretch}{0.9}
\resizebox{\linewidth}{!}{%
\begin{tabular}{@{}lccccc@{}}
\toprule
 Rank/$\alpha$ & $\frac{r}{4}$ & $\frac{r}{2}$ & r & 2r\\
\midrule
4 & 75.6 & \textcolor{blue}{75.9} & \textcolor{red}{76.3} & \textcolor{codegreen}{77.3}\\
8 & 76.4 & \textcolor{red}{77.0} & \textcolor{blue}{76.9} & \textcolor{codegreen}{77.5}\\
16 & 76.6 & \textcolor{blue}{77.7} & \textcolor{red}{78.4} & \textcolor{codegreen}{78.5}\\
\bottomrule
\end{tabular}}}
\end{subtable}
\end{tabular}
\end{table*}

\begin{table*}
\centering
\caption{Test accuracies obtained by fine-tuning DeBERTa v3 on MRPC, CoLA, TREC-50, and RTE using chain LoRA methods CoLA, RAC, and $c^3$LA over varying ranks and chain reset frequencies. We do not observe a clear correlation between the optimal chain reset frequency and rank.}
\label{tab:chain_reset}
\scriptsize
\setlength{\tabcolsep}{3.5pt}
\renewcommand{\arraystretch}{1.04}

% Helper macro: uniform header for inner tables
\newcommand{\CRHeader}{%
\toprule
\multicolumn{2}{c}{} & \multicolumn{6}{c}{Chain Reset Frequency} \\
\cmidrule(lr){3-8}
Variant & Rank 1 & 2 & 5 & 10 & 15 & 20 \\
\midrule
}

\begin{tabular}{@{}c c@{}}
%%%%%%%%%%%%%%%%%%%%%%%%%%%%%%%%%%%%
% MRPC
%%%%%%%%%%%%%%%%%%%%%%%%%%%%%%%%%%%%
\begin{minipage}{0.49\linewidth}
\centering
DeBERTa v3 MRPC\\[2pt]
\begin{tabular}{@{}llcccccc@{}}
\CRHeader
\multirow{3}{*}{CoLA}
  & 4  & \textcolor{blue}{88.0} & 86.8 & \textcolor{codegreen}{89.2} & \textcolor{red}{88.1} & 86.7 & 86.7 \\
  \cmidrule(lr){2-8}
  & 8  & \textcolor{red}{87.8} & \textcolor{codegreen}{88.0} & \textcolor{blue}{87.2} & \textcolor{blue}{87.2} & 86.7 & \textcolor{blue}{87.2} \\
  \cmidrule(lr){2-8}
  & 16  & \textcolor{red}{66.5} & \textcolor{codegreen}{87.2} & \textcolor{codegreen}{87.2} & \textcolor{codegreen}{87.2} & \textcolor{codegreen}{87.2} & \textcolor{codegreen}{87.2} \\
\midrule[\heavyrulewidth]
\multirow{3}{*}{RAC}
  & 4  & 68.3 & 77.4 & \textcolor{blue}{85.0} & \textcolor{red}{85.7} & \textcolor{red}{85.7} & \textcolor{codegreen}{86.6} \\
  \cmidrule(lr){2-8}
  & 8  & 68.1 & 82.0 & \textcolor{red}{85.7} & \textcolor{codegreen}{86.4} & \textcolor{red}{85.7} & \textcolor{blue}{85.6} \\
  \cmidrule(lr){2-8}
  & 16  & 69.1 & 84.2 & 85.6 & \textcolor{blue}{86.1} & \textcolor{codegreen}{86.5} & \textcolor{red}{86.3} \\
\midrule[\heavyrulewidth]
\multirow{3}{*}{c$^{3}$LA}
  & 4  & 84.8 & \textcolor{red}{86.7} & \textcolor{codegreen}{87.2} & 85.2 & \textcolor{blue}{85.8} & 85.2 \\
  \cmidrule(lr){2-8}
  & 8  & 85.2 & \textcolor{codegreen}{87.7} & 86.6 & \textcolor{blue}{86.7} & 85.3 & \textcolor{red}{86.9} \\
  \cmidrule(lr){2-8}
  & 16  & \textcolor{codegreen}{87.6} & \textcolor{red}{87.0} & \textcolor{blue}{86.7} & 86.6 & 86.6 & \textcolor{codegreen}{87.6} \\
\bottomrule
\end{tabular}
\end{minipage}
\vspace{2mm}
&
%%%%%%%%%%%%%%%%%%%%%%%%%%%%%%%%%%%%
% TREC
%%%%%%%%%%%%%%%%%%%%%%%%%%%%%%%%%%%%
\begin{minipage}{0.49\linewidth}
\centering
DeBERTa v3 TREC-50\\[2pt]
\begin{tabular}{@{}llcccccc@{}}
\CRHeader
\multirow{3}{*}{CoLA}
  & 4  & \textcolor{codegreen}{91.3} & \textcolor{red}{91.1} & 89.9 & 88.5 & \textcolor{blue}{90.5} & \textcolor{codegreen}{91.3} \\
  \cmidrule(lr){2-8}
  & 8  & \textcolor{codegreen}{92.7} & \textcolor{red}{91.1} & 85.3 & 10.9 & \textcolor{codegreen}{92.7} & \textcolor{blue}{90.5} \\
  \cmidrule(lr){2-8}
  & 16  & 10.9 & 10.9 & \textcolor{codegreen}{93.1} & \textcolor{blue}{91.7} & \textcolor{red}{92.1} & 65.1 \\
\midrule[\heavyrulewidth]
\multirow{3}{*}{RAC}
  & 4  & 84.3 & 84.1 & \textcolor{blue}{85.5} & 84.1 & \textcolor{red}{86.3} & \textcolor{codegreen}{86.5} \\
  \cmidrule(lr){2-8}
  & 8  & 88.3 & \textcolor{blue}{88.5} & 88.1 & \textcolor{red}{88.7} & 87.7 & \textcolor{codegreen}{88.9} \\
  \cmidrule(lr){2-8}
  & 16  & 87.7 & \textcolor{codegreen}{91.5} & \textcolor{red}{90.3} & \textcolor{blue}{89.9} & \textcolor{blue}{89.9} & 88.9 \\
\midrule[\heavyrulewidth]
\multirow{3}{*}{c$^{3}$LA}
  & 4  & 86.3 & \textcolor{blue}{88.1} & \textcolor{codegreen}{89.3} & 85.9 & \textcolor{red}{88.9} & \textcolor{red}{88.9} \\
  \cmidrule(lr){2-8}
  & 8  & \textcolor{blue}{86.1} & \textcolor{red}{89.3} & 84.7 & 83.7 & \textcolor{blue}{86.1} & \textcolor{codegreen}{90.7} \\
  \cmidrule(lr){2-8}
  & 16  & \textcolor{blue}{89.7} & \textcolor{red}{90.5} & 87.5 & \textcolor{codegreen}{91.1} & 87.3 & 88.1 \\
\bottomrule
\end{tabular}
\end{minipage}
\\[0.8em]
%%%%%%%%%%%%%%%%%%%%%%%%%%%%%%%%%%%%
% CoLA
%%%%%%%%%%%%%%%%%%%%%%%%%%%%%%%%%%%%
\begin{minipage}{0.49\linewidth}
\centering
DeBERTa v3 CoLA\\[2pt]
\begin{tabular}{@{}llcccccc@{}}
\CRHeader
\multirow{3}{*}{CoLA}
  & 4  & \textcolor{red}{86.9} & 86.5 & 86.2 & 86.6 & \textcolor{codegreen}{87.1} & \textcolor{blue}{86.7} \\
  \cmidrule(lr){2-8}
  & 8  & \textcolor{codegreen}{85.5} & \textcolor{red}{85.1} & \textcolor{red}{85.1} & \textcolor{red}{85.1} & \textcolor{red}{85.1} & \textcolor{red}{85.1} \\
  \cmidrule(lr){2-8}
  & 16  & \textcolor{codegreen}{84.2} & \textcolor{red}{69.1} & \textcolor{red}{69.1} & \textcolor{red}{69.1} & \textcolor{red}{69.1} & \textcolor{red}{69.1} \\
\midrule[\heavyrulewidth]
\multirow{3}{*}{RAC}
  & 4  & 87.0 & 86.7 & \textcolor{red}{87.7} & \textcolor{blue}{87.4} & \textcolor{codegreen}{88.0} & \textcolor{red}{87.7} \\
  \cmidrule(lr){2-8}
  & 8  & \textcolor{red}{87.5} & \textcolor{codegreen}{87.8} & \textcolor{codegreen}{87.8} & \textcolor{red}{87.5} & \textcolor{blue}{86.6} & \textcolor{blue}{86.6} \\
  \cmidrule(lr){2-8}
  & 16  & 86.7 & 86.9 & \textcolor{red}{87.3} & \textcolor{blue}{87.0} & \textcolor{blue}{87.0} & \textcolor{codegreen}{87.6} \\
\midrule[\heavyrulewidth]
\multirow{3}{*}{c$^{3}$LA}
  & 4  & \textcolor{red}{86.4} & \textcolor{codegreen}{86.6} & 86.1 & 85.8 & 86.0 & \textcolor{blue}{86.3} \\
  \cmidrule(lr){2-8}
  & 8  & 86.0 & \textcolor{blue}{86.1} & \textcolor{blue}{86.1} & \textcolor{red}{86.2} & \textcolor{codegreen}{86.3} & \textcolor{codegreen}{86.3} \\
  \cmidrule(lr){2-8}
  & 16  & \textcolor{blue}{86.2} & 85.7 & \textcolor{red}{86.3} & 85.6 & 85.4 & \textcolor{codegreen}{86.7} \\
\bottomrule
\end{tabular}
\end{minipage}
&
%%%%%%%%%%%%%%%%%%%%%%%%%%%%%%%%%%%%
% RTE
%%%%%%%%%%%%%%%%%%%%%%%%%%%%%%%%%%%%
\begin{minipage}{0.49\linewidth}
\centering
DeBERTa v3 RTE\\[2pt]
\begin{tabular}{@{}llcccccc@{}}
\CRHeader
\multirow{3}{*}{CoLA}
    & 4  & 82.9 & 84.4 & \textcolor{blue}{85.1} & 83.7 & \textcolor{red}{85.4} & \textcolor{codegreen}{86.2} \\
    \cmidrule(lr){2-8}
  & 8  & \textcolor{codegreen}{88.2} & 84.6 & 84.8 & \textcolor{red}{87.1} & \textcolor{red}{87.1} & \textcolor{blue}{86.7} \\
  \cmidrule(lr){2-8}
  & 16  & \textcolor{codegreen}{85.1} & 52.6 & 81.4 & \textcolor{red}{84.8} & \textcolor{blue}{84.3} & 73.5 \\
\midrule[\heavyrulewidth]
\multirow{3}{*}{RAC}
  & 4  & \textcolor{blue}{82.3} & \textcolor{codegreen}{83.0} & 81.6 & 82.1 & \textcolor{red}{82.4} & \textcolor{red}{82.4} \\
  \cmidrule(lr){2-8}
  & 8  & 85.5 & \textcolor{red}{86.8} & \textcolor{blue}{86.4} & \textcolor{codegreen}{87.5} & \textcolor{codegreen}{87.5} & \textcolor{codegreen}{87.5} \\
  \cmidrule(lr){2-8}
  & 16  & \textcolor{blue}{84.2} & \textcolor{codegreen}{84.4} & 84.1 & 83.5 & 83.7 & \textcolor{red}{84.3} \\
\midrule[\heavyrulewidth]
\multirow{3}{*}{c$^{3}$LA}
  & 4  & \textcolor{red}{79.0} & 77.9 & 72.6 & 71.9 & 74.0 & 72.4 \\
  \cmidrule(lr){2-8}
  & 8  & \textcolor{red}{80.0} & \textcolor{codegreen}{80.3} & \textcolor{blue}{76.6} & 73.9 & 76.2 & 75.4 \\
  \cmidrule(lr){2-8}
  & 16  & \textcolor{codegreen}{85.0} & 82.5 & \textcolor{red}{83.6} & \textcolor{blue}{83.0} & 82.9 & 82.4 \\
\bottomrule
\end{tabular}
\end{minipage}
\end{tabular}

\vspace{2pt}
\end{table*}

\begin{table*}
\centering
\caption{Test accuracies obtained by fine-tuning ViT-Tiny on OfficeHome and CIFAR-10 using chain LoRA methods CoLA, RAC, and $c^3$LA over varying ranks and chain reset frequencies. 
%No clear correlation between optinal chain reset frequency and rank is observed.
}
\label{tab:vit-chain_reset}
\scriptsize
\setlength{\tabcolsep}{3.5pt}
\renewcommand{\arraystretch}{1.04}

% Helper macro: uniform header for inner tables
\newcommand{\CRHeader}{%
\toprule
\multicolumn{2}{c}{} & \multicolumn{6}{c}{Chain Reset Frequency} \\
\cmidrule(lr){3-8}
Variant & Rank & 1 & 2 & 5 & 10 & 15 & 20 \\
\midrule
}

\begin{tabular}{@{}c c@{}}
%%%%%%%%%%%%%%%%%%%%%%%%%%%%%%%%%%%%
% OfficeHome
%%%%%%%%%%%%%%%%%%%%%%%%%%%%%%%%%%%%
\begin{minipage}{0.49\linewidth}
\centering
ViT-Tiny OfficeHome\\[2pt]
\begin{tabular}{@{}llcccccc@{}}
\CRHeader
\multirow{3}{*}{CoLA}
  &  4  & 76.4 & 76.5 & \textcolor{blue}{77.2} & \textcolor{blue}{77.2} & \textcolor{red}{77.6} & \textcolor{codegreen}{77.8}\\
  
  \cmidrule(lr){2-8}
  &  8  & 77.6 & 77.1 & \textcolor{blue}{78.3} & 77.3 & \textcolor{red}{78.7} & \textcolor{codegreen}{79.6}\\
  
  \cmidrule(lr){2-8}
  & 16  & 77.9 & 77.9 & \textcolor{blue}{78.8} & 78.6 & \textcolor{codegreen}{79.4} & \textcolor{red}{79.1}\\
\midrule[\heavyrulewidth]
\multirow{3}{*}{RAC}
  &  4  & 75.5 & \textcolor{red}{75.8} & \textcolor{codegreen}{76.0} & \textcolor{blue}{75.7} & \textcolor{blue}{75.7} & \textcolor{blue}{75.7}\\
  
  \cmidrule(lr){2-8}
  &  8  & \textcolor{codegreen}{77.1} & 76.1 & 76.3 & \textcolor{red}{76.6} & 76.1 & \textcolor{blue}{76.4}\\
  
  \cmidrule(lr){2-8}
  & 16  & \textcolor{red}{77.4} & \textcolor{codegreen}{77.7} & \textcolor{codegreen}{77.7} & 77.0 & \textcolor{blue}{77.3} & 77.0\\
\midrule[\heavyrulewidth]
\multirow{3}{*}{c$^{3}$LA}
  &  4  & \textcolor{red}{76.7} & \textcolor{codegreen}{77.3} & \textcolor{codegreen}{77.3} & 76.3 & 76.5 & \textcolor{blue}{76.6}\\
  
  \cmidrule(lr){2-8}
  &  8  & \textcolor{codegreen}{77.5} & \textcolor{blue}{76.9} & \textcolor{red}{77.2} & 76.7 & 76.8 & 76.8\\
  
  \cmidrule(lr){2-8}
  & 16  & 77.5 & \textcolor{blue}{78.1} & \textcolor{red}{78.4} & \textcolor{codegreen}{78.5} & \textcolor{blue}{78.1} & \textcolor{red}{78.4}\\
\bottomrule
\end{tabular}
\end{minipage}
&
%%%%%%%%%%%%%%%%%%%%%%%%%%%%%%%%%%%%
% CIFAR-10
%%%%%%%%%%%%%%%%%%%%%%%%%%%%%%%%%%%%
\begin{minipage}{0.49\linewidth}
\centering
ViT-Tiny CIFAR-10\\[2pt]
\begin{tabular}{@{}llcccccc@{}}
\CRHeader
\multirow{3}{*}{CoLA}
  &  4  & 94.5 & \textcolor{red}{94.8} & \textcolor{blue}{94.7} & \textcolor{codegreen}{95.3} & 94.0 & 94.0\\
  
  \cmidrule(lr){2-8}
  &  8  & \textcolor{red}{95.1} & \textcolor{red}{95.1} & \textcolor{codegreen}{95.3} & \textcolor{blue}{94.9} & 94.7 & 94.5\\
  
  \cmidrule(lr){2-8}
  & 16  & 95.5 & 95.5 & \textcolor{blue}{95.7} & \textcolor{red}{96.0} & \textcolor{red}{96.0} & \textcolor{codegreen}{96.2}\\
\midrule[\heavyrulewidth]
\multirow{3}{*}{RAC}
  &  4  & \textcolor{codegreen}{94.2} & \textcolor{red}{94.0} & \textcolor{red}{94.0} & \textcolor{blue}{93.4} & 92.4 & 92.5\\
  
  \cmidrule(lr){2-8}
  &  8  & \textcolor{red}{94.5} & \textcolor{codegreen}{94.8} & \textcolor{red}{94.5} & \textcolor{blue}{94.2} & 94.1 & 94.0\\
  
  \cmidrule(lr){2-8}
  & 16  & \textcolor{codegreen}{95.6} & \textcolor{blue}{95.2} & \textcolor{red}{95.3} & 95.1 & 95.0 & 94.3\\
\midrule[\heavyrulewidth]
\multirow{3}{*}{c$^{3}$LA}
  &  4  & \textcolor{codegreen}{94.4} & \textcolor{red}{94.2} & \textcolor{blue}{94.0} & 93.9 & 92.7 & 92.7\\
  
  \cmidrule(lr){2-8}
  &  8  & 93.6 & \textcolor{blue}{94.2} & \textcolor{codegreen}{94.8} & \textcolor{red}{94.5} & 93.4 & 93.4\\
  
  \cmidrule(lr){2-8}
  & 16  & 94.0 & 93.6 & \textcolor{codegreen}{95.3} & \textcolor{red}{95.1} & 94.8 & \textcolor{blue}{95.0}\\
\bottomrule
\end{tabular}
\end{minipage}
\end{tabular}
\end{table*}

\subsection{Computational Cost, Memory, and Efficiency}\label{subsec:efficiency}

In this Section, we discuss the computational cost, peak memory, and throughput of our sparsity-induced LoRA variants. We also contrast these metrics with PaCA and its variants. We start by describing our naïve sparse implementation.

\subsubsection{Naïve sparse implementation.} To highlight the computational benefit of cLA, $c^3$LA, and their random variants' inherently sparse structure, we introduce PEFT methods s-cLA and s-r-$c^3$LA. The only difference is in how we compute the forward pass for each adapted layer. Consider a single layer $W^i\in\mathbb{R}^{n_i\times m_i}$ adapted via r-cLA to $W^i + BA$, where $B\in\mathbb{R}^{n_i\times r}$ and $A$ is constructed as follows: sample $r$ random numbers without replacement from $[m_i]$, $\{c_1, c_2,\cdots,c_r\}\subset [m_i].$ For row $j$ in $A$, the $c_j^{\rm th}$ element is one, and all other elements are zero, thus $Ax = [x_{c_1},\cdots,x_{c_r}]^\top$. For cLA choose $\{c_1,\cdots,c_r\} = \{1,\cdots,r\}.$

In each forward pass, we compute $(W^i)(x) + B(A(x))$ where $x\in\mathbb{R}^{m_i}.$ For cLA, we directly calculate $A(x)$ as a general matrix-matrix multiplication (GEMM) resulting in $r(2m_i - 1)$ floating point operations (FLOPs). For s-cLA, we store the columns $\{c_1,\cdots,c_r\}$ in memory when $A$ is constructed, then directly construct $[x_{c_1},\cdots,x_{c_r}]^\top.$ See Figure~\ref{fig:naive-sparse-design} for a comparative visualization of the two methods.

\subsubsection{Experiments} To showcase the benefit of our sparse implementations, we fine-tune ViT-Tiny and ViT-Base on the OfficeHome and CIFAR-10 datasets, and RoBERTa-Base and RoBERTa-Large on the MRPC and CoLA datasets using full fine-tuning, some base LoRA variants, the sparsity-induced variants and their more optimized counterparts, and PaCA with rank $r = 16$ for $30$ epochs, averaged over three seeds on a single NVIDIA H100 PCIe GPU. To align our experiments with those done in PaCA~\cite{paca}, we adapt all layers of the models, and fully fine-tune the classification heads, then report the throughput, runtime, and trainable parameters in Table~\ref{tab:combined_metrics_all}. To align with the other experiments in our paper, as well as LoRA~\cite{lora}, we adapt only the query and value matrices of each attention head, fully fine-tune the classification heads of the models, and report those results in Table~\ref{tab:combined_metrics_qv}.

\smartparagraph{Discussion of experiments.} When adapting all layers of the models, PaCA is substantially faster than LoRA. In Table~\ref{tab:combined_metrics_all}, the runtime for fine-tuning via PaCA, when normalized to LoRA's time, often is 24 to 25\% faster than LoRA and reduces peak GPU memory by 15-38\%, while the optimized for sparsity LoRA variants (columns s-cLA and s-r-c$^3$LA) are 10 to 15\% faster and reduce peak GPU memory by 15-40\% compared to LoRA. However, due to the overhead of adapting this many layers, FFT was consistently faster than all other methods. In contrast, when adapting far fewer of the model's layers, full fine-tuning was far less competitive. In Table~\ref{tab:combined_metrics_qv}, PaCA is often the fastest method at around 10-11\% faster than LoRA while often reducing peak GPU memory the most by 5-15\%, with the optimized for sparsity LoRA variants generally at 8-10\% faster than LoRA while reducing peak GPU memory almost as much as PaCA, ranging from 5-15\% as well.

\smartparagraph{Key Takeaway.} Given the minimal accuracy gaps between each PEFT method for adapting all of the layers in Table~\ref{tab:vit_acc_all} and adapting only query, value, and classification head in Table~\ref{tab:vit_acc_qv}, from a speed perspective, it is optimal to use sparsity-induced LoRA variants (and thus their PaCA counterparts) for fine-tuning models as shown by Table~\ref{tab:combined_metrics_qv}. However, this is only if the chosen LoRA rank is sufficiently high to ensure the model can adapt to the dataset. For example, when adapting DeepSeekCoder to the DJANGO dataset, we witness in Table~\ref{tab:full-accuracy-table} that $r = 16$ was insufficient for the sparsity-induced methods (the future) to perform comparably to their non-sparse counterparts (the present). This difference disappears as the rank is increased from $r = 16$ to $r = 64$; see Figure~\ref{fig:cla_deepseek}.

\begin{figure*}
     \scalebox{1}{\includegraphics[width=\textwidth]{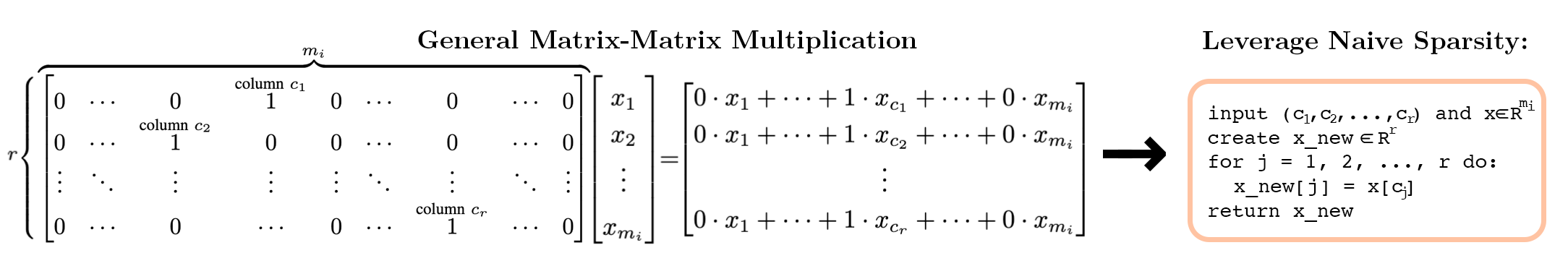}}
  \caption{{\textbf{A comparative visualization of the general matrix-matrix multiplication in r-cLA (left) and the \texttt{gather} (indexing) operation used in s-r-cLA (right).} The \texttt{gather} operation~\cite{gatheroperation} saves $r(2m_i - 1)$ FLOPs for each input $x$ in each adapted layer of the forward pass, at the cost of creating minor memory overhead.}}\label{fig:naive-sparse-design}
  \end{figure*}

\begin{table*}
  \centering
  \tiny
  \setlength{\tabcolsep}{5pt}
  \renewcommand{\arraystretch}{1.25}
\caption{\small{Throughput, runtime, and trainable parameter count, and peak allocated GPU memory for fine-tuning various text and vision models on various datasets using various PEFT methods, including LoRA, PaCA, and their connecting variants, averaged over three seeds (0,1,2). We adapt the token and map embeddings, query, key, value, and output matrices of the attention layers, and both fully connected layers. PaCA reduces peak GPU memory when training ViT-Base on CIFAR-10 by 31\% compared to training via LoRA. This is similar to the results in PaCA~\cite{paca}. While not specifically designed for reducing peak GPU memory, our sparsity-induced LoRA variants produce similar results, reducing the peak GPU memory by 33\% compared to LoRA. }}
\label{tab:combined_metrics_all}
  \vspace{-2mm}
  \scalebox{0.65}{
  \begin{tabular}{@{} l l l c c c c c c c c c c @{}}
    \toprule
    \multirow{2}{*}{\textbf{Datatype}} & \multirow{2}{*}{\textbf{Model}} & \multirow{2}{*}{\textbf{Dataset}} & \multicolumn{1}{c}{} & \multicolumn{3}{c}{\textbf{LoRA Variants}} & \multicolumn{4}{c}{\textbf{Sparsity-Induced Variants}} & \multicolumn{2}{c}{\textbf{PaCA Variants}} \\
    \cmidrule(lr){5-7} \cmidrule(lr){8-11} \cmidrule(lr){12-13}
     &  &  & FFT & LoRA & CoLA & Asym & r-cLA & r-c$^3$LA & s-cLA & s-r-c$^3$LA & PaCA & C-PaCA \\
    \midrule
    \multirow{8}{*}{\shortstack{Throughput\\(Samples\textbackslash s)}} & \multirow{2}{*}{ViT-Tiny} & OfficeHome & 344.4 (1.068) & 322.3 (1.000) & 343.7 (1.066) & 345.3 (1.071) & 346.8 (1.076) & 344.5 (1.069) & \textcolor{blue}{348.6 (1.081)} & 336.7 (1.044) & \textcolor{red}{351.2 (1.090)} & \textcolor{codegreen}{353.0 (1.095)} \\
    \cmidrule(lr){3-13}
     &  & CIFAR-10 & \textcolor{codegreen}{1064.2 (1.502)} & 708.3 (1.000) & 713.5 (1.007) & 713.7 (1.008) & 737.3 (1.041) & 747.2 (1.055) & 782.2 (1.104) & 812.2 (1.147) & \textcolor{blue}{939.2 (1.326)} & \textcolor{red}{944.0 (1.333)} \\
    \cmidrule(lr){2-13}
     & \multirow{2}{*}{ViT-Base} & OfficeHome & \textcolor{codegreen}{362.2 (1.081)} & 335.0 (1.000) & 335.5 (1.001) & 343.5 (1.025) & 341.6 (1.020) & 345.9 (1.032) & 349.0 (1.042) & 346.7 (1.035) & \textcolor{red}{358.0 (1.069)} & \textcolor{blue}{354.2 (1.057)} \\
    \cmidrule(lr){3-13}
     &  & CIFAR-10 & \textcolor{codegreen}{1056.3 (1.530)} & 690.3 (1.000) & 689.0 (0.998) & 703.9 (1.020) & 750.1 (1.087) & 744.2 (1.078) & 788.4 (1.142) & 769.7 (1.115) & \textcolor{red}{930.4 (1.348)} & \textcolor{blue}{859.2 (1.245)} \\
    \cmidrule(lr){2-13}
     & \multirow{2}{*}{RoBERTa-Base} & MRPC & \textcolor{codegreen}{808.4 (1.636)} & 494.2 (1.000) & 493.1 (0.998) & 546.7 (1.106) & 554.3 (1.121) & 556.3 (1.126) & 588.0 (1.190) & 585.5 (1.185) & \textcolor{red}{648.1 (1.311)} & \textcolor{blue}{643.7 (1.302)} \\
    \cmidrule(lr){3-13}
     &  & CoLA & \textcolor{codegreen}{817.6 (1.710)} & 478.0 (1.000) & 498.8 (1.043) & 554.3 (1.160) & 559.1 (1.169) & 556.1 (1.163) & 587.9 (1.230) & 595.2 (1.245) & \textcolor{blue}{632.4 (1.323)} & \textcolor{red}{648.5 (1.357)} \\
    \cmidrule(lr){2-13}
     & \multirow{2}{*}{RoBERTa-Large} & MRPC & \textcolor{codegreen}{391.3 (1.536)} & 254.8 (1.000) & 254.1 (0.997) & 270.0 (1.060) & 271.7 (1.066) & 283.6 (1.113) & 305.6 (1.199) & 305.4 (1.199) & \textcolor{red}{336.7 (1.322)} & \textcolor{blue}{332.0 (1.303)} \\
    \cmidrule(lr){3-13}
     &  & CoLA & \textcolor{codegreen}{394.3 (1.533)} & 257.2 (1.000) & 255.4 (0.993) & 287.5 (1.118) & 288.3 (1.121) & 288.2 (1.121) & 306.4 (1.192) & 309.3 (1.203) & \textcolor{blue}{342.1 (1.330)} & \textcolor{red}{343.8 (1.337)} \\
    \midrule
    \multirow{8}{*}{\shortstack{Runtime\\(Minutes)}} & \multirow{2}{*}{ViT-Tiny} & OfficeHome & 53.2 (0.933) & 57.1 (1.000) & 53.3 (0.935) & 52.9 (0.927) & 52.9 (0.927) & 53.0 (0.928) & \textcolor{red}{52.4 (0.919)} & 54.5 (0.954) & \textcolor{blue}{52.7 (0.924)} & \textcolor{codegreen}{51.9 (0.910)} \\
    \cmidrule(lr){3-13}
     &  & CIFAR-10 & \textcolor{codegreen}{23.1 (0.669)} & 34.5 (1.000) & 34.5 (0.999) & 34.6 (1.003) & 32.8 (0.951) & 32.9 (0.952) & 31.5 (0.913) & 30.4 (0.880) & \textcolor{blue}{26.1 (0.757)} & \textcolor{red}{26.0 (0.753)} \\
    \cmidrule(lr){2-13}
     & \multirow{2}{*}{ViT-Base} & OfficeHome & \textcolor{codegreen}{50.6 (0.932)} & 54.3 (1.000) & 54.2 (0.999) & 53.1 (0.978) & 53.2 (0.980) & 52.8 (0.972) & 52.3 (0.963) & 52.6 (0.970) & \textcolor{red}{50.9 (0.937)} & \textcolor{blue}{51.4 (0.947)} \\
    \cmidrule(lr){3-13}
     &  & CIFAR-10 & \textcolor{codegreen}{23.3 (0.664)} & 35.1 (1.000) & 35.4 (1.008) & 34.5 (0.983) & 32.3 (0.921) & 32.7 (0.932) & 31.1 (0.886) & 31.7 (0.903) & \textcolor{red}{26.3 (0.751)} & \textcolor{blue}{28.7 (0.817)} \\
    \cmidrule(lr){2-13}
     & \multirow{2}{*}{RoBERTa-Base} & MRPC & \textcolor{codegreen}{2.5 (0.619)} & 4.0 (1.000) & 4.0 (1.000) & 3.6 (0.907) & 3.6 (0.905) & 3.6 (0.901) & \textcolor{blue}{3.4 (0.849)} & \textcolor{blue}{3.4 (0.854)} & \textcolor{red}{3.1 (0.768)} & \textcolor{red}{3.1 (0.769)} \\
    \cmidrule(lr){3-13}
     &  & CoLA & \textcolor{codegreen}{5.5 (0.588)} & 9.3 (1.000) & 8.9 (0.964) & 8.0 (0.867) & 8.1 (0.873) & 8.0 (0.864) & 7.7 (0.826) & 7.5 (0.812) & \textcolor{blue}{7.0 (0.760)} & \textcolor{red}{6.9 (0.749)} \\
    \cmidrule(lr){2-13}
     & \multirow{2}{*}{RoBERTa-Large} & MRPC & \textcolor{codegreen}{5.1 (0.652)} & 7.8 (1.000) & 7.8 (1.003) & 7.4 (0.954) & 7.4 (0.949) & 7.1 (0.906) & 6.6 (0.842) & 6.6 (0.844) & \textcolor{red}{5.9 (0.759)} & \textcolor{blue}{6.0 (0.766)} \\
    \cmidrule(lr){3-13}
     &  & CoLA & \textcolor{codegreen}{11.3 (0.653)} & 17.3 (1.000) & 17.4 (1.004) & 15.6 (0.898) & 15.5 (0.895) & 15.6 (0.898) & \textcolor{blue}{14.5 (0.839)} & \textcolor{blue}{14.5 (0.837)} & \textcolor{red}{13.0 (0.752)} & \textcolor{red}{13.0 (0.750)} \\
    \midrule
    \multirow{8}{*}{\shortstack{Trainable\\Parameters}} & \multirow{2}{*}{ViT-Tiny} & OfficeHome & \ensuremath{5.7e^{6}} (6.675) & \ensuremath{8.6e^{5}} (1.000) & \ensuremath{8.6e^{5}} (1.000) & \ensuremath{5.2e^{5}} (0.613) & \ensuremath{5.2e^{5}} (0.613) & \ensuremath{5.2e^{5}} (0.613) & \ensuremath{5.2e^{5}} (0.613) & \ensuremath{5.2e^{5}} (0.613) & \ensuremath{5.2e^{5}} (0.613) & \ensuremath{5.2e^{5}} (0.613) \\
    \cmidrule(lr){3-13}
     &  & CIFAR-10 & \ensuremath{5.5e^{6}} (8.304) & \ensuremath{6.7e^{5}} (1.000) & \ensuremath{6.7e^{5}} (1.000) & \ensuremath{3.3e^{5}} (0.501) & \ensuremath{3.3e^{5}} (0.501) & \ensuremath{3.3e^{5}} (0.501) & \ensuremath{3.3e^{5}} (0.501) & \ensuremath{3.3e^{5}} (0.501) & \ensuremath{3.3e^{5}} (0.501) & \ensuremath{3.3e^{5}} (0.501) \\
    \cmidrule(lr){2-13}
     & \multirow{2}{*}{ViT-Base} & OfficeHome & \ensuremath{8.7e^{7}} (25.288) & \ensuremath{3.4e^{6}} (1.000) & \ensuremath{3.4e^{6}} (1.000) & \ensuremath{2.1e^{6}} (0.612) & \ensuremath{2.1e^{6}} (0.612) & \ensuremath{2.1e^{6}} (0.612) & \ensuremath{2.1e^{6}} (0.612) & \ensuremath{2.1e^{6}} (0.612) & \ensuremath{2.1e^{6}} (0.612) & \ensuremath{2.1e^{6}} (0.612) \\
    \cmidrule(lr){3-13}
     &  & CIFAR-10 & \ensuremath{8.6e^{7}} (32.235) & \ensuremath{2.7e^{6}} (1.000) & \ensuremath{2.7e^{6}} (1.000) & \ensuremath{1.3e^{6}} (0.501) & \ensuremath{1.3e^{6}} (0.501) & \ensuremath{1.3e^{6}} (0.501) & \ensuremath{1.3e^{6}} (0.501) & \ensuremath{1.3e^{6}} (0.501) & \ensuremath{1.3e^{6}} (0.501) & \ensuremath{1.3e^{6}} (0.501) \\
    \cmidrule(lr){2-13}
     & \multirow{2}{*}{RoBERTa-Base} & MRPC & \ensuremath{1.2e^{8}} (38.396) & \ensuremath{3.2e^{6}} (1.000) & \ensuremath{3.2e^{6}} (1.000) & \ensuremath{1.9e^{6}} (0.591) & \ensuremath{1.9e^{6}} (0.591) & \ensuremath{1.9e^{6}} (0.591) & \ensuremath{1.9e^{6}} (0.591) & \ensuremath{1.9e^{6}} (0.591) & \ensuremath{1.9e^{6}} (0.591) & \ensuremath{1.9e^{6}} (0.591) \\
    \cmidrule(lr){3-13}
     &  & CoLA & \ensuremath{1.2e^{8}} (38.396) & \ensuremath{3.2e^{6}} (1.000) & \ensuremath{3.2e^{6}} (1.000) & \ensuremath{1.9e^{6}} (0.591) & \ensuremath{1.9e^{6}} (0.591) & \ensuremath{1.9e^{6}} (0.591) & \ensuremath{1.9e^{6}} (0.591) & \ensuremath{1.9e^{6}} (0.591) & \ensuremath{1.9e^{6}} (0.591) & \ensuremath{1.9e^{6}} (0.591) \\
    \cmidrule(lr){2-13}
     & \multirow{2}{*}{RoBERTa-Large} & MRPC & \ensuremath{3.6e^{8}} (43.712) & \ensuremath{8.1e^{6}} (1.000) & \ensuremath{8.1e^{6}} (1.000) & \ensuremath{4.6e^{6}} (0.565) & \ensuremath{4.6e^{6}} (0.565) & \ensuremath{4.6e^{6}} (0.565) & \ensuremath{4.6e^{6}} (0.565) & \ensuremath{4.6e^{6}} (0.565) & \ensuremath{4.6e^{6}} (0.565) & \ensuremath{4.6e^{6}} (0.565) \\
    \cmidrule(lr){3-13}
     &  & CoLA & \ensuremath{3.6e^{8}} (43.712) & \ensuremath{8.1e^{6}} (1.000) & \ensuremath{8.1e^{6}} (1.000) & \ensuremath{4.6e^{6}} (0.565) & \ensuremath{4.6e^{6}} (0.565) & \ensuremath{4.6e^{6}} (0.565) & \ensuremath{4.6e^{6}} (0.565) & \ensuremath{4.6e^{6}} (0.565) & \ensuremath{4.6e^{6}} (0.565) & \ensuremath{4.6e^{6}} (0.565) \\
    \midrule
    \multirow{8}{*}{\shortstack{Peak GPU\\Memory (GB)}} & \multirow{2}{*}{ViT-Tiny} & OfficeHome & 2.22 (1.015) & \textcolor{blue}{2.19 (1.000)} & \textcolor{blue}{2.19 (1.000)} & \textcolor{red}{1.72 (0.786)} & \textcolor{red}{1.72 (0.786)} & \textcolor{red}{1.72 (0.786)} & \textcolor{codegreen}{1.68 (0.769)} & \textcolor{codegreen}{1.68 (0.769)} & \textcolor{codegreen}{1.68 (0.770)} & \textcolor{codegreen}{1.68 (0.770)} \\
    \cmidrule(lr){3-13}
     &  & CIFAR-10 & 2.29 (1.060) & 2.16 (1.000) & 2.16 (1.000) & \textcolor{red}{1.59 (0.738)} & \textcolor{red}{1.59 (0.738)} & \textcolor{red}{1.59 (0.738)} & \textcolor{codegreen}{1.57 (0.728)} & \textcolor{codegreen}{1.57 (0.729)} & \textcolor{blue}{1.60 (0.744)} & \textcolor{blue}{1.60 (0.744)} \\
    \cmidrule(lr){2-13}
     & \multirow{2}{*}{ViT-Base} & OfficeHome & 7.29 (0.887) & 8.22 (1.000) & 8.22 (1.000) & \textcolor{red}{5.02 (0.611)} & \textcolor{red}{5.02 (0.611)} & \textcolor{red}{5.02 (0.611)} & \textcolor{codegreen}{4.88 (0.594)} & \textcolor{codegreen}{4.88 (0.594)} & \textcolor{blue}{5.15 (0.627)} & \textcolor{blue}{5.15 (0.627)} \\
    \cmidrule(lr){3-13}
     &  & CIFAR-10 & 7.05 (1.061) & 6.65 (1.000) & 6.65 (1.000) & \textcolor{red}{4.49 (0.676)} & \textcolor{red}{4.49 (0.676)} & \textcolor{red}{4.49 (0.676)} & \textcolor{codegreen}{4.48 (0.674)} & \textcolor{codegreen}{4.48 (0.674)} & \textcolor{blue}{4.61 (0.694)} & 4.62 (0.695) \\
    \cmidrule(lr){2-13}
     & \multirow{2}{*}{RoBERTa-Base} & MRPC & 3.82 (1.433) & 2.66 (1.000) & 2.76 (1.036) & \textcolor{blue}{2.22 (0.835)} & \textcolor{blue}{2.22 (0.835)} & \textcolor{blue}{2.22 (0.835)} & \textcolor{red}{2.20 (0.827)} & \textcolor{red}{2.20 (0.827)} & \textcolor{codegreen}{2.08 (0.783)} & \textcolor{codegreen}{2.08 (0.783)} \\
    \cmidrule(lr){3-13}
     &  & CoLA & 3.34 (1.632) & 2.05 (1.000) & 2.05 (1.000) & 1.89 (0.924) & 1.80 (0.877) & 1.80 (0.877) & \textcolor{blue}{1.77 (0.866)} & 1.82 (0.891) & \textcolor{red}{1.73 (0.847)} & \textcolor{codegreen}{1.64 (0.800)} \\
    \cmidrule(lr){2-13}
     & \multirow{2}{*}{RoBERTa-Large} & MRPC & 9.34 (1.508) & 6.19 (1.000) & 6.19 (1.000) & 4.98 (0.804) & \textcolor{blue}{4.85 (0.784)} & 4.91 (0.794) & \textcolor{red}{4.84 (0.782)} & \textcolor{red}{4.84 (0.782)} & \textcolor{codegreen}{4.30 (0.694)} & \textcolor{codegreen}{4.30 (0.694)} \\
    \cmidrule(lr){3-13}
     &  & CoLA & 7.93 (1.782) & 4.45 (1.000) & 4.45 (1.000) & \textcolor{blue}{3.76 (0.846)} & \textcolor{blue}{3.76 (0.846)} & \textcolor{blue}{3.76 (0.846)} & 3.81 (0.857) & 3.85 (0.865) & \textcolor{codegreen}{3.19 (0.718)} & \textcolor{red}{3.34 (0.751)} \\
    \bottomrule
  \end{tabular}}
  \vspace{-2mm}
\end{table*}

\begin{table*}
  \centering
  \tiny
  \setlength{\tabcolsep}{5pt}
  \renewcommand{\arraystretch}{1.25}
\caption{\small{Throughput, runtime, and trainable parameter count, and peak allocated GPU memory for fine-tuning various text and vision models on various datasets using various PEFT methods, including LoRA, PaCA, and their connecting variants, averaged over three seeds (0,1,2). We adapt the query and value matrices of the attention layers as well as the classification head. When adapting a small enough number of layers such that the PEFT methods are faster than FFT, PaCA reduces peak GPU memory by around 5--15\% compared to LoRA, which is only a 1.4\% greater reduction on average than the sparsity-induced LoRA variants.}}
\label{tab:combined_metrics_qv}
  \vspace{0mm}
  \scalebox{0.65}{
  \begin{tabular}{@{} l l l c c c c c c c c c c @{}}
    \toprule
    \multirow{2}{*}{\textbf{Datatype}} & \multirow{2}{*}{\textbf{Model}} & \multirow{2}{*}{\textbf{Dataset}} & \multicolumn{1}{c}{} & \multicolumn{3}{c}{\textbf{LoRA Variants}} & \multicolumn{4}{c}{\textbf{Sparsity-Induced Variants}} & \multicolumn{2}{c}{\textbf{PaCA Variants}} \\
    \cmidrule(lr){5-7} \cmidrule(lr){8-11} \cmidrule(lr){12-13}
     &  &  & FFT & LoRA & CoLA & Asym & r-cLA & r-c$^3$LA & s-cLA & s-r-c$^3$LA & PaCA & C-PaCA \\
    \midrule
    \multirow{8}{*}{\shortstack{Throughput\\(Samples\textbackslash s)}} & \multirow{2}{*}{ViT-Tiny} & OfficeHome & 360.4 (1.004) & 358.8 (1.000) & \textcolor{blue}{361.6 (1.008)} & 361.5 (1.007) & 359.9 (1.003) & 359.4 (1.002) & 358.3 (0.998) & \textcolor{red}{362.1 (1.009)} & 344.3 (0.959) & \textcolor{codegreen}{364.1 (1.015)} \\
    \cmidrule(lr){3-13}
     &  & CIFAR-10 & 1068.6 (1.006) & 1062.4 (1.000) & 1004.0 (0.945) & 1094.2 (1.030) & 1083.4 (1.020) & 1081.9 (1.018) & \textcolor{blue}{1111.8 (1.046)} & 1109.5 (1.044) & \textcolor{codegreen}{1159.5 (1.091)} & \textcolor{red}{1139.6 (1.073)} \\
    \cmidrule(lr){2-13}
     & \multirow{2}{*}{ViT-Base} & OfficeHome & \textcolor{codegreen}{362.6 (1.004)} & \textcolor{blue}{361.3 (1.000)} & 357.5 (0.990) & 359.3 (0.994) & 361.2 (1.000) & 360.4 (0.997) & 360.8 (0.999) & \textcolor{red}{362.5 (1.003)} & 351.4 (0.973) & 342.2 (0.947) \\
    \cmidrule(lr){3-13}
     &  & CIFAR-10 & 1060.4 (1.013) & 1047.2 (1.000) & 1031.7 (0.985) & 1077.7 (1.029) & 1070.3 (1.022) & 1061.7 (1.014) & 1038.9 (0.992) & \textcolor{blue}{1104.3 (1.054)} & \textcolor{red}{1132.6 (1.082)} & \textcolor{codegreen}{1139.5 (1.088)} \\
    \cmidrule(lr){2-13}
     & \multirow{2}{*}{RoBERTa-Base} & MRPC & 806.5 (0.974) & 828.3 (1.000) & 725.5 (0.876) & 875.6 (1.057) & 829.3 (1.001) & 880.3 (1.063) & \textcolor{blue}{920.9 (1.112)} & 910.6 (1.099) & \textcolor{codegreen}{942.0 (1.137)} & \textcolor{red}{930.7 (1.124)} \\
    \cmidrule(lr){3-13}
     &  & CoLA & 816.1 (1.125) & 725.5 (1.000) & 834.8 (1.151) & \textcolor{blue}{894.2 (1.233)} & 829.1 (1.143) & 834.3 (1.150) & \textcolor{red}{929.0 (1.281)} & \textcolor{codegreen}{929.8 (1.282)} & 765.3 (1.055) & 875.3 (1.206) \\
    \cmidrule(lr){2-13}
     & \multirow{2}{*}{RoBERTa-Large} & MRPC & 387.7 (0.888) & 436.6 (1.000) & 437.8 (1.003) & 466.9 (1.069) & 468.9 (1.074) & 439.3 (1.006) & \textcolor{red}{483.6 (1.108)} & 453.9 (1.040) & \textcolor{codegreen}{499.7 (1.144)} & \textcolor{blue}{480.8 (1.101)} \\
    \cmidrule(lr){3-13}
     &  & CoLA & 394.7 (0.893) & 441.9 (1.000) & 417.9 (0.946) & 468.1 (1.059) & 468.1 (1.059) & \textcolor{blue}{471.2 (1.066)} & 463.1 (1.048) & 454.3 (1.028) & \textcolor{red}{480.3 (1.087)} & \textcolor{codegreen}{506.0 (1.145)} \\
    \midrule
    \multirow{8}{*}{\shortstack{Runtime\\(Minutes)}} & \multirow{2}{*}{ViT-Tiny} & OfficeHome & 50.6 (0.989) & 51.2 (1.000) & \textcolor{blue}{50.5 (0.986)} & \textcolor{red}{50.4 (0.984)} & 51.0 (0.996) & 50.6 (0.989) & 50.8 (0.992) & \textcolor{codegreen}{50.3 (0.983)} & 53.8 (1.050) & \textcolor{codegreen}{50.3 (0.983)} \\
    \cmidrule(lr){3-13}
     &  & CIFAR-10 & 23.0 (0.992) & 23.2 (1.000) & 24.7 (1.067) & 22.4 (0.969) & 22.7 (0.979) & 22.7 (0.980) & \textcolor{blue}{22.2 (0.956)} & \textcolor{blue}{22.2 (0.959)} & \textcolor{codegreen}{21.3 (0.919)} & \textcolor{red}{21.5 (0.928)} \\
    \cmidrule(lr){2-13}
     & \multirow{2}{*}{ViT-Base} & OfficeHome & \textcolor{codegreen}{50.3 (0.999)} & \textcolor{red}{50.4 (1.000)} & 51.0 (1.012) & 50.8 (1.008) & \textcolor{blue}{50.6 (1.004)} & \textcolor{blue}{50.6 (1.005)} & \textcolor{red}{50.4 (1.000)} & \textcolor{red}{50.4 (1.000)} & 52.7 (1.047) & 53.3 (1.058) \\
    \cmidrule(lr){3-13}
     &  & CIFAR-10 & 23.2 (0.987) & 23.5 (1.000) & 23.5 (0.999) & 22.8 (0.969) & 22.9 (0.976) & 22.9 (0.974) & 23.8 (1.011) & \textcolor{blue}{22.2 (0.946)} & \textcolor{red}{21.7 (0.922)} & \textcolor{codegreen}{21.6 (0.918)} \\
    \cmidrule(lr){2-13}
     & \multirow{2}{*}{RoBERTa-Base} & MRPC & \textcolor{blue}{2.5 (1.016)} & \textcolor{blue}{2.5 (1.000)} & 2.9 (1.155) & \textcolor{red}{2.3 (0.941)} & \textcolor{blue}{2.5 (1.016)} & \textcolor{red}{2.3 (0.950)} & \textcolor{red}{2.3 (0.926)} & \textcolor{red}{2.3 (0.919)} & \textcolor{codegreen}{2.2 (0.890)} & \textcolor{codegreen}{2.2 (0.885)} \\
    \cmidrule(lr){3-13}
     &  & CoLA & 5.5 (0.878) & 6.2 (1.000) & 5.4 (0.864) & \textcolor{red}{5.1 (0.813)} & 5.5 (0.885) & 5.5 (0.878) & \textcolor{codegreen}{4.9 (0.791)} & \textcolor{codegreen}{4.9 (0.786)} & 5.9 (0.951) & \textcolor{blue}{5.2 (0.833)} \\
    \cmidrule(lr){2-13}
     & \multirow{2}{*}{RoBERTa-Large} & MRPC & 5.1 (1.100) & 4.7 (1.000) & 4.7 (1.004) & \textcolor{blue}{4.5 (0.961)} & \textcolor{red}{4.4 (0.952)} & 4.8 (1.037) & \textcolor{codegreen}{4.2 (0.909)} & 4.7 (1.001) & \textcolor{codegreen}{4.2 (0.902)} & \textcolor{red}{4.4 (0.939)} \\
    \cmidrule(lr){3-13}
     &  & CoLA & 11.3 (1.110) & 10.1 (1.000) & 10.8 (1.068) & 9.7 (0.951) & 9.7 (0.953) & \textcolor{blue}{9.6 (0.949)} & 9.9 (0.972) & 9.9 (0.977) & \textcolor{red}{9.5 (0.938)} & \textcolor{codegreen}{9.0 (0.888)} \\
    \midrule
    \multirow{8}{*}{\shortstack{Trainable\\Parameters}} & \multirow{2}{*}{ViT-Tiny} & OfficeHome & \ensuremath{5.7e^{6}} (16.793) & \ensuremath{3.4e^{5}} (1.000) & \ensuremath{3.4e^{5}} (1.000) & \ensuremath{2.7e^{5}} (0.783) & \ensuremath{2.7e^{5}} (0.783) & \ensuremath{2.7e^{5}} (0.783) & \ensuremath{2.7e^{5}} (0.783) & \ensuremath{2.7e^{5}} (0.783) & \ensuremath{2.7e^{5}} (0.783) & \ensuremath{2.7e^{5}} (0.783) \\
    \cmidrule(lr){3-13}
     &  & CIFAR-10 & \ensuremath{5.5e^{6}} (36.994) & \ensuremath{1.5e^{5}} (1.000) & \ensuremath{1.5e^{5}} (1.000) & \ensuremath{7.6e^{4}} (0.506) & \ensuremath{7.6e^{4}} (0.506) & \ensuremath{7.6e^{4}} (0.506) & \ensuremath{7.6e^{4}} (0.506) & \ensuremath{7.6e^{4}} (0.506) & \ensuremath{7.6e^{4}} (0.506) & \ensuremath{7.6e^{4}} (0.506) \\
    \cmidrule(lr){2-13}
     & \multirow{2}{*}{ViT-Base} & OfficeHome & \ensuremath{8.7e^{7}} (63.708) & \ensuremath{1.4e^{6}} (1.000) & \ensuremath{1.4e^{6}} (1.000) & \ensuremath{1.1e^{6}} (0.783) & \ensuremath{1.1e^{6}} (0.783) & \ensuremath{1.1e^{6}} (0.783) & \ensuremath{1.1e^{6}} (0.783) & \ensuremath{1.1e^{6}} (0.783) & \ensuremath{1.1e^{6}} (0.783) & \ensuremath{1.1e^{6}} (0.783) \\
    \cmidrule(lr){3-13}
     &  & CIFAR-10 & \ensuremath{8.6e^{7}} (143.606) & \ensuremath{6e^{5}} (1.000) & \ensuremath{6e^{5}} (1.000) & \ensuremath{3e^{5}} (0.506) & \ensuremath{3e^{5}} (0.506) & \ensuremath{3e^{5}} (0.506) & \ensuremath{3e^{5}} (0.506) & \ensuremath{3e^{5}} (0.506) & \ensuremath{3e^{5}} (0.506) & \ensuremath{3e^{5}} (0.506) \\
    \cmidrule(lr){2-13}
     & \multirow{2}{*}{RoBERTa-Base} & MRPC & \ensuremath{1.2e^{8}} (105.459) & \ensuremath{1.2e^{6}} (1.000) & \ensuremath{1.2e^{6}} (1.000) & \ensuremath{8.9e^{5}} (0.750) & \ensuremath{8.9e^{5}} (0.750) & \ensuremath{8.9e^{5}} (0.750) & \ensuremath{8.9e^{5}} (0.750) & \ensuremath{8.9e^{5}} (0.750) & \ensuremath{8.9e^{5}} (0.750) & \ensuremath{8.9e^{5}} (0.750) \\
    \cmidrule(lr){3-13}
     &  & CoLA & \ensuremath{1.2e^{8}} (105.459) & \ensuremath{1.2e^{6}} (1.000) & \ensuremath{1.2e^{6}} (1.000) & \ensuremath{8.9e^{5}} (0.750) & \ensuremath{8.9e^{5}} (0.750) & \ensuremath{8.9e^{5}} (0.750) & \ensuremath{8.9e^{5}} (0.750) & \ensuremath{8.9e^{5}} (0.750) & \ensuremath{8.9e^{5}} (0.750) & \ensuremath{8.9e^{5}} (0.750) \\
    \cmidrule(lr){2-13}
     & \multirow{2}{*}{RoBERTa-Large} & MRPC & \ensuremath{3.6e^{8}} (135.401) & \ensuremath{2.6e^{6}} (1.000) & \ensuremath{2.6e^{6}} (1.000) & \ensuremath{1.8e^{6}} (0.700) & \ensuremath{1.8e^{6}} (0.700) & \ensuremath{1.8e^{6}} (0.700) & \ensuremath{1.8e^{6}} (0.700) & \ensuremath{1.8e^{6}} (0.700) & \ensuremath{1.8e^{6}} (0.700) & \ensuremath{1.8e^{6}} (0.700) \\
    \cmidrule(lr){3-13}
     &  & CoLA & \ensuremath{3.6e^{8}} (135.401) & \ensuremath{2.6e^{6}} (1.000) & \ensuremath{2.6e^{6}} (1.000) & \ensuremath{1.8e^{6}} (0.700) & \ensuremath{1.8e^{6}} (0.700) & \ensuremath{1.8e^{6}} (0.700) & \ensuremath{1.8e^{6}} (0.700) & \ensuremath{1.8e^{6}} (0.700) & \ensuremath{1.8e^{6}} (0.700) & \ensuremath{1.8e^{6}} (0.700) \\
    \midrule
    \multirow{8}{*}{\shortstack{Peak GPU\\ Memory (GB)}} & \multirow{2}{*}{ViT-Tiny} & OfficeHome & 2.22 (1.280) & \textcolor{blue}{1.74 (1.000)} & \textcolor{blue}{1.74 (1.000)} & \textcolor{codegreen}{1.58 (0.911)} & \textcolor{codegreen}{1.58 (0.911)} & \textcolor{codegreen}{1.58 (0.911)} & \textcolor{codegreen}{1.58 (0.912)} & \textcolor{codegreen}{1.58 (0.912)} & \textcolor{red}{1.61 (0.927)} & \textcolor{red}{1.61 (0.927)} \\
    \cmidrule(lr){3-13}
     &  & CIFAR-10 & \textcolor{blue}{2.29 (1.340)} & \textcolor{red}{1.71 (1.000)} & \textcolor{red}{1.71 (1.000)} & \textcolor{codegreen}{1.57 (0.918)} & \textcolor{codegreen}{1.57 (0.918)} & \textcolor{codegreen}{1.57 (0.918)} & \textcolor{codegreen}{1.57 (0.918)} & \textcolor{codegreen}{1.57 (0.918)} & \textcolor{codegreen}{1.57 (0.918)} & \textcolor{codegreen}{1.57 (0.918)} \\
    \cmidrule(lr){2-13}
     & \multirow{2}{*}{ViT-Base} & OfficeHome & 7.29 (1.221) & \textcolor{blue}{5.97 (1.000)} & \textcolor{blue}{5.97 (1.000)} & \textcolor{red}{5.12 (0.857)} & \textcolor{red}{5.12 (0.857)} & \textcolor{red}{5.12 (0.857)} & \textcolor{red}{5.12 (0.857)} & \textcolor{red}{5.12 (0.857)} & \textcolor{codegreen}{5.11 (0.856)} & \textcolor{codegreen}{5.11 (0.856)} \\
    \cmidrule(lr){3-13}
     &  & CIFAR-10 & 7.05 (1.437) & \textcolor{blue}{4.91 (1.000)} & \textcolor{blue}{4.91 (1.000)} & \textcolor{red}{4.44 (0.904)} & \textcolor{red}{4.44 (0.904)} & \textcolor{red}{4.44 (0.904)} & \textcolor{red}{4.44 (0.904)} & \textcolor{red}{4.44 (0.905)} & \textcolor{codegreen}{4.40 (0.896)} & \textcolor{codegreen}{4.40 (0.896)} \\
    \cmidrule(lr){2-13}
     & \multirow{2}{*}{RoBERTa-Base} & MRPC & 3.82 (1.658) & 2.30 (1.000) & 2.30 (1.000) & \textcolor{red}{2.18 (0.948)} & \textcolor{red}{2.18 (0.948)} & \textcolor{red}{2.18 (0.948)} & \textcolor{blue}{2.28 (0.990)} & \textcolor{blue}{2.28 (0.990)} & \textcolor{codegreen}{2.16 (0.940)} & \textcolor{codegreen}{2.16 (0.940)} \\
    \cmidrule(lr){3-13}
     &  & CoLA & 3.34 (1.812) & \textcolor{blue}{1.84 (1.000)} & \textcolor{blue}{1.84 (1.000)} & \textcolor{red}{1.76 (0.953)} & \textcolor{red}{1.76 (0.953)} & \textcolor{red}{1.76 (0.954)} & \textcolor{red}{1.76 (0.953)} & \textcolor{red}{1.76 (0.953)} & \textcolor{codegreen}{1.74 (0.943)} & \textcolor{codegreen}{1.74 (0.943)} \\
    \cmidrule(lr){2-13}
     & \multirow{2}{*}{RoBERTa-Large} & MRPC & 9.34 (1.783) & 5.24 (1.000) & 5.17 (0.988) & \textcolor{blue}{4.81 (0.919)} & \textcolor{blue}{4.81 (0.919)} & 4.94 (0.944) & \textcolor{blue}{4.81 (0.919)} & \textcolor{blue}{4.81 (0.919)} & \textcolor{codegreen}{4.69 (0.895)} & \textcolor{red}{4.77 (0.910)} \\
    \cmidrule(lr){3-13}
     &  & CoLA & 7.93 (2.026) & 3.91 (1.000) & 3.91 (1.000) & \textcolor{blue}{3.71 (0.949)} & \textcolor{blue}{3.71 (0.949)} & \textcolor{blue}{3.71 (0.949)} & \textcolor{blue}{3.71 (0.948)} & \textcolor{blue}{3.71 (0.948)} & \textcolor{codegreen}{3.59 (0.917)} & \textcolor{red}{3.66 (0.934)} \\
    \bottomrule
  \end{tabular}}
  \vspace{-2mm}
\end{table*}

\subsection{Performance Analysis---Continued}\label{subsec:geom}
We extend \S\ref{sec:performance analysis} by reporting additional empirical results regarding PEFT methods, including prediction capacity and model behaviors.

\subsubsection{Loss Landscape---Continued}\label{appendix:loss-landscape}

\smartparagraph{3D landscapes.} We obtained the top two principle directions of the model's update path via PCA of the update matrix $[\mathbf{W}^0-\mathbf{W}^T;...;\mathbf{W}^{T-1}-\mathbf{W}^T],$ where $\{\mathbf{W}^t\}_{t=0}^T.$ are the model weight's update steps. Let $\delta,\eta$ be those two directions. For random directions, we generate them via a Gaussian distribution. For LoRA methods, we merged the adapters into the base weights before calculating. We normalize the directions similar to the methods of~\cite{losslandscape}. We plot the function $f(\alpha, \beta):= \cL (\textbf{W} + \alpha\delta + \beta \eta)$ over a $51^2$ grid of $\alpha,\beta$ values uniformly distributed over $[-2,2]\times[-2,2]$, we use mini-batches of size $12$ when finding the values for $\mathcal{L}.$ 

\smartparagraph{Comparison between using random or PCA directions.} To understand the differences between the loss landscapes of the models in the PCA directions compared to random directions, we plotted the loss landscape of ViT-Base fine-tuned on CIFAR-10 in both PCA directions (top) and random directions (bottom) in Figure~\ref{fig:Implementation-comparison}. For random directions, the FFT landscape is substantially smoother; this is consistent with~\cite{losslandscape}, but this is inconsistent with the loss landscapes of RoBERTa-Base with random direction in Figure~\ref{fig:direct-chain-method-comparison}, where chain methods produce spikier landscapes with no substantial change in generalizability.

\smartparagraph{2D landscapes.} The initial setup is identical to the 3D landscape. We obtain the same principal directions and plot the same function. For 2D landscapes, when generating our $\alpha,\beta$ grid of values, we uniformly distribute over $[-m,m]\times[-m,m]$ where $m$ is chosen to ensure the optimizer trajectory (blue arrows) is entirely contained in the image. As shown in Figure~\ref{fig:roberta-2D-landscapes}, chain methods have more diverse loss landscapes than their non-chain counterparts due to their overall update to the pre-trained weights having a higher effective rank~\cite{cola}.

\begin{figure*}
    \includegraphics[width=\textwidth, height=3in]{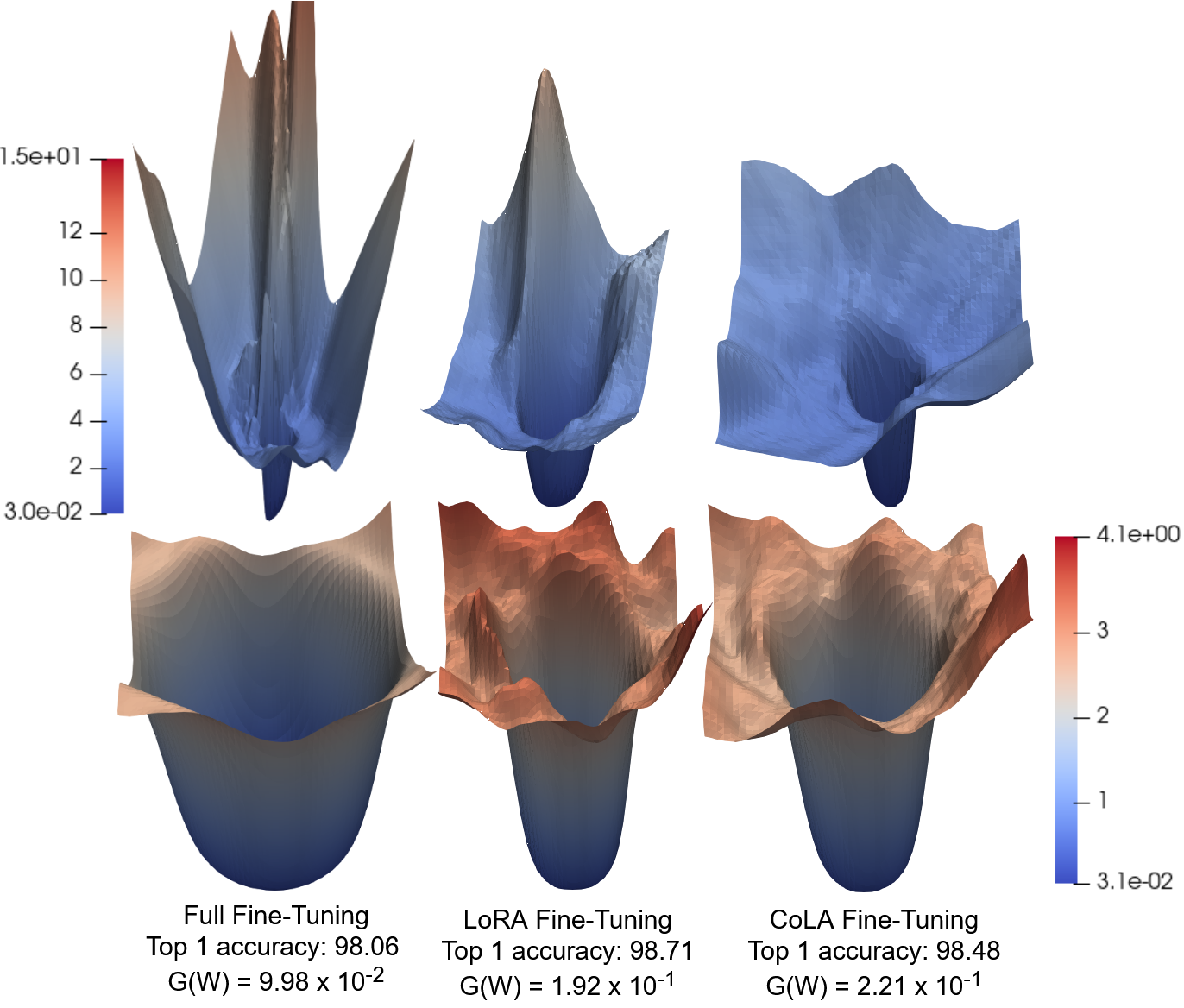}
    \caption{\small{3D loss landscapes of ViT-Base~\cite{dosovitskiy2020vit} pretrained on ImageNet-1K~\cite{deng2009imagenet} and fine-tuned on CIFAR-10~\cite{cifar10-dataset} using the PCA directions of the model's weights updates (top) and random directions (bottom).}}
    \label{fig:Implementation-comparison}
\end{figure*}

\begin{figure*}
    \includegraphics[width=\textwidth, height=3in]{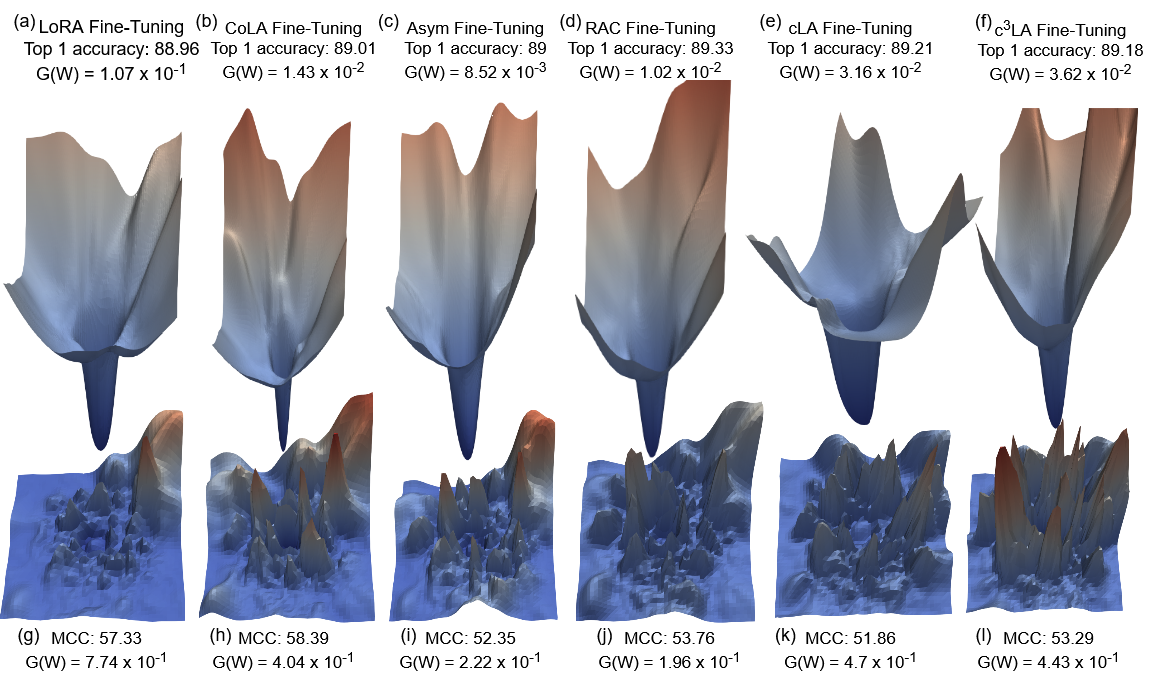}
    \caption{\small{3D loss landscapes of ViT-Base~\cite{dosovitskiy2020vit} pretrained on ImageNet-1K~\cite{deng2009imagenet} and fine-tuned on OfficeHome~\cite{officehome-dataset} (top) and RoBERTa-Base~\cite{roberta} pretrained on a corpus of English text fine-tuned on CoLA~\cite{wang2018glue} (bottom) using the non-chain then chain variants of each LoRA method. The chain variants consistently produce sharper landscapes than the non-chain variants. In asymmetric LoRA methods, this often correlates to worse generalizability, but not in symmetric methods where $B$,$A$ are both trained as shown in~\ref{tab:combined}.}}
    \label{fig:direct-chain-method-comparison}
\end{figure*}

\begin{figure*}
%\vspace{-2ex}
\centering
      \scalebox{0.90}{\centering
	\begin{subfigure}{0.3\linewidth}
		\includegraphics[width=\linewidth]{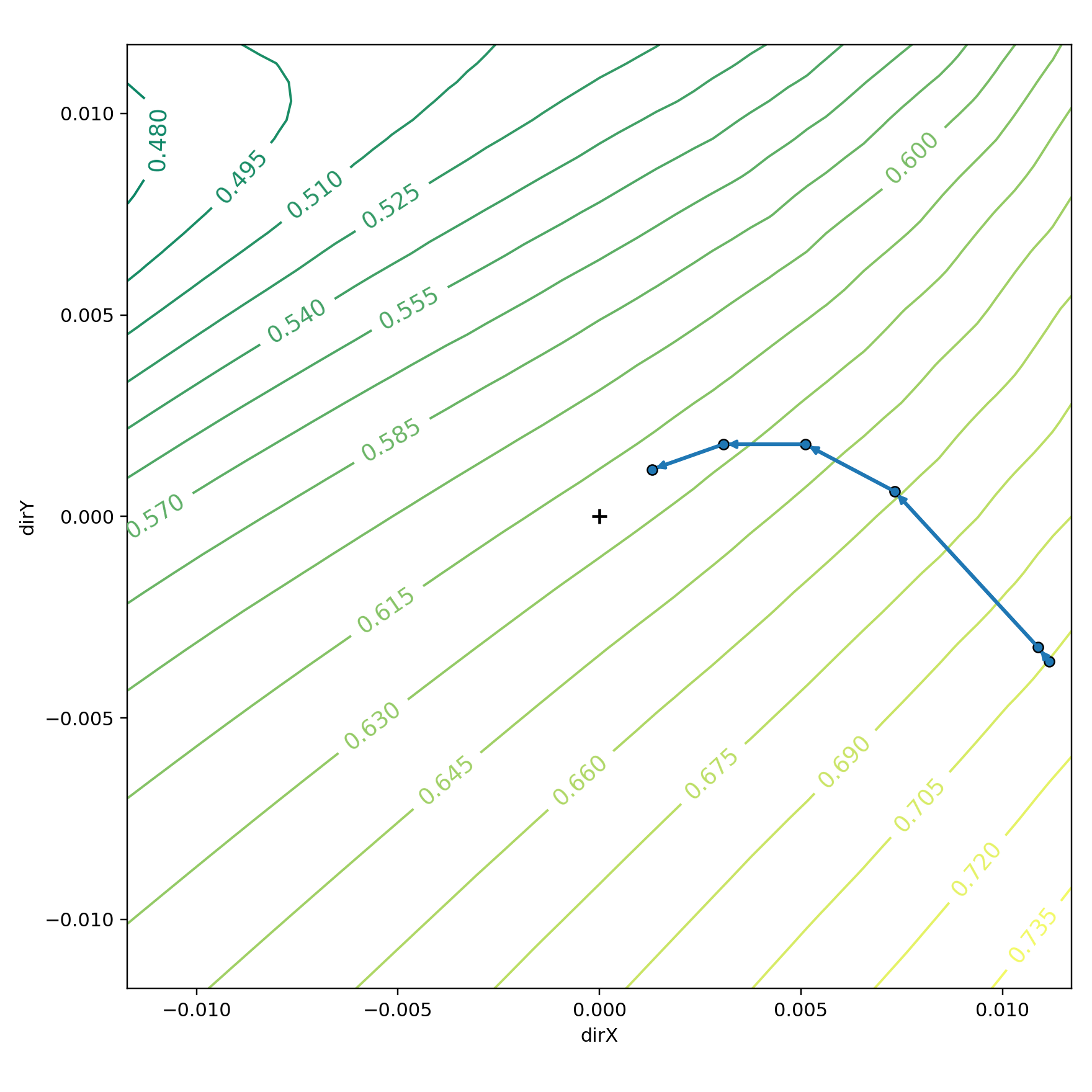}\caption{LoRA}\label{fig:lora}
	\end{subfigure}
	\begin{subfigure}{0.3\linewidth}
		\includegraphics[width=\linewidth]{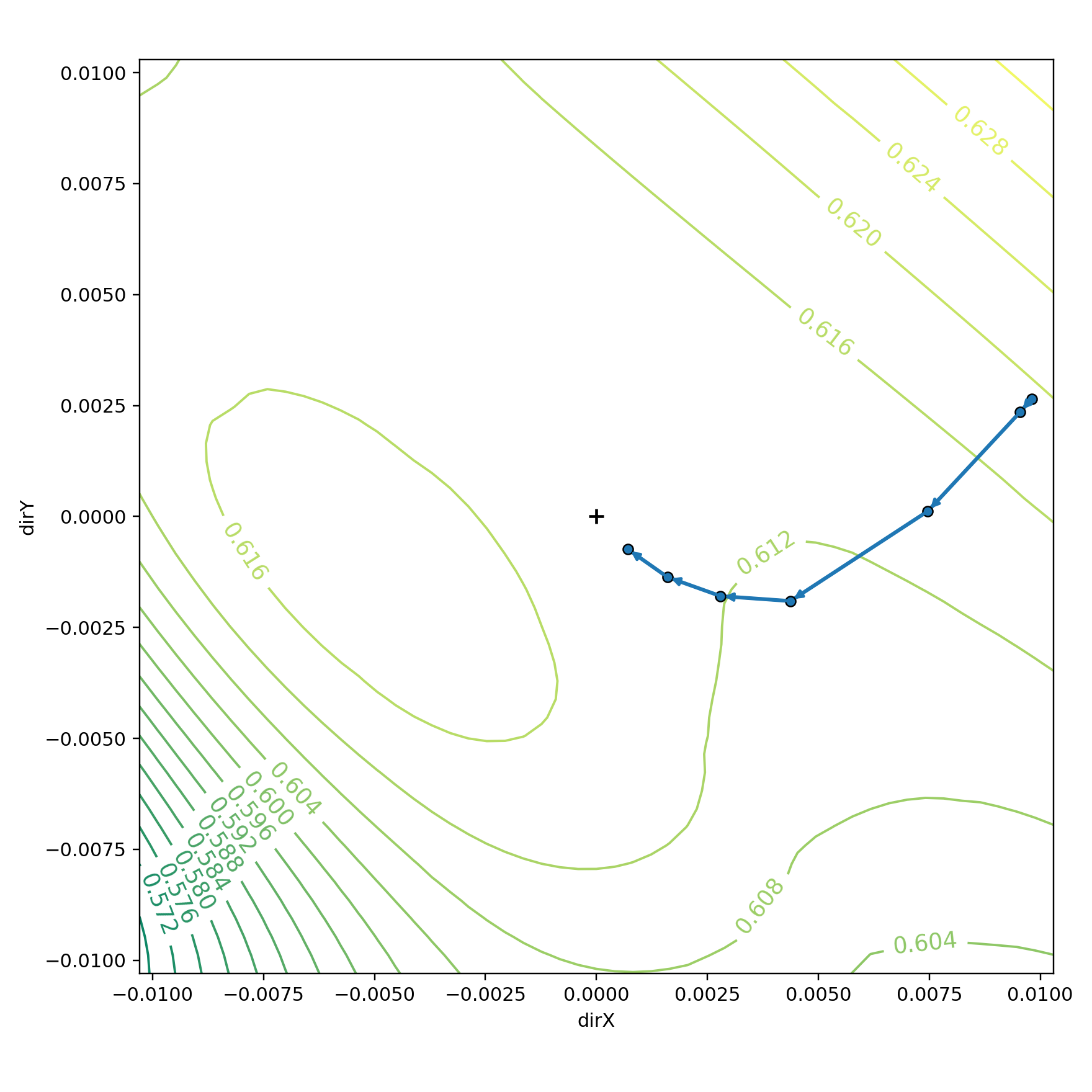}\caption{Asym LoRA}\label{fig:asym-b}
	\end{subfigure}
    \begin{subfigure}{0.3\linewidth}
		\includegraphics[width=\linewidth]{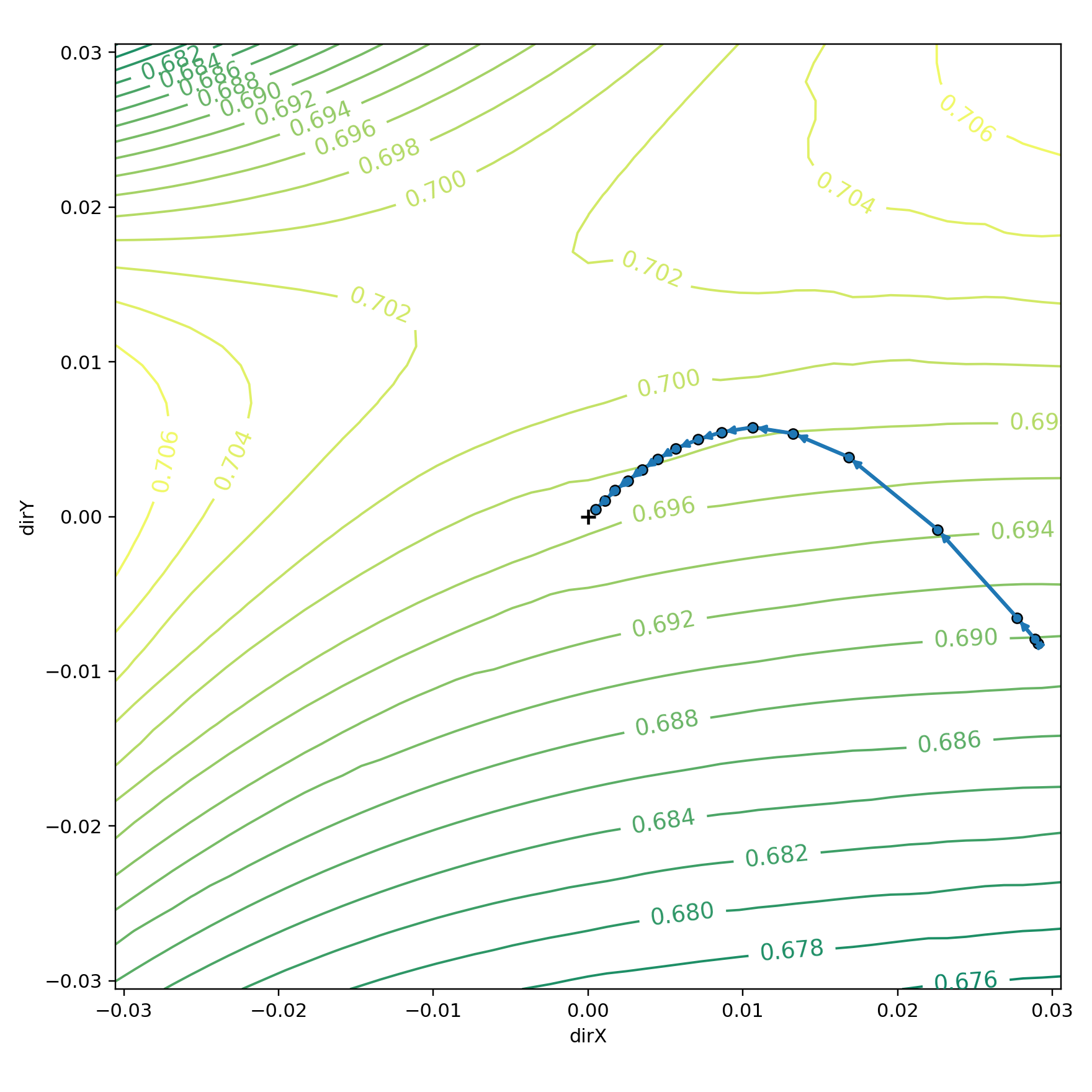}\caption{cLA}\label{fig:cheap}
	\end{subfigure}}
    \scalebox{0.90}{\begin{subfigure}{0.3\linewidth}
		\includegraphics[width=\linewidth]{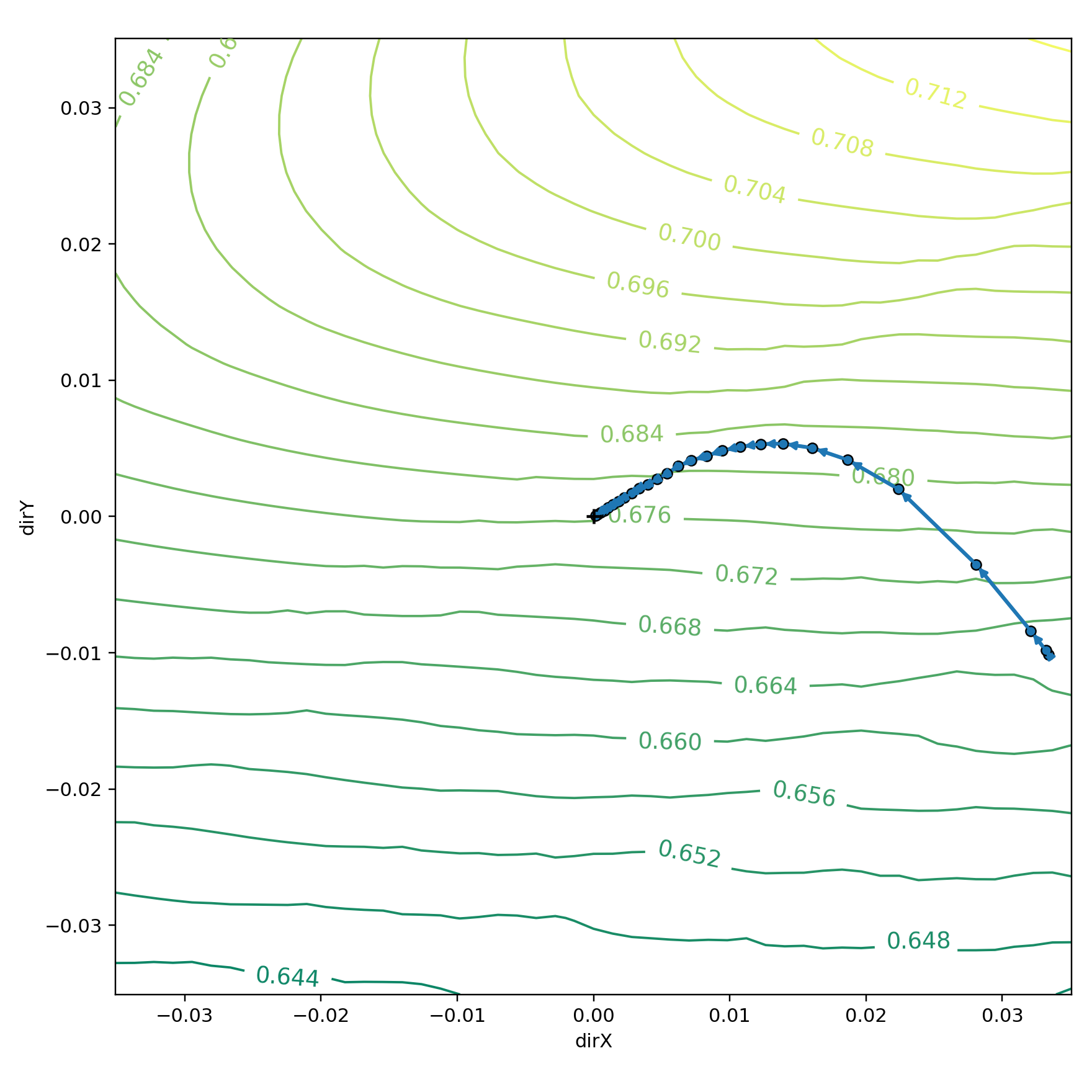}\caption{r-cLA}\label{fig:rac-lora}
	\end{subfigure}
	\begin{subfigure}{0.3\linewidth}
		\includegraphics[width=\linewidth]{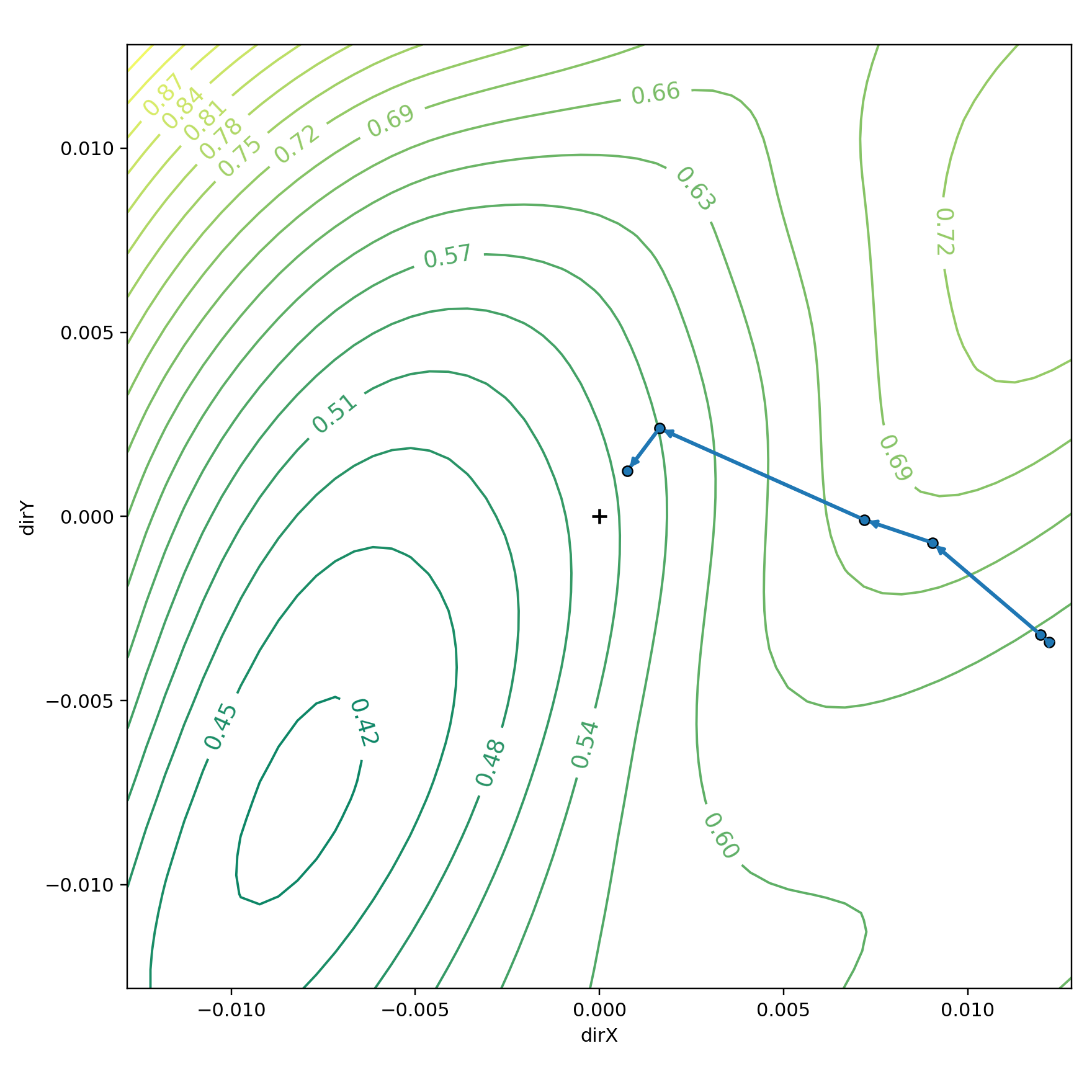}\caption{CoLA}\label{fig:cola}
	\end{subfigure}
    \begin{subfigure}{0.3\linewidth}
		\includegraphics[width=\linewidth]{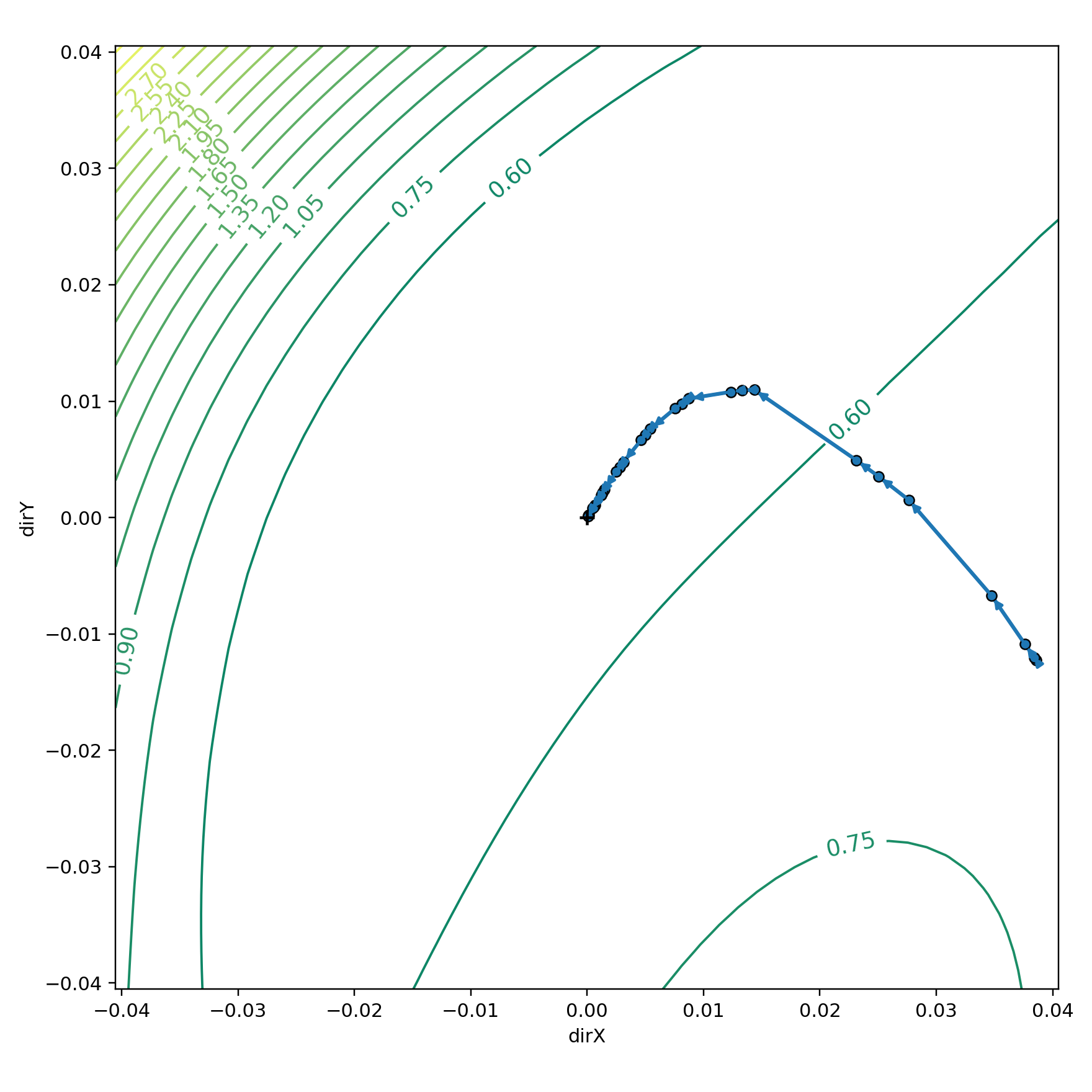}\caption{$c^3$LA}\label{fig:c3la}
	\end{subfigure}}
   \scalebox{0.90}{ \begin{subfigure}{0.3\linewidth}
		\includegraphics[width=\linewidth]{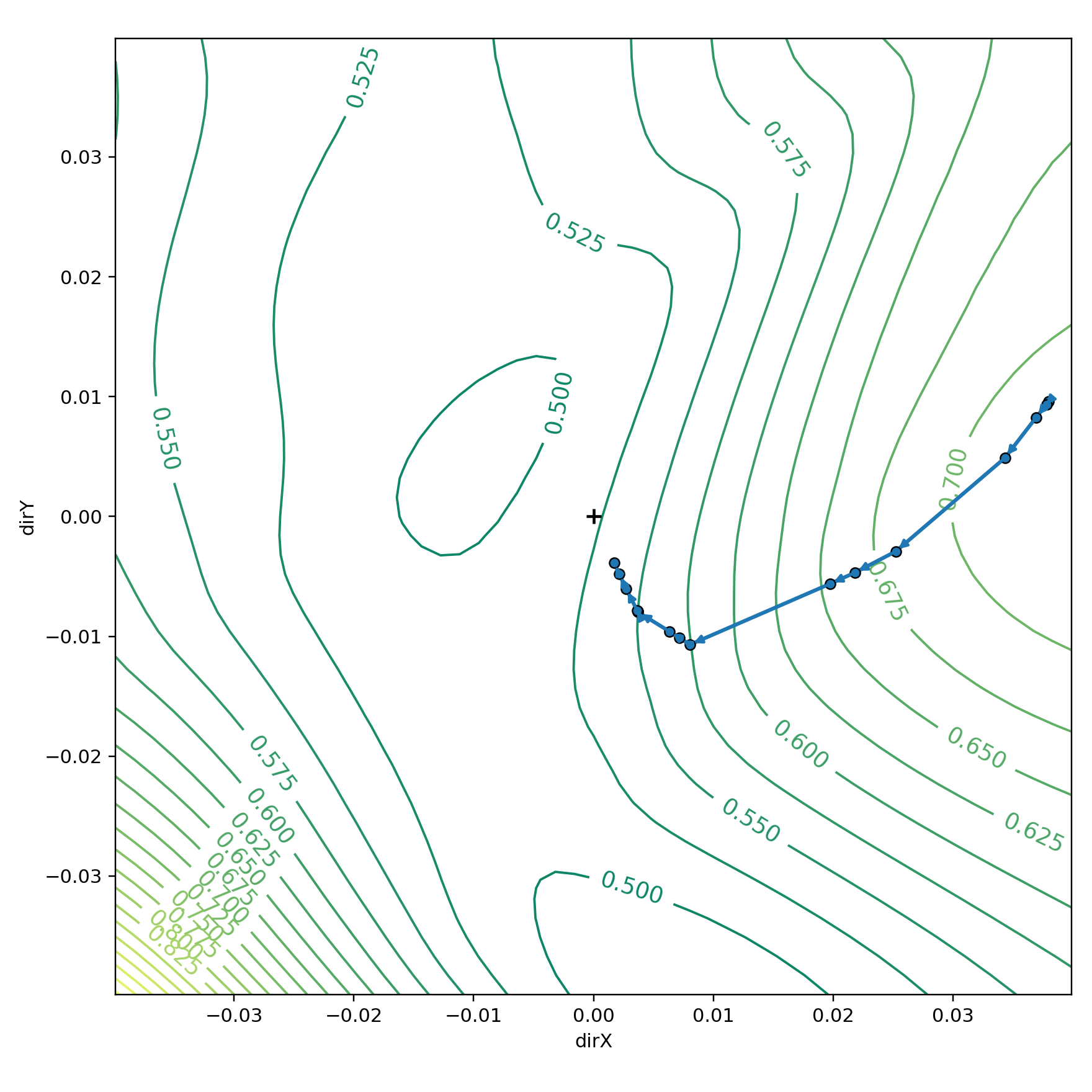}\caption{r-$c^3$LA}\label{fig:r-c3la-lora}
	\end{subfigure}
	\begin{subfigure}{0.3\linewidth}
		\includegraphics[width=\linewidth]{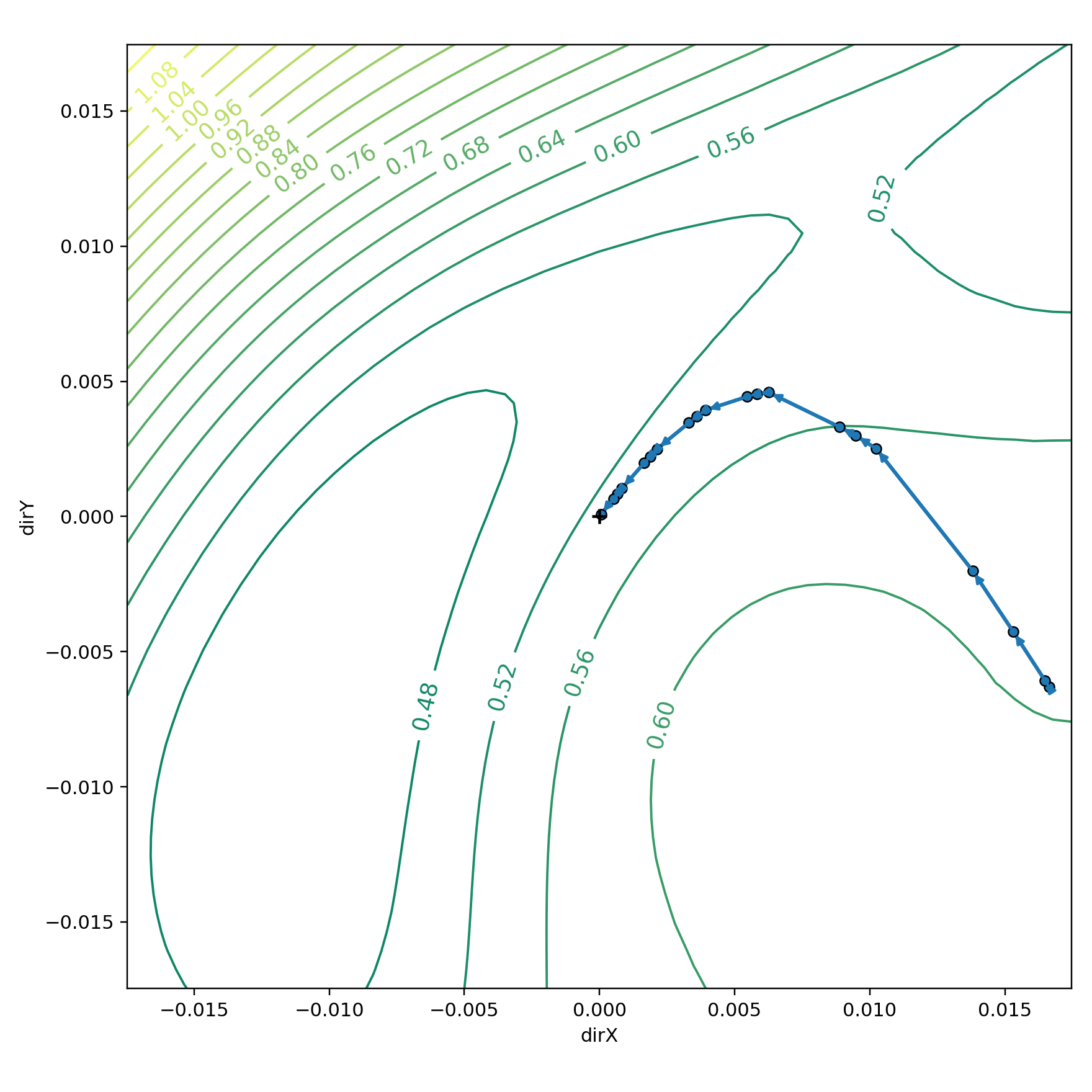}\caption{RAC}\label{fig:rac}
	\end{subfigure}
    \begin{subfigure}{0.3\linewidth}
		\includegraphics[width=\linewidth]{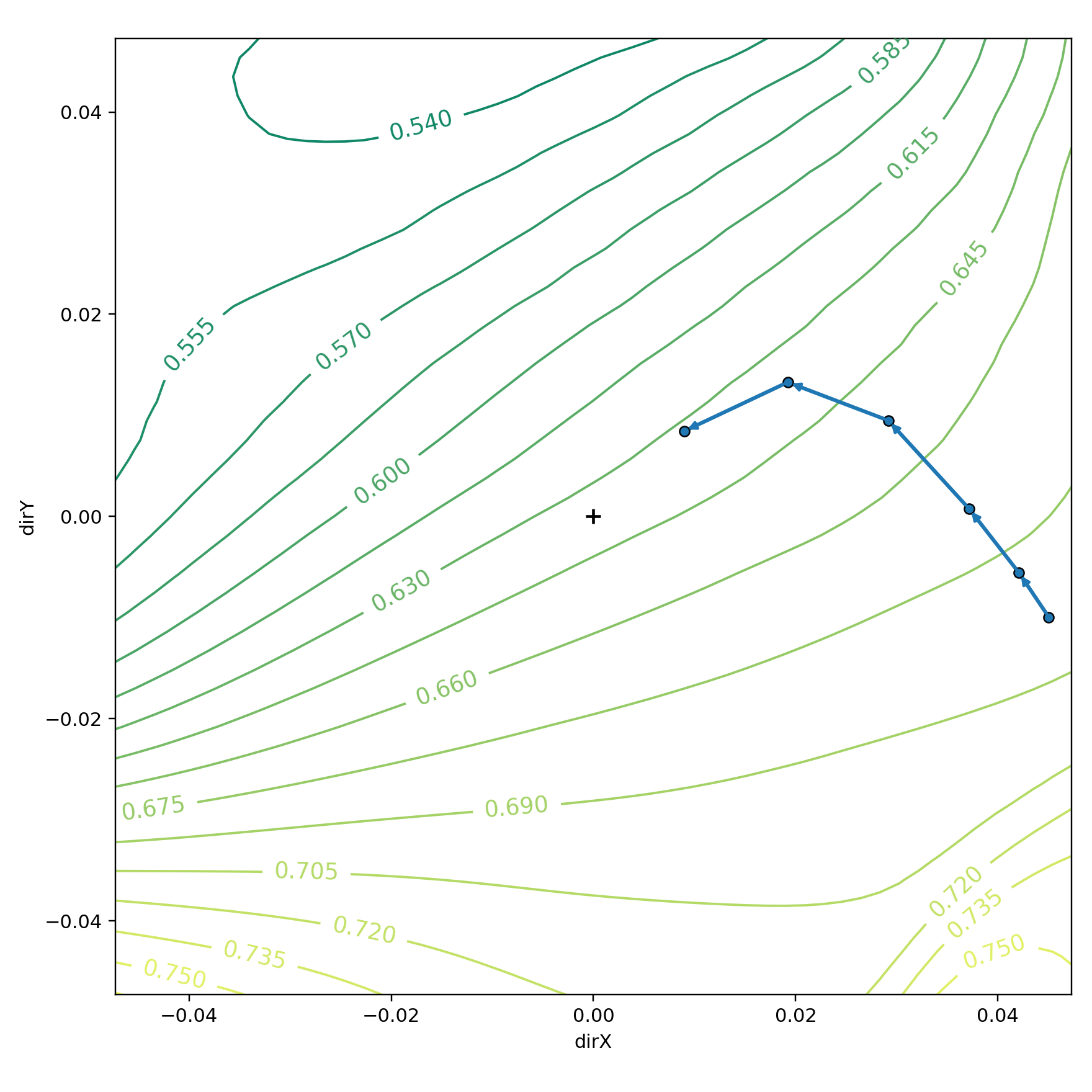}\caption{LoRA+}\label{fig:lora-plus}
	\end{subfigure}}
   \begin{subfigure}{0.3\linewidth}
		\includegraphics[width=\linewidth]{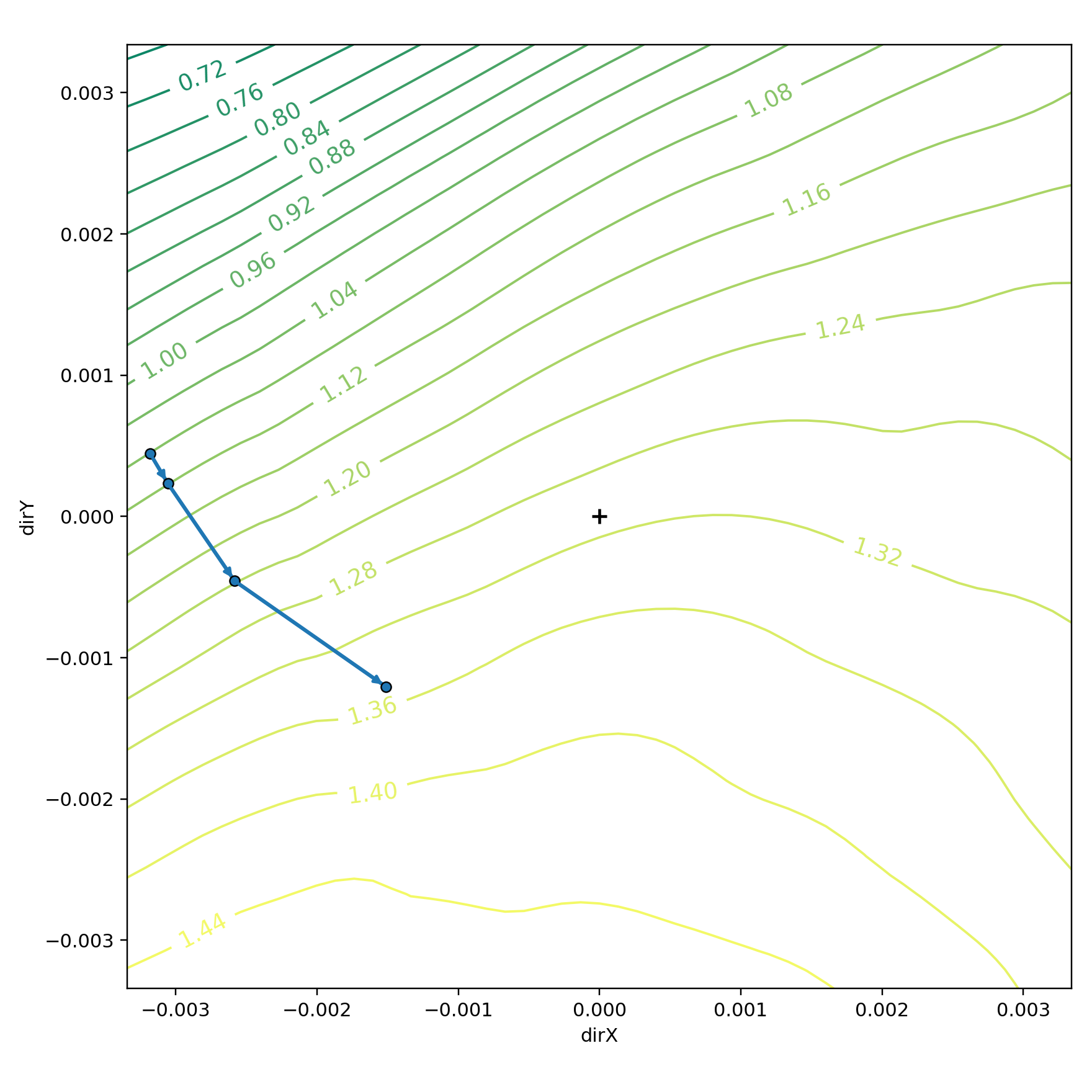}\caption{Full Fine-Tuning}\label{fig:fft}
	\end{subfigure}
	\caption{2D loss landscapes of RoBERTa-Base fine-tuned on CoLA for FFT and other PEFT methods. The axes dirX and dirY are the constants we scale the top two PCA components of the weight displacement matrix. The range was chosen to contain the entire gradient path. The top row is the non-chain variant of the bottom row, save for the last column. The center is marked with a cross for visibility; the arrows indicate the direction of the model's updates.}\label{fig:roberta-2D-landscapes}
\end{figure*}

\subsubsection{Intruder Dimension implementation}\label{appendix:Cosine Similarity}

\begin{table*}
  \centering
  \tiny
  \setlength{\tabcolsep}{2pt}
  \renewcommand{\arraystretch}{1}
  \caption{\small{\textbf{Generalization error approximations} (${\cG}(\textbf{W})\approx\mathbb{E}(\mathcal{L_{\rm test}}) - \mathcal{L_{\rm train}}$) on the past (FFT, LoRA), the present (CoLA, Asymmetric LoRA, RAC, LoRA+), and the future (cLA, $c^3$LA, r-cLA, r-$c^3$LA) fine-tuning methods over various models and datasets. The color \textcolor{codegreen}{green} indicates the best result for each particular model and dataset combination, \textcolor{red}{red} is the second-best result, and \textcolor{blue}{blue} the third.}}
  \label{tab:combined}

  \begin{tabular}{@{} ll *{2}{r} *{4}{r} *{4}{r} @{} }
  \toprule
  \multirow{2}{*}{\textbf{Model}} & \multirow{2}{*}{\textbf{Dataset}} & \multicolumn{2}{c}{\textbf{The Past}} & \multicolumn{4}{c}{\textbf{The Present}} & \multicolumn{4}{c}{\textbf{The Future}} \\
  \cmidrule(lr){3-4} \cmidrule(lr){5-8} \cmidrule(lr){9-12}
   &  & FFT & LoRA & CoLA & Asymm & RAC & LoRA+ & cLA & $c^3$LA & r-cLA & r-$c^3$LA \\
  \midrule
    ViT-Tiny \cite{dosovitskiy2020vit} & OfficeHome & 4.85$e^{-1}$ & 6.96$e^{-2}$ & \textcolor{codegreen}{9.55$e^{-3}$} & 7.22$e^{-2}$ & 6.17$e^{-2}$ & 7.39$e^{-2}$ & \textcolor{red}{1.98$e^{-2}$} & 3.40$e^{-2}$ & \textcolor{blue}{2.16$e^{-2}$} & 3.51$e^{-2}$ \\
    & CIFAR-10 & \textcolor{codegreen}{1.42$e^{-1}$} & \textcolor{red}{2.64$e^{-1}$} & 2.87$e^{-1}$ & 3.36$e^{-1}$ & 3.18$e^{-1}$ & \textcolor{blue}{2.80$e^{-1}$} & 3.13$e^{-1}$ & 3.03$e^{-1}$ & 3.12$e^{-1}$ & 2.92$e^{-1}$ \\
    \midrule
    ViT-Base \cite{dosovitskiy2020vit} & OfficeHome & 3.66$e^{-1}$ & 1.07$e^{-1}$ & \textcolor{blue}{1.43$e^{-2}$} & \textcolor{codegreen}{8.52$e^{-3}$} & \textcolor{red}{1.02$e^{-2}$} & 1.41$e^{-1}$ & 3.16$e^{-2}$ & 3.62$e^{-2}$ & 5.53$e^{-2}$ & 3.00$e^{-2}$ \\
    & CIFAR-10 & \textcolor{codegreen}{9.98$e^{-2}$} & \textcolor{blue}{1.92$e^{-1}$} & 2.21$e^{-1}$ & 2.38$e^{-1}$ & 2.30$e^{-1}$ & \textcolor{red}{1.84$e^{-1}$} & 2.33$e^{-1}$ & 2.34$e^{-1}$ & 2.26$e^{-1}$ & 2.15$e^{-1}$ \\
    \midrule
    DeBERTa v2 XXL \cite{deberta} & MRPC & 8.15$e^{-2}$ & \textcolor{red}{6.89$e^{-2}$} & \textcolor{codegreen}{6.53$e^{-2}$} & 8.09$e^{-2}$ & \textcolor{blue}{8.02$e^{-2}$} & 9.08$e^{-2}$ & 9.31$e^{-2}$ & 1.10$e^{-1}$ & 9.47$e^{-2}$ & 1.22$e^{-1}$\\
    & TREC50 & 3.38$e^{-1}$& 2.36$e^{-1}$& \color{codegreen}{7.04$e^{-2}$} & \textcolor{blue}{1.53$e^{-1}$} & 2.24$e^{-1}$& \color{red}{1.36$e^{-1}$}& 1.85$e^{-1}$& 2.22$e^{-1}$& 1.93$e^{-1}$& 1.92$e^{-1}$\\
    & PAWS & 6.07$e^{-2}$& \textcolor{red}{1.99$e^{-2}$} & 3.63$e^{-2}$ & \textcolor{blue}{3.26$e^{-2}$} & 3.95$e^{-2}$ & 5.41$e^{-2}$ & 6.68$e^{-2}$ & 5.11$e^{-2}$ & \textcolor{codegreen}{1.98$e^{-2}$} & 6.99 $e^{-2}$\\
    \midrule
    DeBERTa v3 Base \cite{debertav3} & MRPC & 1.06$e^{-1}$ & 8.90$e^{-2}$ & {2.59$e^{-2}$} & 7.28$e^{-2}$ & 9.86$e^{-2}$ & \textcolor{red}{1.52$e^{-2}$} & {2.58$e^{-2}$} & \textcolor{codegreen}{8.52$e^{-3}$} & 1.16$e^{-1}$ & \textcolor{blue}{2.57$e^{-2}$}\\
    & TREC50 & 4.56$e^{-1}$ & 2.73$e^{-1}$ & 3.99$e^{-1}$ & \textcolor{blue}{2.16$e^{-1}$} & 2.67$e^{-1}$ & \textcolor{red}{2.61$e^{-2}$} & 2.25$e^{-1}$ & {3.70$e^{-1}$} & 3.36$e^{-1}$ & \textcolor{codegreen}{2.63$e^{-2}$}\\
    & PAWS & 2.62$e^{-2}$ & 6.43$e^{-2}$ & \textcolor{codegreen}{2.40$e^{-2}$} & 6.27$e^{-2}$ & 8.17$e^{-2}$ & \textcolor{red}{5.55$e^{-2}$} & 7.39$e^{-2}$ & \textcolor{blue}{5.77$e^{-2}$} & 1.01$e^{-1}$ &  5.82$e^{-2}$\\
    \midrule
    RoBERTa-Base \cite{roberta} & MRPC & 9.48$e^{-1}$ & 6.01$e^{-1}$ & \textcolor{red}{2.05$e^{-1}$} & \textcolor{codegreen}{1.64$e^{-1}$} & \textcolor{blue}{2.20$e^{-1}$} & 5.33$e^{-1}$ & 4.37$e^{-1}$ & 3.78$e^{-1}$ & 3.35$e^{-1}$ & 3.21$e^{-1}$ \\
    & CoLA & 1.39 & 7.74$e^{-1}$ & 4.04$e^{-1}$ & \textcolor{red}{2.22$e^{-1}$} & \textcolor{codegreen}{1.96$e^{-1}$} & 8.10$e^{-1}$ & 4.70$e^{-1}$ & 4.43$e^{-1}$ & 4.38$e^{-1}$ & \textcolor{blue}{4.01$e^{-1}$} \\
    \midrule
    RoBERTa-Large \cite{roberta} & MRPC & 7.29$e^{-1}$ & 4.64$e^{-1}$ & 4.71$e^{-1}$ & \textcolor{blue}{2.77$e^{-1}$} & \textcolor{red}{2.68$e^{-1}$} & \textcolor{codegreen}{2.64$e^{-1}$} & 6.54$e^{-1}$ & 5.57$e^{-1}$ & 5.27$e^{-1}$ & 3.84$e^{-1}$ \\
    & CoLA & 8.06$e^{-1}$ & 4.25$e^{-1}$ & 4.18$e^{-1}$ & \textcolor{blue}{2.36$e^{-1}$} & \textcolor{codegreen}{1.75$e^{-1}$} & \textcolor{red}{2.28$e^{-1}$} & 4.96$e^{-1}$ & 4.56$e^{-1}$ & 6.14$e^{-1}$ & 4.05$e^{-1}$ \\
    \midrule
    TinyLlama  \cite{zhang2024tinyllama} & OpenBookQA & \textcolor{blue}{1.78$e^{-1}$} & 2.82$e^{-1}$ & 3.41$e^{-1}$ & {2.15$e^{-1}$} & 1.86$e^{-1}$ & 2.07$e^{-1}$ & \textcolor{red}{1.51$e^{-1}$} & 2.20$e^{-1}$ & 3.16$e^{-1}$ & \textcolor{codegreen}{7.59$e^{-2}$} \\
    & FOLIO & {1.82$e^{-1}$} & 2.37$e^{-1}$ & {2.17$e^{-1}$} &  \textcolor{blue}{1.75$e^{-1}$} & 1.93$e^{-1}$ & \textcolor{codegreen}{5.11$e^{-2}$} & 2.35$e^{-1}$ & 1.91$e^{-1}$ & \textcolor{red}{1.05$e^{-1}$} & {2.49$e^{-1}$}\\
    & LogiQA & 3.61$e^{-1}$ & \textcolor{codegreen}{6.12$e^{-3}$} & 1.45$e^{-1}$ & \textcolor{red}{1.16$e^{-2}$} & 1.75$e^{-1}$ & 2.37$e^{-1}$ & 8.60$e^{-2}$ & 1.1$e^{-1}$ & 6.64$e^{-2}$ & \textcolor{blue}{6.25$e^{-2}$}\\
    & CLUTRR & 4.29 & 2.25 & \textcolor{codegreen}{1.55} & 2.34 & 2.27 & 5.48 & \textcolor{red}{2.16} & \textcolor{blue}{2.19} & 2.59 & 4.23\\
    \midrule
    Llama 3~\cite{grattafiori2024llama} & OpenBookQA & 2.65$e^{-1}$ & 2.54$e^{-1}$ & \textcolor{blue}{2.32$e^{-1}$} & 2.63$e^{-1}$ & \textcolor{codegreen}{1.67$e^{-1}$} & \textcolor{red}{1.92$e^{-1}$} & 3.55$e^{-1}$ & 2.69$e^{-1}$ & 2.79$e^{-1}$ & 3.61 \\
    & CLUTRR & \textcolor{red}{2.53} & 2.66 & 2.97 & 2.9 & 5.49 & \textcolor{blue}{2.65} & 2.69 & 5.02 & \textcolor{codegreen}{2.51} & 4.33 \\
    \midrule
    DeepseekCoder \cite{guo2024deepseekcoder} & DJANGO & \textcolor{red}{3.48$e^{-2}$} & {4.65$e^{-2}$} & \textcolor{codegreen}{3.4$e^{-2}$} & 5.16$e^{-2}$ & 4.64$e^{-2}$ & {3.87$e^{-2}$} & {4.19$e^{-2}$} & {3.89$e^{-2}$} & {3.64$e^{-2}$} & \textcolor{blue}{3.62$e^{-2}$}\\
    \midrule
    GPT2-Small \cite{gpt2} & E2E & \textcolor{codegreen}{1.65$e^{-1}$} & 1.93$e^{-1}$ & 1.85$e^{-1}$ & 1.83$e^{-1}$ & 1.85$e^{-1}$ & 1.87$e^{-1}$ & \textcolor{red}{1.77$e^{-1}$} & \textcolor{blue}{1.82$e^{-1}$} & 1.88$e^{-1}$ & \textcolor{blue}{1.82$e^{-1}$} \\
  %\midrule
  \bottomrule
  \end{tabular}
\end{table*}

Given the pretrained and fine-tuned models, $\textbf{W}_0$ and $\textbf{W}_0 + \Delta \textbf{W}$ we find intruder dimensions as follows: first, we decompose each layer of $\textbf{W}_0$ and $\textbf{W}_0 + \Delta \textbf{W}$ into their corresponding SVDs, $U^i\Sigma^i {V^i}^T_{(\textbf{W}_0)^i}$ and $U^i\Sigma^i {V^i}^T_{(\textbf{W}_0+\Delta \textbf{W})^i},i\in[L]$, respectively. Then, given a threshold $\varepsilon\in(0, 1)$, a singular vector $u^{j,i}_{(\textbf{W}_0+\Delta \textbf{W})}$ in $U^i_{(\textbf{W}_0+\Delta \textbf{W})}$ is an intruder dimension if for all $u^{k,i}_{(\textbf{W}_0)}$ in $U^i_{(\textbf{W}_0)}$, the expression, $\tfrac{|\langle u^{j,i}_{(\textbf{W}_0+\Delta \textbf{W]})},u^{k,i}_{(\textbf{W}_0)}\rangle|}{\|u^{j,i}_{(\textbf{W}_0+\Delta \textbf{W})}\|\|u^{k,i}_{(\textbf{W}_0)}\|}|<\varepsilon$. For $\varepsilon$ small enough, this indicates the vector $u^{j,i}_{(\textbf{W}_0+\Delta \textbf{W})}$ is almost orthogonal to all vectors in $U^i_{(\textbf{W}_0)}$. We denote these vectors as \emph{intruder dimensions.}

\subsection{Generalization Error---Continued}
\label{subsec:genconnect}

Let $\mathcal{X} \times \mathcal{Y}$ be our input space and label space with $\nu$ distribution of pairs $(x, y) \in \mathcal{X} \times     \mathcal{Y}$, our dataset $N= \{ (x_1,y_1), ..., (x_n, y_n) \}$ where each $(x_i, y_i)$ is i.i.d. from $\nu$ distribution of $\mathcal{X} \times \mathcal{Y}$, thus the distribution over our dataset does not represent the true distribution of input-output pairs from our instance space. Let $\mathcal{H}$ be our hypothesis space, where $w \in \mathcal{H}; w(x_i)=\hat{y}_i$ thus, we are concerned with how accurately $w$ can adapt to the true distribution $\nu$ of $\mathcal{X} \times \mathcal{Y}$. This can be addressed by the generalization error of our hypothesis $w\in \mathcal{H}$ given our loss function $\ell$. The true risk of $w$ over $\mathcal{X} \times \mathcal{Y}$ given $\ell$ is $\cL_{\rm global}(w):=\mathbb{E}_{\mathcal{X}, \mathcal{Y}}[\ell(w(x), y)]=\int_{\mathcal{X \times \mathcal{Y}}} \ell(w(x), y)d\nu$, while empirical risk is $\cL :=\frac{1}{n}\sum^{n}\ell(w(x_i), y_i); (x_i, y_i) \in N$. Let $M$ denote the full dataset, where $M = N \cup T$, $N$ being the train dataset, and $T$ being the test dataset. In practice, the empirical risk can be computed based on $N$, and the test dataset, $T$, can be used to show how well the model has generalized. $N$ and $T$ are independent samples from $\nu$; their distributions approximate $\nu$ but differ due to random and finite sampling. Although $\mathcal{L}_{\rm test} - \mathcal{L}_{\rm train}$ is not a true testament for calculating the generalization error of a model, it can be used as a heuristic for determining generalization. Understanding how stable these models are to small weight perturbations provides insight into their reliability and reputability for practical use.

{
\subsubsection{Normalized Generalization Results}
\label{subsubsec:gen_norm}
We normalize the generalization gaps with respect to LoRA, reporting the ratio $\frac{\cG(\cdot)}{\cG(\mathrm{LoRA})}$ in Table~\ref{tab:gen_normed}. Under this normalization, many PEFT methods exhibit similar generalization behavior relative to LoRA. Throughout all models, the effect of chaining behavior tended to coincide; if CoLA generalized better than LoRA, then RAC often generalizes better than Asymmetric LoRA, cLA than c$^3$LA, r-cLA, and r-c$^3$LA. 
}

\begin{table*}[b]
  \centering
  \scriptsize
  \setlength{\tabcolsep}{2pt}
  \renewcommand{\arraystretch}{0.8}
  \caption{\small{\textbf{Normalized generalization error approximations with respect to LoRA ($\frac{\cG(\cdot)}{\cG(\rm LoRA)}$)}, on the past (FFT, LoRA), the present (CoLA, Asymmetric LoRA, RAC, LoRA+), and the future (cLA, $c^3$LA, r-cLA, r-$c^3$LA) fine-tuning methods over various models and datasets. The color \textcolor{codegreen}{green} indicates the best result for each particular model and dataset combination, \textcolor{red}{red} is the second-best result, and \textcolor{blue}{blue} the third.}}
  \label{tab:gen_normed}
  \scalebox{1.2}{
  \begin{tabular}{@{} ll *{2}{r} *{4}{r} *{4}{r} @{} }
  \toprule
  \multirow{2}{*}{\textbf{Model}} & \multirow{2}{*}{\textbf{Dataset}} & \multicolumn{2}{c}{\textbf{The Past}} & \multicolumn{4}{c}{\textbf{The Present}} & \multicolumn{4}{c}{\textbf{The Future}} \\
  \cmidrule(lr){3-4} \cmidrule(lr){5-8} \cmidrule(lr){9-12}
   &  & FFT & LoRA & CoLA & Asymm & RAC & LoRA+ & cLA & $c^3$LA & r-cLA & r-$c^3$LA \\
  \midrule
    ViT-Tiny \cite{dosovitskiy2020vit} & OfficeHome & 6.97 & 1.00 & \textcolor{codegreen}{0.14} & 1.04 & 0.89 & 1.06 & \textcolor{red}{0.28} & 0.49 & \textcolor{blue}{0.31} & 0.50 \\
    & CIFAR-10 & \textcolor{codegreen}{0.54} & \textcolor{red}{1.00} & 1.09 & 1.27 & 1.20 & \textcolor{blue}{1.06} & 1.19 & 1.15 & 1.18 & 1.11 \\
    \midrule
    ViT-Base \cite{dosovitskiy2020vit} & OfficeHome & 3.42 & 1.00 & \textcolor{blue}{0.13} & \textcolor{codegreen}{0.08} & \textcolor{red}{0.10} & 1.32 & 0.30 & 0.34 & 0.52 & 0.28 \\
    & CIFAR-10 & \textcolor{codegreen}{0.52} & \textcolor{blue}{1.00} & 1.15 & 1.24 & 1.20 & \textcolor{red}{0.96} & 1.21 & 1.22 & 1.18 & 1.12 \\
    \midrule
    DeBERTa v2 XXL \cite{deberta} & MRPC & 1.18 & \textcolor{red}{1.00} & \textcolor{codegreen}{0.95} & 1.17 & \textcolor{blue}{1.16} & 1.32 & 1.35 & 1.60 & 1.37 & 1.77\\
    & TREC50 & 1.43& 1.00& \color{codegreen}{0.30} & \textcolor{blue}{0.65} & 0.95& \color{red}{0.58}& 0.78& 0.94& 0.82& 0.81\\
    & PAWS & 3.05& \textcolor{red}{1.00} & 1.82 & \textcolor{blue}{1.64} & 1.98 & 2.72 & 3.36 & 2.57 & \textcolor{codegreen}{0.99} & 3.51\\
    \midrule
    DeBERTa v3 Base \cite{debertav3} & MRPC & 1.19 & 1.00 & {0.29} & 0.82 & 1.11 & \textcolor{red}{0.17} & {0.29} & \textcolor{codegreen}{0.10} & 1.30 & \textcolor{blue}{0.29}\\
    & TREC50 & 1.67 & 1.00 & 1.46 & \textcolor{blue}{0.79} & 0.98 & \textcolor{red}{0.10} & 0.82 & {1.36} & 1.23 & \textcolor{codegreen}{0.10}\\
    & PAWS & 0.41 & 1.00 & \textcolor{codegreen}{0.37} & 0.98 & 1.27 & \textcolor{red}{0.86} & 1.15 & \textcolor{blue}{0.90} & 1.57 &  0.91\\
    \midrule
    RoBERTa-Base \cite{roberta} & MRPC & 1.58 & 1.00 & \textcolor{red}{0.34} & \textcolor{codegreen}{0.27} & \textcolor{blue}{0.37} & 0.89 & 0.73 & 0.63 & 0.56 & 0.53 \\
    & CoLA & 1.80 & 1.00 & 0.52 & \textcolor{red}{0.29} & \textcolor{codegreen}{0.25} & 1.05 & 0.61 & 0.57 & 0.57 & \textcolor{blue}{0.52} \\
    \midrule
    RoBERTa-Large \cite{roberta} & MRPC & 1.57 & 1.00 & 1.02 & \textcolor{blue}{0.60} & \textcolor{red}{0.58} & \textcolor{codegreen}{0.57} & 1.41 & 1.20 & 1.14 & 0.83 \\
    & CoLA & 1.90 & 1.00 & 0.98 & \textcolor{blue}{0.56} & \textcolor{codegreen}{0.41} & \textcolor{red}{0.54} & 1.17 & 1.07 & 1.44 & 0.95 \\
    \midrule
    TinyLlama  \cite{zhang2024tinyllama} & OpenBookQA & \textcolor{blue}{0.63} & 1.00 & 1.21 & {0.76} & 0.66 & 0.73 & \textcolor{red}{0.54} & 0.78 & 1.12 & \textcolor{codegreen}{0.27} \\
    & FOLIO & {0.77} & 1.00 & {0.92} &  \textcolor{blue}{0.74} & 0.81 & \textcolor{codegreen}{0.22} & 0.99 & 0.81 & \textcolor{red}{0.44} & {1.05}\\
    & LogiQA & 58.99 & \textcolor{codegreen}{1.00} & 23.69 & \textcolor{red}{1.90} & 28.59 & 38.73 & 14.05 & 17.97 & 10.85 & \textcolor{blue}{10.21}\\
    & CLUTRR & 1.91 & 1.00 & \textcolor{codegreen}{0.69} & 1.04 & 1.01 & 2.44 & \textcolor{red}{0.96} & \textcolor{blue}{0.97} & 1.15 & 1.88\\
    \midrule
    Llama 3~\cite{grattafiori2024llama} & OpenBookQA
      & 1.04 & 1.00 & \textcolor{blue}{0.91} & 1.04 & \textcolor{codegreen}{0.66} & \textcolor{red}{0.76} & 1.40 & 1.06 & 1.10 & 14.21 \\
        & CLUTRR
      & \textcolor{red}{0.95} & 1.00 & 1.12 & 1.09 & 2.06 & \textcolor{blue}{1.00} & 1.01 & 1.89 & \textcolor{codegreen}{0.94} & 1.63 \\
    \midrule
    DeepseekCoder \cite{guo2024deepseekcoder} & DJANGO & \textcolor{red}{0.75} & {1.00} & \textcolor{codegreen}{0.73} & 1.11 & 1.00 & {0.83} & {0.90} & {0.84} & {0.78} & \textcolor{blue}{0.78}\\
    \midrule
    GPT2-Small \cite{gpt2} & E2E & \textcolor{codegreen}{0.85} & 1.00 & 0.96 & 0.95 & 0.96 & 0.97 & \textcolor{red}{0.92} & \textcolor{blue}{0.94} & 0.97 & \textcolor{blue}{0.94} \\
  %\midrule
  \bottomrule
  \end{tabular}}
\end{table*}

\begin{table*}[t]
  \centering
  \tiny
  \setlength{\tabcolsep}{2pt}
  \renewcommand{\arraystretch}{0.8}
  \caption{\small{Extended Table 2, performance of fine-tuned models with adapter rank $r=16$.~We use \textcolor{codegreen}{green}, \textcolor{red}{red}, and \textcolor{blue}{blue} to indicate the best, second best, and third best result. For the sparse variants, \reddown~indicates the accuracy drop percentage compared to the best.}}
  \label{tab:extendedMo-accuracy-table}
  \vspace{-2mm}
\begin{tabular}{@{} ll *{2}{r} *{4}{r} | *{4}{r} @{} }
  \toprule
  \multirow{2}{*}{\textbf{Model}} & \multirow{2}{*}{\textbf{Dataset}} & \multicolumn{2}{c}{\textbf{The Past}} & \multicolumn{4}{c}{\textbf{The Present}} & \multicolumn{4}{c}{\textbf{The Future}} \\
  \cmidrule(lr){3-4} \cmidrule(lr){5-8} \cmidrule(lr){9-12}
   &  & FFT & LoRA & CoLA & Asym & RAC & LoRA+ & cLA & $c^3$LA & r-cLA & r-$c^3$LA \\
  \midrule
    ViT-Tiny \cite{dosovitskiy2020vit} & OfficeHome \cite{officehome-dataset} & \textcolor{red}{79.68} & \textcolor{codegreen}{80.13} & \textcolor{blue}{79.54} & 78.02 & 78.55 & 77.87 & 78.01 (\reddown2.65\%) & 78.69 (\reddown1.80\%) & 78.01 (\reddown2.65\%) & 79.32 (\reddown1.01\%) \\
    & CIFAR10 \cite{cifar10-dataset} & \textcolor{codegreen}{96.59} & \textcolor{red}{96.17} & \textcolor{blue}{95.85} & 94.80 & 95.36 & 95.29 & 94.94 (\reddown1.71\%) & 95.23 (\reddown1.41\%) & 95.12 (\reddown1.52\%) & 95.22 (\reddown1.42\%) \\
    \midrule
    ViT-Base \cite{dosovitskiy2020vit} & OfficeHome
      & 86.42 & 88.96 & 89.01 & 89.00 & \textcolor{codegreen}{89.33} & 87.87 & \textcolor{red}{89.21} & \textcolor{blue}{89.18} & 88.83 & 89.17 \\
    & CIFAR10
      & 98.06 & 98.71 & 98.48 & 98.68 & \textcolor{red}{98.73} & 98.36 & 98.63 & 98.54 & \textcolor{codegreen}{98.78} & \textcolor{blue}{98.72} \\
    \midrule
    DeBERTa v2 XXL \cite{deberta} & MRPC \cite{wang2018glue} & \textcolor{blue}{87.49} & \textcolor{codegreen}{88.28} & 87.47 & 87.03 & 86.97 & \textcolor{red}{87.53} & 86.13 (\reddown2.44\%) & 85.11 (\reddown3.59\%) & 85.55 (\reddown3.09\%) & 85.15 (\reddown3.55\%) \\
    & TREC-50 \cite{li2006learning} & \textcolor{blue}{91.99} & 91.47 & 85.65 & \textcolor{codegreen}{92.26} & \textcolor{red}{92.02} & 84.92 & 91.73 (\reddown0.57\%) & 90.87 (\reddown1.51\%) & 91.67 (\reddown0.64\%) & 91.07 (\reddown1.29\%) \\
    & PAWS \cite{zhang2019paws} & 94.69 & \textcolor{blue}{94.97} & \textcolor{codegreen}{95.22} & {94.95} & 94.66 & \textcolor{red}{95.20} & 94.77 (\reddown0.47\%) & 94.90 (\reddown0.34\%) & 94.38 (\reddown0.88\%) & 94.71 (\reddown0.54\%) \\
    \midrule
    DeBERTa v3 Base \cite{debertav3} & MRPC & {85.80} & \textcolor{codegreen}{88.33} & \textcolor{red}{87.91} & \textcolor{blue}{86.40} & {86.34} & 84.51 & 84.43 (\reddown4.42\%) & {80.22} (\reddown9.18\%) & 85.42 (\reddown3.29\%) & 84.17 (\reddown4.71\%) \\
    & RTE \cite{wang2018glue} & 82.47 & \textcolor{codegreen}{86.34} & \textcolor{blue}{83.80} & 78.94 & 79.40 & \textcolor{red}{84.72} & 76.00 (\reddown11.98\%) & 75.08 (\reddown13.04\%) & 79.40 (\reddown8.04\%) & 79.40 (\reddown8.04\%) \\
    & STS-B \cite{wang2018glue} & \textcolor{codegreen}{89.52} & {89.09} & \textcolor{red}{89.34} & 89.04 & 88.71 & \textcolor{blue}{89.15} & 87.56 (\reddown2.19\%) & 87.90 (\reddown1.81\%) & 88.05 (\reddown1.64\%) & 87.90 (\reddown1.81\%) \\
    & TREC-50 & \textcolor{red}{90.15} & {89.29} & \textcolor{blue}{89.88} & \textcolor{codegreen}{90.67} & {89.22} & 85.52 & 86.04 (\reddown5.11\%) & 87.96 (\reddown2.99\%) & 86.04 (\reddown5.11\%) & {87.70} (\reddown3.28\%) \\
    & PAWS & \textcolor{codegreen}{94.76} & \textcolor{red}{94.62} & {94.40} & {94.48} & {94.45} & {94.44} & 94.23 & \textcolor{blue}{94.60} & 94.36 & 94.42\\
    \midrule
    RoBERTa-Base \cite{roberta} & MRPC & \textcolor{codegreen}{87.40} & 86.34 & \textcolor{red}{86.76} & 86.40 & \textcolor{blue}{86.67} & 84.29 & 84.83 (\reddown2.94\%) & 84.39 (\reddown3.44\%) & 85.08 (\reddown2.65\%) & 85.33 (\reddown2.37\%) \\
    & CoLA \cite{wang2018glue} & \textcolor{blue}{56.08} & \textcolor{red}{57.33} & \textcolor{codegreen}{58.39} & 52.35 & 53.76 & 50.40 & 51.86 (\reddown11.18\%) & 53.29 (\reddown8.73\%) & 52.56 (\reddown9.98\%) & 53.10 (\reddown9.06\%) \\
    \midrule
    RoBERTa-Large \cite{roberta} & MRPC
      & 87.57 & \textcolor{codegreen}{88.46} & \textcolor{red}{88.43} & 87.56 & 87.69 & 72.91 & \textcolor{blue}{87.81} & 86.36 & 86.24 & 86.59 \\
   & CoLA & \textcolor{codegreen}{64.58} & \textcolor{blue}{62.42} & 60.03 & \textcolor{red}{63.42} & 59.84 & 28.80 & 59.47 (\reddown7.91\%) & 59.60 (\reddown7.71\%) & 58.60 (\reddown9.26\%) & 60.24 (\reddown6.72\%) \\
    \midrule
    TinyLlama \cite{zhang2024tinyllama} 
    & OpenBookQA \cite{openbookQA} & \textcolor{codegreen}{55.47} & 52.41 & \textcolor{blue}{52.47} & 45.96 & 47.59 & \textcolor{red}{53.26} & 44.92(\reddown19.02\%) & 45.12(\reddown18.66\%) & 47.07(\reddown15.14\%) & 27.34(\reddown50.71\%) \\
    & FOLIO \cite{han2022folio} & \textcolor{codegreen}{60.71} & 57.59 & {59.40} & {58.33} & 55.45 & 54.17 & \textcolor{blue}{58.97} & 58.01 & 54.81 & \textcolor{red}{59.82}\\
    & LogiQA \cite{liu2020logiqa} & \textcolor{codegreen}{47.54} & 41.54 & \textcolor{blue}{43.70} & 41.50 & 40.86 & \textcolor{red}{45.83} & 39.09 (\reddown17.77\%) & 39.30 (\reddown17.33\%) & 39.09 (\reddown17.77\%) & 39.31 (\reddown17.31\%) \\
    & CLUTRR \cite{sinha2019clutrr} & \textcolor{codegreen}{42.01} & 37.44 & \textcolor{red}{39.38} & 37.98 & 37.98 & 38.10 & \textcolor{blue}{39.12} & 37.79 & 36.23 & 37.03\\
    \midrule
    Llama3-8B~\cite{grattafiori2024llama} & OpenBookQA & \textcolor{codegreen}{88.80} & \textcolor{blue}{87.53} & 86.47 & \textcolor{red}{88.47} & 87.33 & 86.87 & 87.33(\reddown1.65\%) & 85.07(\reddown4.20\%) & 86.07(\reddown3.07\%) & 53.69(\reddown39.54\%) \\
    & CLUTRR & 50.29 & 48.7 & 47.65 & 51.69 & 49.65 & \textcolor{blue}{52.89} & \textcolor{codegreen}{55.53} & {52.04} & \textcolor{red}{54.9} & 49.94\\
    \midrule
    DeepseekCoder \cite{guo2024deepseekcoder} & DJANGO \cite{oda2015learning} & 22.73 & 23.60 & 19.79 & \textcolor{codegreen}{35.12} & \textcolor{red}{30.27} & \textcolor{blue}{27.27} & 7.83 (\reddown77.71\%) & 19.48 (\reddown44.53\%) & {19.36} (\reddown44.87\%) & 15.34 (\reddown56.32\%) \\
    \midrule
    GPT2-Small \cite{gpt2} & E2E \cite{novikova2017e2e} & \textcolor{codegreen}{2.98} & \textcolor{blue}{3.18} & 3.29 & 3.36 & 3.34 & \textcolor{red}{3.23} & 3.34(\textcolor{red}{$\uparrow$}12.08\%) & 3.29(\textcolor{red}{$\uparrow$}10.4\%) & 3.30(\textcolor{red}{$\uparrow$}10.7\%) & 3.29(\textcolor{red}{$\uparrow$}10.4\%) \\
  %\midrule
  \bottomrule
  \end{tabular}
  \vspace{-3mm}
\end{table*}

\section{Limitations and Discussion}\label{app:limitations}
cLA and c$^3$LA particularly train only a small subsection of our pretrained model at a time, leading to underperformance on lower ranks in comparison to alternate LoRA variants. Often, we observed that cLA and c$^3$LA performed nearly as well as their non-sparse counterparts, Asymmetric LoRA and RAC, while being less expensive. The nature of the methods they were inspired by already had a frozen matrix component; we leave it up to researchers to study more potential identity-based LoRA variants to save computational resources.

% We emphasize that some of the analytical tools in \S\ref{sec:performance analysis} are not necessarily strong indicators of a model's performance. However, they tell us about some of the fine-tuned model's subspace properties, such as changes in direction and magnitude. Particularly, they depict how a fine-tuned model deviates from a pretrained model. This is relevant if the preservation of structure is important for alternative purposes such as cross-training, hybrid fine-tuning, or preservation of historical datasets.

% \section{Impact Statement}\label{app:impact_statement}
% This paper analyzes the capability of SOTA LoRA variants under structured sparsity restrictions from generalization, compute efficiency, and performance-metric perspectives. This paper primarily provides theoretical insights and connects them through empirical validation. As authors, we do not foresee any societal harm that this work may pose. Also, we do not foresee any high risk for misuse.

% We leave it up to future researchers to see if similar trends and analysis results come from alternate LoRA methods. Further research can be done on the use of any reparameterization fine-tuning technique \cite{PEFT}.

\clearpage
\section{Table of Notations}\label{app:notation}

\begin{center}
\small
  \captionof{table}{Table of notations.}\label{tab:notation}
  \renewcommand{\arraystretch}{1.06}% local only
  \begin{tabularx}{0.92\linewidth}{
    @{} p{0.24\linewidth} @{\hspace{1em}} X @{}}
    \toprule
    \textbf{Notation} & \textbf{Definition} \\
    \midrule
     $\|x\|$ & The $\ell_2$ norm of a vector, $x$\\
     \midrule
     $\|A\|$ & The Frobenius norm of a matrix, $A$\\
     \midrule
     $\|A\|_2$ & The spectral norm of a matrix, $A$\\
     \midrule
     $L$ & Number of layers in a deep neural network\\
    \midrule
    $W^i$ & $i^{\rm th}$ layer of network\\
    \midrule
    $\mathbf{W}$ & ($W^1,...,W^{L})$ \\
    \midrule
    $x$  & Input to the network \\
    \midrule
    $f_{\mathbf{W}}(x)$ & $\sigma_L(W^L\cdots\sigma_3(W^3\sigma_2(W^2\sigma_1(W^1(x))...)))$\\
    \midrule
    $\sigma_i(\cdot)$ & $i^{\rm th}$ layer non-linear activation function\\
    \midrule
    $N_{\rm pre}$ & pre-training dataset $(x_i,y_i)_{i=1}^{|N_{\rm pre}|}$\\
    \midrule
    $\ell_{\rm pre}(\cdot)$ & pre-training loss function\\
    \midrule
    $\mathbf{W}_0$ & pre-training weights\\
    \midrule
    $\Delta \mathbf{W}$ & FFT weight-update\\
    \midrule
    $\Delta \hat {\mathbf{W}}$ & FFT argmin update\\
    \midrule
    $\ell(\cdot)$ & fine-tuning loss function\\
    \midrule
    $\mathbf{BA}$ & LoRA weight-update\\
    \midrule
    $\hat {\mathbf{B}} \hat {\mathbf{A}}$ & LoRA argmin weight update\\
    \midrule
    $k$ & Chain-length of chain methods (CoLA, RAC, $c^3$LA)\\
    \midrule
    ${\mathbf{B}}^j\mathbf{A}^j$ & CoLA $j^{\rm th}$ chain weight update\\
    \midrule
    $\hat {\mathbf{B}}^j\hat{\mathbf{A}}^j$ & CoLA $j^{\rm th}$ chain argmin weight update\\
    \midrule
    $\mathbf{W}_0^{(k,BA)}$ & $k$ chains of CoLA updates, where $\mathbf{W}_0^{(k,BA)} := \mathbf{W}_0 + \sum_{j=1}^{k}\hat{\mathbf{B}}^j\hat{\mathbf{A}}^j$\\
    \midrule
    $\mathbf{A}_0$ & Frozen $A$ layers\\
    \midrule
    $\mathbf{BA}_0$ & Assymetric LoRA weight update\\
    \midrule
    $\hat{\mathbf{B}}\mathbf{A}_0$ & Assymetric LoRA argmin weight update\\
    \midrule
    $\mathbf{B}^j\mathbf{A}_0^j$ & RAC-LoRA $j^{\rm th}$ chain weight update\\
    \midrule
    $\hat{\mathbf{B}}^j\mathbf{A}_0^j$ & RAC-LoRA $j^{\rm th}$ chain argmin weight update\\
    \midrule
    $\mathbf{W}_0^{(k,B)}$ & $k$ chains of RAC-LoRA updates, where $\mathbf{W}_0^{(k,B)} := \mathbf{W}_0 + \sum_{j=1}^{k}\hat{\mathbf{B}}^j\mathbf{A}_0^j$\\
    \midrule
    $\mathbf{B}^{c}$ & Cheap LoRA (cLA) weight update\\
    \midrule
    $\hat{\mathbf{B}}^{c}$ & cLA argmin weight update\\
    \midrule
    $\mathbf{B}^{c^{3},j}$ & Circulant chain of cheap LoRA's ($c^3$LA) $j^{\rm th}$ chain weight update\\
    \midrule
    $\hat{\mathbf{B}}^{c^3}$ & $c^3$LA $j^{\rm th}$ chain argmin weight update\\
    \midrule
    $\mathbf{W}_0^{(k,B^{c^3})}$ & $k$ chains of $c^3$LA updates, where $\mathbf{W}_0^{(k,B^{c^3})} := \mathbf{W}_0 + \sum_{j=1}^{k}\hat{\mathbf{B}}^{c^{3},j}$\\
    \midrule
    $L_G$ & Lipschitz constant for the gradient of the loss function.\\
    \midrule
    $\mathcal{X}$ & feature space of the network\\
    \midrule
    $\mathcal{Y}$ & label space of the network\\
    \midrule
    $\hat{\mathcal{L}}_{\rm global}(\cdot)$ & true risk of an input network\\ 
  \bottomrule
  \end{tabularx}
\end{center}

% \begin{figure}[t]
%     \centering
%         \centering
%         \includegraphics[width=\linewidth]{Images/updated_loss_difference.png}
%     \caption{DeBERTa v3 BASE MRPC and TREC (Test Loss - Train Loss).}
%     \label{fig:testloss-trainloss}
% \end{figure}
%\clearpage
%\newpage
%\input{checklist.tex}

\end{document}